\documentclass{article}

\usepackage[preprint,nonatbib]{nips_2018}

\usepackage[utf8]{inputenc} 
\usepackage[T1]{fontenc}    
\usepackage{hyperref}       
\usepackage{url}            
\usepackage{booktabs}       
\usepackage{amsfonts}       
\usepackage{nicefrac}       
\usepackage{microtype}      

\usepackage{amsmath}
\usepackage{url}
\usepackage{caption}
\usepackage{subcaption}
\usepackage{float}
\usepackage{graphicx}
\usepackage{mathabx}
\usepackage{longtable}
\usepackage{cases}

\usepackage{multirow}
\usepackage[utf8]{inputenc}
\usepackage[english]{babel}
\usepackage{tikz}

\usepackage{amsmath,amsfonts,amssymb,amsthm}

\newtheorem{theorem}{Theorem}

\newtheorem{lemma}{\hspace{0.0cm}Lemma}

\newtheorem{assum}{Assumption}

\usepackage[shortlabels]{enumitem}

\DeclareMathOperator*{\argmin}{arg\,min}
 
\usepackage{todonotes}

\usepackage{xcolor,colortbl}

\usepackage{forloop}
\newcounter{loopcntr}

\definecolor{orange}{HTML}{FF7F00}

\title{Robust Implicit Backpropagation}

\author{
   Francois Fagan \\
   Columbia University \\
   New York, NY, 10027 \\
   \texttt{ff2316@columbia.edu} \\
   \And
   Garud Iyengar \\
   Columbia University \\
   New York, NY, 10027 \\
   \texttt{garud@ieor.columbia.edu} \\
}

\begin{document}

\maketitle

\begin{abstract}
Arguably the biggest challenge in applying neural networks is tuning the hyperparameters, in particular the learning rate. The sensitivity to the learning rate is due to the reliance on backpropagation to train the network. In this paper we present the first application of Implicit Stochastic Gradient Descent (ISGD) to train neural networks, a method known in convex optimization to be \emph{unconditionally stable} and \emph{robust to the learning rate}. Our key contribution is a novel layer-wise approximation of ISGD which makes its updates tractable for neural networks. Experiments show that our method is more robust to high learning rates and generally outperforms standard backpropagation on a variety of tasks. 
\end{abstract}

\section{Introduction}
Despite decades of research, most neural networks are still optimized using minor variations on the backpropagation method proposed by Rumelhart, Hinton and Williams in 1986~\cite{rumelhart1986learning}. Since backpropagation is a stochastic first order method, its run time per iteration is independent of the number of training datapoints. It is this key property that makes it able to ingest the vast quantities of data required to train neural networks on complex tasks like speech and image recognition. 

A serious limitation of backpropagation being a first order method is its inability to use higher order information. This leads to multiple problems: the
need to visit similar datapoints multiple times in
order to converge to a good solution, instability
due to ``exploding'' gradients \cite[Sec. 3]{pascanu2013difficulty},
and high sensitivity to the learning rate
\cite[Sec. 11.4.1]{Goodfellow-et-al-2016}. 
A number of different approaches have been suggested to deal with these problems.
Adaptive learning rate methods, like Adam~\cite{kingma2014adam} and
Adagrad~\cite{duchi2011adaptive}, estimate appropriate per-parameter
learning rates; and momentum accelerates backpropagation in a common
direction of descent~\cite{zhang2017yellowfin}. Gradient clipping is a
heuristic which ``clips'' the gradient magnitude at a pre-specified
threshold and has been shown to help deal with exploding
gradients~\cite[Sec. 3]{pascanu2013difficulty}. Although these approaches
partially address the problems of backpropagation, neural network training
remains unstable and highly sensitive to the learning
rate \cite[Sec. 11]{Goodfellow-et-al-2016}.

The key research question here is how to add higher order information to
stabilize backpropagation while keeping the per iteration run time independent of the number of
datapoints.
A technique that has recently emerged that addresses this same question in the context of convex optimization is Implicit Stochastic Gradient Descent (ISGD).
ISGD is known  in convex optimization
to be \emph{robust to the learning rate} and \emph{unconditionally stable} for
convex optimization problems \cite[Sec. 5]{ryu2014stochastic}\cite[Sec. 3.1]{toulis2015scalable}. A natural question is whether ISGD can be used to improve the stability of neural network optimization.

In this paper, we show how ISGD can be applied to neural network training. To
the best of our knowledge this is the first time ISGD has been applied to
this problem.\footnote{\cite{toulisstable} recently remarked that ISGD hasn't yet been applied to neural networks and is an open research~question.} The main challenge in applying ISGD is solving its implicit update
equations. This step is difficult even for most
convex optimization problems. We leverage the special structure of neural
networks by constructing  a novel layer-wise
approximation for the ISGD updates. The resulting algorithm,
\emph{Implicit Backpropagation} (IB), is a good trade-off: it has almost the same run time as
the standard ``Explicit'' Backpropagation (EB), and yet enjoys many of the
desirable features of exact ISGD. IB is compatible with many activation
functions such as the relu, arctan, hardtanh and smoothstep; however, in its
present form, it cannot be
applied to convolutional layers. It is possible to use IB for some 
layers 
and EB for the other layers; thus, IB is partially applicable to virtually all neural network
architectures.  

Our numerical experiments demonstrate that IB is stable for much higher
learning rates as compared 
to EB on classification, autoencoder and music prediction tasks. In all of
these examples the learning rate at which IB begins to diverge is
20\%-200\% higher than for EB. We note that for small-scale classification
tasks EB and IB have similar performance.
IB performs particularly well for
RNNs, where exploding gradients are most troublesome. 
We also investigate IB's compatibility with clipping. We find that IB outperforms EB with clipping on RNNs, where clipping is most commonly used, and that clipping benefits both IB and EB for classification and autoencoding tasks.
Overall, IB is clearly beneficial for RNN training and shows promise for
classification and autoencoder feedforward neural networks. We believe
that more refined implementations of ISGD to neural networks than IB are likely to lead to even better results --- a topic for future research. 

The rest of this paper is structured as follows. Section \ref{ib_sec:ISGD_related} reviews the literature on ISGD and related methods. Section \ref{ib_sec:IB} develops IB as approximate ISGD, with Section \ref{ib_sec:ISGD_updates} deriving IB updates for multiple activation functions. The empirical performance of IB is investigated in Section \ref{ib_sec:experiments} and we conclude with mentions of further work in Section \ref{ib_sec:conclusion}.

\section{ISGD and related methods}\label{ib_sec:ISGD_related}
\subsection{ISGD method}
The standard objective in most machine learning models, including neural
networks, is the ridge-regularized loss
\vspace{-0.2cm}
\begin{equation*}
\ell(\theta) = \frac{1}{N} \sum_{i=1}^N \ell_i(\theta) + \frac{\mu}{2}\|\theta\|_2^2,
\end{equation*}
where $\ell_i(\theta) = \ell_\theta(x_i,y_i)$ is the loss associated with $i^{th}$ datapoint $(x_i,y_i)$ and $\theta$ comprises the weight and bias parameters in the neural network.

 The method ISGD uses to minimize $\ell(\theta)$ is similar to that of
 standard ``Explicit'' SGD (ESGD).   
In each iteration of ESGD, we first sample a random datapoint $i$ and then
update the parameters as $\theta^{(t+1)} = \theta^{(t)} - \eta_t
(\nabla_\theta \ell_i(\theta^{(t)}) + \mu \theta^{(t)})$, where $\eta_t$ is the learning rate
at time~$t$. ISGD also samples a random datapoint $i$ but employs the update $\theta^{(t+1)} = \theta^{(t)} - \eta_t
(\nabla_\theta \ell_i(\theta^{(t+1)}) + \mu \theta^{(t+1)})$, or equivalently,
\begin{equation}\label{ib_eq:Implicit_SGD_formula}
\theta^{(t+1)} = \argmin_\theta \left\{2\eta_t (\ell_i(\theta) +
  \frac{\mu}{2} \|\theta\|_2^2)+ \|\theta - 
  \theta^{(t)}\|_2^2\right\}. 
\end{equation} 

%


The main motivation of ISGD over ESGD is its \emph{robustness to learning rates}, \emph{numerical stability} and \emph{transient convergence behavior}~\cite{bertsekas2011incremental,patrascu2017nonasymptotic,ryu2014stochastic}.
The increased robustness of ISGD over ESGD can be illustrated with a simple quadratic loss $\ell(\theta) = \frac{1}{2}\|\theta\|_2^2$, as displayed in Figure~\ref{fig:SGD_vs_Implicit_SGD}. Here the ISGD step $\theta^{(t+1)} = \theta^{(t)}/(1+ \eta_t)$ is stable for any learning rate whereas the ESGD step $\theta^{(t+1)} = \theta^{(t)}(1 - \eta_t)$ diverges when $\eta_t>2$.

Since ISGD becomes equivalent to ESGD when the learning rate is small, there is no difference in their asymptotic convergence rate for decreasing learning rates. However, it is often the case that in the initial iterations, when the learning rate is still large, ISGD outperforms ESGD.

\begin{figure}[t!]
\centering
\begin{minipage}{.49\textwidth}
  \centering
  \includegraphics[width=.99\linewidth]{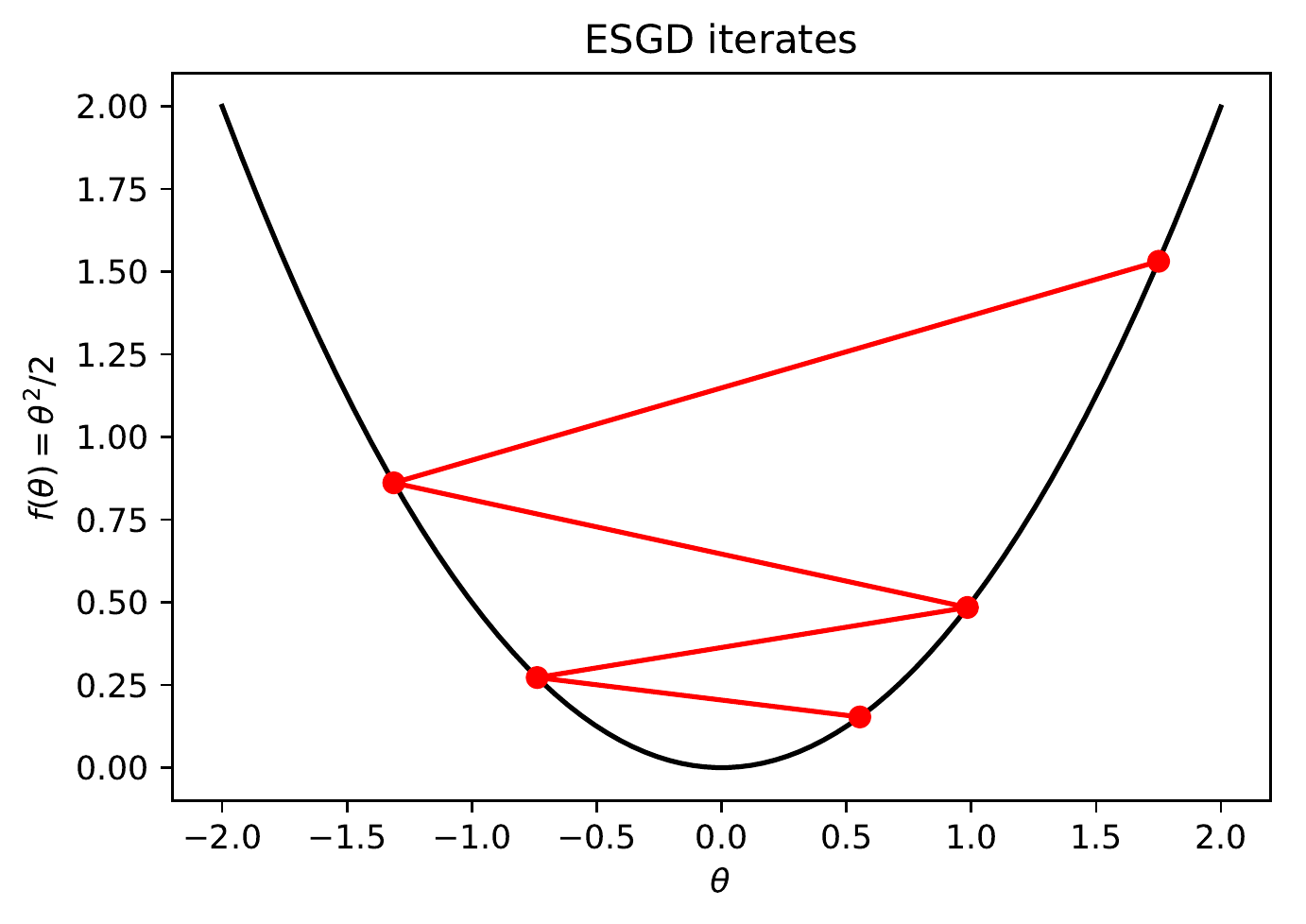}
\end{minipage}%
\hfill
\begin{minipage}{.49\textwidth}
  \centering
  \includegraphics[width=.99\linewidth]{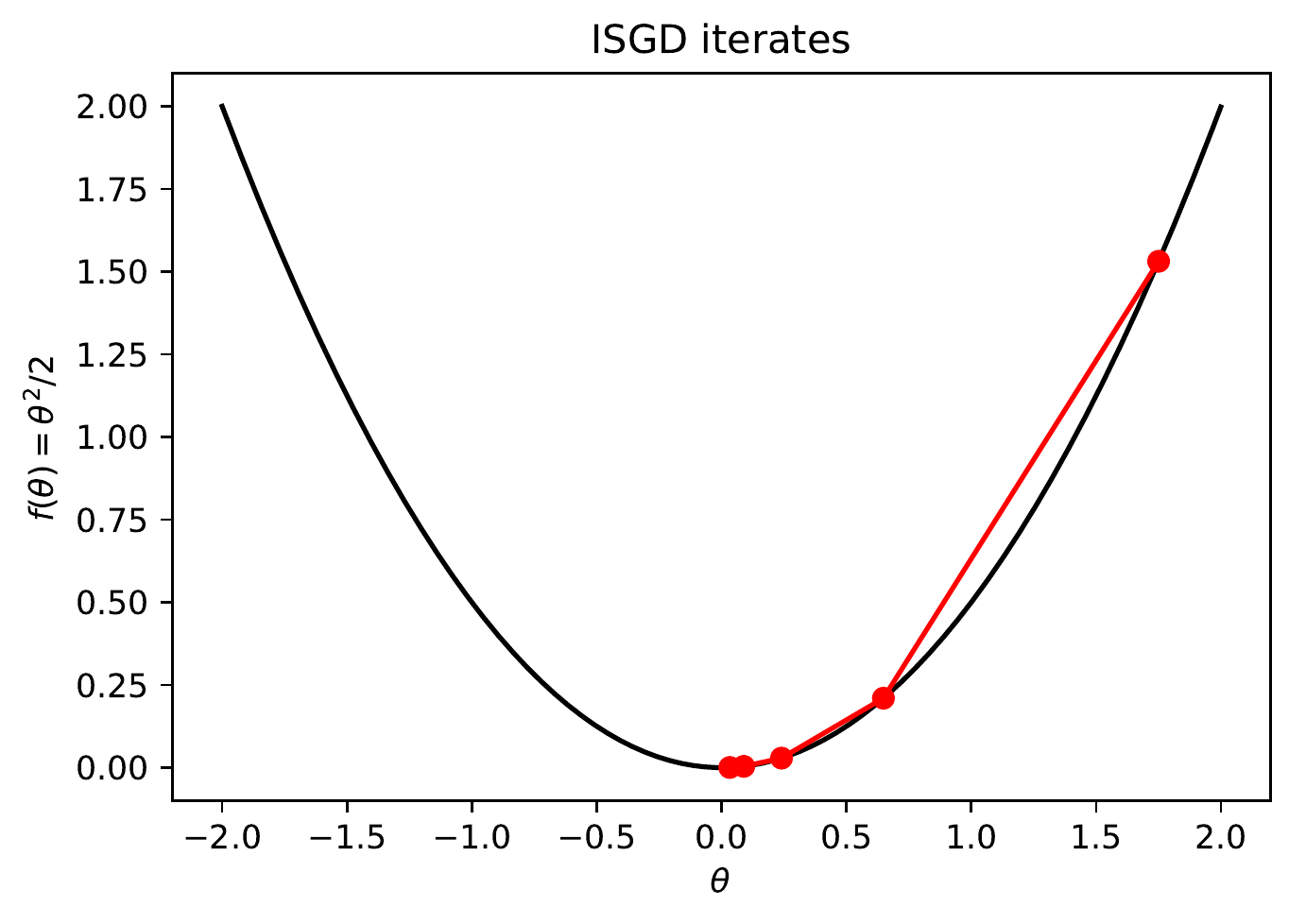}
\end{minipage}
\caption{Illustration of the difference between ESGD and ISGD in optimizing ${f(\theta)=\theta^2/2}$. The learning rate is $\eta=1.75$ in both cases.}
\label{fig:SGD_vs_Implicit_SGD}
\end{figure}

The main drawback of ISGD is that the implicit update 
(\ref{ib_eq:Implicit_SGD_formula}) can be expensive to compute, whereas the
update for ESGD is usually trivial. If this update is expensive, then ESGD 
may converge faster than ISGD in terms of wall clock time, even if ISGD 
converges faster per epoch. Thus, in order for ISGD to be
effective, one needs to be able to solve the
update~\eqref{ib_eq:Implicit_SGD_formula} efficiently. 
Indeed, 
the focus of this paper is to develop a methodology for efficiently
approximating the ISGD update for neural networks.

\subsection{Related methods}
ISGD has been successfully applied to
several machine
learning tasks. Cheng et al. \cite{cheng2007implicit} applied ISGD to learning online kernels and He \cite{he2014stochastic} to SVMs, while Kulis and Bartlett  \cite{kulis2010implicit} consider a range of problems including online metric learning. ISGD has also been used to improve the stability of temporal difference algorithms in reinforcement learning \cite{iwaki2017implicit,tamar2014implicit}.  
For more recent advances in ISGD see \cite{bertsekas2015incremental,lin2017catalyst,paquette2017catalyst,toulis2016towards, wang2017memory}.

Although ISGD has never been applied to neural networks, closely related
methods have been investigated. ISGD may be viewed as a trust-region
method, where $1/2\eta$ 
is the optimal dual variable in the Langrangian, 
\begin{equation*}
\argmin_\theta \left\{\ell_i(\theta) + \frac{\mu}{2}\|\theta\|_2^2: ~ \|\theta - \theta^{(t)}\|_2^2 \leq
  r\right\} = \argmin_\theta \left\{\ell_i(\theta) + \frac{\mu}{2}\|\theta\|_2^2 +
  \frac{1}{2\eta(r)}\|\theta - \theta^{(t)}\|_2^2\right\}. 
\end{equation*}
Trust-region methods for optimizing neural networks have been effectively
used to stabilize policy optimization in reinforcement learning
\cite{schulman2015trust, wu2017scalable}. Clipping, which truncates the
gradient at a pre-specified threshold, may also be viewed as a
computationally efficient approximate trust-region
method~\cite{pascanu2013difficulty}. It was explicitly designed to address
the exploding gradient problem and achieves this by truncating
the exploded step. In the experiments section we will investigate the difference in the effect of IB and clipping. 

An example of a non-trust region method for optimizing neural networks 
using higher order information is Hessian-Free optimization
\cite{martens2010deep,martens2011learning}. These methods directly
estimate second order information of a neural network. They have been
shown to make training more stable and require fewer epochs for
convergence. However they come at a much higher per iteration cost than
first order methods, which offsets this benefit
\cite{bengio2013advances}.

\section{Implicit Backpropagation}\label{ib_sec:IB}
In this section we develop Implicit Backpropagation (IB) as an approximate
ISGD implementation for neural networks. It retains many of the desirable
characteristics of ISGD, while being virtually as fast as the standard
``Explicit'' Backpropagation (EB). 

Consider a $d$-layered neural network 
$f_\theta(x) = f_{\theta_d}^{(d)} \circ f_{\theta_{d-1}}^{(d-1)} \circ
... \circ f_{\theta_1}^{(1)}(x)$  
where $f_{\theta_k}^{(k)}:\mathbb{R}^{D_k} \to  \mathbb{R}^{D_{k+1}}$
represents the $k^{th}$ layer with parameters
$\theta_k\in\mathbb{R}^{P_k}$, ~ $\theta=(\theta_1,...,\theta_d)$,
$x\in\mathbb{R}^{D_1}$ is the input to the neural network, and $\circ$
denotes composition. Let the loss associated with a datapoint $(x,y)$ be
$\ell_\theta(x,y) =~ \ell(y, \cdot) \circ  f_\theta(x)$ for some $\ell(y,
\cdot):\mathbb{R}^{D_{d+1}} \to \mathbb{R}$. Later in this section, we
will want to extract the effect of the $k^{th}$ layer on the loss. To this
end, we can rewrite the loss as 
$\ell_\theta(x,y)  =~ \ell_{\theta_{d:k+1},y}^{(d:k+1)} \circ f_{\theta_{k}}^{(k)}\circ f_{\theta_{k-1:1}}^{(k-1:1)}(x)$
where $f_{\theta_{i:j}}^{(i:j)} = f_{\theta_i}^{(i)} \circ
f_{\theta_{i-1}}^{(i-1)} \circ ... f_{\theta_j}^{(j)}:
\mathbb{R}^{D_{i}}\!\to\!\mathbb{R}^{D_{j+1}}$ and
$\ell_{\theta_{d:j},y}^{(d:j)} = \ell(y, \cdot) \circ
f_{\theta_{d:j}}^{(d:j)}: \mathbb{R}^{D_{j}}\!\to\!\mathbb{R}$. 

The complexity of computing the ISGD update depends on the functional form of the loss
$\ell_\theta(x,y)$. Although it is possible in some cases to compute the
ISGD update explicitly, this is not the case for neural networks. Even
computing the solution numerically is very expensive. Hence,
it is necessary to approximate the ISGD update in order for it to be 
computationally tractable. We introduce the following two approximations in IB:
\begin{enumerate}[(a)]
\item  We update parameters layer by layer. 
When
  updating parameter $\theta_k$ associated with 
  layer $k$, all the parameters $\theta_{-k}= \{\theta_i:~ i\neq k\}$
  corresponding to the other layers are kept fixed. Under this approximation the loss
  when updating the $k^{th}$ layer is  
  \begin{equation}
    \ell_k(\theta_k; x,y,\theta^{(t)}_{-k}) :=
    \ell_{\theta_{d:k+1}^{(t)},y}^{(d:k+1)} \circ
    f_{\theta_{k}}^{(k)}\circ
    f_{\theta^{(t)}_{k-1:1}}^{(k-1:1)}(x).\label{ib_eq:exact_ISGD_layer_loss} 
  \end{equation}
\item Building on (a), 
we linearize the  higher layers $\ell_{\theta^{(t)}_{d:k+1},y}^{(d:k+1)}$, but keep the layer being updated $f_{\theta_{k}}^{(k)}$ as non-linear.
  The loss from  (\ref{ib_eq:exact_ISGD_layer_loss})
  reduces to\footnote{Note that the derivative $\nabla
    \ell_{\theta^{(t)}_{d:k+1},y}^{(d:k+1)}$ is taken with respect to its
    argument $f_{\theta^{(t)}_{k:1}}^{(k:1)}(x)\in\mathbb{R}^{D_{k+1}}$,
    not $\theta_{d:k+1}$.} 
  \begin{equation}
    \tilde{\ell}_k(\theta_k; x,y,\theta^{(t)}_{-k}) :=  \ell_{\theta^{(t)}}(x,y) +  \nabla \ell_{\theta^{(t)}_{d:k+1},y}^{(d:k+1)\top}
    ( f_{\theta_{k}}^{(k)}\circ f_{\theta^{(t)}_{k-1:1}}^{(k-1:1)}(x) - f_{\theta^{(t)}_{k:1}}^{(k:1)}(x)). \label{ib_eq:Taylor_series}
  \end{equation}
  This approximation can be validated via a Taylor series expansion where
  the error in (\ref{ib_eq:Taylor_series}) compared to
  (\ref{ib_eq:exact_ISGD_layer_loss}) is $O(\|\theta_k -
  \theta^{(t)}_k\|_2^2)$. 
\end{enumerate}
The IB approximation to the ISDG update is, thus, given by 
\begin{equation}\label{ib_eq:ISGD_NN_update_equation}
\theta_k^{(t+1)} = \argmin_{\theta_k} \left\{2\eta \left(\tilde{\ell}_k(\theta_k; x,y,\theta^{(t)}_{-k})  + \frac{\mu}{2}\|\theta_k\|_2^2 \right) + \|\theta_k - \theta_k^{(t)}\|_2^2\right\}.
\end{equation}
In Appendix \ref{app:convergence} we present a simple theorem, which
leverages the fact that IB converges to EB in the limit of small learning rates,
to show that IB converges to a stationary point of $\ell(\theta)$ for
appropriately decaying learning rates. 

In the next section we show that the IB update can be efficiently computed for a variety of activation functions. 
The IB approximations thus make ISGD practically feasible to implement.
However, the approximation is not without drawbacks.
The layer-by-layer update from (a) removes all higher order information
 along directions perpendicular to the parameter space of the layer being updated, and the
linearization in (b) loses information about the non-linearity in the
higher layers. The hope is that IB retains enough of the beneficial properties of
ISGD to have noticeable benefits over EB. In our experiments we show
that this is indeed the case.  

Our IB formulation is, to our knowledge, novel. The most similar update in the literature is for composite convex optimization with just two layers,
where the lower, not higher, layer is linearized~\cite{duchi2017stochastic}. 

\section{Implicit Backpropagation updates for various activation functions}\label{ib_sec:ISGD_updates}
Since neural networks are applied to extremely large datasets, it is
important that the IB updates can be computed efficiently.  In this
section we show that the IB update
(\ref{ib_eq:ISGD_NN_update_equation}) can be greatly simplified, resulting in
fast analytical updates for activation functions such as the relu and
arctan. For those activation functions that do not have an analytical IB
update, IB can easily be applied on a piecewise-cubic approximation of the
activation function. This makes IB practically applicable to virtually any
element-wise acting activation function. 

Although it is possible to apply IB to layers with weight sharing or
non-element-wise acting activation functions, the updates tend to be
complex and expensive to compute. For example, the IB update for a
convolutional layer with max-pooling and a relu activation function
involves solving a quadratic program with binary variables (see
Appendix~\ref{app:CNNs} for the derivation). Thus, we will only focus on
updates for activation functions that are applied element-wise and have no
shared weights. 

\subsection{Generic updates}\label{ib_sec:generic_update_statements}
Here we derive the IB updates for a generic layer $k$ with element-wise
acting activation function $\sigma$. Let the parameters in the
$k^{th}$ layer be $\theta_k = (W_k, B_k)\in \mathbb{R}^{D_{k+1} \times
  (1+D_k)}$ where $W_k \in \mathbb{R}^{D_{k+1} \times D_k}$ is the weight
matrix and $B_k \in \mathbb{R}^{D_{k+1}}$ is the bias. We'll use the
shorthand notation $z^{ki} = (f^{(k-1:1)}_{\theta^{(t)}_{k-1:1}}(x_i),~1)
\in \mathbb{R}^{1+D_{k}}$ for the input to the $k^{th}$ layer and $b^{ki}
= \nabla \ell_{\theta^{(t)}_{d:k+1},y}^{(d:k+1)} 
\in \mathbb{R}^{D_{k+1}}$ for the backpropagated gradient. The output of
the $k^{th}$ layer is thus $\sigma(\theta_k
z^{ki})$ where $\sigma$ is applied element-wise and $\theta_k z^{ki}$ is a
matrix-vector product. Using this notation the IB update from  (\ref{ib_eq:ISGD_NN_update_equation}) becomes:
\begin{equation}\label{ib_eq:ISGD_main_starting_point}
\theta^{(t+1)}_k = \argmin_{\theta_k}\left\{2\eta_t b^{ki\top} \sigma(\theta_k z^{ki}) +\eta_t\mu\|\theta_k\|_2^2 + \|\theta_k - \theta_k^{(t)}\|_2^2 \right\},
\end{equation}
where we have dropped the terms $\ell_{\theta^{(t)}}(x,y)$ and $b^{ki\top}
f_{\theta^{(t)}_{k:1}}^{(k:1)}(x)$ from (\ref{ib_eq:Taylor_series}) as they
are constant with respect to $\theta_k$. 
Now that we have written the IB update in more convenient notation, we can
begin to simplify it. Due to the fact that $\sigma$ is applied
element-wise, (\ref{ib_eq:ISGD_main_starting_point}) breaks up into $D_{k+1}$
separate optimization problems, one for each output node
$j\in\{1,...,D_{k+1}\}$: 
\begin{equation}\label{ib_eq:node_j_update}
\theta_{kj}^{(t+1)}  = \argmin_{\theta_{kj}}\left\{2\eta_t b_j^{ki}
  \sigma(\theta_{kj}^\top z^{ki}) +\eta_t\mu\|\theta_{kj}\|_2^2 +
  \|\theta_{kj} - \theta_{kj}^{(t)}\|_2^2 \right\}, 
\end{equation}
where $\theta_{kj} = (W_{kj},B_{kj}) \in\mathbb{R}^{1+D_k}$ are the parameters corresponding to the 
$j^{th}$ output node. Using simple calculus we can write the solution
to $\theta_{kj}^{(t+1)}$ as 
\begin{equation}\label{ib_eq:theta_k_solution}
\theta_{kj}^{(t+1)}  = \frac{\theta^{(t)}_{kj}}{1+\eta_t\mu} 
- \eta_t \alpha^{ki}_jz^{ki}
\end{equation}
where $\alpha^{ki}_j\in\mathbb{R}$ is the solution to the one-dimensional optimization problem
\begin{equation}\label{ib_eq:alpha_original}
\alpha^{ki}_j = \argmin_\alpha \left\{b^{ki}_j\cdot \sigma\left(\frac{\theta^{(t)\top}_{kj} z^{ki}}{1+\eta_t\mu} - \alpha\cdot \eta_t\|z^{ki}\|_2^2 \right) + \eta_t(1+\eta_t\mu) \|z^{ki}\|_2^2 \frac{\alpha^2}{2} \right\}.
\end{equation}
See Appendix~\ref{app:derivation_generic_update} for the derivation. 

To connect the IB update to EB, observe that if we do a first order Taylor expansion of (\ref{ib_eq:theta_k_solution}) and (\ref{ib_eq:alpha_original}) in $\eta_t$ we recover the EB update:
\begin{equation*}\label{ib_eq:EB_first_order_update}
\theta_{kj}^{(t+1)}  = \theta^{(t)}_{kj}(1-\eta_t\mu) 
- \eta_t \sigma'\left(\theta^{(t)\top}_{kj} z^{ki} \right) b^{ki}_j z^{ki} + O(\eta_t^2),
\end{equation*}
where $\sigma'$ denotes the derivative of $\sigma$. Thus we can think of IB as a higher order update than EB.

In summary, the original $D_{k+1}\times D_k$ dimensional IB update from (\ref{ib_eq:ISGD_main_starting_point}) has been reduced to $D_{k+1}$ separate one-dimensional optimization problems in the form of (\ref{ib_eq:alpha_original}). 
The difficulty of solving (\ref{ib_eq:alpha_original}) depends on the
activation function $\sigma$. Since  (\ref{ib_eq:alpha_original}) is a
one-dimensional problem, an optimal $\alpha$ can always be computed numerically using
the bisection method, although this may be slow. Fortunately there are
certain important activation functions for which $\alpha$ can be computed
analytically. In the subsections below we derive analytical updates for
$\alpha$ when $\sigma$ is the relu and arctan functions as well as a
general formula for piecewise-cubic functions. 

Before proceeding to these updates, we can observe directly from
(\ref{ib_eq:theta_k_solution}) and (\ref{ib_eq:alpha_original}) that IB will be
robust to high learning rates. Unlike EB, in which the step size
increases linearly with the learning rate, IB has a bounded step size even
for infinite learning rates. As the learning rate increases
(\ref{ib_eq:theta_k_solution}) becomes 
\begin{equation*}
\theta_{kj}^{(t+1)}  \xrightarrow[]{\eta_t \to \infty} \argmin_\beta \left\{b^{ki}_j\cdot \sigma\left(\beta\cdot \|z^{ki}\|_2^2 \right) + \mu \|z^{ki}\|_2^2 \frac{\beta^2}{2} \right\} z^{ki}
\end{equation*}
where $\beta = -\eta_t \alpha$. This update is finite as long as $\mu >0$
and $\sigma$ is asymptotically linear. 

\subsection{Relu update}\label{ib_sec:Relu_alpha}
Here we give the solution to $\alpha$ from (\ref{ib_eq:alpha_original}) for
the relu activation function, $\sigma(x) = \max\{x,0\}$. We will drop the
super- and sub-scripts from (\ref{ib_eq:alpha_original}) for notational convenience. When $sign(b)=+1$ there are three cases and when $sign(b)=+1$ there are two cases for the solution to $\alpha$. The updates are given in Table \ref{tb:relu}.

\begin{table}[!h!]
\centering
\caption{IB relu updates}
\setlength{\tabcolsep}{12pt} 
\label{tb:relu}
\begin{tabular}{ll}
\toprule
$Sign(b)=+1$ & $Sign(b)=-1$ \tabularnewline
\midrule
$\alpha = \begin{cases}
0 &\mbox{if }\theta^{\top} z\leq 0 \\
\frac{\theta^{\top} z}{1+\eta\mu} \frac{1}{\eta \|z\|_2^2} &\mbox{if }0< \theta^{\top} z\leq \eta \|z\|_2^2b \\
\frac{b}{1+\eta\mu} &\mbox{if } \theta^{\top} z > \eta \|z\|_2^2 b
\end{cases}$ & $\alpha = \begin{cases}
0 &\mbox{if } \theta^{\top} z\leq \frac{1}{2}\eta \|z\|_2^2b \\
\frac{b}{1+\eta\mu} &\mbox{if } \theta^{\top} z > \frac{1}{2}\eta \|z\|_2^2b
\end{cases}$
\tabularnewline
\bottomrule
\end{tabular}
\end{table}


\begin{figure}[b]
\centering
\begin{minipage}{.49\textwidth}
  \centering
  \includegraphics[width=.99\linewidth]{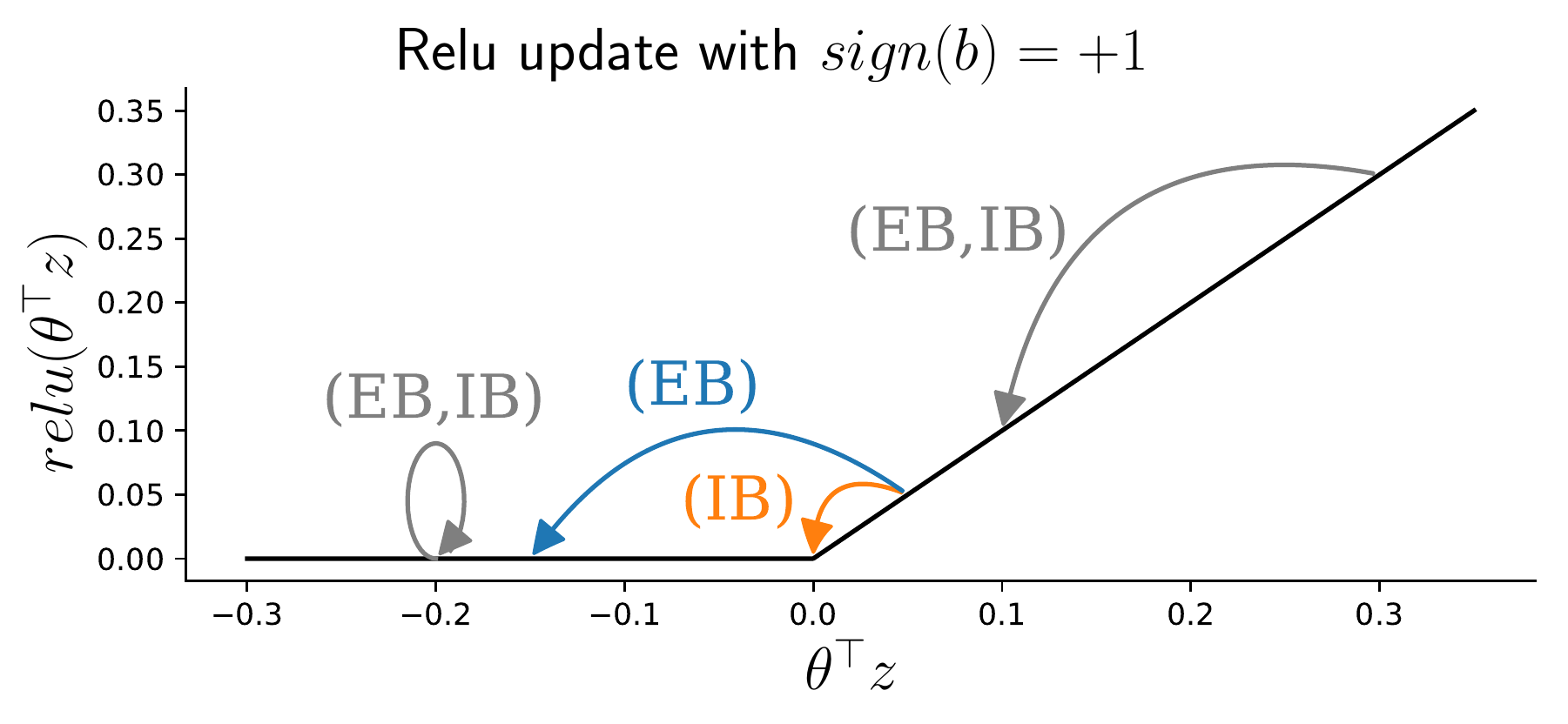}
\end{minipage}%
\hfill
\begin{minipage}{.49\textwidth}
  \centering
  \includegraphics[width=.99\linewidth]{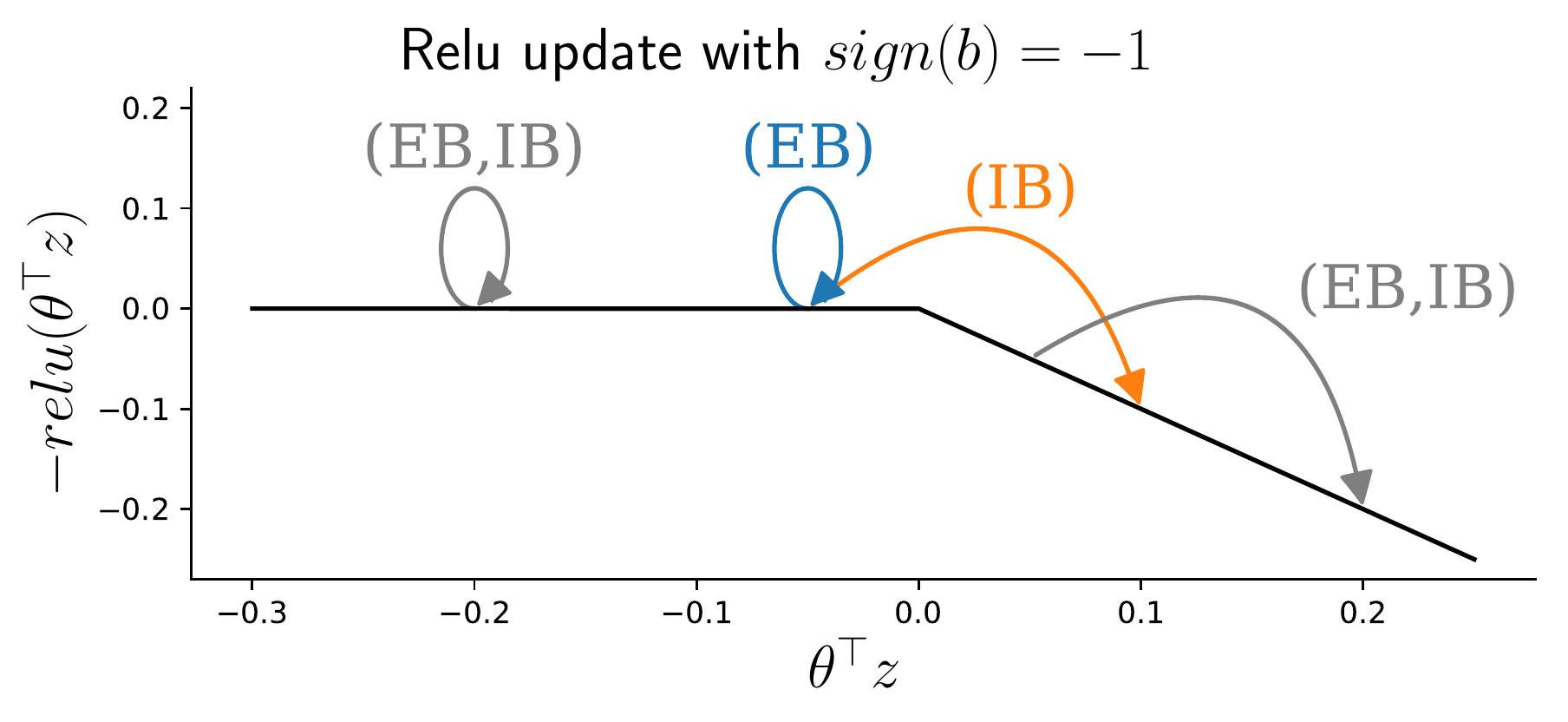}
\end{minipage}
\caption{Relu updates for EB and IB with $\mu=0$. Lower values are better.}
\label{fig:Relu_updates}
\end{figure}

The difference between the EB and IB updates is illustrated in Figure
\ref{fig:Relu_updates}. When $sign(b)=+1$ and $\theta^{\top} z$ is on the
slope but close to the hinge, the EB step overshoots the hinge point,
making a far larger step than is necessary to reduce the relu to 0. IB, on
the other hand, stops at the hinge. The IB update is better from two
perspectives. First, it is able to improve the loss on the given
datapoint just as much as EB, but without taking as large a step. Assuming
that the current value of $\theta$ is close to a minimizer of the average
loss of the other datapoints, an unnecessarily large step will likely take
$\theta$ away its (locally) optimum value. 
An example of where this property might be particularly important is for the
``cliff'' of ``exploding gradient'' problem in RNNs \cite{pascanu2013difficulty}.
And second, the IB step size is a
continuous function of $\theta$, unlike in EB where the step size has a
discontinuity at the origin. This should make the IB update more robust to
perturbations in the data. 
 
When $sign(b)=-1$ and $\theta^{\top} z$ is on the flat, EB isn't able to
descend as the relu has ``saturated'' (i.e. is flat). IB, on the other
hand, can look past the hinge and is still able to descend down the slope,
thereby decreasing the loss. IB thus partially solves the saturating gradient problem \cite{pascanu2013difficulty}.

\subsection{Arctan update}\label{ib_sec:arctan_alpha}
Although the IB update is not analytically available for all sigmoidal
activation functions, it is available when $\sigma$ is the arctan. 
For the arctan
the value of $\alpha$ becomes the root of a
cubic equation which can be solved for analytically. Since the arctan
function is non-convex there may be up to three real solutions for
$\alpha$. Under the axiom that smaller step sizes that achieve the same
decrease in the objective are better (as argued in Section
\ref{ib_sec:Relu_alpha}), we always choose the value of $\alpha$ closest
to zero. 

\subsection{Piecewise-cubic function  update}\label{ib_sec:piecewise}
Many activation functions are piecewise-cubic, including the
hardtanh and smoothstep. Furthermore, all common activation functions can be approximated arbitrarily well with a piecewise-cubic. Being able to do IB updates for piecewise-cubic functions thus extends its applicability to virtually all activation functions. 

Let $\sigma(x) = \sum_{m=1}^MI[B_m\leq x < B_{m+1}]\cdot \sigma_m(x)$ where $\sigma_m(x)$ is a cubic function and $B_m\in[-\infty,\infty]$ defines the bounds on each piece. 
The optimal value of $\alpha$ for each $m$ can be found by evaluating
$\sigma_m$ at its boundaries and stationary points (which may be found by
solving a quadratic equation). The value of $\alpha$ can then be solved for by
iterating over $m=1,...,M$ and taking the minimum over all the
pieces. Since there are $M$ pieces, the time to calculate $\alpha$ scales
as $O(M)$. 

\subsection{Relative run time difference of IB vs EB measured in flops}\label{ib_sec:runtime_flop_bounds}
A crucial aspect of any neural network algorithm is not only its
convergence rate per epoch, but also its run time. Since IB does
extra calculations, it is slower than EB. Here we show that the difference
in floating point operations (flops) between EB and IB is typically small,
on the order of 30\% or less. 

For any given layer, let the input dimension be denoted as $n$ and the output
dimension as $m$. The run time of both EB and IB are dominated by three
operations costing $nm$ flops: multiplying the weight matrix and input in
the forward propagation, multiplying the backpropagated gradient and input
for the weight matrix gradient, and multiplying the backpropagated
gradient and weight matrix for the backpropagated gradient in the lower
layer. IB has two extra calculations as compared to  EB: calculating
$\|z^{ki}\|_2^2$, costing $2n$ flops, and calculating $\alpha_j^{ki}$ a
total of $m$ times, once for each output node. Denoting the number of
flops to calculate each $\alpha_j^{ki}$ with activation function $\sigma$
as $c_\sigma$, the relative increase in run time of IB over EB is upper
bounded by $(2n + c_\sigma m)/(3nm)$. 

The relative run time increase of IB over EB depends on the values of $n,m$
and $c_\sigma$. When $\sigma$ is the relu, then $c_\sigma$ is small, no
more than $10$ flops; whereas when $\sigma$ is the arctan $c_\sigma$ is
larger, costing just less than $100$ flops.\footnote{The number of flops for
  arctan was counted using Cardano's method for optimizing the cubic
  equation.} Taking these flop values as upper bounds, if $n=m=100$ then
the relative run time increase is upper bounded by $4$\% for relu and $34$\%
for arctan. These bounds diminish as $n$ and $m$ grow. If $n=m=1000$, then the bounds become just $0.4$\% for relu and $3.4$\% for
arctan. Thus, IB's run time is virtually the same as EB's for large neural network tasks.

If $n$ and $m$ are small then the arctan IB update might be too slow relative to EB for the IB update to be worthwhile.
In this case simpler sigmoidal activation functions, such as the hardtanh or smoothstep, may be preferable for IB.
The hardtanh has been used before in neural networks, mainly in the context of binarized networks \cite{courbariaux2015binaryconnect}.
It has the form
\begin{equation*}
\sigma(x) = 
\begin{cases}
-1 &\mbox{ if }x<-1\\
x &\mbox{ if }-1\leq x \leq 1\\
1 &\mbox{ if }x>1
\end{cases}
\end{equation*}
for which $c_\sigma$ is no more than 15 flops (using the piecewise-cubic function update from Section~\ref{ib_sec:piecewise}).
The smoothstep is like the hardtanh but is both continuous and has continuous first derivatives,
\begin{equation*}
\sigma(x) = 
\begin{cases}
-1 &\mbox{ if }x<-1\\
\frac{3}{2}x - \frac{1}{2}x^3 &\mbox{ if }-1\leq x \leq 1\\
1 &\mbox{ if }x>1,
\end{cases}
\end{equation*}
with $c_\sigma$ being no more than 25 flops.
The relative increase in run time of IB over EB for the hardtanh and smoothstep is about the same as for the relu.
This makes the IB update with the hardtanh or smoothstep practical even if $n$ and $m$ are small.

\section{Experiments}\label{ib_sec:experiments}
This section details the results of three sets of experiments where the robustness to the learning rate of IB is compared to that of EB.\footnote{A more extensive description of the experimental setup and results are given in Appendices~\ref{app:experimental_setup} and~\ref{app:results}. }
Since IB is equivalent to EB when the learning rate is
small, we expect  little difference between the methods in the limit of small
learning rates. However, we expect that IB will be more stable and have lower loss than EB for larger learning~rates.

\paragraph{Classification, autoencoding and music prediction tasks.}  For the first set of experiments, we applied IB and EB to three different but common machine learning tasks. 
The first task was image classification on the MNIST dataset \cite{lecun1998mnist} with an architecture consisting of two convolutional layers, an arctan layer and a relu
layer. The second task also uses the MNIST dataset, but for an 8 layer relu
autoencoding architecture. The third task involves music prediction
on four music datasets, JSB Chorales, MuseData, Nottingham and Piano-midi.de \cite{boulanger2012modeling}, for which a simple RNN architecture is used with an arctan activation function.

For each dataset-architecture pair we investigated the performance of EB and IB over a range of learning rates where EB performs well (see Appendix \ref{ib_sec:learning_rates} for more details on how these learning rates were chosen). Since IB and EB have
similar asymptotic convergence rates, the difference between the methods will be most
evident in the initial epochs of training. In our experiments we 
only focus on the performance after the first epoch of training for the
MNIST datasets, and after the fifth epoch for the music
datasets.\footnote{The music datasets have fewer training examples and so
  more epochs are needed to see convergence.} Five random seeds are used
for each configuration in order to understand the variance of the performance
of the methods.\footnote{The same seeds are used for both EB and IB. Five seeds are used for all experiments, except MNIST-classification where twenty seeds are used.}

\begin{figure}[t]
\centering
\begin{minipage}{.33\textwidth}
  \centering
  \includegraphics[width=.95\linewidth]{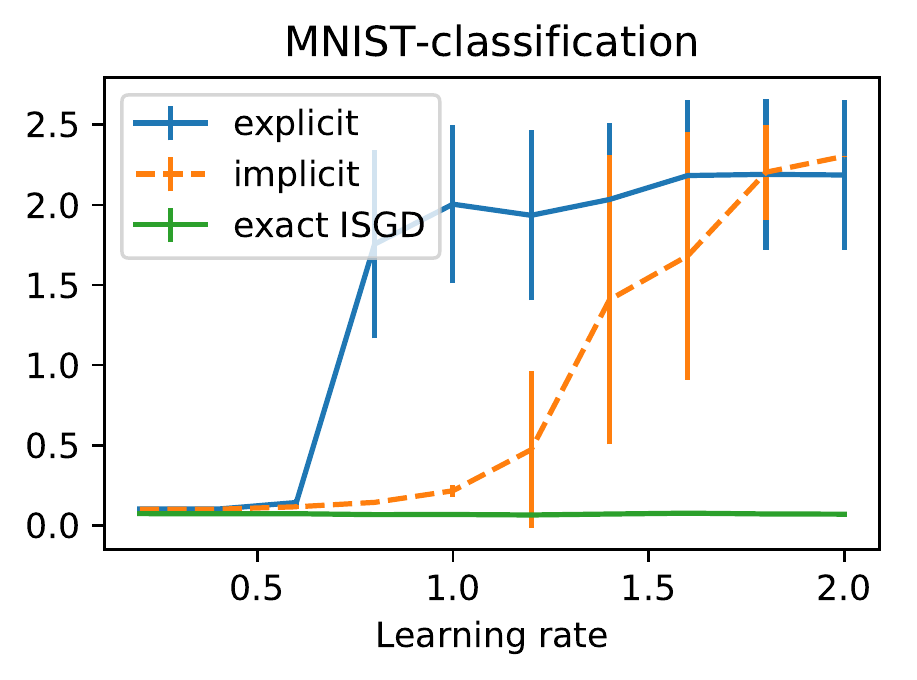}
\end{minipage}%
\hfill
\begin{minipage}{.33\textwidth}
  \centering
  \includegraphics[width=.99\linewidth]{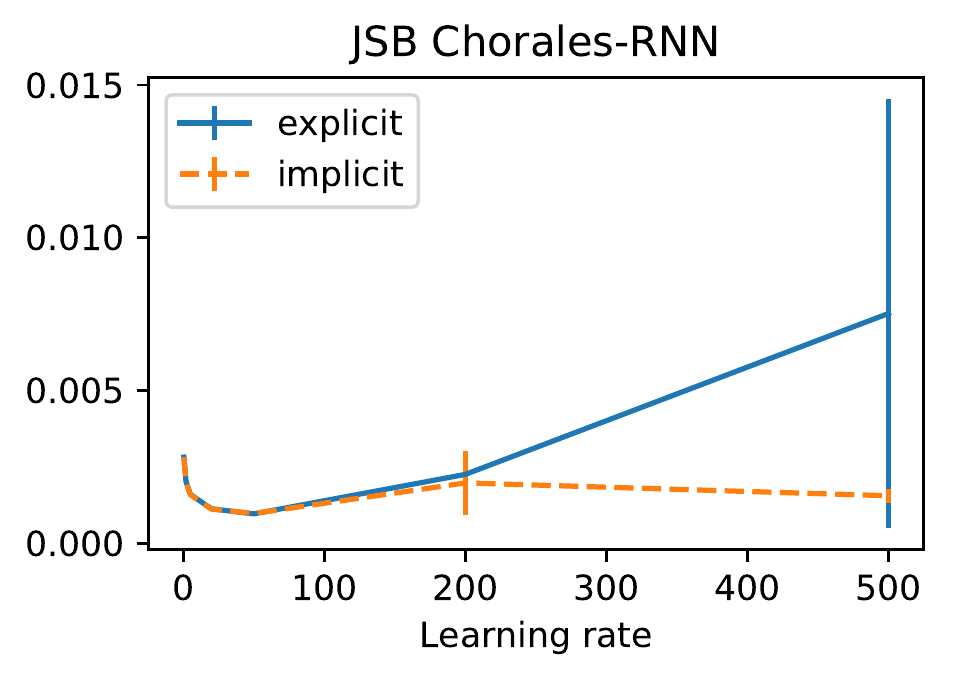}
\end{minipage}
\begin{minipage}{.33\textwidth}
  \centering
  \includegraphics[width=.99\linewidth]{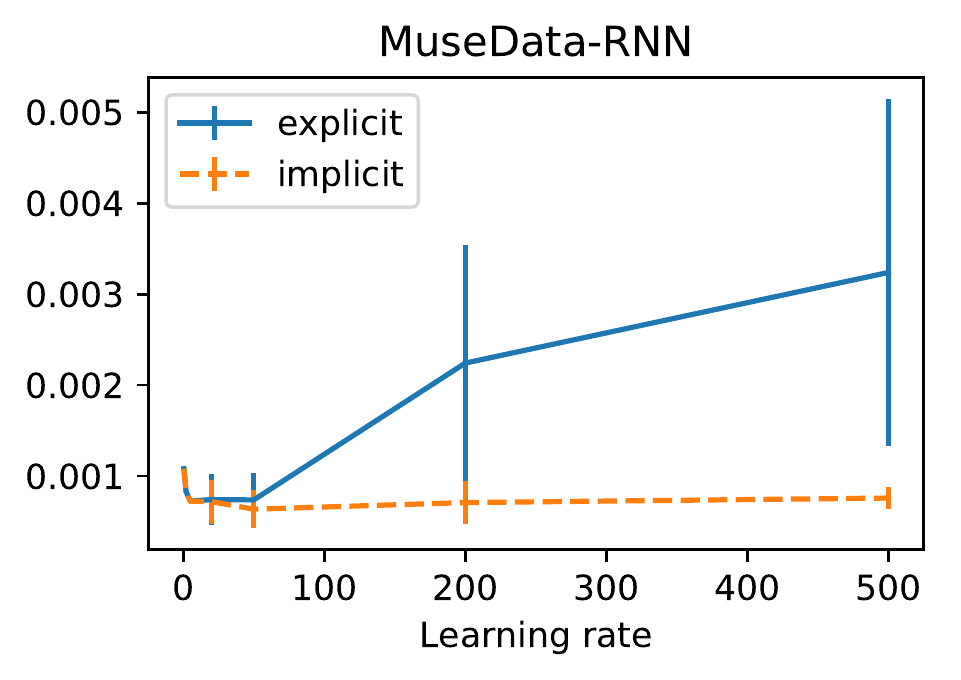}
\end{minipage}%
\hfill
\begin{minipage}{.33\textwidth}
  \centering
  \includegraphics[width=.9\linewidth]{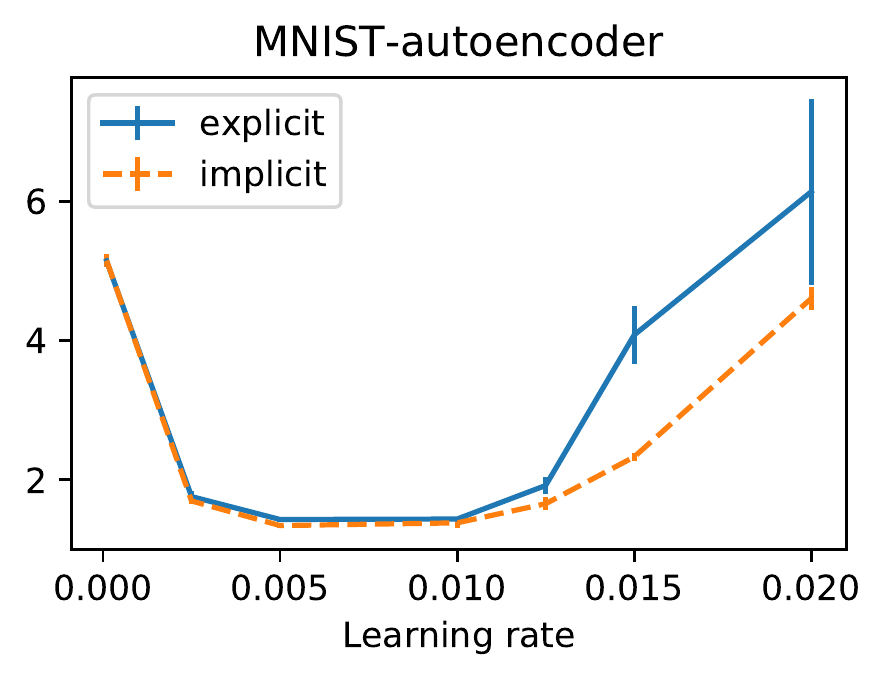}
\end{minipage}
\begin{minipage}{.33\textwidth}
  \centering
  \includegraphics[width=.99\linewidth]{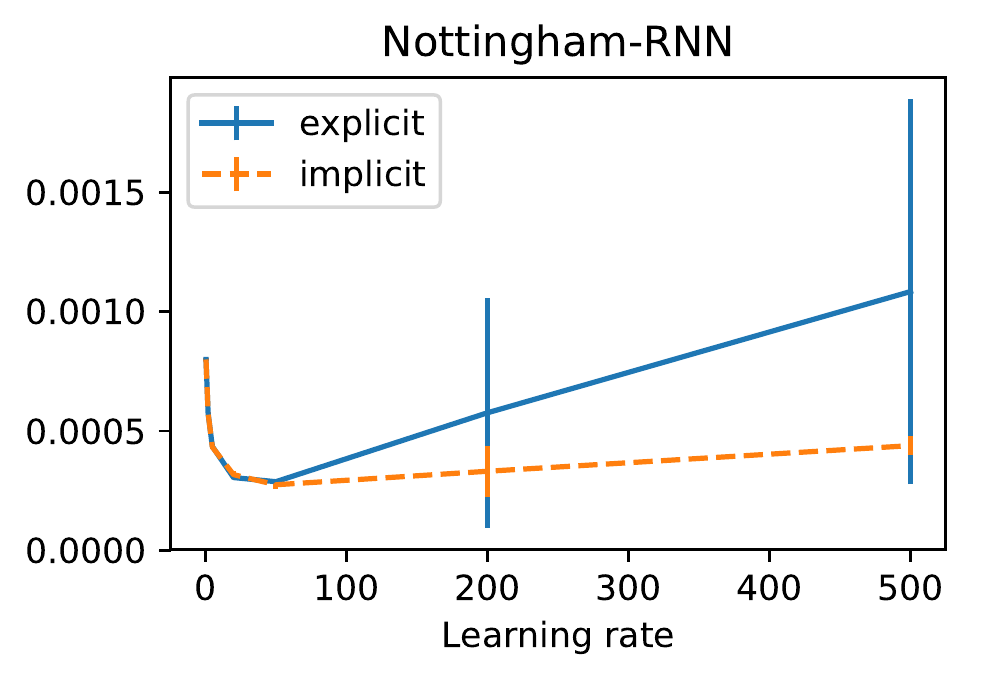}
\end{minipage}%
\hfill
\begin{minipage}{.33\textwidth}
  \centering
  \includegraphics[width=.99\linewidth]{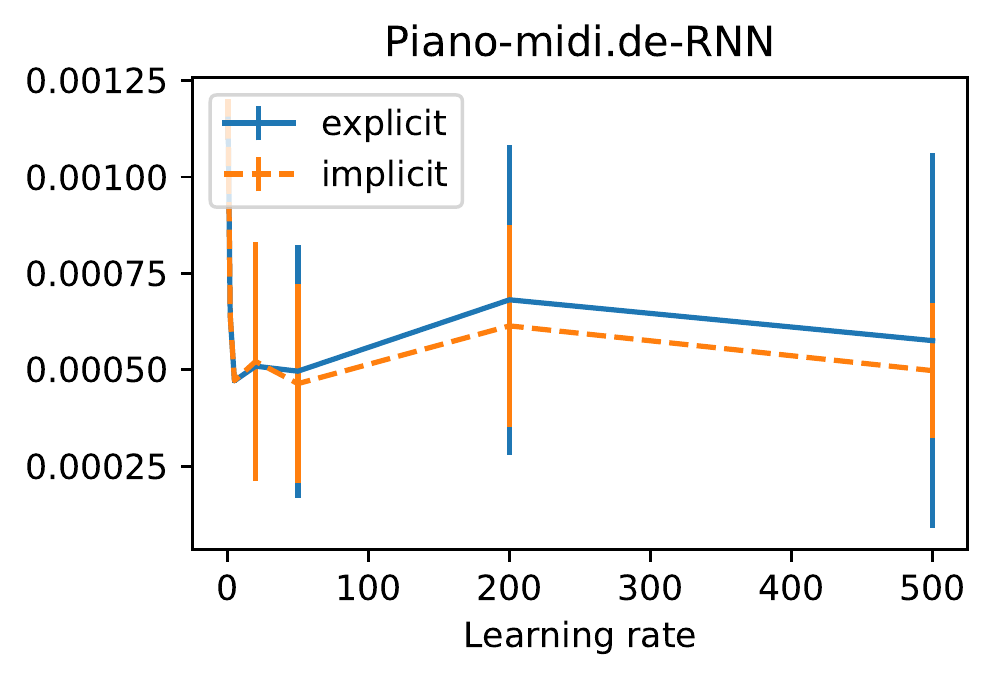}
\end{minipage}
\caption{Training loss of EB and IB (lower is better). The plots display the mean performance and one standard deviation errors. The MNIST-classification plot also shows an ``exact'' version of ISGD with an inner gradient descent optimizer.}
\label{fig:main_results_no_clipping}
\end{figure}

Figure~\ref{fig:main_results_no_clipping} displays the results for the
experiments. EB and IB have near-identical
performance when the learning rate is small. However, as the learning rate
increases, the performance of EB deteriorates far more quickly as compared
to IB. Over the six datasets, the learning rate at which IB starts to
diverge is at least 20\% higher than that of EB.\footnote{The threshold
  for divergence that we use is when the mean loss exceeds the average of
  EB's minimum and maximum mean losses measured on that dataset 
  over the various learning rates.} The
benefit of IB over EB is most noticeable for the music prediction problems
where for high learning rates IB has much better mean performance and
lower variance than EB. A potential explanation for this behaviour is that
IB is better 
able to deal with exploding
gradients, which are more prevalent in RNN training. 

\paragraph{Exact ISGD.} For MNIST-classification we also investigate the potential performance of
\emph{exact} ISGD. Instead of using our IB approximation for the ISGD update, we
directly optimize (\ref{ib_eq:Implicit_SGD_formula}) using gradient
descent. For each ISGD update we take a total of $100$ gradient descent
steps of (\ref{ib_eq:Implicit_SGD_formula}) at a learning rate $10$ times
smaller than the ``outer'' learning rate $\eta_t$. It is evident from
Figure \ref{fig:main_results_no_clipping} that this method achieves the
best performance and is remarkably robust to the learning rate. Since
exact ISGD uses $100$ extra gradient descent steps per iteration, it is 
100 times slower than the other methods, and is thus impractically
slow. However, its impressive performance indicates that 
ISGD-based methods have great potential for neural networks. 

\paragraph{Run times.} According to the bounds derived in Section \ref{ib_sec:runtime_flop_bounds}, the run time of IB should be no more than 12\% longer per epoch than EB on any of our experiments. With our basic Pytorch implementation IB took between 16\% and 216\% longer in practice, depending on the architecture used (see Appendix \ref{app:runtimes} for more details). With a more careful implementation of IB we expect these run times to decrease to at least the levels indicated by the bounds. Using activation functions with more efficient IB updates, like the smoothstep instead of arctan, would further reduce the run time.

\paragraph{UCI datasets.} Our second set of experiments is on $121$ classification datasets from the UCI database \cite{fernandez2014we}. We consider a 4 layer feedforward neural network run for 10 epochs on each dataset. In contrast to the above experiments, we use the same coarse grid of 10 learning rates between 0.001 and 50 for all datasets. 
For each algorithm and dataset the best performing learning rate was found on the training set (measured by the performance on the training set). The neural network trained with this learning rate was then applied to the test set.
Overall we found IB to have a $0.13\%$ higher average accuracy on the test set. 
The similarity in performance of IB and EB is likely due to the small size of the datasets (some datasets have as few as 4 features) and relatively shallow architecture making the network relatively easy to train, as well as the coarseness of the learning rate grid.

\paragraph{Clipping.} In our final set of experiments we investigated the
effect of clipping on IB and EB. Both IB and clipping can be
interpreted as approximate trust-region methods. Consequently, we expect IB
to be less influenced by clipping than EB. 
This was indeed observed in our experiments. A
total of $9$ experiments were run with different clipping thresholds applied
to RNNs on the music datasets (see Appendix \ref{app:experimental_setup} for details). Clipping improved EB's performance for higher learning rates in $7$ out of the $9$ experiments, whereas IB's performance was only improved in $2$. IB without clipping had an equal or lower loss than EB with clipping for all learning rates in all experiments except
for one (Piano-midi.de with a clipping threshold of~0.1). This suggests
that IB is a more effective method for training RNNs than EB with clipping. 

The effect of clipping on IB and EB applied to MNIST-classification and
MNIST-autoencoder is more complicated. In both cases clipping enabled IB
and EB to have lower losses for higher learning rates. For
MNIST-classification it is still the case that IB has uniformly superior
performance to EB, but for MNIST-autoencoder this is reversed. It is not
unsurprising that EB  with  clipping may outperform IB  with  clipping. If the
clipping threshold is small enough then the clipping induced trust region will be
smaller than that induced by IB. This makes EB  with  clipping and IB  with 
clipping act the same for large gradients; however, below the clipping
threshold EB's unclipped steps may be able to make more progress than IB's
dampened steps. See Figure \ref{fig:main_results_with_clipping} for plots
of EB and IB's performance with clipping. 

\begin{figure}[t]
\centering
\begin{minipage}{.3175\textwidth}
  \centering
  \includegraphics[width=.99\linewidth]{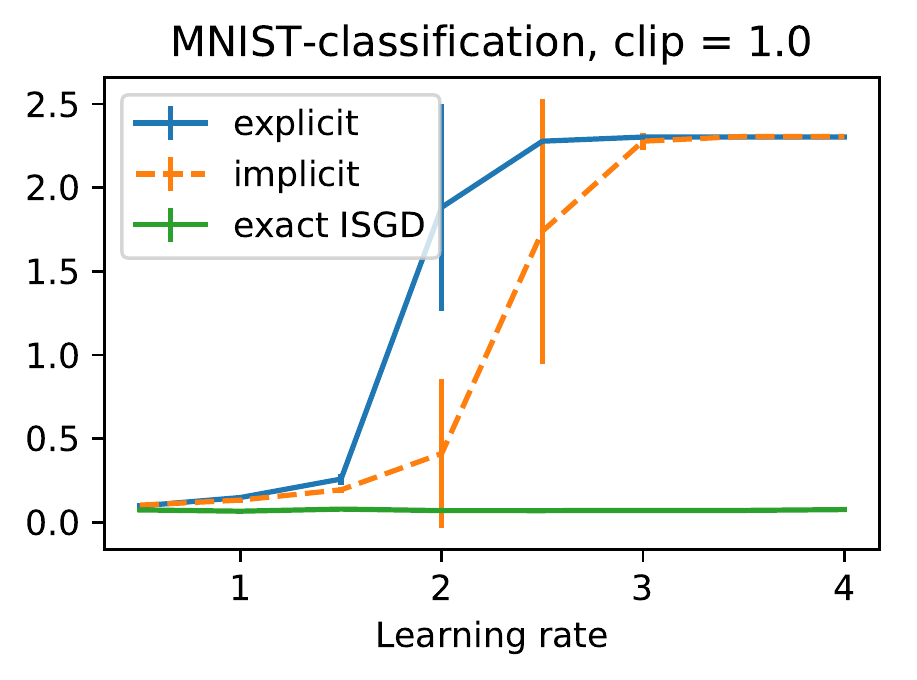}
\end{minipage}%
\hfill
\begin{minipage}{.3175\textwidth}
  \centering
  \includegraphics[width=.99\linewidth]{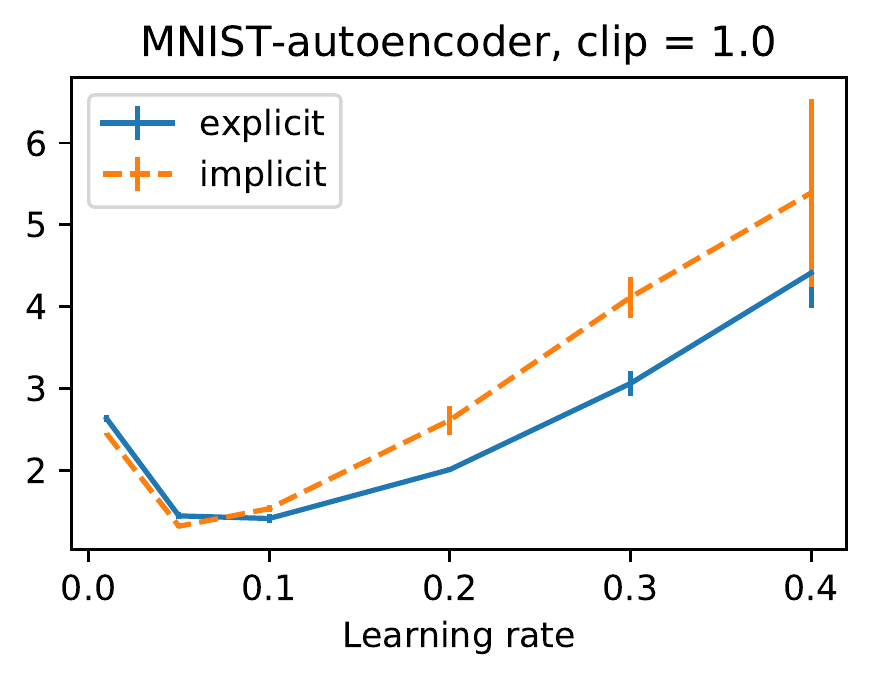}
\end{minipage}
\hfill
\begin{minipage}{.355\textwidth}
  \centering
  \includegraphics[width=.99\linewidth]{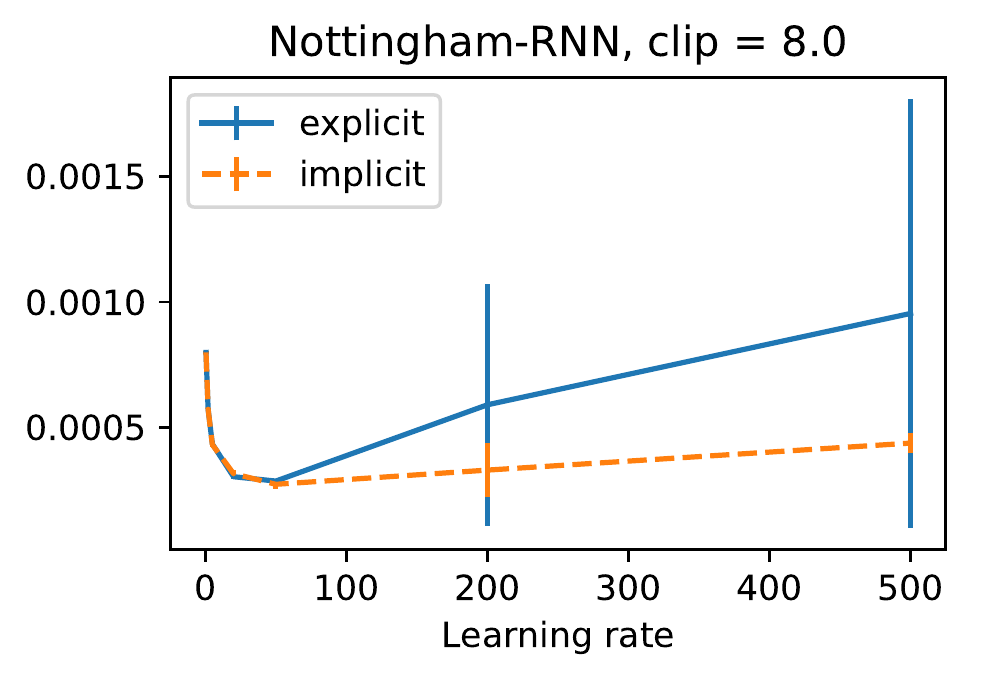}
\end{minipage}
\caption{Training loss of EB  with  clipping and IB  with  clipping on three dataset-architecture pairs. }
\label{fig:main_results_with_clipping}
\end{figure}

\paragraph{Summary.} We ran a total of 17 experiments on the MNIST and music datasets\footnote{MNIST-classification with and without clipping, MNIST-autoencoder with and without clipping, the four music datasets without clipping and nine experiments on the music datasets with clipping.}. IB outperformed EB in 15 out of these. On the UCI datasets IB had slightly better performance than EB on average. The advantage of IB is most pronounced for RNNs, where for large learning rates IB has much lower losses and even consistently outperforms EB with clipping. Although IB takes slightly longer to train, this is offset by its ability to get good performance with higher learning rates, which should enable it to get away with less hyperparameter tuning.

%

\section{Conclusion}\label{ib_sec:conclusion}
In this paper we developed the first method for applying ISGD to neural networks. We showed that, through careful approximations, ISGD can be made to run nearly as quickly as standard backpropagation while still retaining the
property of being more robust to high learning rates. The resulting
method, which we call Implicit Backpropagation, consistently matches or
outperforms standard backpropagation on image recognition, autoencoding
and music prediction tasks; and is particularly effective for robust RNN
training. 

The success of IB demonstrates the potential of ISGD methods
to improve neural network training. It may be the case that there are
better ways to approximate ISGD than IB, which could produce even better results. 
For example, the techniques behind Hessian-Free methods could be used to make a quadratic approximation of the higher layers in IB (opposed to the linear approximation currently used); or a second order approximation could be made directly to the ISGD formulation in~(\ref{ib_eq:Implicit_SGD_formula}).
Developing and testing such methods is a ripe area for future research. 


\clearpage
\bibliographystyle{plain}
\bibliography{isgdnn_bib}

\clearpage
\appendix

\section{Derivation of generic update equations}\label{app:derivation_generic_update}
In this section we will derive the generic IB update equations, starting from equation (\ref{ib_eq:ISGD_main_starting_point}) and ending at equation (\ref{ib_eq:alpha_original}). 
For notational simplicity we will drop superscripts and subscripts where they are clear from the context. Let a tilde denote the current iterate, i.e. $\tilde{\theta}_k =\theta_k^{(t)}$. 
With this notation, the IB update from (\ref{ib_eq:ISGD_main_starting_point}) becomes
\begin{align}
&\argmin_{\theta}\left\{2\eta b^\top \sigma(\theta z) +\eta\mu\|\theta\|_2^2 + \|\theta - \tilde{\theta}\|_2^2 \right\} \label{ib_eq:ISGD_starting_point} \\
&=\argmin_{\theta}\left\{\sum_{j=1}^{D_{k+1}}2\eta b_j \sigma(\theta_j^\top z) +\eta\mu\|\theta_j\|_2^2 + \|\theta_j - \tilde{\theta}_j\|_2^2 \right\}\nonumber
\end{align}
where $\theta_j$ is the $j^{th}$ row of $\theta$ corresponding to the $j^{th}$ node in the layer.
The minimization splits into separate minimization problems, one for each $j$:
\begin{equation}\label{ib_eq:split_eq}
\argmin_{\theta_j}\left\{2\eta b_j \sigma(\theta_j^\top z) +\eta\mu\|\theta_j\|_2^2 + \|\theta_j - \tilde{\theta}_j\|_2^2 \right\}.
\end{equation}
Since $\theta_j\in\mathbb{R}^{1+D_k}$ this is a $1+D_k$-dimensional problem. However, we will be able to reduce it to just a one-dimensional problem. We begin by introducing an auxiliary variable $q_j =\theta_j^\top z$ and rewriting (\ref{ib_eq:split_eq}) as
\begin{equation}\label{ib_eq:regularized_q_equation_original}
\min_{q_j}\left\{2\eta b_j \sigma( q_j) + \min_{\theta_j}\{\eta\mu\|\theta_j\|_2^2 + \|\theta_j - \tilde{\theta}_j\|_2^2:~q_j =\theta_j^\top z\}  \right\}.
\end{equation}
We will first solve the inner minimization over $\theta_j$ as a function of $q_j$, and then solve the outer minimization over $q_j$.

\subsubsection*{Inner minimization}
The inner minimization can be solved by taking the dual:
\begin{align}
&\min_{\theta_j}\{\eta\mu\|\theta_j\|_2^2 + \|\theta_j - \tilde{\theta}_j\|_2^2:~q_j =\theta_j^\top z\}\nonumber\\
&=\max_{\lambda_j\in\mathbb{R}}\min_{\theta_j} \{\eta\mu\|\theta_j\|_2^2 + \|\theta_j - \tilde{\theta}_j\|_2^2 + 2\lambda_j(q_j -\theta_j^\top z)\}\nonumber\\
&=\max_{\lambda_j\in\mathbb{R}} \left\{2\lambda_j q_j +  \min_{\theta_j}\{ \eta\mu\|\theta_j\|_2^2 + \|\theta_j - \tilde{\theta}_j\|_2^2 - 2\lambda_j\theta_j^\top z\}\right\}.\label{ib_eq:mu_theta}
\end{align}
The solution for $\theta_j$ is 
\begin{equation}\label{ib_eq:theta_as_fn_of_u}
\theta_j = \frac{\tilde{\theta}_j + \lambda_j z}{1+\eta\mu}.
\end{equation}
Substituting (\ref{ib_eq:theta_as_fn_of_u}) into (\ref{ib_eq:mu_theta}) and simplifying yields
\begin{align*}
&-\frac{1}{{1+\eta\mu}}\min_{\lambda_j\in\mathbb{R}} \left\{\lambda_j^2\|z\|_2^2 + 2\lambda_j(\tilde{\theta}_j^\top z-(1+\eta\mu) q_j) -\eta\mu\|\tilde{\theta}_j\|_2^2\right\}.
\end{align*}
This is a quadratic in $\lambda_j$, which is easily minimized. The value for $\lambda_j$ at the minimum is, 
\begin{equation}\label{ib_eq:u_solution}
\lambda_j = \frac{(1+\eta\mu) q_j - \tilde{\theta}_j^\top z}{\|z\|_2^2}.
\end{equation}
Substituting (\ref{ib_eq:u_solution}) into (\ref{ib_eq:mu_theta}) yields the minimal value of the inner minimization problem as a function of $q_j$,
\begin{equation}\label{ib_eq:inner_min_solution}
\frac{1+\eta\mu}{\|z\|_2^2} \left( q_j -\frac{\tilde{\theta}_j^\top z}{1+\eta\mu}\right)^2  +\frac{\eta\mu}{1+\eta\mu}\|\tilde{\theta}_j\|_2^2.
\end{equation} 

\subsubsection*{Outer minimization} 
Replacing the inner minimization with (\ref{ib_eq:inner_min_solution}), dropping the constant $\frac{\eta\mu}{1+\eta\mu}\|\tilde{\theta}_j\|_2^2$ term and dividing everything by $2\eta$,  (\ref{ib_eq:regularized_q_equation_original}) becomes 
\begin{align*}
&\argmin_{q_j}\left\{b_j \sigma( q_j) + \frac{1+\eta\mu}{2\eta\|z\|_2^2} \left( q_j -\frac{\tilde{\theta}_j^\top z}{1+\eta\mu}\right)^2   \right\}.
\end{align*}
Reparameterizing $q_j$ as 
\begin{equation}\label{ib_eq:alpha_regularized_def}
\alpha_j = \frac{1}{\eta\|z\|_2^2} \left( \frac{\tilde{\theta}_j^\top z}{1+\eta\mu} - q_j\right),
\end{equation}
we arrive at our simplified update from (\ref{ib_eq:alpha_original}):
\begin{equation*}
\alpha_j = \argmin_\alpha \left\{b_j\cdot \sigma\left(\frac{\tilde{\theta}^{\top}_j z}{1+\eta\mu} - \alpha\cdot \eta\|z\|_2^2 \right) + \eta(1+\eta\mu) \|z\|_2^2 \frac{\alpha^2}{2} \right\}.
\end{equation*}
Once we have solved for $\alpha_j$ we can recover the optimal $\theta_j$ by using (\ref{ib_eq:alpha_regularized_def}) to find $q_j$, (\ref{ib_eq:u_solution}) to find $\lambda_j$ and (\ref{ib_eq:theta_as_fn_of_u}) to find $\theta_j$. The resulting formula is
\begin{equation*}
\theta_{j}  = \frac{\tilde{\theta}_{j}}{1+\eta\mu} - \eta \alpha_j z,
\end{equation*}
as was stated in (\ref{ib_eq:theta_k_solution}).

\section{IB for convolutional neural networks}\label{app:CNNs}
Here we consider applying IB to a Convolutional Neural Network (CNNs). As each filter is applied independently in a CNN, the IB updates decouple into separate updates for each filter. Since a filter uses shared weights, we cannot use the generic update equations in Section~\ref{ib_sec:generic_update_statements}. Instead we have to derive the updates starting from (\ref{ib_eq:ISGD_starting_point})
\begin{equation}\label{ib_eq:CNN_abstract}
\argmin_{\theta}2\eta b^\top \sigma(Z \theta) +\eta\mu\|\theta\|_2^2 + \|\theta - \tilde{\theta}\|_2^2 .
\end{equation}
Here $Z$ is a matrix where each row corresponds to one patch over which the convolution vector $\theta$ is multiplied.\footnote{Note that $Z$ in general will have repeated entries and may have a column of ones appended to it to account for a bias.} The activation function $\sigma$ has components $[\sigma(x)]_m = \max\{B_mx\}$ where $B_m$ is a pooling matrix with elements in $\{0,1\}$. 
We will assume that $B_m$ has a row of all zeros so that the max-pooling effectively includes a relu non-linearity, since $\max\{(B_m^\top,0)^\top x\} = \max\{\mbox{relu}(B_mx)\}$. 

We can expand (\ref{ib_eq:CNN_abstract}) into a quadratic program:
\begin{align*}
&\argmin_{\theta}2\eta \sum_{m=1}^M b_m \max\{B_m Z\theta\} +\eta\mu\|\theta\|_2^2 + \|\theta - \tilde{\theta}\|_2^2 \\
&=\argmin_{a,\theta}2\eta \sum_{m=1}^M |b_m|a_m +\eta\mu\|\theta\|_2^2 + \|\theta - \tilde{\theta}\|_2^2 \\
&~~~~~~s.t.~~~~~~~~a_m \geq 
\begin{cases}
\max\{B_m Z\theta\} &\mbox{if }b_m\geq 0\\
\min\{-B_m Z\theta\} &\mbox{if }b_m< 0
\end{cases}\\
&=\argmin_{a,y,\theta}2\eta \sum_{m=1}^M |b_m|a_m +\eta\mu\|\theta\|_2^2 + \|\theta - \tilde{\theta}\|_2^2 \\
&~~~~~~s.t.~~~~~~~~a_m \geq 
\begin{cases}
B_m Z_j\theta &\mbox{if }b_m\geq 0\\
-B_m Z_j\theta - M(1-y_{mj}) &\mbox{if }b_m< 0
\end{cases}\qquad\mbox{for all }j\\
&~~~~~~~~~~~~~~~~~~\sum_j y_{mj}=1\\
&~~~~~~~~~~~~~~~~~~y_{mj}\in\{0,1\},
\end{align*}
where $M> 0$ is a large constant. This is clearly an expensive problem to solve each iteration. 

Note that if the convolution did not include max-pooling, but just used shared weights, then the problem would become a quadratic program with continuous variables and no constraints, which could be solved analytically. On the other hand if the convolution had max-pooling, but no shared weights, then the generic IB updates from (\ref{ib_eq:theta_k_solution}) would apply.
Thus the difficulty in solving the IB convolutional update comes from doing the max-pooling and weight sharing in the \emph{same} layer.

\section{Convergence theorem}\label{app:convergence}
In this section we present a simple theorem that leverages the fact that IB converges to EB in the limit of small learning rates, to show that IB converges to a stationary point of the loss function for
appropriately decaying learning rates. First we will introduce some useful notation, after which we state the conditions under which the convergence theorem holds. After a few lemmas, we prove the desired result.

\paragraph{Notation.} Let $g_i(\theta^{(t)};\eta_t)$ denote the gradient used to take a step in IB when datapoint $i$ is sampled with learning rate $\eta$, i.e. $\theta^{(t+1)} = \theta^{(t)} - \eta_t g_i(\theta^{(t)};\eta_t)$. 
Let $\ell$ be the loss function from Sections \ref{ib_sec:ISGD_related} and \ref{ib_sec:IB}.
Define the level set $$C= \{\theta: \|\theta\|_2^2 \leq \frac{2}{\mu}\ell(0)\}$$
and the restarting function
\begin{equation*}
R(\theta) = 
\begin{cases}
\theta &\mbox{if }\theta \in C\\
0 &\mbox{otherwise}.
\end{cases}
\end{equation*}
The set $C$ depends on the value of $\ell(0)$. When $\theta=0$ the output of the neural network is independent of its input and so $\ell(0) = \frac{1}{N}\sum_{i=1}^N\ell(y_i,f_0(x_i)) = \frac{1}{N}\sum_{i=1}^N\ell(y_i,f_0(0))$ can be quickly calculated.
Finally define the extended level-set $$\bar{C}(\eta) = \left\{\theta:~\|\theta\|_2\leq \sqrt{\frac{2}{\mu}\ell(0)} + \eta\cdot \max_{i,\theta \in C}\{\|g_i(\theta;\eta)\|_2\}\right\}$$ to contain all points that can be reached from $C$ in one IB iteration (without restarting).

We will assume the following conditions.

\begin{assum}\label{ass:main}
The objective function $\ell(\theta) = \frac{1}{N}\sum_{i=1}^N \ell_\theta(x_i,y_i) + \frac{\mu}{2}\|\theta\|_2^2$, IB gradients $g_i(\theta^{(t)};\eta)$ and learning rate sequence $\{\eta_t\}_{t=1}^\infty$ satisfy the following:
\begin{enumerate}[(a)]

\item The loss at each datapoint is non-negative, i.e. $\ell_\theta(x,y)\geq 0$ for all $x,y,\theta$.

\item The gradient function is Lipschitz continuous with a Lipschitz constant $0<L(\eta)<\infty$ that is monotonically decreasing in $\eta$. That is
\begin{equation*}
\|\nabla \ell(\theta) - \nabla \ell(\bar{\theta})\|_2 \leq L(\eta) \|\theta - \bar{\theta}\|_2,
\end{equation*}
for all $\{\theta, \bar{\theta}\}\subset \bar{C}(\eta)$ and $L(\eta)\leq L(\bar{\eta})$ if $\eta\leq\bar{\eta}$.

\item The gradients of the stochastic functions $\tilde{\ell}_k(\theta_k; x,y,\theta^{(t)}_{-k})$ in (\ref{ib_eq:Taylor_series}) are Lipschitz continuous with a Lipschitz constant $0<\tilde{L}(\eta)<\infty$ that is monotonically decreasing in~$\eta$. That is
\begin{equation*}
\|\nabla \tilde{\ell}_k(\theta_k; x,y,\theta^{(t)}_{-k}) - \nabla \tilde{\ell}_k(\bar{\theta}_k; x,y,\theta^{(t)}_{-k})\|_2 \leq \tilde{L}(\eta) \|\theta_k - \bar{\theta}_k\|_2,
\end{equation*}
for all $k\in\{1,...,d\}$, $(x,y) \in\mathcal{D}$, $\theta^{(t)}\in C$ and $\{\theta, \bar{\theta}\}\subset \bar{C}(\eta)$;  and $\tilde{L}(\eta)\leq \tilde{L}(\bar{\eta})$ if~$\eta\leq\bar{\eta}$.

\item The learning rate sequence is monotonically decreasing with $\eta_1 \tilde{L}(\eta_1)<1$, $\sum_{t=1}^\infty \eta_t = \infty$ and $\sum_{t=1}^\infty \eta_t^2 < \infty$.

\end{enumerate}
\end{assum}
A few comments on the assumptions. Assumption (a) is effectively equivalent to the loss being lower bounded by a deterministic constant $\ell_\theta(x,y)\geq B$, as one can always define an equivalent loss $\ell_\theta'(x,y) = \ell_\theta(x,y)- B\geq 0$. Most standard loss functions, such as the square loss, quantile loss, logistic loss and multinomial loss, satisfy assumption (a). Assumptions (b) and (c) will be valid for any neural network whose activation functions have Lipschitz gradients. Assumption (d) is standard in SGD proofs (except for the $\eta_1 \tilde{L}(\eta_1)<1$ assumption which is particular to us).

Let ``restarting IB'' refer to IB where the restarting operator $R$ is applied each iteration.  We now state the IB convergence theorem:
\begin{theorem}\label{thm:main}
Under Assumptions~\ref{ass:main}, restarting IB converges to a stationary point in the sense that
\begin{equation*}
\lim_{T\to \infty}\frac{\sum_{t=1}^T \eta_t \mathbb{E}[\|\nabla \ell(\theta^{(t)})\|_2^2]}{\sum_{t=1}^T \eta_t} = 0.
\end{equation*}
\end{theorem}
The proof of Theorem \ref{thm:main} is given below, after a few helpful lemmas.

\begin{lemma}\label{lm:restarting}
The restarting operator $R$ does not increase the loss $\ell$.
\end{lemma}

\begin{proof}
If $\theta \in C$ then $\ell(R(\theta)) = \ell(\theta)$ and the loss stays the same, otherwise
\begin{equation*}
\ell(R(\theta)) = \ell(0) < \frac{\mu}{2}\|\theta\|_2^2 \leq \ell(\theta),
\end{equation*}
where the first inequality is by the definition of $C$ and the second is from the assumption that $\ell_\theta(x,y)\geq 0$.
\end{proof}

\begin{lemma}\label{lm:bounded_gradients}
The gradient of the loss function at any point in $C$ is bounded by $B_\ell(\eta) = L(\eta)\sqrt{\frac{2}{\mu}\ell(0)} + \|\nabla \ell(0)\|_2$.
\end{lemma}

\begin{proof}
By the triangle inequality and Lipschitz assumption
\begin{align*}
\|\nabla \ell(\theta)\|_2 &= \|\nabla \ell(\theta) - \nabla \ell(0) + \nabla \ell(0)\|_2\\
&\leq \|\nabla \ell(\theta) - \nabla \ell(0)\|_2 + \|\nabla \ell(0)\|_2\\
&\leq L(\eta)\|\theta - 0\|_2 + \|\nabla \ell(0)\|_2\\
&\leq L(\eta)\sqrt{\frac{2}{\mu}\ell(0)} + \|\nabla \ell(0)\|_2\\
&= B_\ell(\eta)
\end{align*}
for all $\theta\in C$.
\end{proof}

\begin{lemma}\label{lm:bound_inner_product}
If $\|x-y\|_2\leq z$ then for any $v$ we have $y^\top v\geq  x^\top v - z\|v\|_2$.
\end{lemma}

\begin{proof}
\begin{align*}
y^\top v &\geq \min_s \{s^\top v:~\|x-s\|_2^2\leq z^2\}\\
&= \max_{\lambda\geq0} \min_s \{s^\top v + \lambda(\|x-s\|_2^2- z^2)\}\\
& = x^\top v - z\|v\|_2
\end{align*}
where the final line follows from basic algebraic and calculus.
\end{proof}

\begin{lemma}\label{lm:ell_g_difference}
The 2-norm difference between the EB and IB gradients at $\theta^{(t)}$ with learning rate $\eta_t$ is bounded by $\eta_t\tilde{L}(\eta_t)\| g_i(\theta^{(t)};\eta_t)\|_2$.
\end{lemma}

\begin{proof}
Let $\nabla \ell_{ik}(\theta)$ denote the components of $\nabla \ell_i(\theta)$ corresponding to the parameters of the $k^{th}$ layer $\theta_k$, i.e. $\nabla \ell_i(\theta) = (\nabla \ell_{i1}(\theta), ..., \nabla \ell_{id}(\theta))$. 
By construction 
\begin{align*}
\nabla \ell_{ik}(\theta^{(t)}) &= \nabla_{\theta_k} \tilde{\ell}_k(\theta_k; x,y,\theta^{(t)}_{-k})\big|_{\theta_k = \theta_k^{(t)}}\\
g_i(\theta^{(t)};\eta_t) &= \nabla_{\theta_k} \tilde{\ell}_k(\theta_k; x,y,\theta^{(t)}_{-k})\big|_{\theta_k = \theta_k^{(t+1)}}
\end{align*} 
where, in a slight abuse of notation, $\theta^{(t+1)}$ refers to the value of the next IB iterate \emph{before} the application of the restarting operator. 
By the Lipschitz assumption on $\tilde{\ell}_k$ we have
\begin{align*}
\|\nabla \ell_i(\theta^{(t)}) - g_i(\theta^{(t)};\eta_t)\|_2^2 &= \sum_{k=1}^d\|\nabla \ell_{ik}(\theta^{(t)}) - g_{ik}(\theta^{(t)};\eta_t)\|_2^2\\
&= \sum_{k=1}^d\|\nabla_{\theta_k} \tilde{\ell}_k(\theta_k; x,y,\theta^{(t)}_{-k})\big|_{\theta_k = \theta_k^{(t)}} - \nabla_{\theta_k} \tilde{\ell}_k(\theta_k; x,y,\theta^{(t)}_{-k})\big|_{\theta_k = \theta_k^{(t+1)}}\|_2^2\\
& \leq \sum_{k=1}^d\tilde{L}(\eta_t)^2\|\theta_k^{(t)} - \theta_k^{(t+1)}\|_2^2\\
& = \tilde{L}(\eta_t)^2\|\theta^{(t)} - \theta^{(t+1)}\|_2^2\\
& = \tilde{L}(\eta_t)^2\|-\eta_t g_i(\theta^{(t)};\eta_t)\|_2^2\\
& = \eta_t^2\tilde{L}(\eta_t)^2\| g_i(\theta^{(t)};\eta_t)\|_2^2.
\end{align*}

\end{proof}

\begin{lemma}\label{lm:bounded_g}
The 2-norm of the IB gradient $g_i(\theta;\eta_t)$ at any point in $C$ is bounded by $B_g(\eta_t) = \frac{B_\ell(\eta_t)}{1-\eta_t \tilde{L}(\eta_t)}$.
\end{lemma}

\begin{proof}
By Lemmas~\ref{lm:bounded_gradients},~\ref{lm:ell_g_difference} and the triangle inequality,
\begin{align}
\|g_i(\theta;\eta_t)\|_2 &= \|g_i(\theta;\eta_t) - \nabla \ell_i(\theta) + \nabla \ell_i(\theta)\|_2\nonumber\\
&\leq \|g_i(\theta;\eta_t) - \nabla \ell_i(\theta)\|_2 + \|\nabla \ell_i(\theta)\|_2\nonumber\\
&\leq \eta_t \tilde{L}(\eta_t)\| g_i(\theta;\eta_t)\|_2 + B_\ell(\eta_t).\label{ib_eq:lemma_subtract}
\end{align}
Note that $\eta_1 \tilde{L}(\eta_1)<1$ by assumption. Since $\eta_t\leq \eta_1$ and $\tilde{L}(\eta)$ is monotonically decreasing in $\eta$, we have that $\eta_t \tilde{L}(\eta_t)<1$ for all $t\geq 1$. Thus subtracting $\eta_t \tilde{L}(\eta_t)\| g_i(\theta;\eta_t)\|_2$ from both sides of (\ref{ib_eq:lemma_subtract}) and dividing by $1-\eta_t \tilde{L}(\eta_t)>0$ yields the desired result.
\end{proof}


\begin{proof}[Proof of Theorem~\ref{thm:main}]
We upper bound the loss of restarting-IB as
\begin{align}
\mathbb{E}[\ell(\theta^{(t+1)})] &= \mathbb{E}[\ell(R(\theta^{(t)} - \eta_t g_i(\theta^{(t)};\eta_t)))] \nonumber\\
&\leq \mathbb{E}[\ell(\theta^{(t)} - \eta_t g_i(\theta^{(t)};\eta_t))] \nonumber\\
&\leq \mathbb{E}[\ell(\theta^{(t)}) - \eta_t g_i(\theta^{(t)};\eta_t)^\top \nabla \ell(\theta^{(t)}) + \frac{1}{2}\eta_t^2 L(\eta) \|g_i(\theta^{(t)};\eta_t)\|_2^2] \nonumber\\
&= \mathbb{E}[\ell(\theta^{(t)})] - \eta_t \mathbb{E}[g_i(\theta^{(t)};\eta_t)^\top \nabla \ell(\theta^{(t)})] + \frac{1}{2}\eta_t^2 L(\eta) \mathbb{E}[\|g_i(\theta^{(t)};\eta_t)\|_2^2].\label{ib_eq:ell_taylor_proof}
\end{align}
Lets focus on the second term in (\ref{ib_eq:ell_taylor_proof}), $\mathbb{E}[g_i(\theta^{(t)};\eta_t)^\top \nabla \ell(\theta^{(t)})]$. 
By Lemmas~\ref{lm:bound_inner_product} and \ref{lm:ell_g_difference} we have that
\begin{align*}
g_i(\theta^{(t)};\eta_t)^\top \nabla \ell(\theta^{(t)}) \geq \nabla \ell_{i}(\theta^{(t)})^\top \nabla \ell(\theta^{(t)}) - \eta_t \tilde{L}(\eta)\| g_i(\theta^{(t)};\eta_t)\|_2\|\nabla \ell(\theta^{(t)})\|_2
\end{align*}
and so 
\begin{align*}
\mathbb{E}[g_i(\theta^{(t)};\eta_t)^\top \nabla \ell(\theta^{(t)})] &\geq  \|\nabla \ell(\theta^{(t)})\|_2^2- \eta_t \tilde{L}(\eta) \mathbb{E}[\| g_i(\theta^{(t)};\eta_t)\|_2]\|\nabla \ell(\theta^{(t)})\|_2\\
&\geq  \|\nabla \ell(\theta^{(t)})\|_2^2- \eta_t \tilde{L}(\eta) B_g(\eta_t)B_\ell(\eta_t)
\end{align*}
where $B_\ell(\eta_t)$ is as defined in Lemma~\ref{lm:bounded_gradients} and $B_g(\eta_t)$ in Lemma~\ref{lm:bounded_g}.
Moving onto the third term in (\ref{ib_eq:ell_taylor_proof}), we have 
\begin{equation*}
\mathbb{E}[\|g_i(\theta^{(t)};\eta_t)\|_2^2]  \leq B^2_g(\eta_t)
\end{equation*}
by Lemma~\ref{lm:bounded_g}. 
Putting all the terms in (\ref{ib_eq:ell_taylor_proof}) together
\begin{align*}
\mathbb{E}[\ell(\theta^{(t+1)})] &\leq \mathbb{E}[\ell(\theta^{(t)})] - \eta_t (\|\nabla \ell(\theta^{(t)})\|_2^2- \eta_t \tilde{L}(\eta_t) B_g(\eta_t)B_\ell(\eta_t)) + \frac{1}{2}\eta_t^2 L(\eta_t) B^2_g(\eta_t)\\
&= \mathbb{E}[\ell(\theta^{(t)})] - \eta_t \|\nabla \ell(\theta^{(t)})\|_2^2+ \eta_t^2(\tilde{L}(\eta_t)B_g(\eta_t) B_\ell(\eta_t)  + \frac{1}{2}L(\eta_t)B^2_g(\eta_t) )\\
&\leq \mathbb{E}[\ell(\theta^{(t)})] - \eta_t \|\nabla \ell(\theta^{(t)})\|_2^2+ \eta_t^2(\tilde{L}(\eta_1)B_g(\eta_1) B_\ell(\eta_1)  + \frac{1}{2}L(\eta_1)B_g(\eta_1)^2 )
\end{align*}
where we have used the assumption that $L$ and $\tilde{L}$ are monotonically decreasing in $\eta$ (that $B_\ell$ and $B_g$ are also monotonically decreasing in $\eta$ follows from the assumptions on $L$ and $\tilde{L}$). Using a telescoping sum and rearranging yields
\begin{align}
\sum_{t=1}^T \eta_t \mathbb{E}[\|\nabla \ell(\theta^{(t)})\|_2^2] &\leq \mathbb{E}[\ell(\theta^{(1)})] - \mathbb{E}[\ell(\theta^{(T+1)})]\nonumber\\
&\qquad+(\tilde{L}(\eta_1)B_g(\eta_1) B_\ell(\eta_1)  + \frac{1}{2}L(\eta_1)B_g(\eta_1)^2 )\sum_{t=1}^T \eta_t^2 \nonumber\\
&\leq \mathbb{E}[\ell(\theta^{(1)})]+(\tilde{L}(\eta_1)B_g(\eta_1) B_\ell(\eta_1)  + \frac{1}{2}L(\eta_1)B_g(\eta_1)^2 )\sum_{t=1}^T \eta_t^2 \label{ib_eq:bounded_terms}
\end{align}
where the second inequality follows from the assumption that  $\ell(\theta) \geq 0$. Both of the terms in (\ref{ib_eq:bounded_terms}) are deterministically bounded for all $T$ and so it must be the case that 
\begin{equation*}
\sum_{t=1}^\infty \eta_t \mathbb{E}[\|\nabla \ell(\theta^{(t)})\|_2^2] < \infty.
\end{equation*}
Finally, by the assumption that $\sum_{t=1}^\infty \eta_t = \infty$ we have
\begin{equation*}
\lim_{T\to \infty}\frac{\sum_{t=1}^T \eta_t \mathbb{E}[\|\nabla \ell(\theta^{(t)})\|_2^2]}{\sum_{t=1}^T \eta_t} = 0 .
\end{equation*}

\end{proof}

\section{Experimental setup}\label{app:experimental_setup}
In this section we describe the datasets, hyperparameters, run times and results of the experiments from Section \ref{ib_sec:experiments} in greater detail. 

\subsection{Datasets}
The experiments use three types of dataset: MNIST for image classification and autoencoding, 4 polyphonic music dataset for RNN prediction and 121 UCI datasets for classification. These are standard benchmark datasets that are often used in the literature: MNIST is arguably the most used dataset in machine learning \cite{lecun1998mnist}, the polyphonic music datasets are a standard for testing real world RNNs \cite{martens2011learning, pascanu2013difficulty, boulanger2012modeling} and the  UCI datasets are an established benchmark for comparing the performance of classification algorithms \cite{fernandez2014we, klambauer2017self}.

The sources of the datasets are given in Table \ref{tb:data_sources} along with a basic description of their characteristics in Table \ref{tb:characteristics}. The MNIST and UCI datasets were pre-split into training and test sets. For the music datasets we used a random 80\%-20\% train-test split.

\begin{table}[h]
\centering
\caption{Data sources}
\label{tb:data_sources}
\begin{tabular}{p{0.3\textwidth}p{0.7\textwidth}}
\toprule
{Dataset} & {Source url} \tabularnewline
\midrule
MNIST & \url{http://yann.lecun.com/exdb/mnist/}
\tabularnewline
Polyphonic music & \url{http://www-etud.iro.umontreal.ca/~boulanni/icml2012}. \tabularnewline
UCI classification & \url{https://github.com/bioinf-jku/SNNs}. \tabularnewline
\bottomrule
\end{tabular}
\end{table}

\begin{table}[h]
\centering
\caption{Data characteristics. An asterisk $^\ast$ indicates the average length of the musical piece. 88 is the number of potential notes in each chord. For information on the UCI classification datasets, see \cite[Appendix 4.2]{klambauer2017self}.\\}
\label{tb:characteristics}
\begin{tabular}{lllll}
\toprule
{Dataset} & {Train} & {Test} & {Dimension} & {Classes} \tabularnewline
\midrule
MNIST & 60,000 & 10,000 & 28x28 & 10 \tabularnewline
JSB Chorales & 229 ($60^\ast$) & 77 ($61^\ast$) &  88 &  88  \tabularnewline
MuseData & 524 ($468^\ast$) & 124 ($519^\ast$) &  88 & 88
\tabularnewline
Nottingham & 694 ($254^\ast$) & 170 ($262^\ast$) &  88 & 88  \tabularnewline
Piano-midi.de & 87 ($873^\ast$) & 25 ($761^\ast$) &  88 & 88  \tabularnewline
\bottomrule
\end{tabular}
\end{table}

\subsection{Loss metrics}
For all of the experiments we used standard loss metrics. For image classification we used the cross-entropy and for autoencoders the mean squared error. As suggested by Boulanger-Lewandowski et al.  \cite{boulanger2012modeling} on their website \url{http://www-etud.iro.umontreal.ca/~boulanni/icml2012}, we used the expected frame level accuracy as the loss for the music prediction problems. That is, at each time step the RNN outputs a vector in $[0,1]^{88}$ representing the probability of each of the 88 notes in the next chord being played or not. This naturally defines a likelihood of the next chord. We use the log of this likelihood as our loss. Also as suggested by Boulanger-Lewandowski et al.   \cite{boulanger2012modeling}, the loss for each piece is divided by its length, so that the loss is of similar magnitude for all pieces.

\subsection{Architectures}
We endeavored to use standard architectures so that the experiments would be a fair representation how the algorithms might perform in practical applications. The architecture for each type of dataset is given below.

\paragraph{MNIST classification. } The architecture is virtually identical to that given in the MNIST classification Pytorch tutorials available at \url{https://github.com/pytorch/examples/blob/master/mnist/main.py}.
The input is a $28 \times 28$ dimensional image. We first apply two 2d-convolutions with a relu activation function and max-pooling. The convolution kernel size is 5 with 10 filters for the first convolution and 20 filters for the second, and the max-pooling kernel size is 2. After that we apply one arctan layer with input size 320 and output size 50, followed by dropout at a rate of 0.5 and a relu layer with output size 10. This is fed into the softmax function to get probabilities over the 10 classes.

Since IB is not applicable to convolutional layers, for the IB implementation we use EB updates for the convolutional layers and IB for the arctan and relu layers. For the EB implementation, we use EB updates for all of the layers.

\paragraph{MNIST autoencoder. } The autoencoder architecture just involves relus and has a 784:500:300:100:30:100:300:500:784 structure. This is similar to, but deeper than, the structure used in \cite{martens2010deep}.

\paragraph{Music prediction. } The simple RNN architecture used is identical to that in \cite{pascanu2013difficulty}, except for the fact that we use an arctan activation function instead of tanh (since IB is compatible with arctan but not with tanh). 300 hidden units are used.

\paragraph{UCI datasets. } The architecture used for the UCI datasets consists of three arctan layers followed by a linear layer and the softmax. The number of nodes in the arctan layers are equal to the number of features in the dataset, with the final linear layer having input size equal to the number of features and output size equal to the number of classes. This architecture is based on \cite{klambauer2017self}, who also use the same UCI datasets with a similar architecture.

\subsection{Hyperparameters and initialization details}
The hyperparameters and parameter initializations schemes used are as follows:
\begin{itemize}
\item Batch size: 100 (except for RNNs which had a batch-size of 1)
\item Dropout: None (unless otherwise stated)
\item Momentum: None
\item Ridge-regularization ($\mu$): 0
\item Weight matrix initialization:  Each element is sampled independently from \\$unif\left(-\sqrt{\frac{6}{n+m}},\sqrt{\frac{6}{n+m}}\right)$, where $n$ = input size, $m$ = output size to layer. This follows advice of \cite[p. 299]{Goodfellow-et-al-2016} 
\item Bias initialization: 0. Again, this follows advice of \cite[p. 299]{Goodfellow-et-al-2016}
\end{itemize}

\subsection{Learning rates}\label{ib_sec:learning_rates}
The process used to decide on the learning rates for the MNIST and music experiments is as follows. First a very coarse grid of learning rates was tested using EB to ascertain the range of learning rates where EB performed reasonably well. We then constructed a finer grid of learning rates around where EB performed well. The finer grid was constructed so that EB demonstrated in a U-shape of losses, with lower and higher learning rates having higher losses than learning rates in the middle, as in \cite[Fig 11.1, p. 425]{Goodfellow-et-al-2016}. It was this finer grid that was used to generate the final results. Note that at no stage was IB involved in constructing the grid and thus, if anything, the grid is biased in favor of~EB.

For the UCI datasets the following set of learning rates were used: 0.001, 0.01, 0.1, 0.5, 1.0, 3.0, 5.0, 10.0, 30.0, 50.0. This is quite a coarse grid over a very large range of values.

\subsection{Clipping}
Clipping was not used except for those experiments where the effect of clipping was explicitly being investigated. When clipping was used,  the gradients were clipping according to their norm (opposed to each component being clipped separately). For EB we applied clipping in the usual way: first calculating the gradient, clipping it and then taking the step using the clipped gradient. To implement clipping for IB we used an alternative definition of the IB step from  (\ref{ib_eq:Implicit_SGD_formula}): $\theta^{(t+1)} = \theta^{(t)} - \eta_t (\nabla_\theta \ell_i(\theta^{(t+1)}) +\mu \theta^{(t+1)})$, where we highlight that the gradient is evaluated at the next, not current, value of $\theta$. The IB gradient can be inferred using
\begin{equation}\label{ib_eq:IB_gradient}
\nabla_\theta \ell_i(\theta^{(t+1)})  +\mu \theta^{(t+1)} = (\theta^{(t)} - \theta^{(t+1)})/\eta_t
\end{equation}
where $\theta^{(t+1)}$ is calculated using (\ref{ib_eq:ISGD_NN_update_equation}). When applying clipping to IB we first calculate $\theta^{(t+1)}$ using the IB update, infer the IB gradient using (\ref{ib_eq:IB_gradient}), clip it, and finally take a step equal to the learning rate multiplied by the clipped gradient.


\section{Results}\label{app:results}
Here we will give the full results of all of the experiments. We begin with giving the run times of each experiment after which we present  the performance on the MNIST and music datasets. Finally we give results on the UCI datasets.

\subsection{Run times}\label{app:runtimes}
In Section \ref{ib_sec:runtime_flop_bounds} upper bounds on the relative increase in run time for IB as compared to EB were derived. The bounds  for each experiment are displayed in Table \ref{tbl:runtimes} along with the empirically measured relative run times of our basic Pytorch implementation. The Pytorch run times are higher than in the theoretical upper bounds. This shows that IB could be more efficiently implemented.

\begin{table}[t]
\centering
\caption{Theoretical and empirical relative run time of IB vs EB. Empirical measured on AWS p2.xlarge with our basic Pytorch implementation.\\}
\label{tbl:runtimes}
\begin{tabular}{lcc}
{Dataset} & {Theoretical upper bound} & {Empirical}\tabularnewline
\midrule
MNIST classification & 6.27\% &16.82\%\tabularnewline
MNIST autoencoder & 0.60\% &99.68\% \tabularnewline
JSB Chorals & 11.33\% &152.51\% \tabularnewline
MuseData & 11.33\% &207.48\% \tabularnewline
Nottingham & 11.33\% &215.85\% \tabularnewline
Piano-midi.de & 11.33\% &213.98\% \tabularnewline
UCI classification sum & - &58.99\% \tabularnewline
\bottomrule
\end{tabular}
\end{table}



\subsection{Results from MNIST and music experiments}\label{app:all_plots}
In this section the plots for all of the experiments are presented. For each experiment we have multiple plots: a line plot showing the mean performance along with 1 standard deviation error bars; a scatter plot showing the performance for each random seed; and, where applicable, line and scatter plots showing the difference between the experiment with and without clipping. In general 5 seeds are used per learning rate, although for MNIST-classification we use 20 seeds.

\vspace{3cm}

The experiments are presented in the following order
\begin{enumerate}
\item MNIST classification without clipping
\item MNIST classification with clipping threshold = 1.0
\item MNIST autoencoder without clipping
\item MNIST autoencoder with clipping threshold = 1.0
\item JSB Chorales  without clipping
\item JSB Chorales with clipping threshold = 8.0
\item JSB Chorales with clipping threshold = 1.0
\item JSB Chorales with clipping threshold = 0.1
\item MuseData without clipping
\item MuseData with clipping threshold = 8.0
\item Nottingham without clipping
\item Nottingham with clipping threshold = 8.0
\item Piano-midi.de without clipping
\item Piano-midi.de with clipping threshold = 8.0
\item Piano-midi.de with clipping threshold = 1.0
\item Piano-midi.de with clipping threshold = 0.1
\item Piano-midi.de with clipping threshold = 0.01
\end{enumerate}

For the MNIST experiments we only used a clipping threshold of 1.0.
 For the music datasets we first use a clipping threshold of 8.0 as suggested by the authors of \cite{pascanu2013difficulty}. As this clipping threshold didn't much affect the performance of either EB or IB, we also considered lower clipping thresholds.

It is evident from the plots without clipping that on Piano-midi.de the algorithms are less well converged than the other datasets (which is probably due to Piano-midi.de having only 87 training datapoints). It was therefore of interest to see if further clipping could help stabilize the algorithms on Piano-midi.de. We can see from the random seed scatter plots that the effect of clipping helped a little with stabilization, but not to the extent that the results were significantly better. As JSB Chorales has the second fewest number of datapoints, we also tried lower clipping thresholds on it, and found it to often hurt the performance of EB as much as it helped (depending on the random seed).


\clearpage
\subsection{MNIST classification}
\begin{figure}[h]
\centering
\begin{minipage}{.49\textwidth}
  \centering
  \includegraphics[width=.79\linewidth]{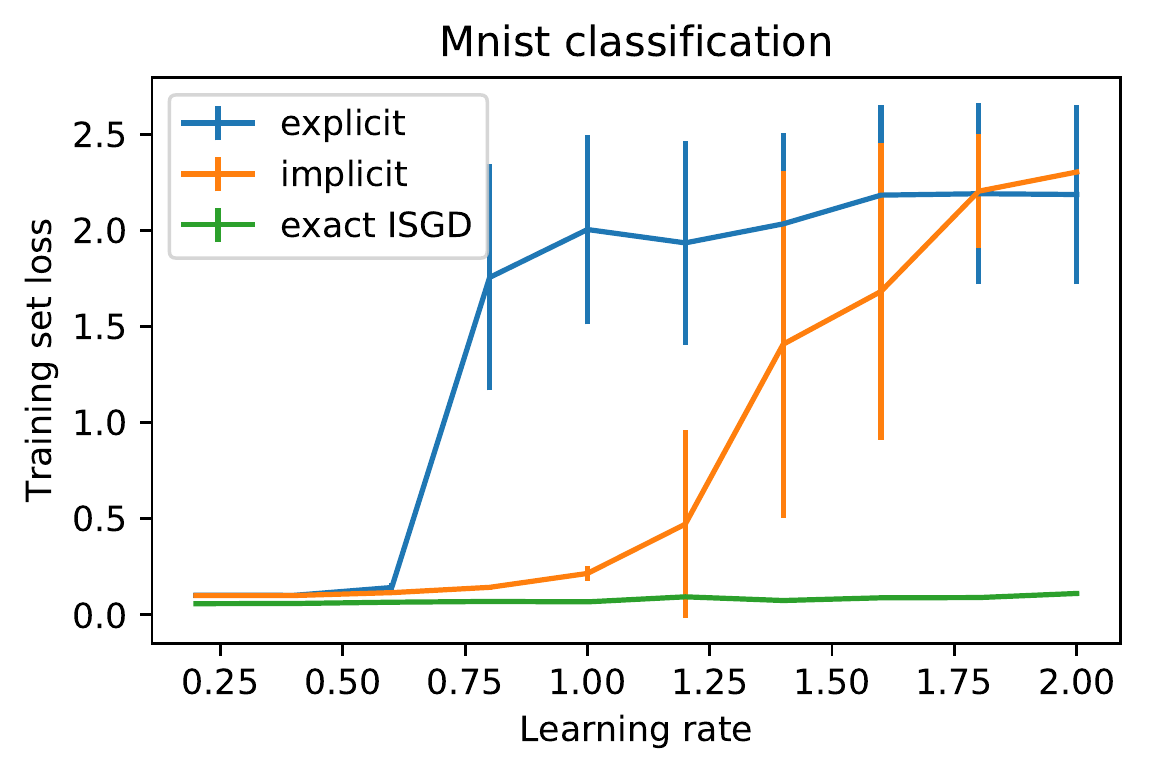}
\end{minipage}%
\hfill
\begin{minipage}{.49\textwidth}
  \centering
  \includegraphics[width=.79\linewidth]{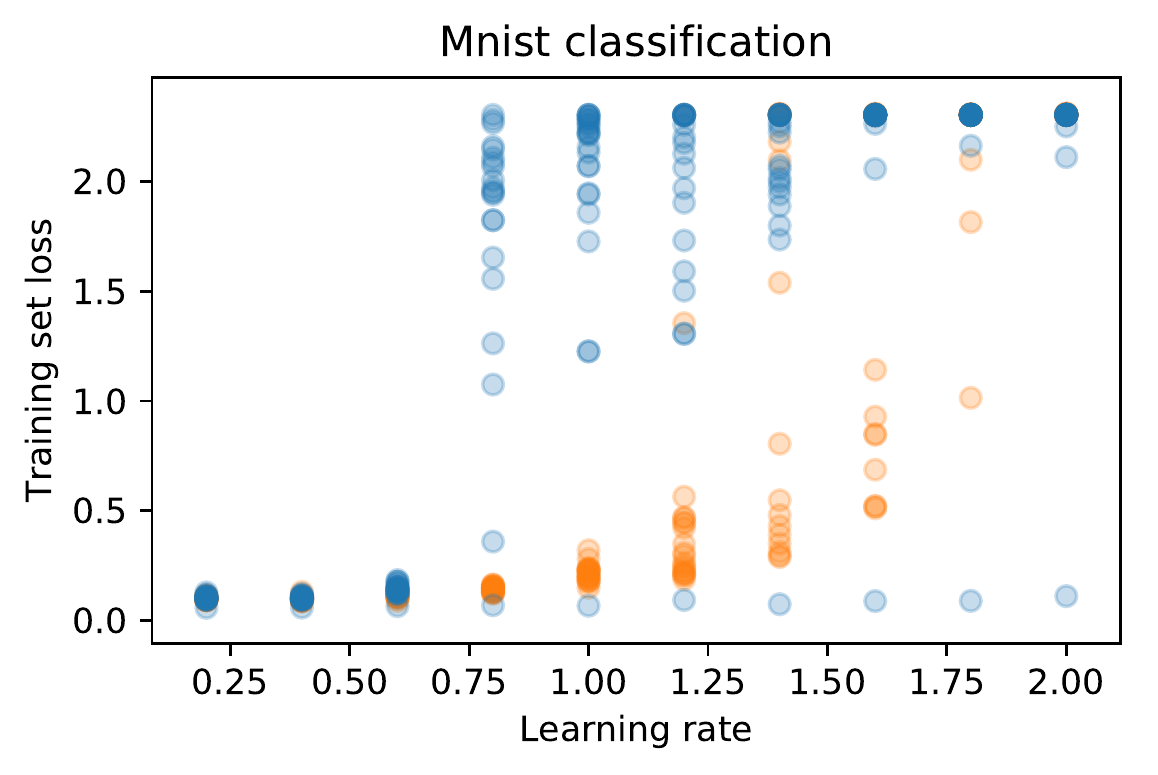}
\end{minipage}
\begin{minipage}{.49\textwidth}
  \centering
  \includegraphics[width=.79\linewidth]{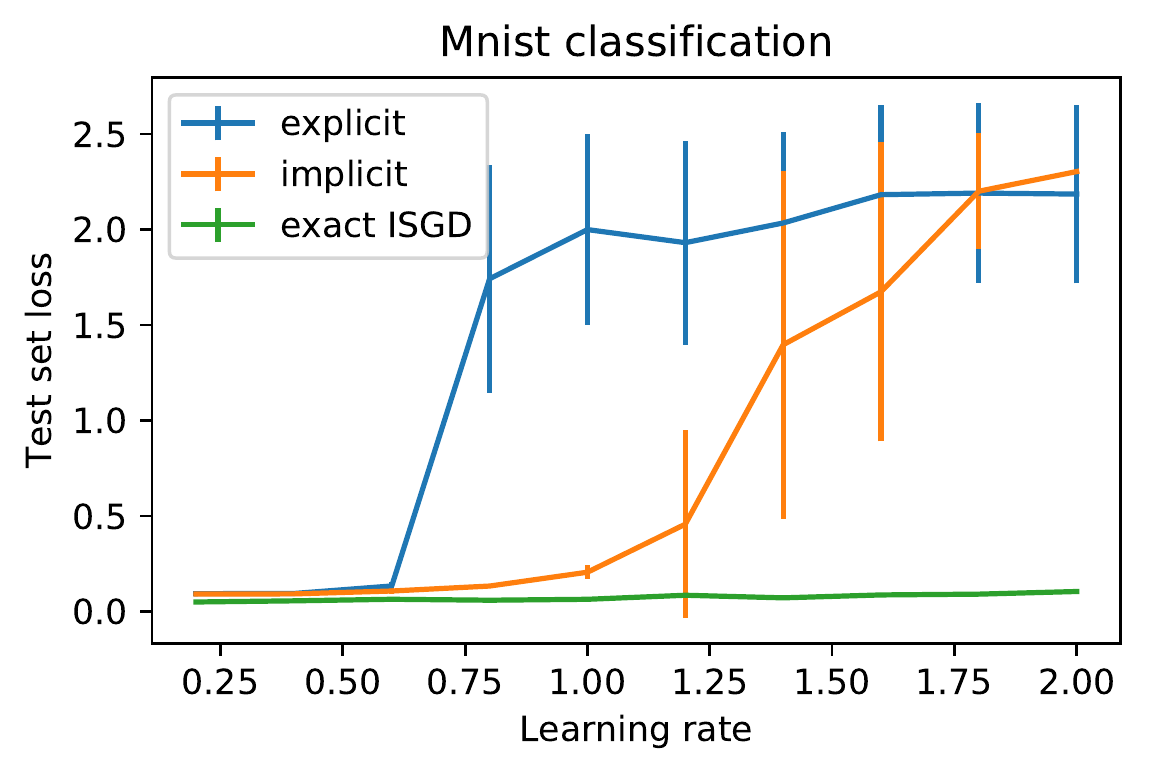}
\end{minipage}%
\hfill
\begin{minipage}{.49\textwidth}
  \centering
  \includegraphics[width=.79\linewidth]{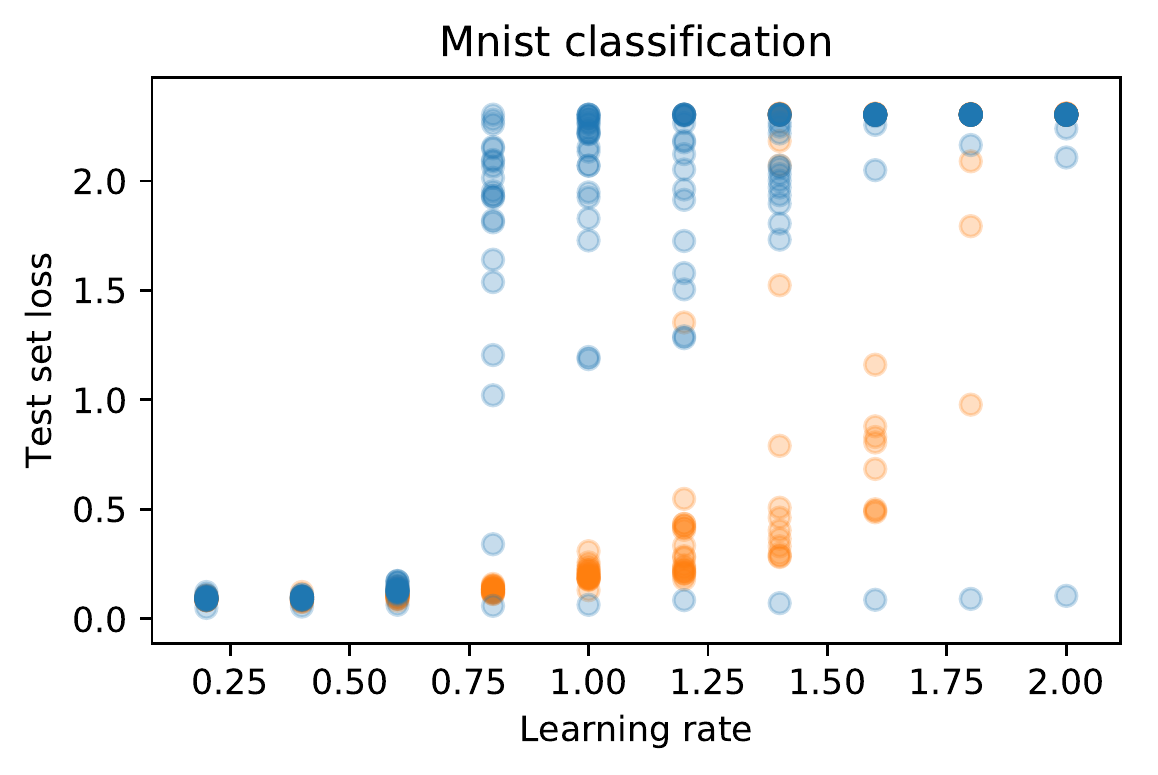}
\end{minipage}
\begin{minipage}{.49\textwidth}
  \centering
  \includegraphics[width=.79\linewidth]{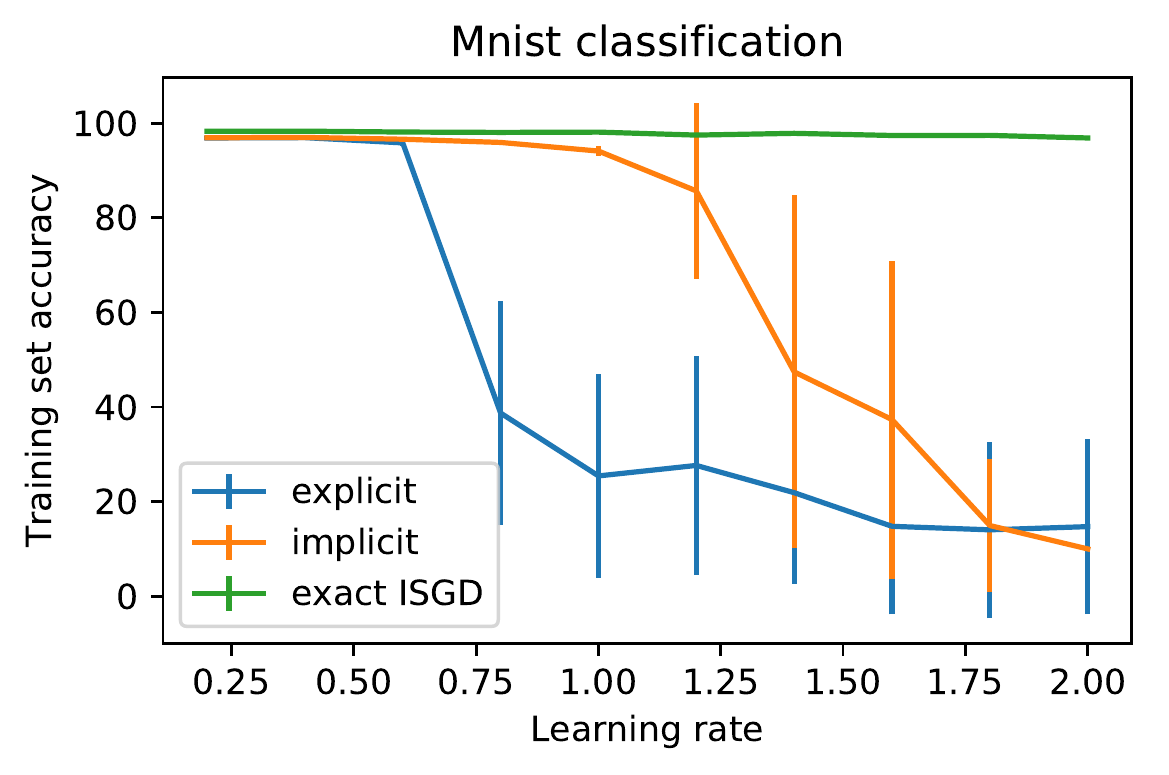}
\end{minipage}%
\hfill
\begin{minipage}{.49\textwidth}
  \centering
  \includegraphics[width=.79\linewidth]{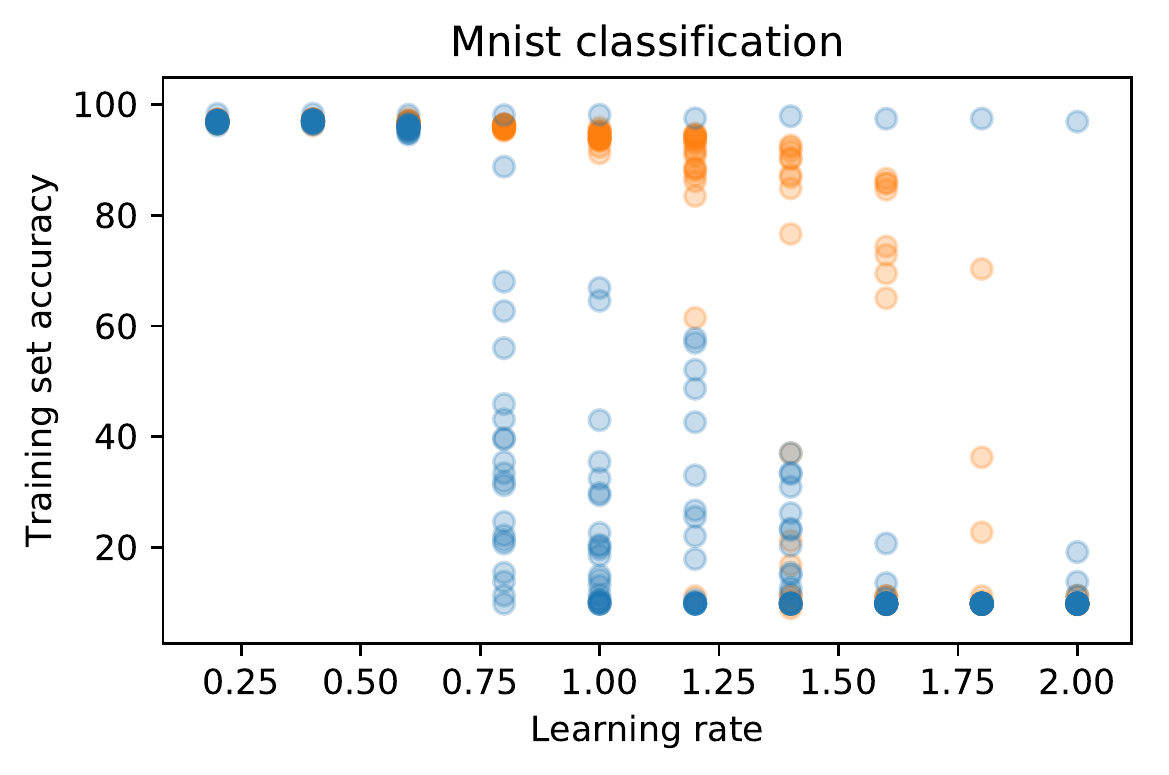}
\end{minipage}
\begin{minipage}{.49\textwidth}
  \centering
  \includegraphics[width=.79\linewidth]{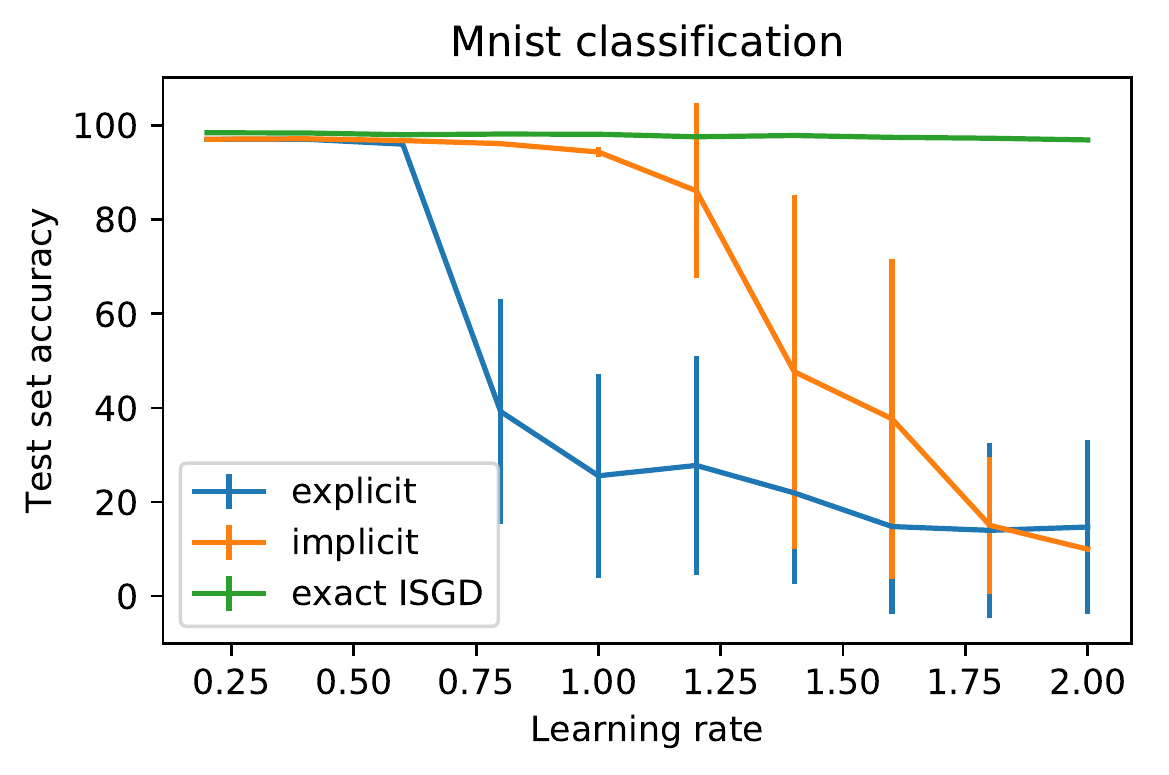}
\end{minipage}%
\hfill
\begin{minipage}{.49\textwidth}
  \centering
  \includegraphics[width=.79\linewidth]{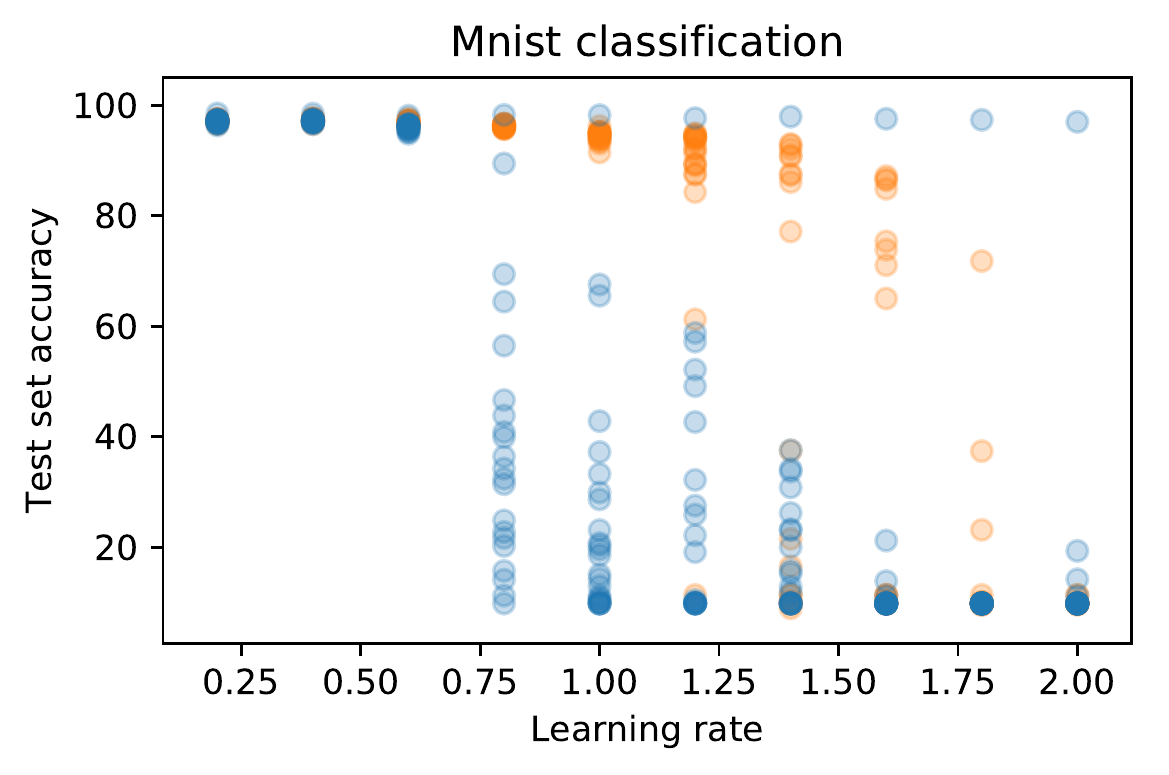}
\end{minipage}
\end{figure}

%

\begin{figure}[h]
\centering
\begin{minipage}{.49\textwidth}
  \centering
  \includegraphics[width=.79\linewidth]{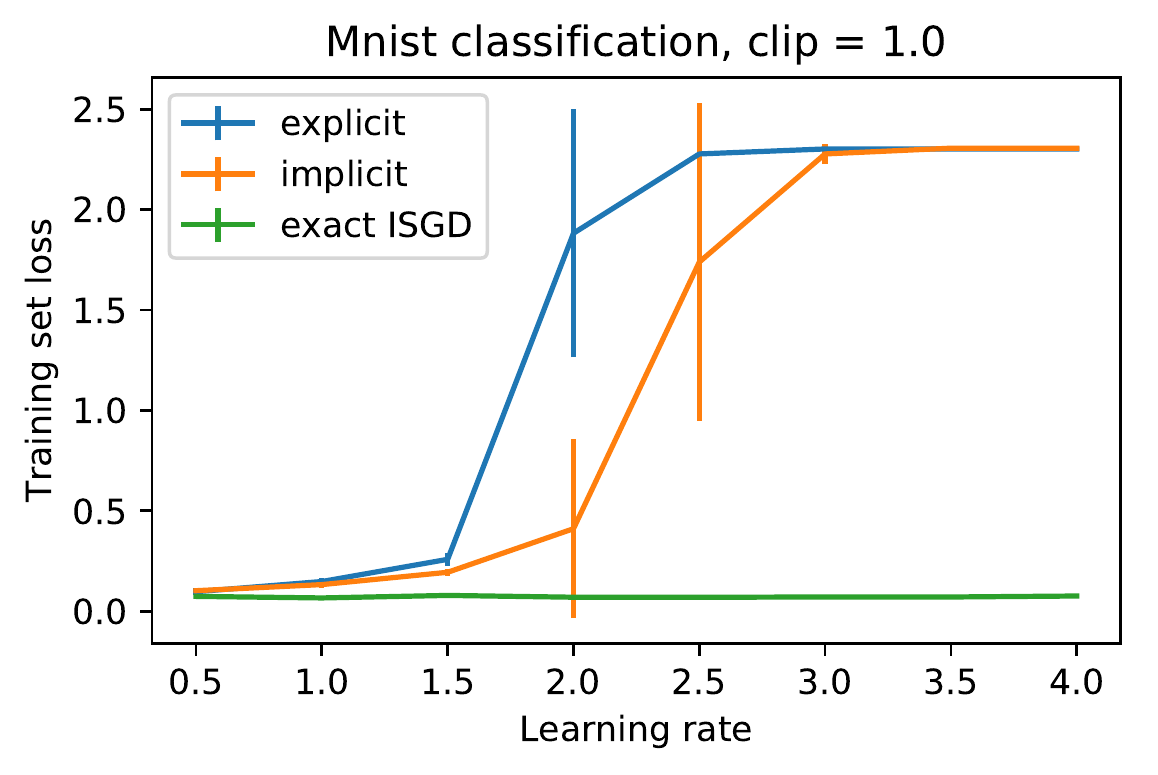}
\end{minipage}%
\hfill
\begin{minipage}{.49\textwidth}
  \centering
  \includegraphics[width=.79\linewidth]{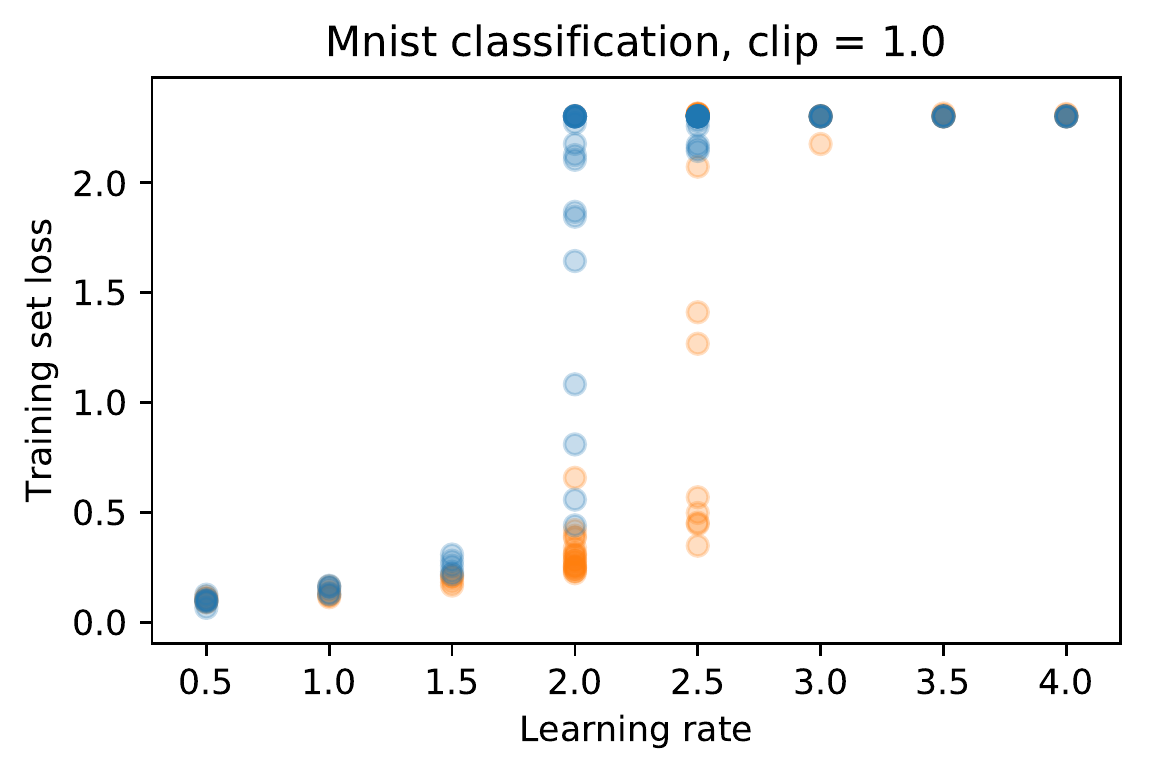}
\end{minipage}
\begin{minipage}{.49\textwidth}
  \centering
  \includegraphics[width=.79\linewidth]{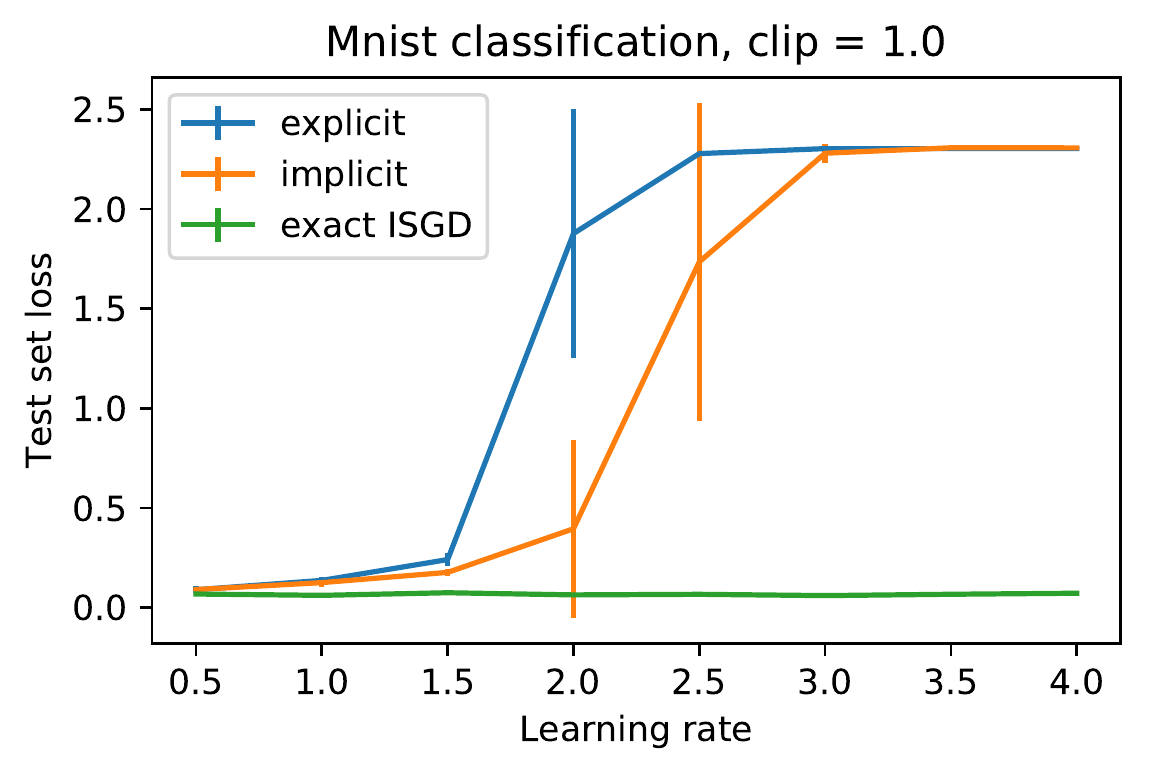}
\end{minipage}%
\hfill
\begin{minipage}{.49\textwidth}
  \centering
  \includegraphics[width=.79\linewidth]{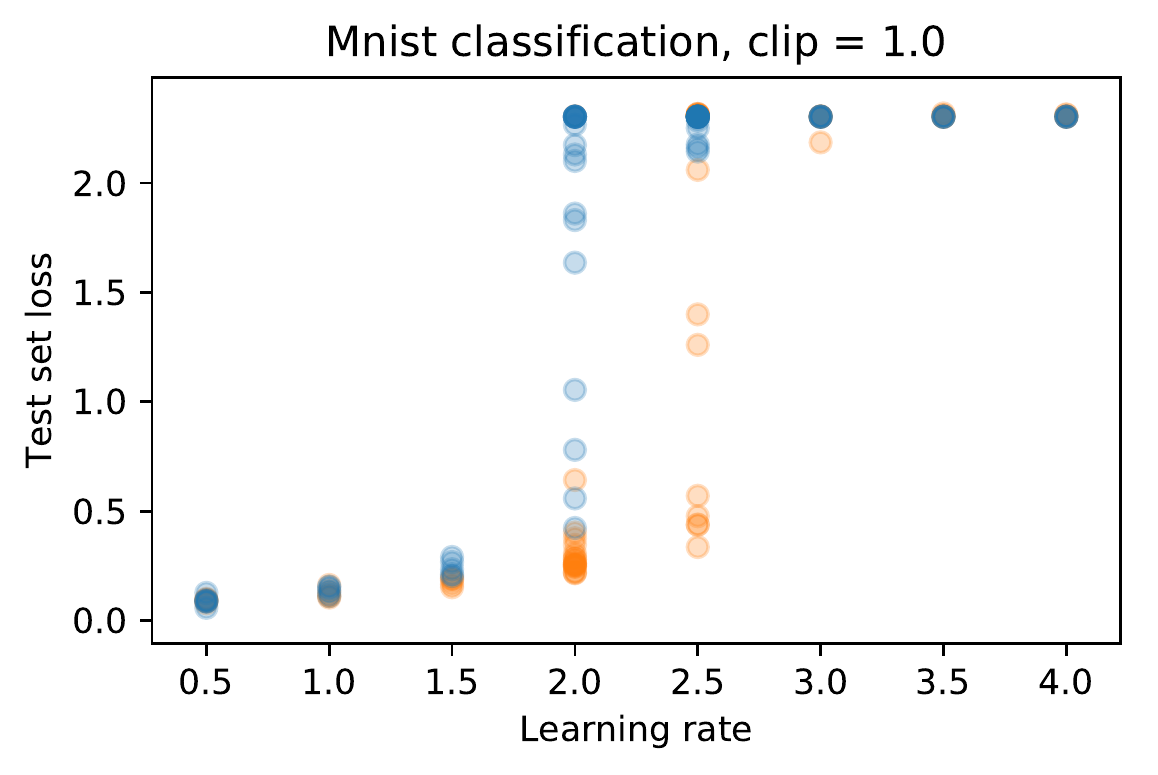}
\end{minipage}

\vspace{0cm}

\begin{minipage}{.49\textwidth}
  \centering
  \includegraphics[width=.79\linewidth]{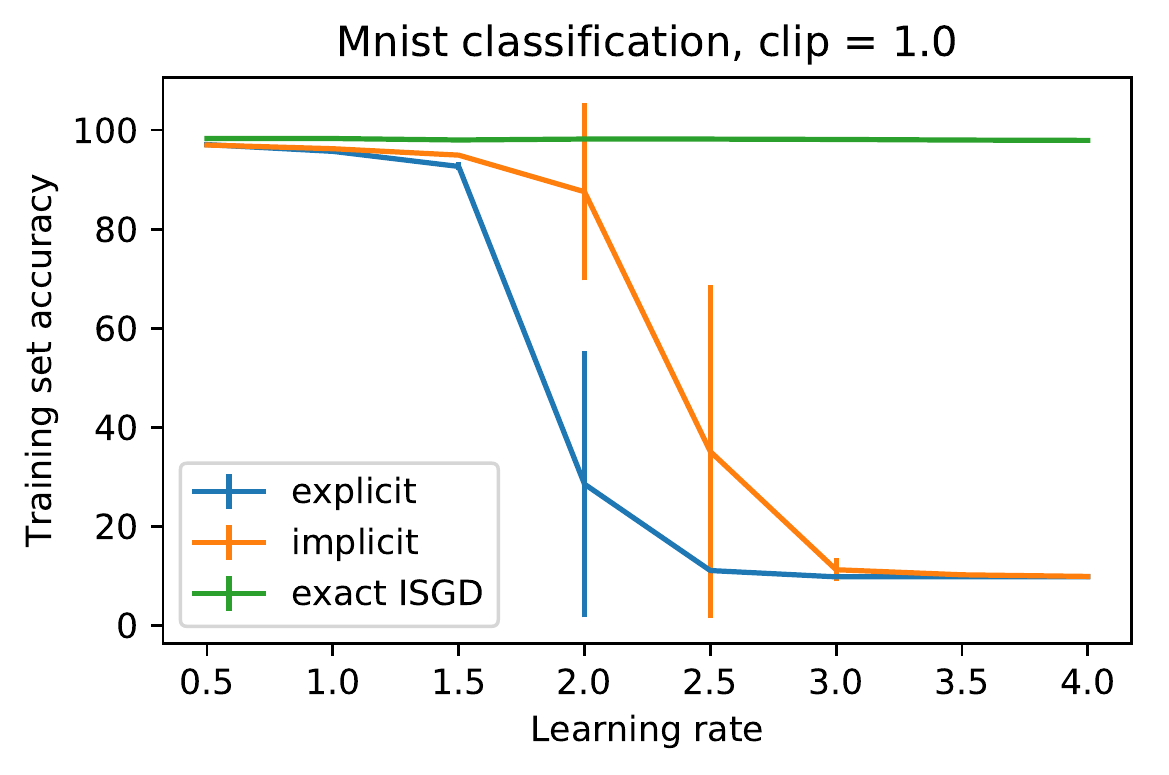}
\end{minipage}%
\hfill
\begin{minipage}{.49\textwidth}
  \centering
  \includegraphics[width=.79\linewidth]{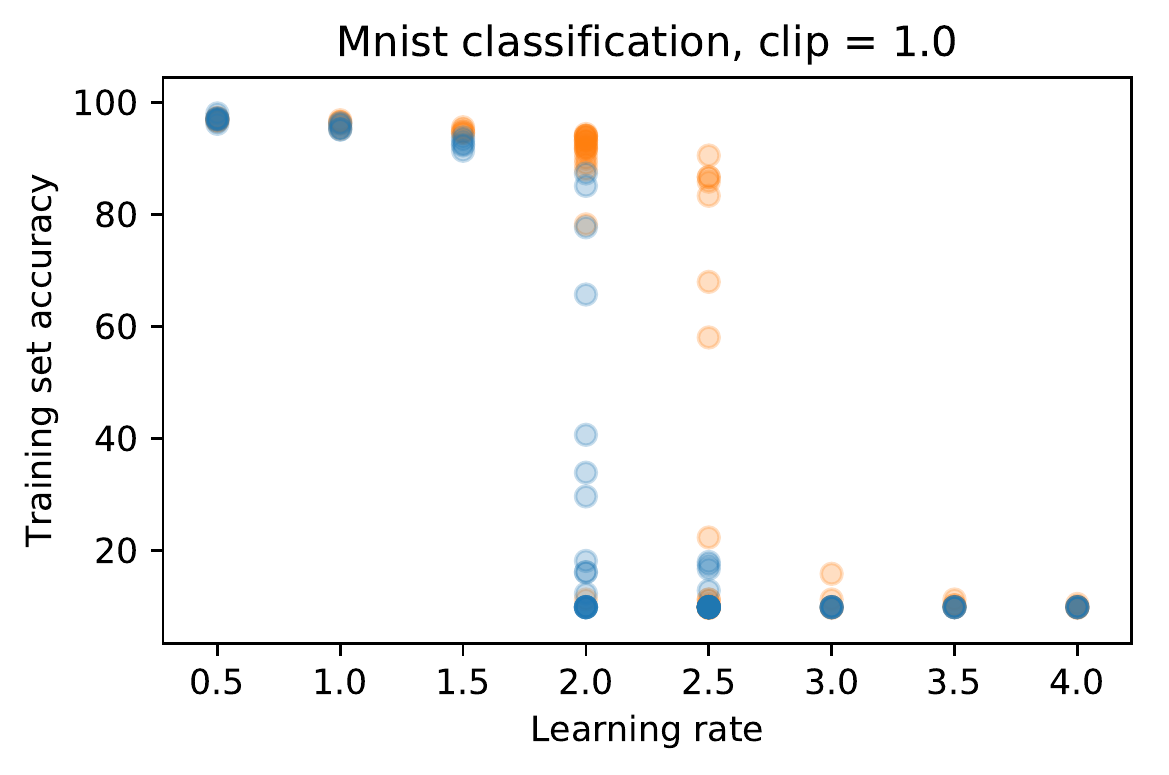}
\end{minipage}
\begin{minipage}{.49\textwidth}
  \centering
  \includegraphics[width=.79\linewidth]{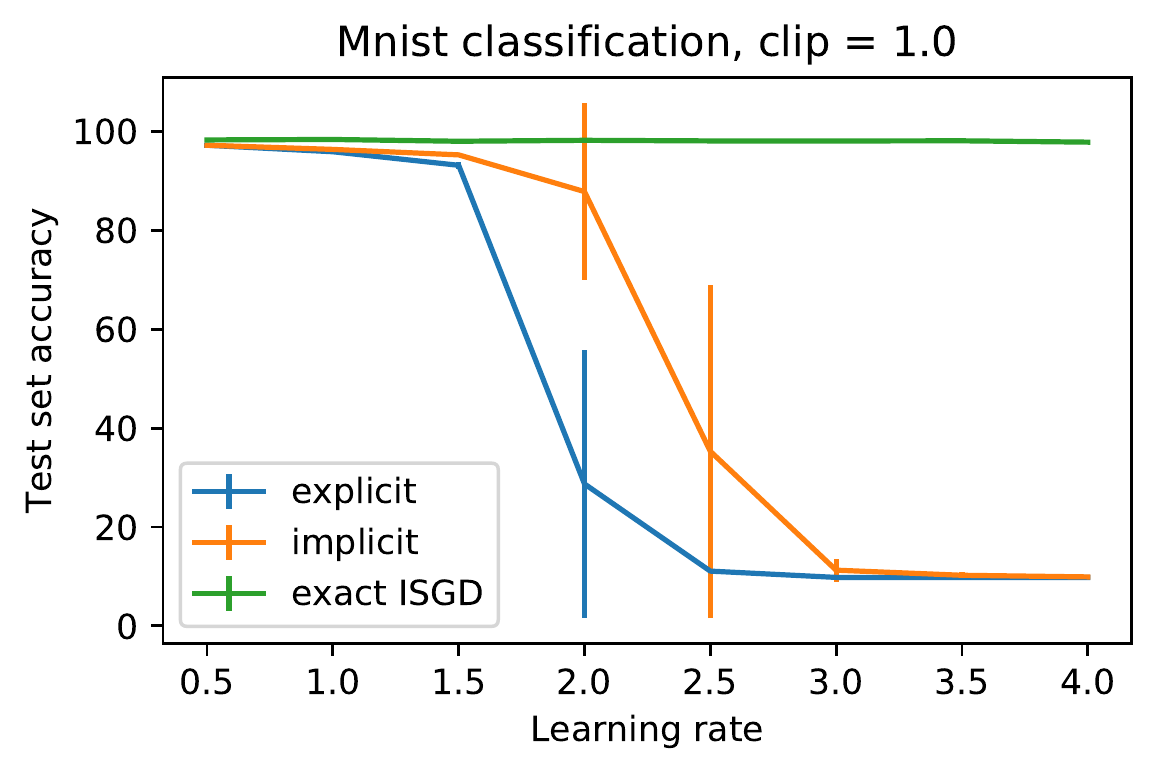}
\end{minipage}%
\hfill
\begin{minipage}{.49\textwidth}
  \centering
  \includegraphics[width=.79\linewidth]{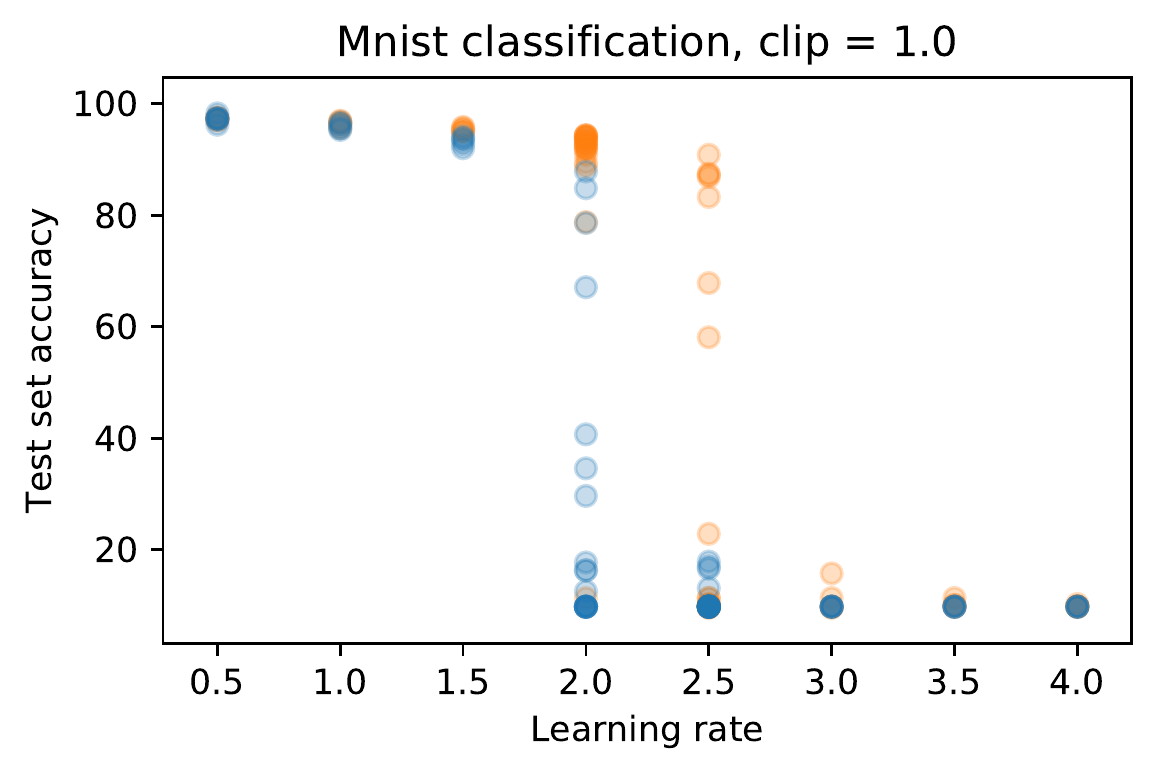}
\end{minipage}
\end{figure}

\clearpage
\subsection{MNIST autoencoder}

\begin{figure}[h]
\centering
\begin{minipage}{.49\textwidth}
  \centering
  \includegraphics[width=.79\linewidth]{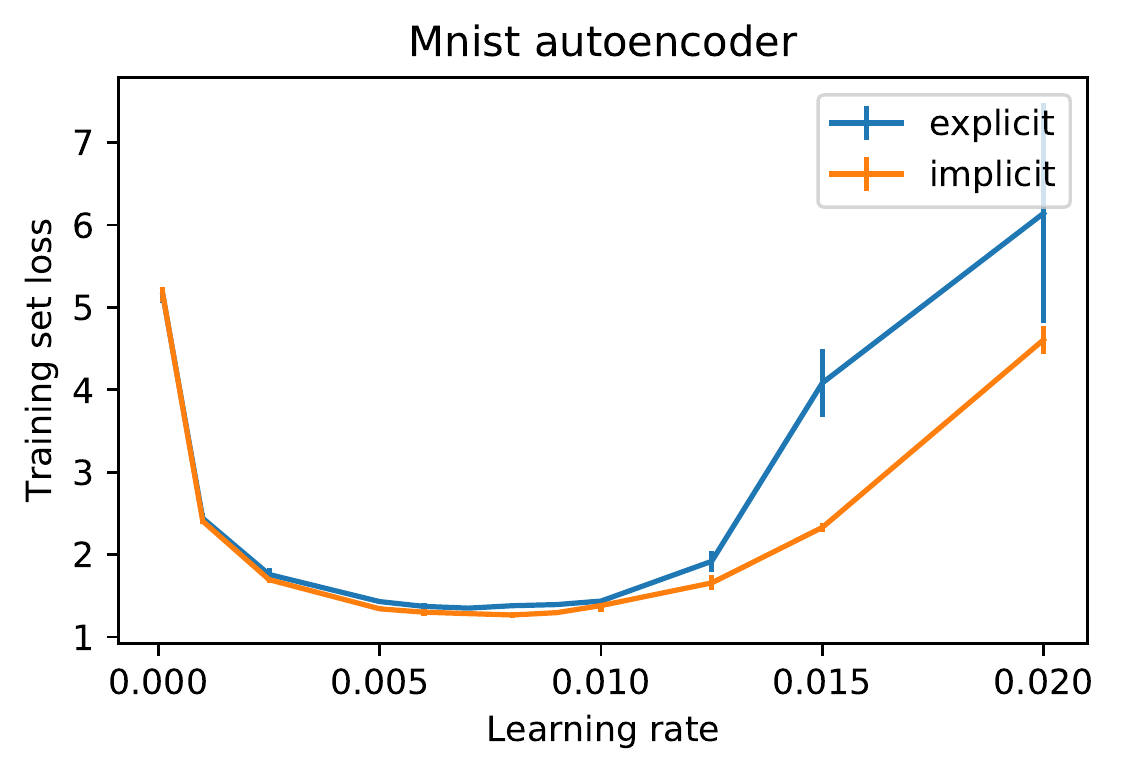}
\end{minipage}%
\hfill
\begin{minipage}{.49\textwidth}
  \centering
  \includegraphics[width=.79\linewidth]{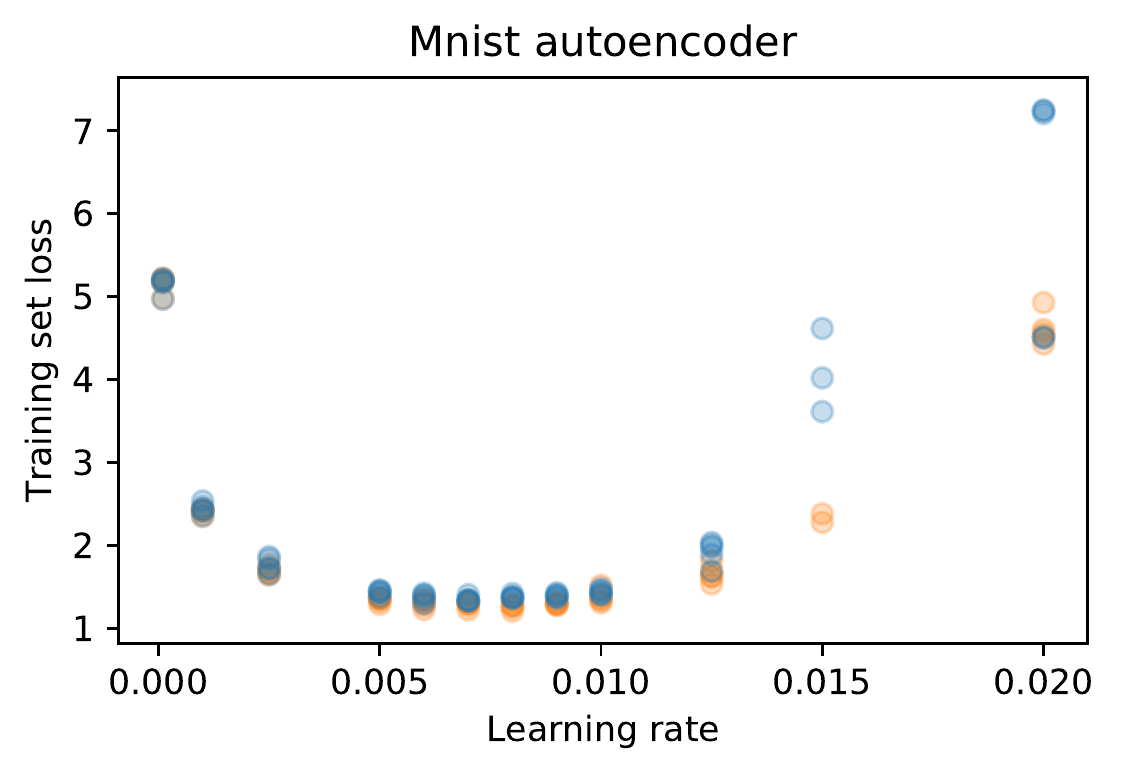}
\end{minipage}
\begin{minipage}{.49\textwidth}
  \centering
  \includegraphics[width=.79\linewidth]{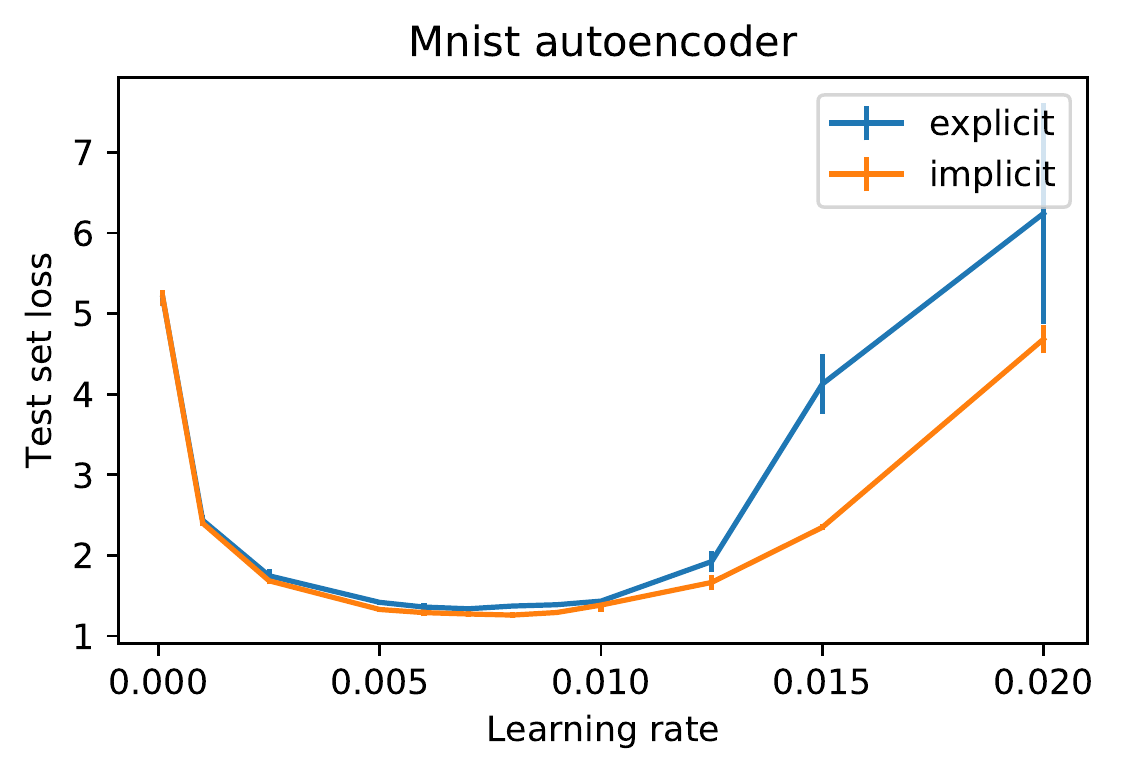}
\end{minipage}%
\hfill
\begin{minipage}{.49\textwidth}
  \centering
  \includegraphics[width=.79\linewidth]{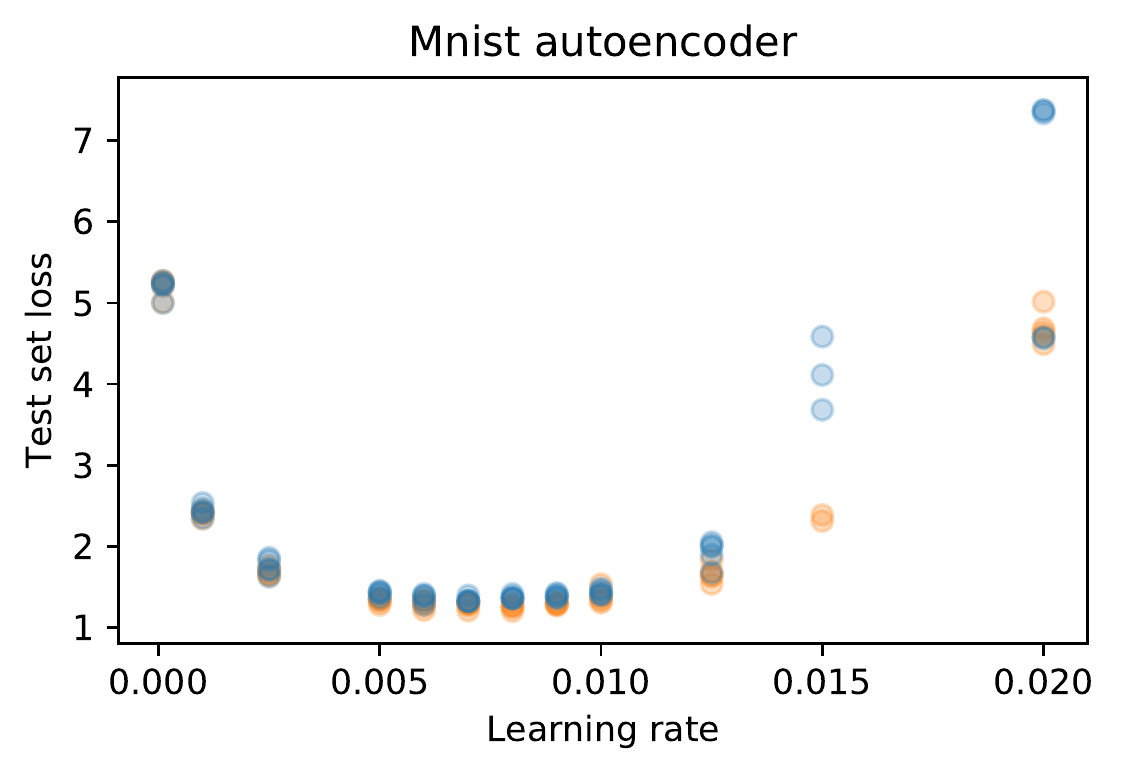}
\end{minipage}

\vspace{0cm}

\begin{minipage}{.49\textwidth}
  \centering
  \includegraphics[width=.79\linewidth]{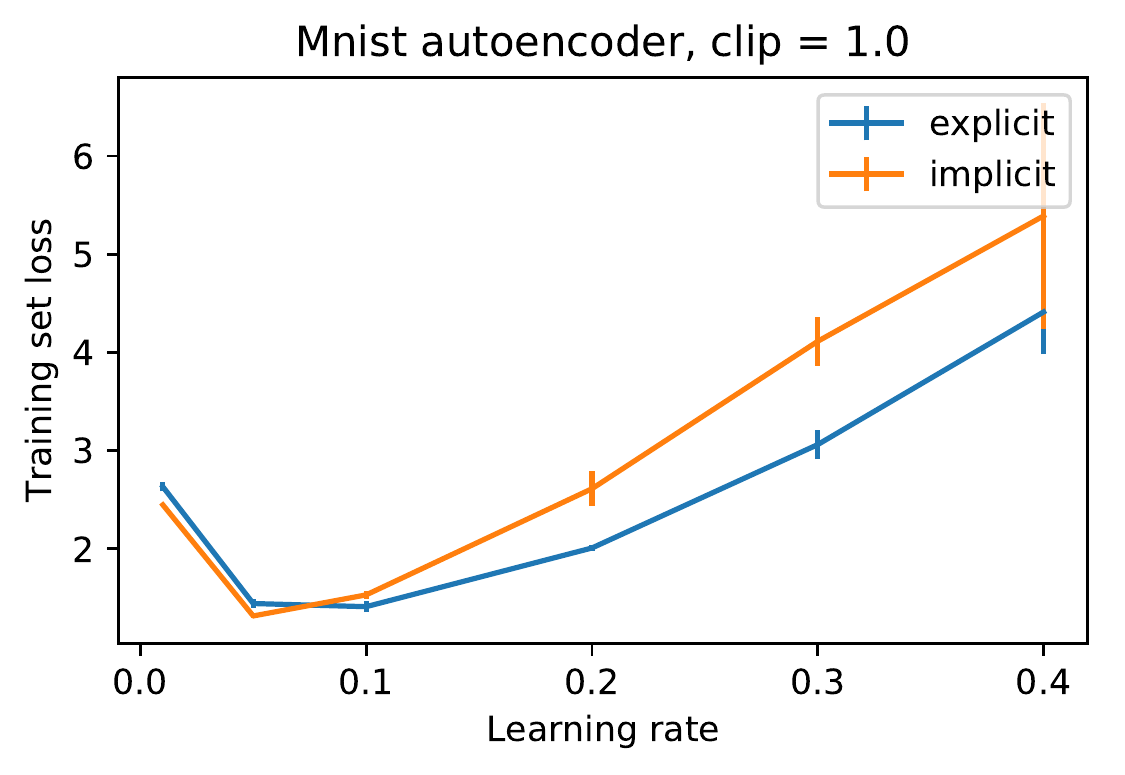}
\end{minipage}%
\hfill
\begin{minipage}{.49\textwidth}
  \centering
  \includegraphics[width=.79\linewidth]{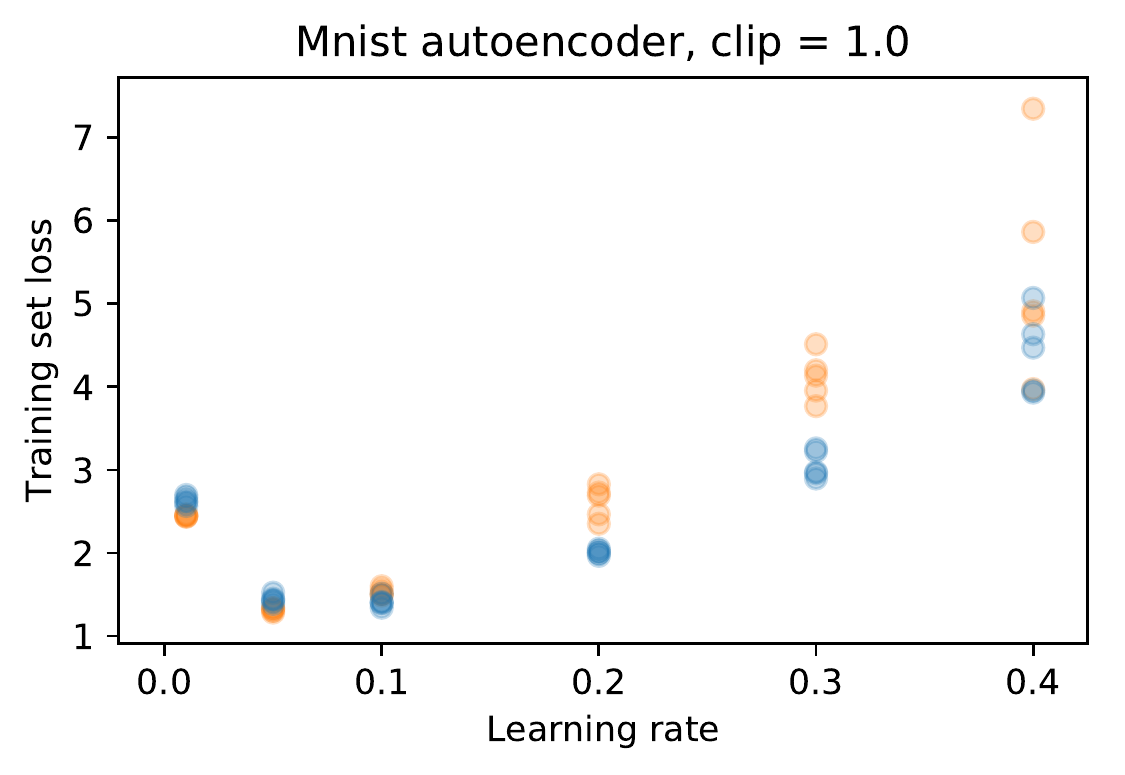}
\end{minipage}
\begin{minipage}{.49\textwidth}
  \centering
  \includegraphics[width=.79\linewidth]{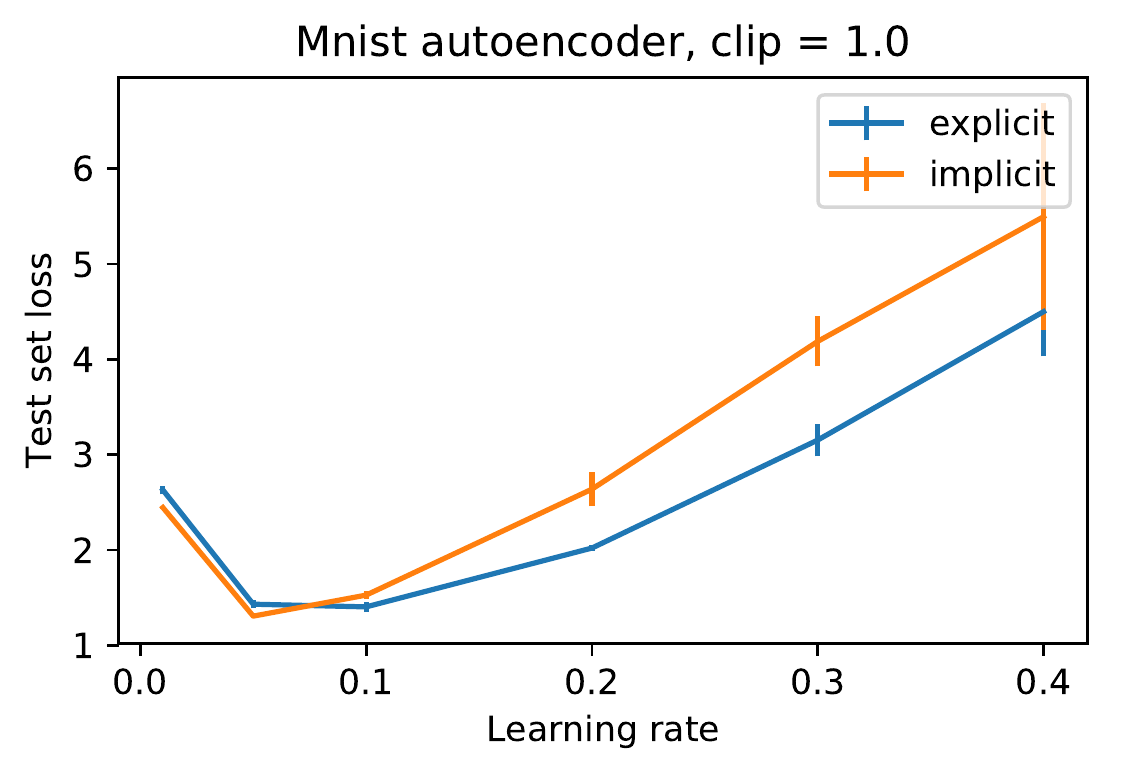}
\end{minipage}%
\hfill
\begin{minipage}{.49\textwidth}
  \centering
  \includegraphics[width=.79\linewidth]{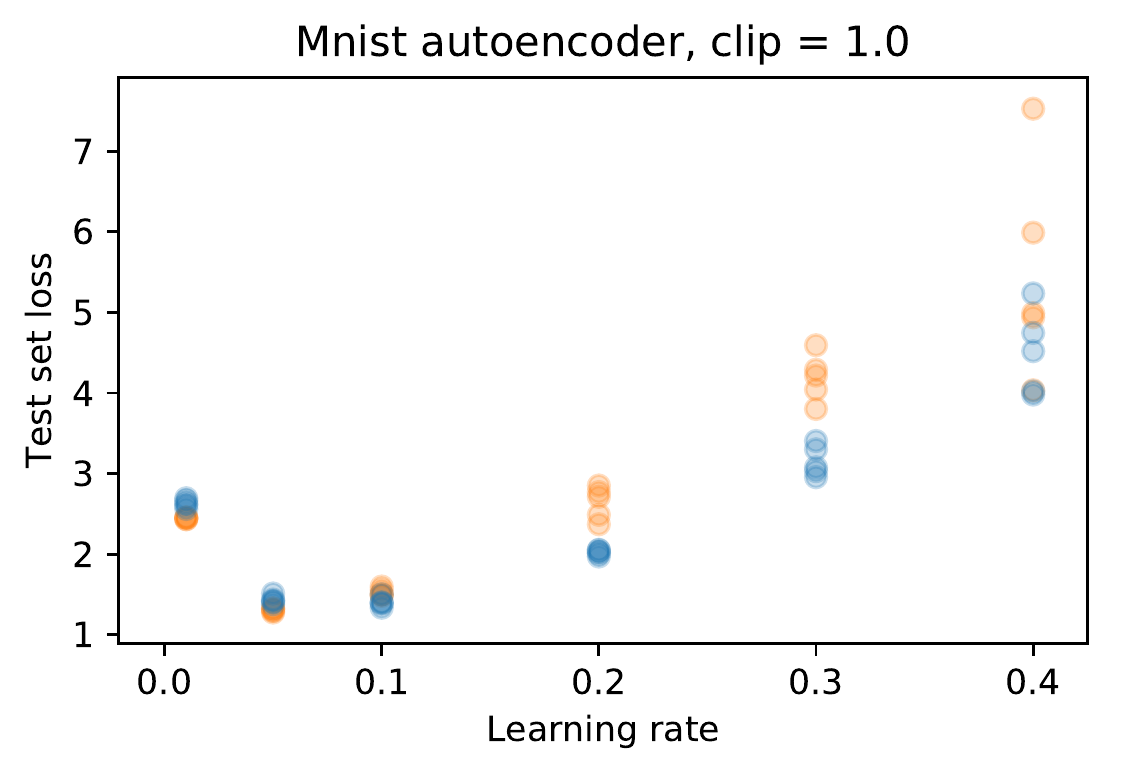}
\end{minipage}
\end{figure}

\clearpage
\subsection{JSB Chorales}

\begin{figure}[h]
\centering
\begin{minipage}{.49\textwidth}
  \centering
  \includegraphics[width=.79\linewidth]{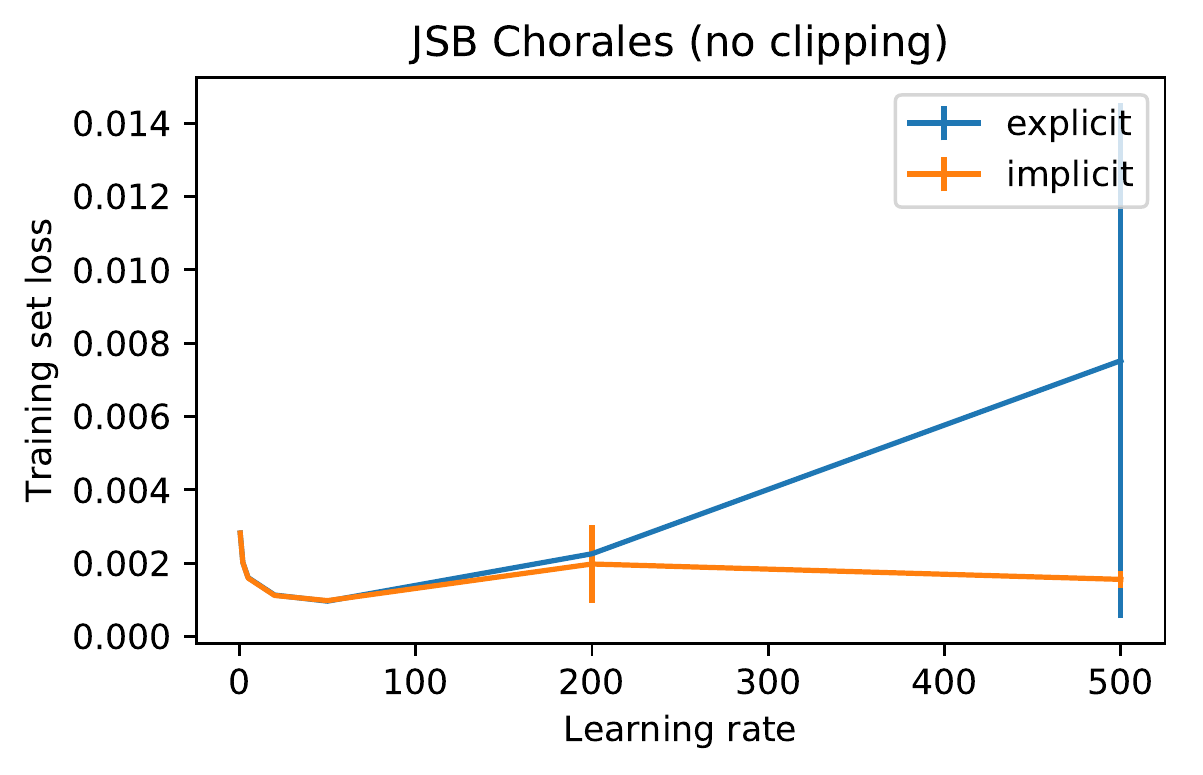}
\end{minipage}%
\hfill
\begin{minipage}{.49\textwidth}
  \centering
  \includegraphics[width=.79\linewidth]{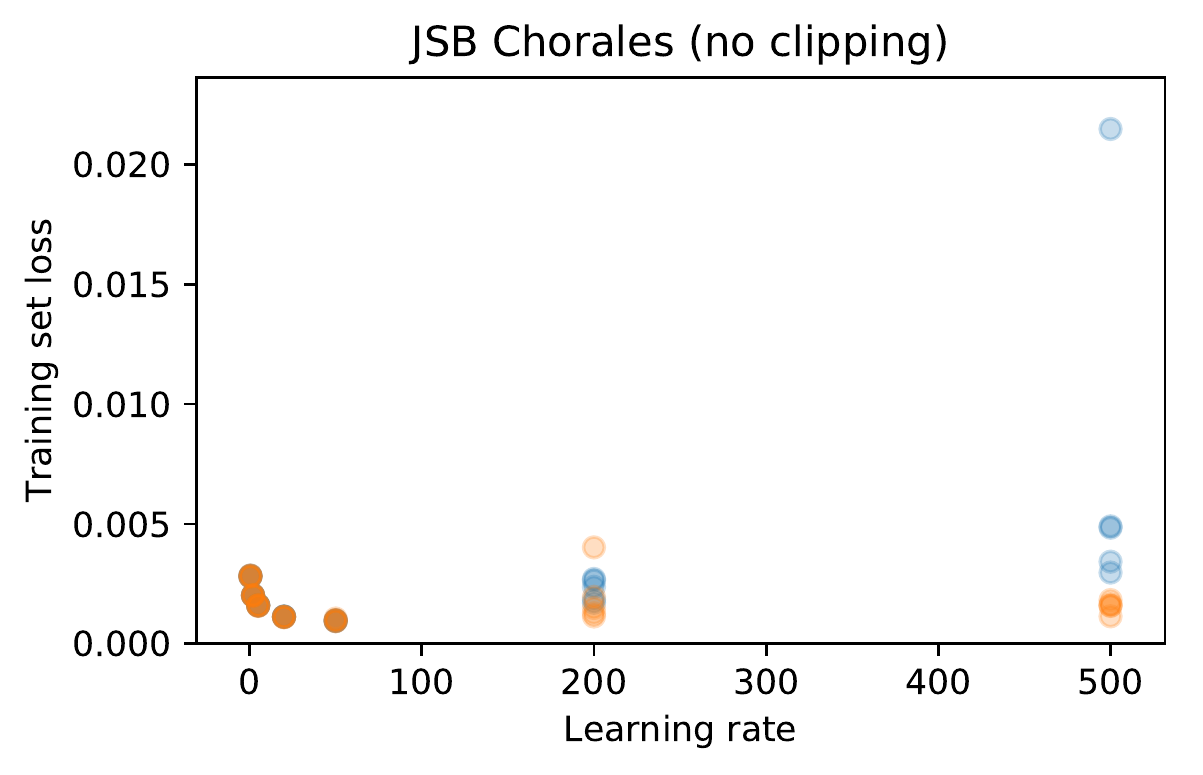}
\end{minipage}
\begin{minipage}{.49\textwidth}
  \centering
  \includegraphics[width=.79\linewidth]{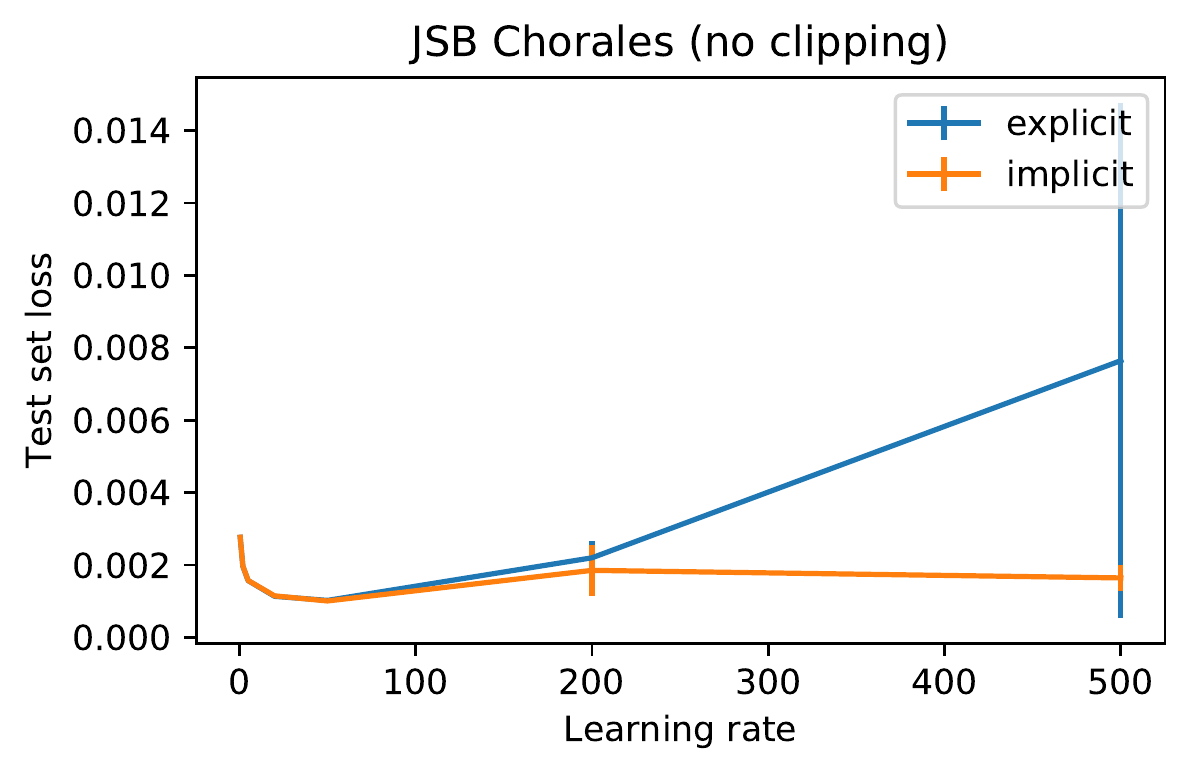}
\end{minipage}%
\hfill
\begin{minipage}{.49\textwidth}
  \centering
  \includegraphics[width=.79\linewidth]{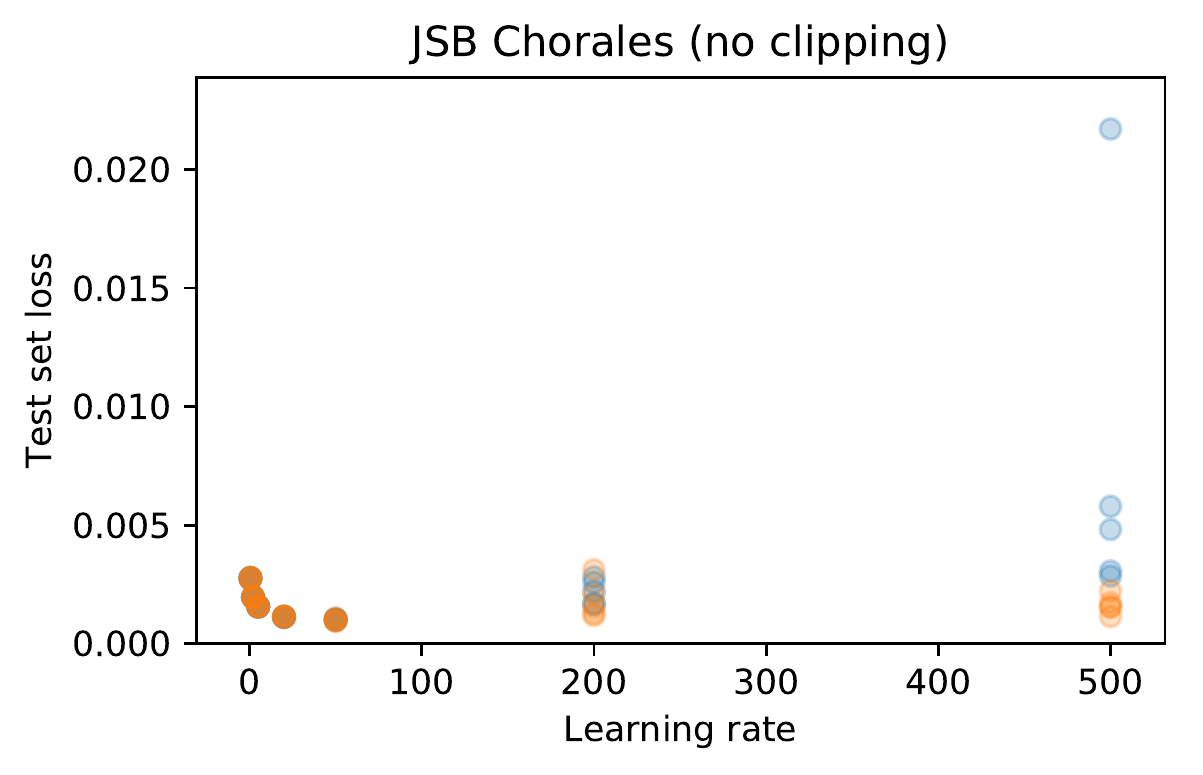}
\end{minipage}
\end{figure}

\begin{figure}[h]
\centering
\begin{minipage}{.49\textwidth}
  \centering
  \includegraphics[width=.79\linewidth]{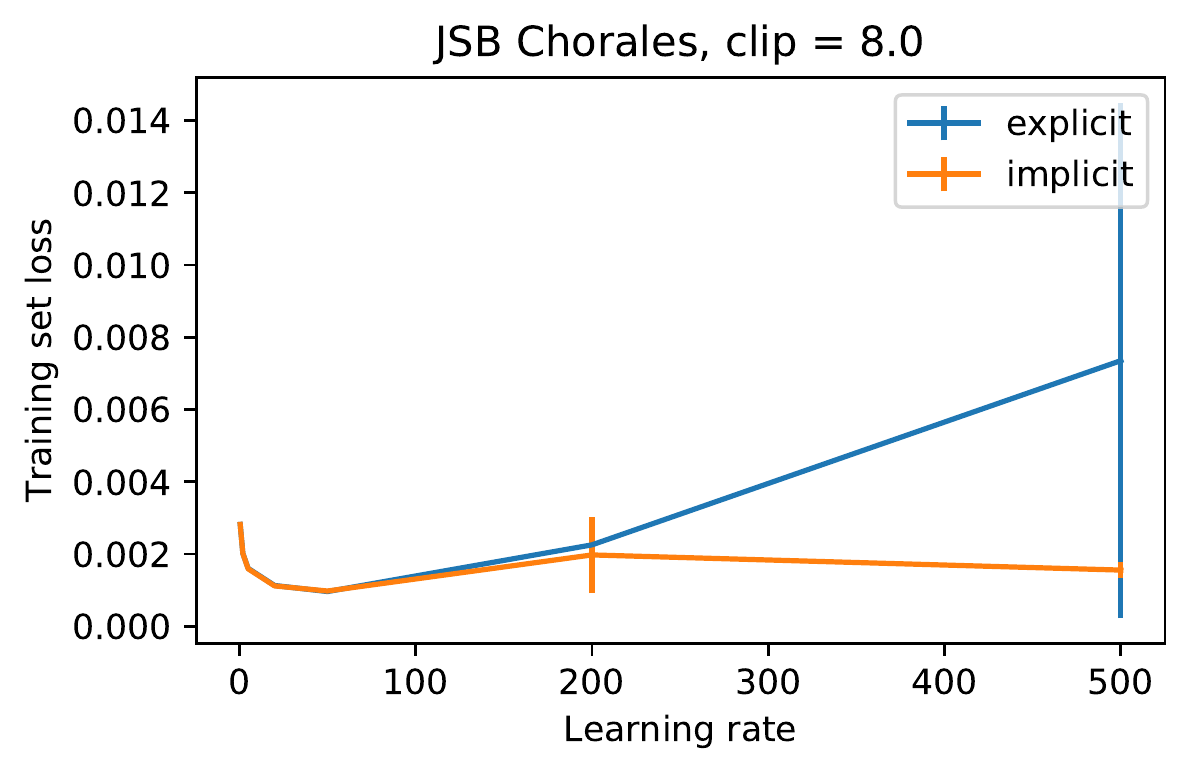}
\end{minipage}%
\hfill
\begin{minipage}{.49\textwidth}
  \centering
  \includegraphics[width=.79\linewidth]{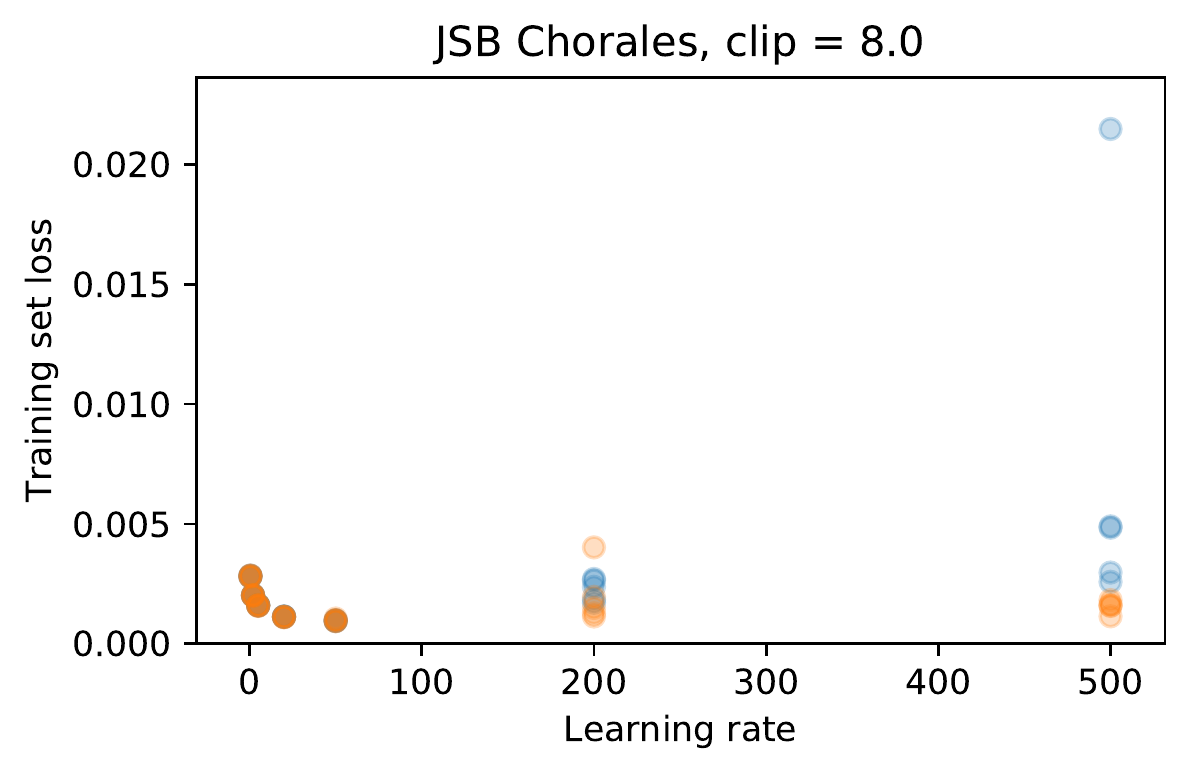}
\end{minipage}
\begin{minipage}{.49\textwidth}
  \centering
  \includegraphics[width=.79\linewidth]{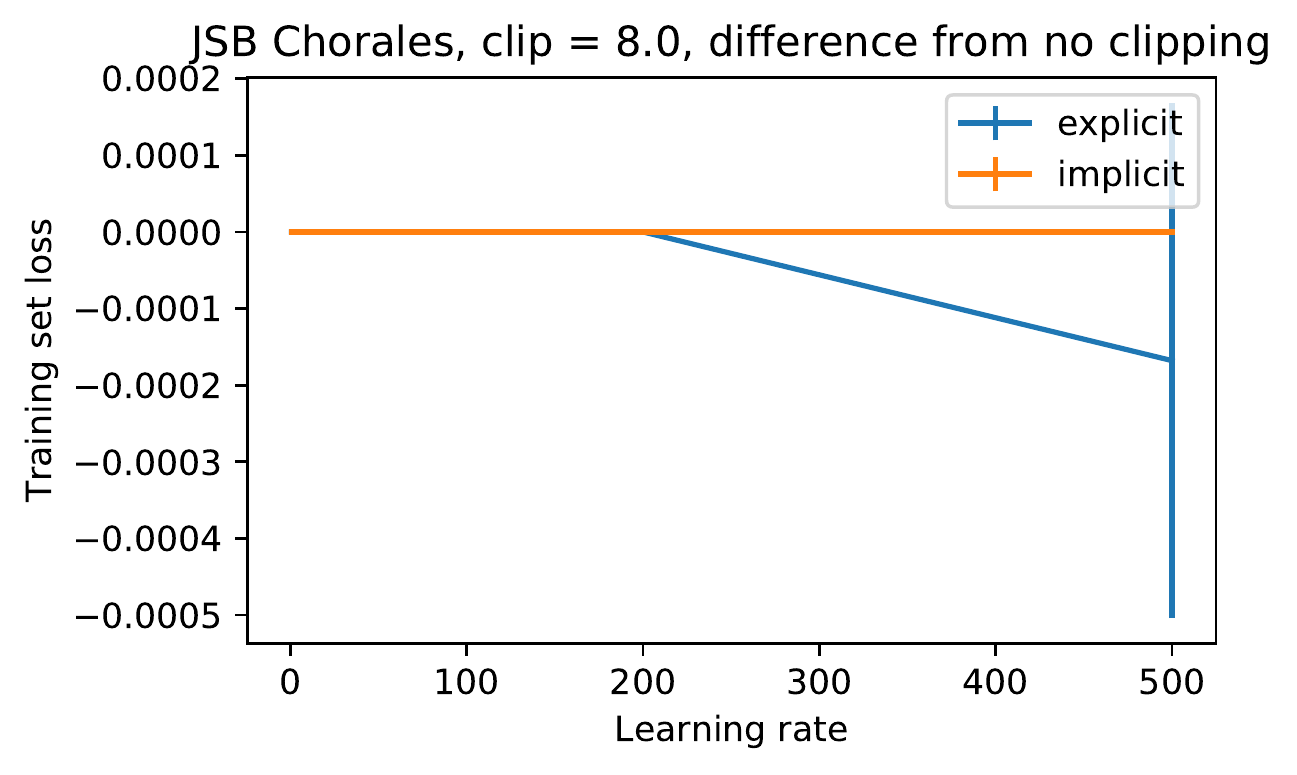}
\end{minipage}%
\hfill
\begin{minipage}{.49\textwidth}
  \centering
  \includegraphics[width=.79\linewidth]{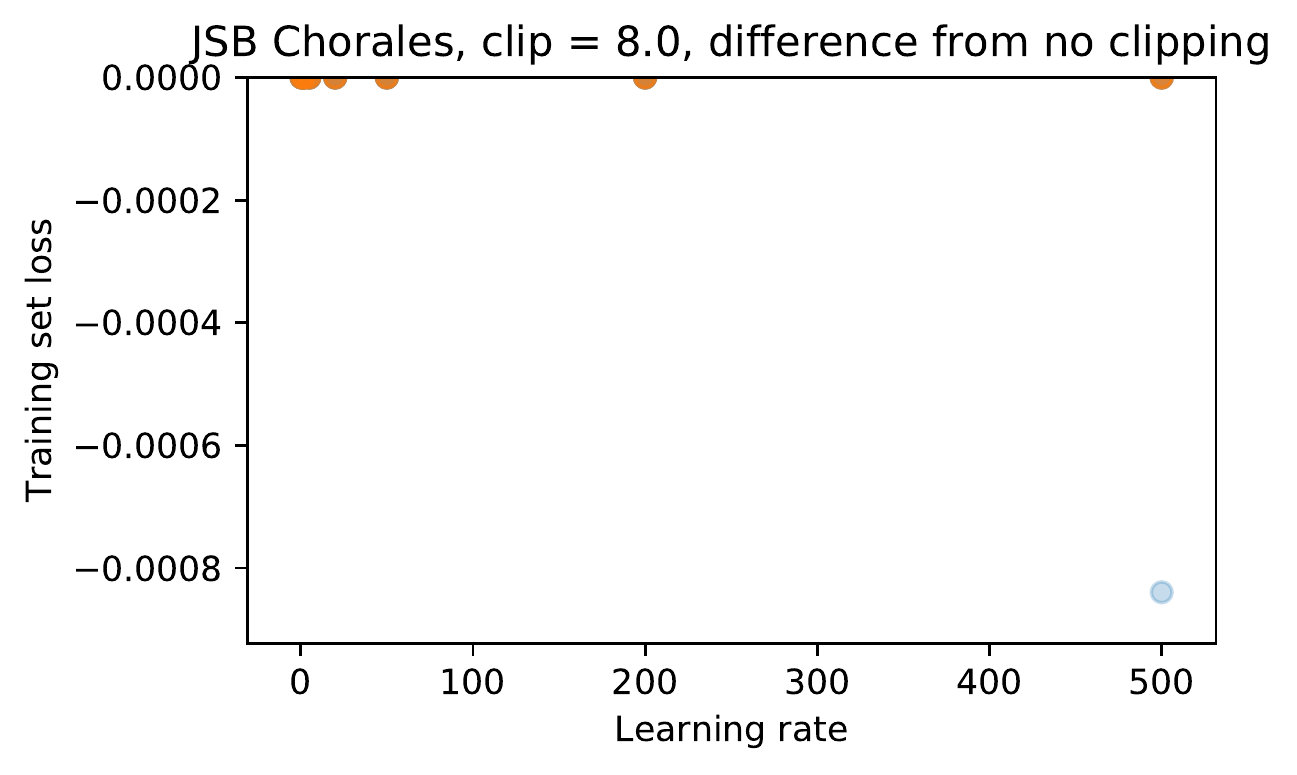}
\end{minipage}
\end{figure}

\vspace{0cm}

\begin{figure}[h]
\begin{minipage}{.49\textwidth}
  \centering
  \includegraphics[width=.79\linewidth]{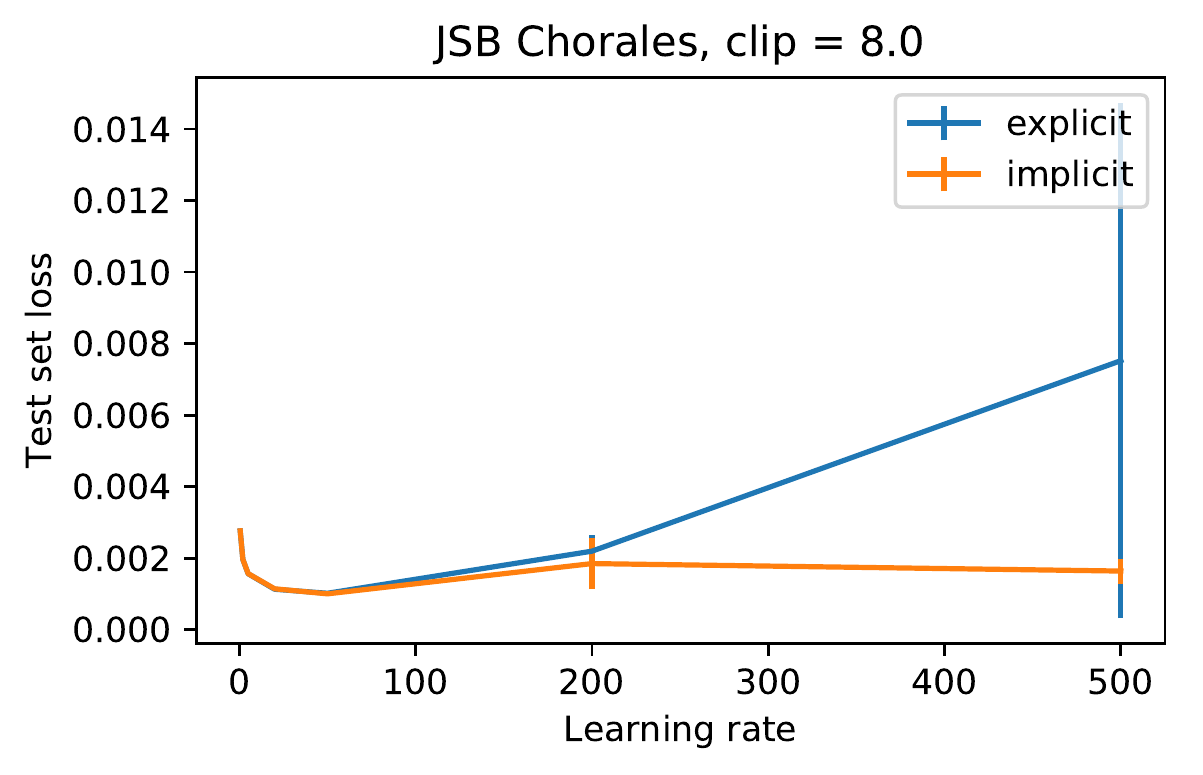}
\end{minipage}%
\hfill
\begin{minipage}{.49\textwidth}
  \centering
  \includegraphics[width=.79\linewidth]{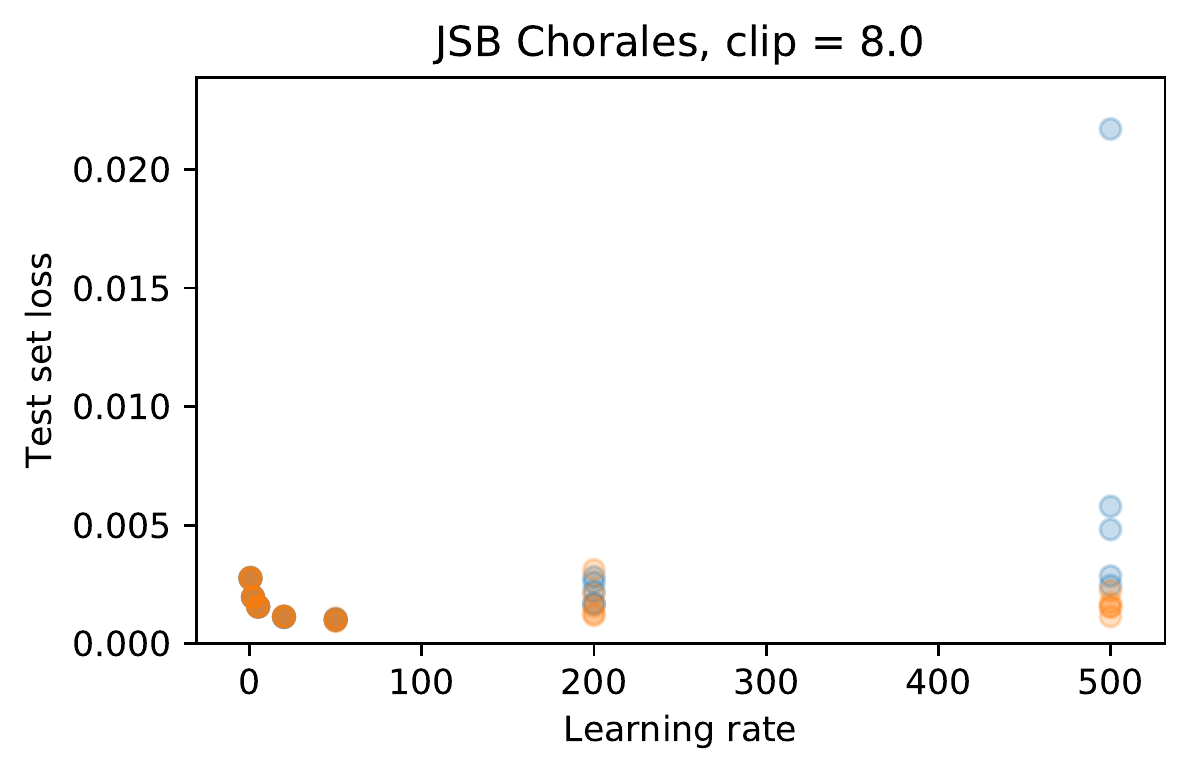}
\end{minipage}
\begin{minipage}{.49\textwidth}
  \centering
  \includegraphics[width=.79\linewidth]{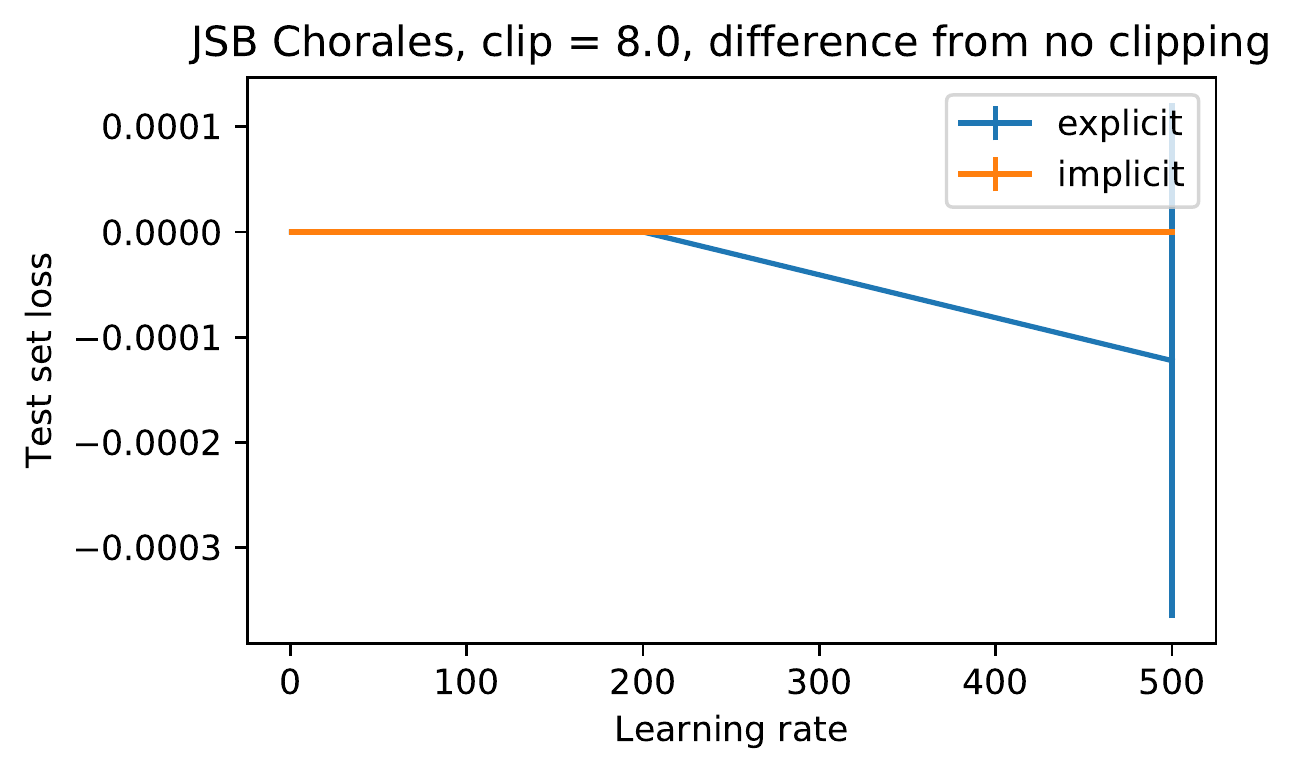}
\end{minipage}%
\hfill
\begin{minipage}{.49\textwidth}
  \centering
  \includegraphics[width=.79\linewidth]{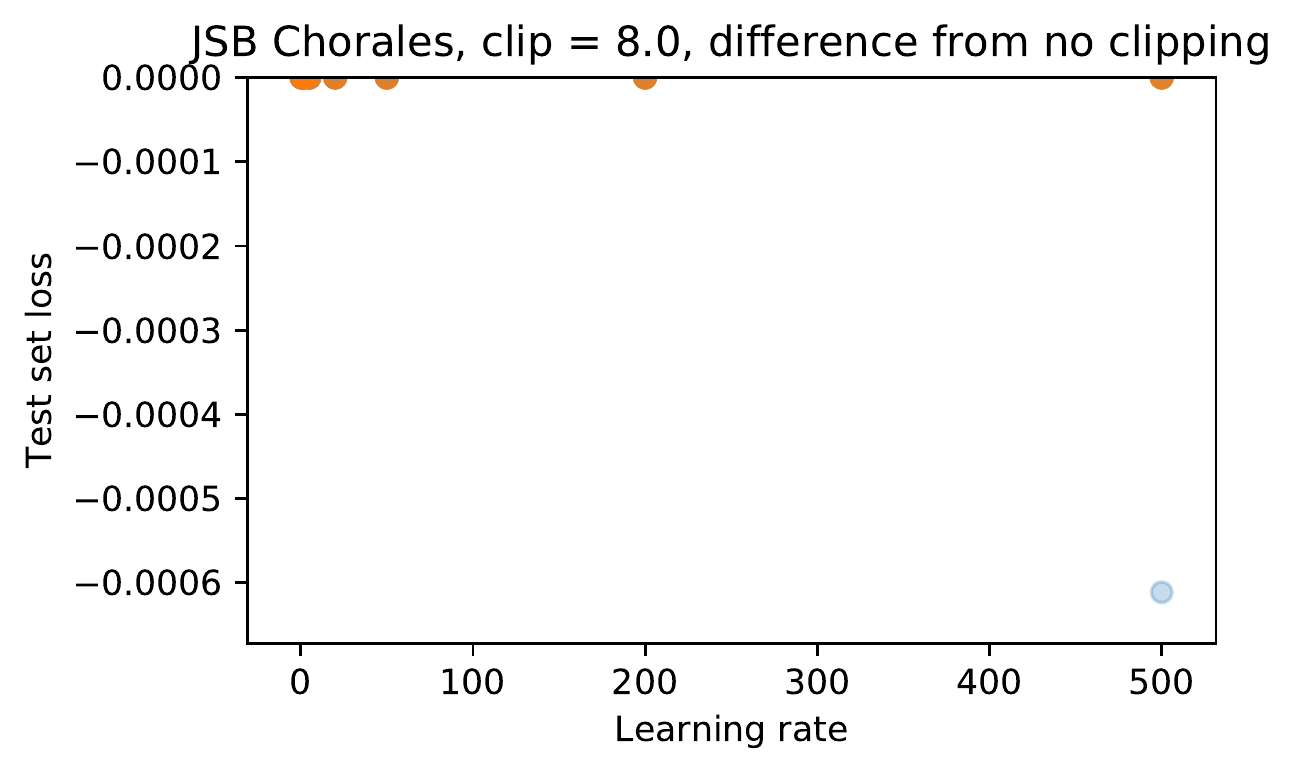}
\end{minipage}
\end{figure}

\begin{figure}[h]
\centering
\begin{minipage}{.49\textwidth}
  \centering
  \includegraphics[width=.79\linewidth]{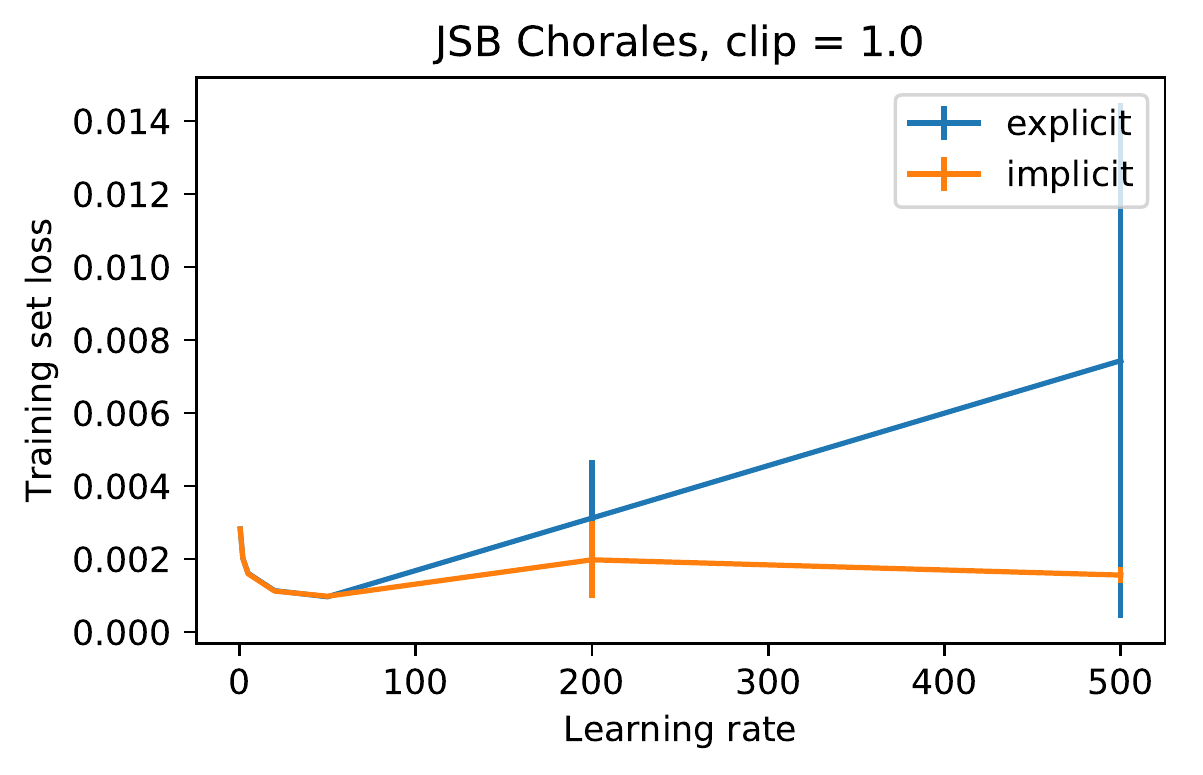}
\end{minipage}%
\hfill
\begin{minipage}{.49\textwidth}
  \centering
  \includegraphics[width=.79\linewidth]{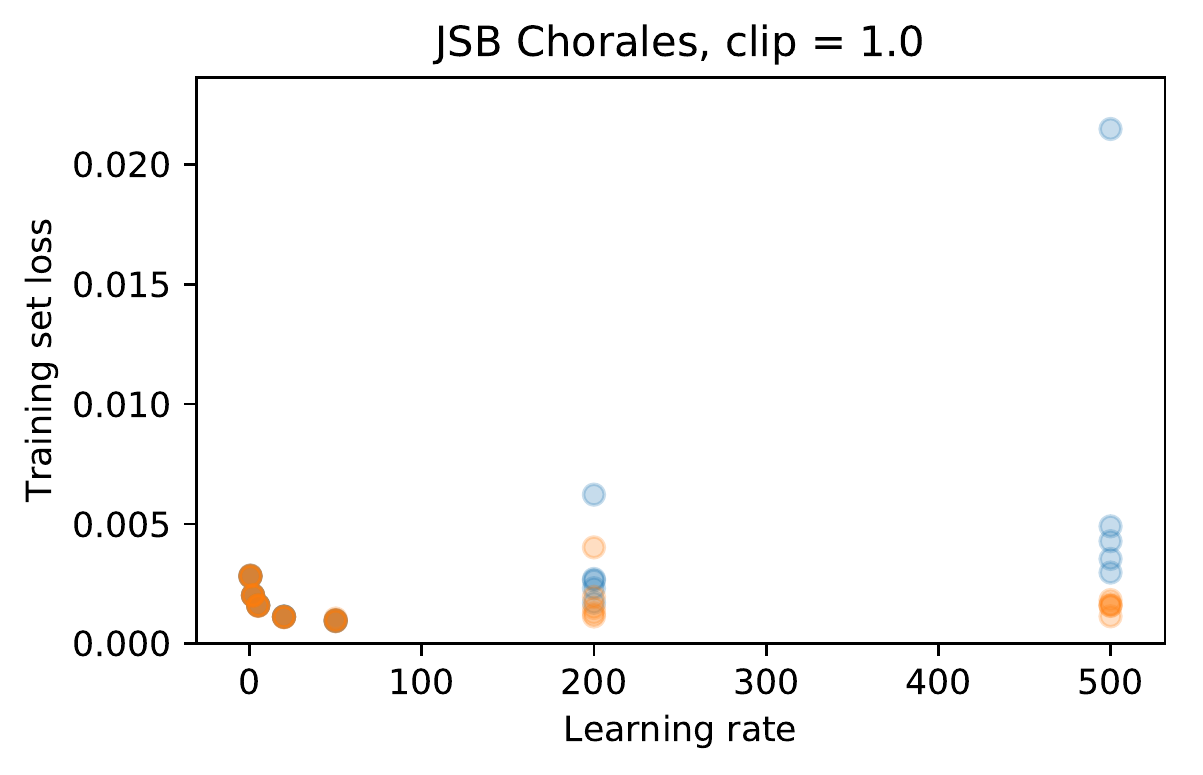}
\end{minipage}
\begin{minipage}{.49\textwidth}
  \centering
  \includegraphics[width=.79\linewidth]{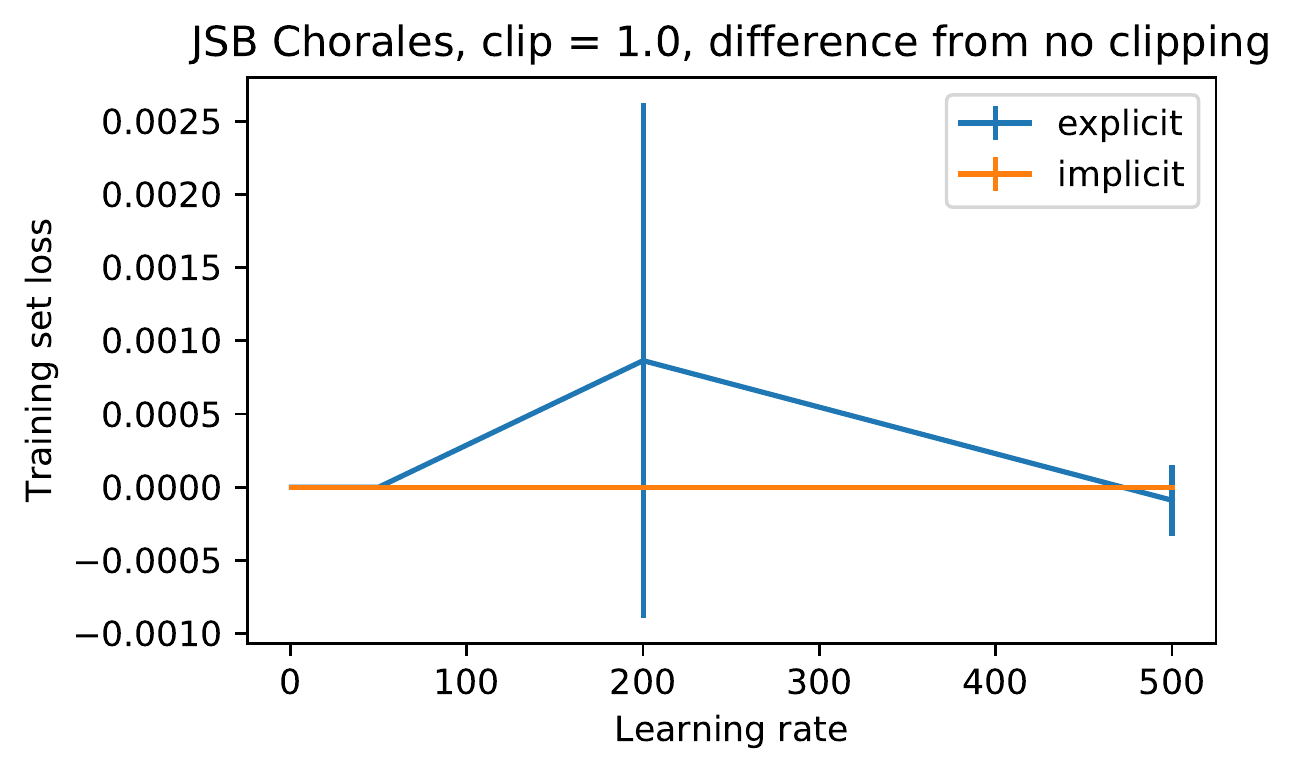}
\end{minipage}%
\hfill
\begin{minipage}{.49\textwidth}
  \centering
  \includegraphics[width=.79\linewidth]{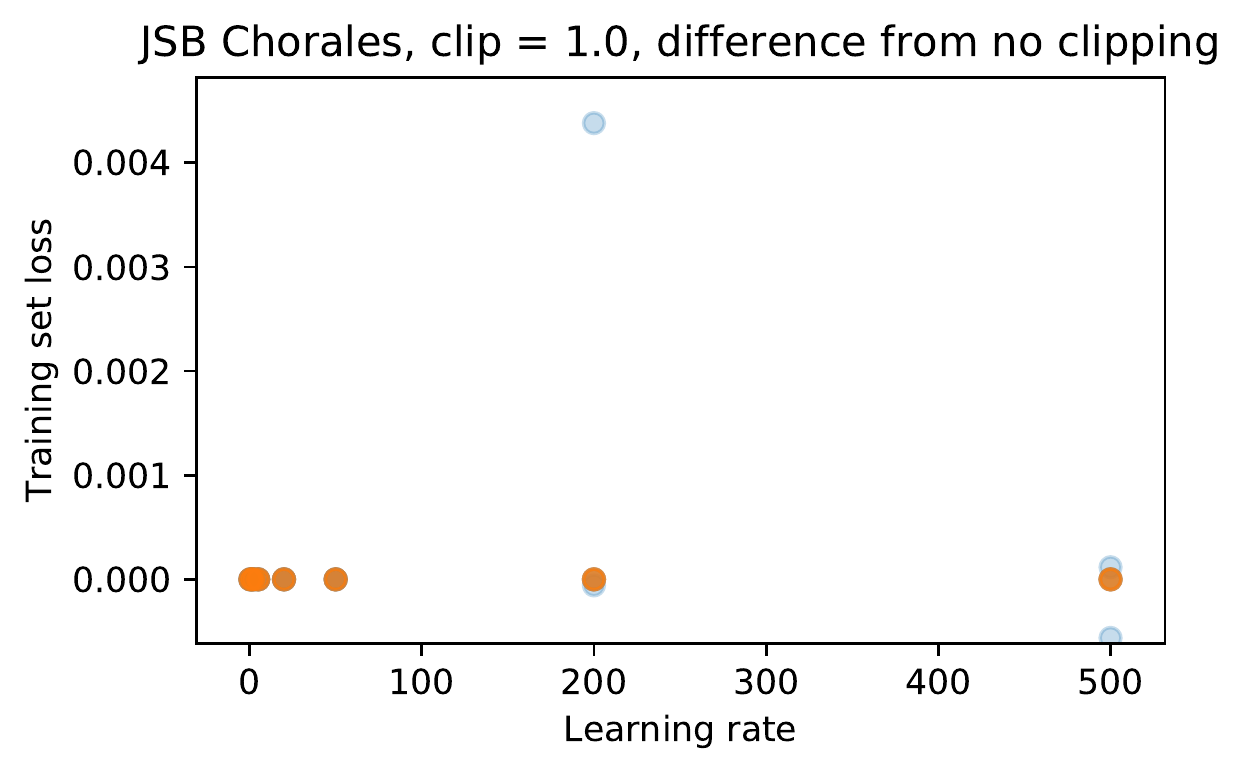}
\end{minipage}
\end{figure}

\vspace{0cm}

\begin{figure}[h]
\begin{minipage}{.49\textwidth}
  \centering
  \includegraphics[width=.79\linewidth]{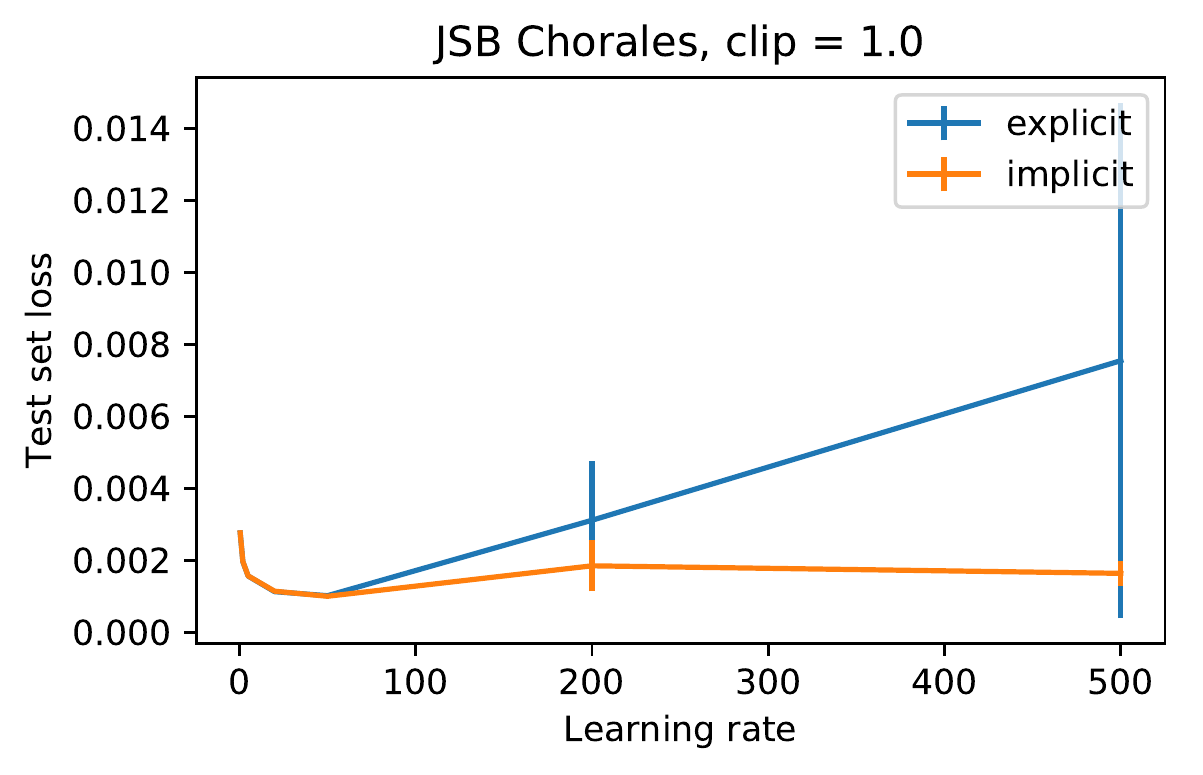}
\end{minipage}%
\hfill
\begin{minipage}{.49\textwidth}
  \centering
  \includegraphics[width=.79\linewidth]{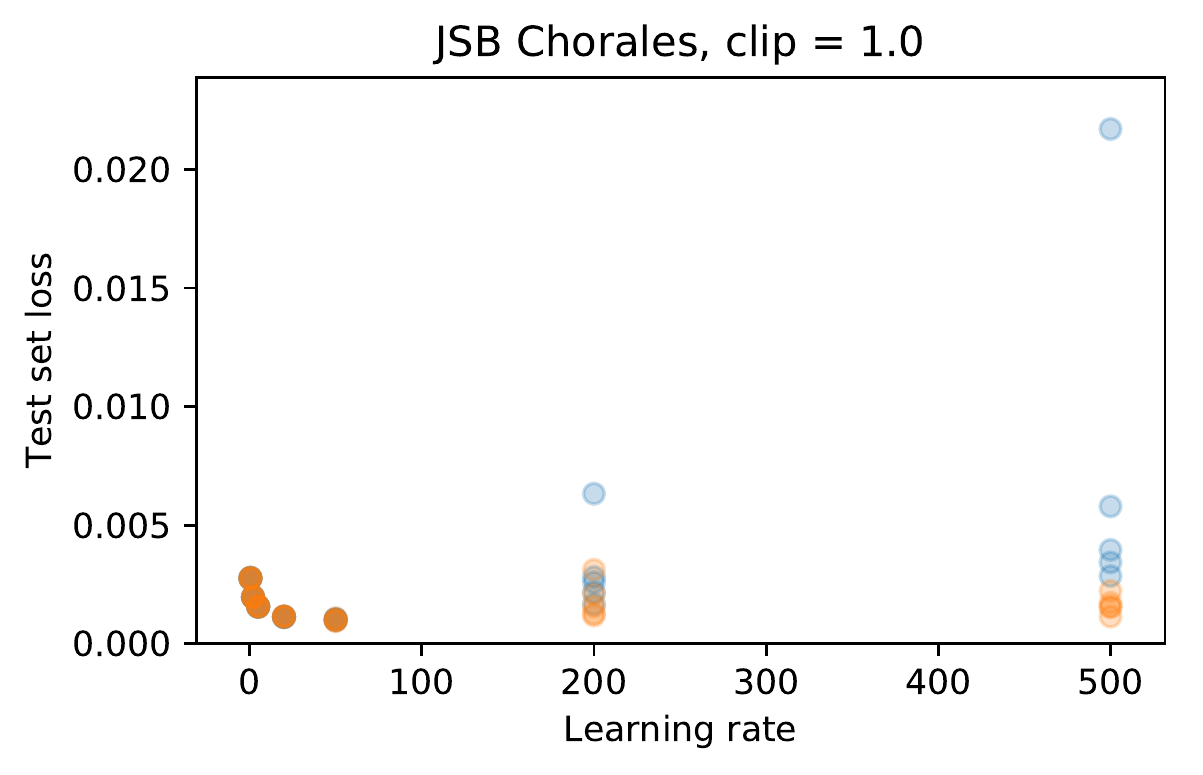}
\end{minipage}
\begin{minipage}{.49\textwidth}
  \centering
  \includegraphics[width=.79\linewidth]{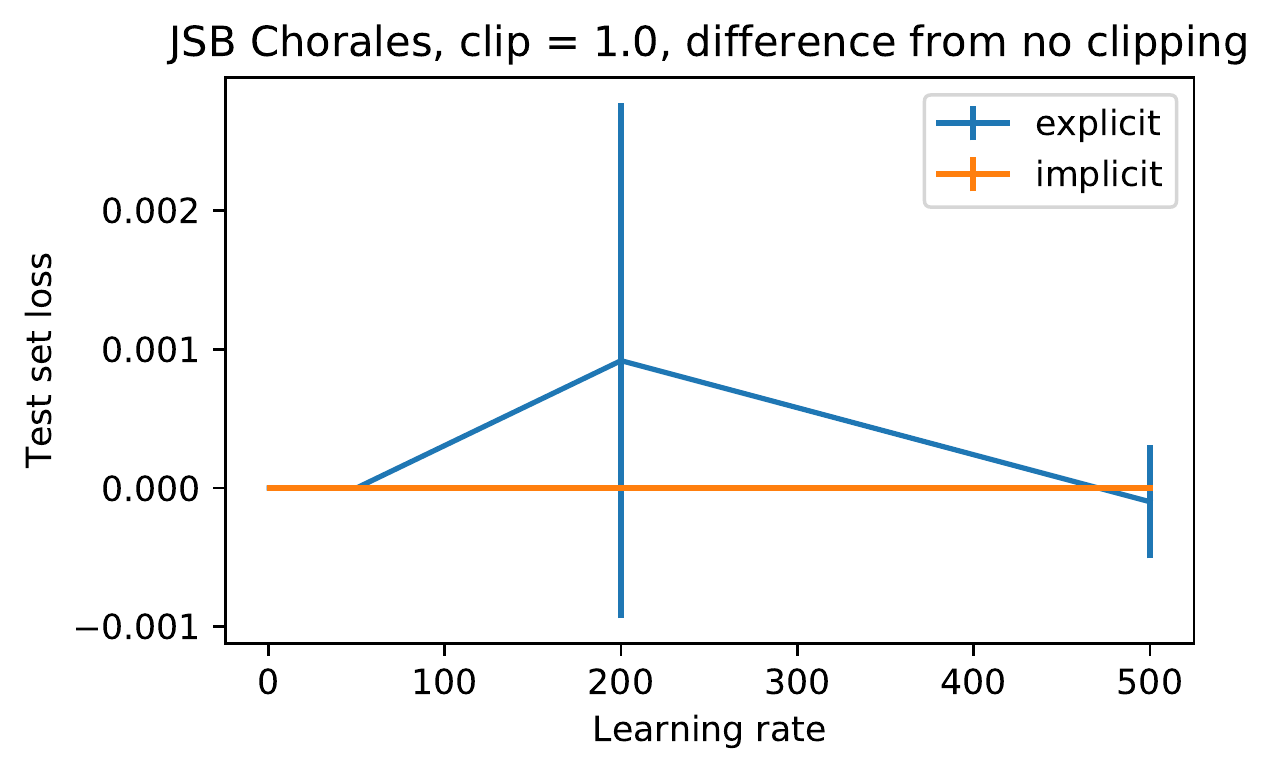}
\end{minipage}%
\hfill
\begin{minipage}{.49\textwidth}
  \centering
  \includegraphics[width=.79\linewidth]{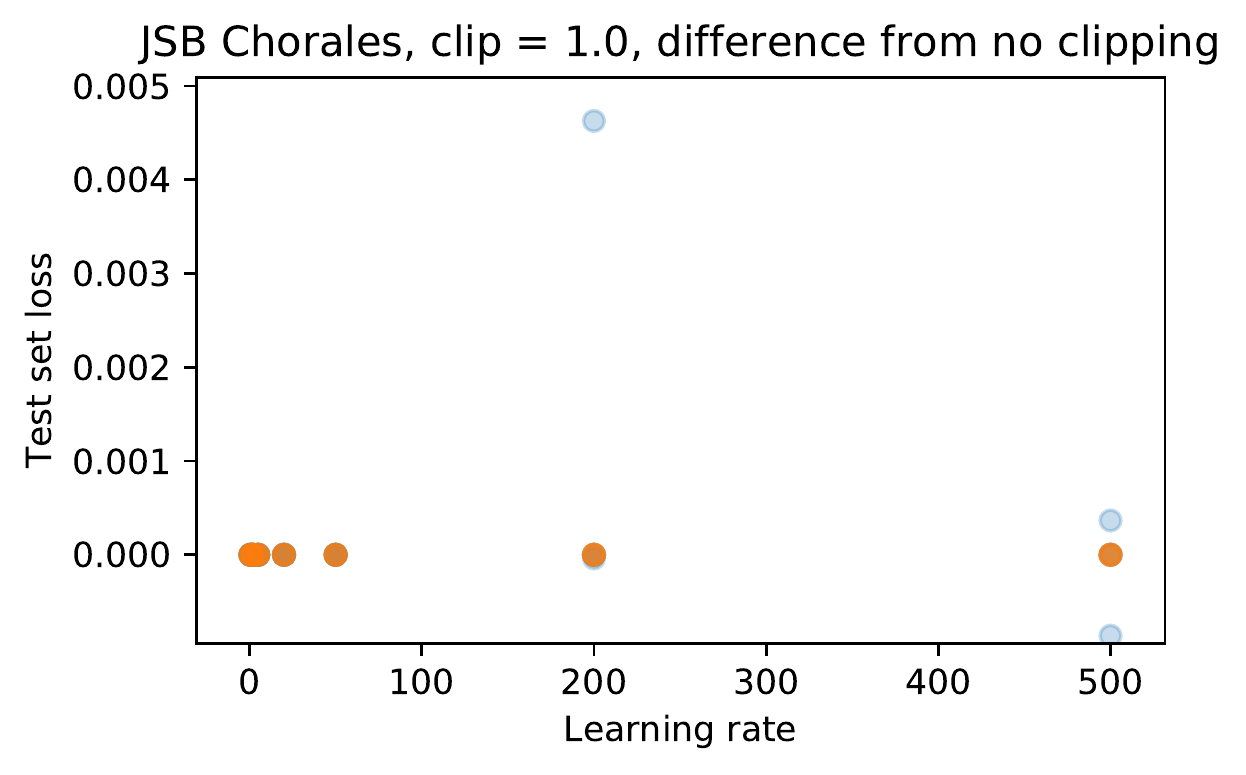}
\end{minipage}
\end{figure}

\begin{figure}[h]
\centering
\begin{minipage}{.49\textwidth}
  \centering
  \includegraphics[width=.79\linewidth]{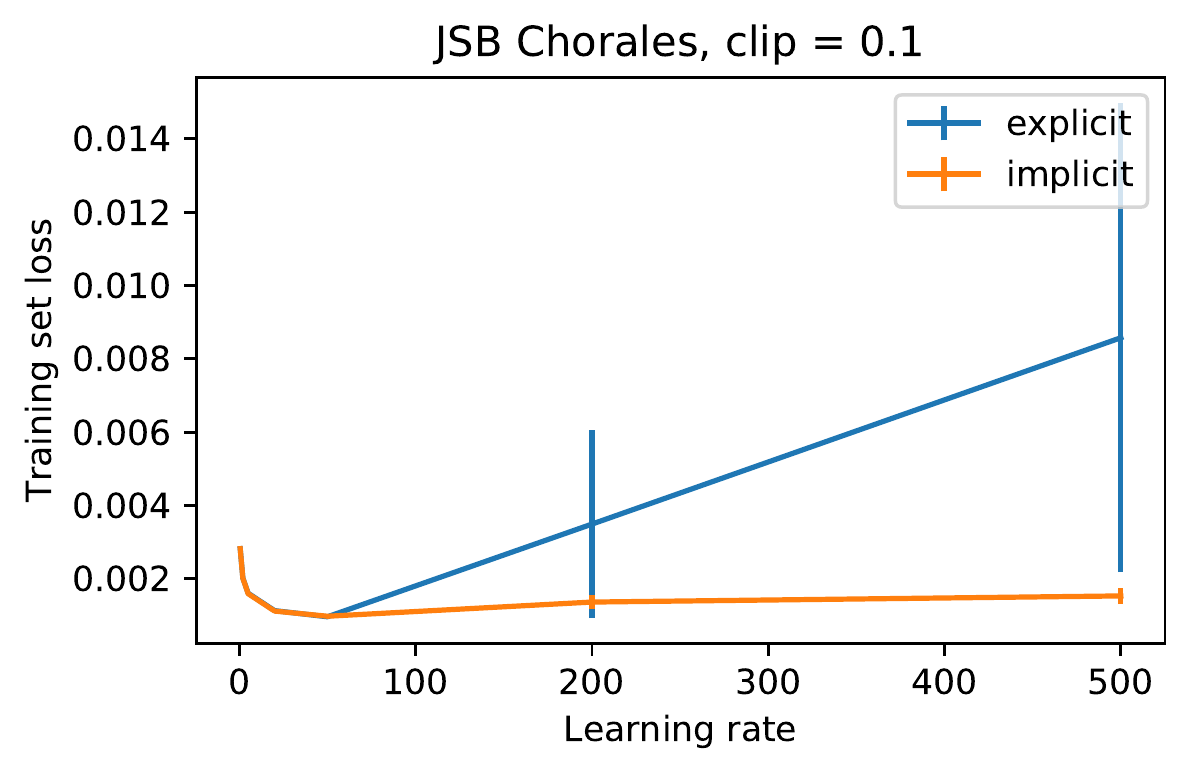}
\end{minipage}%
\hfill
\begin{minipage}{.49\textwidth}
  \centering
  \includegraphics[width=.79\linewidth]{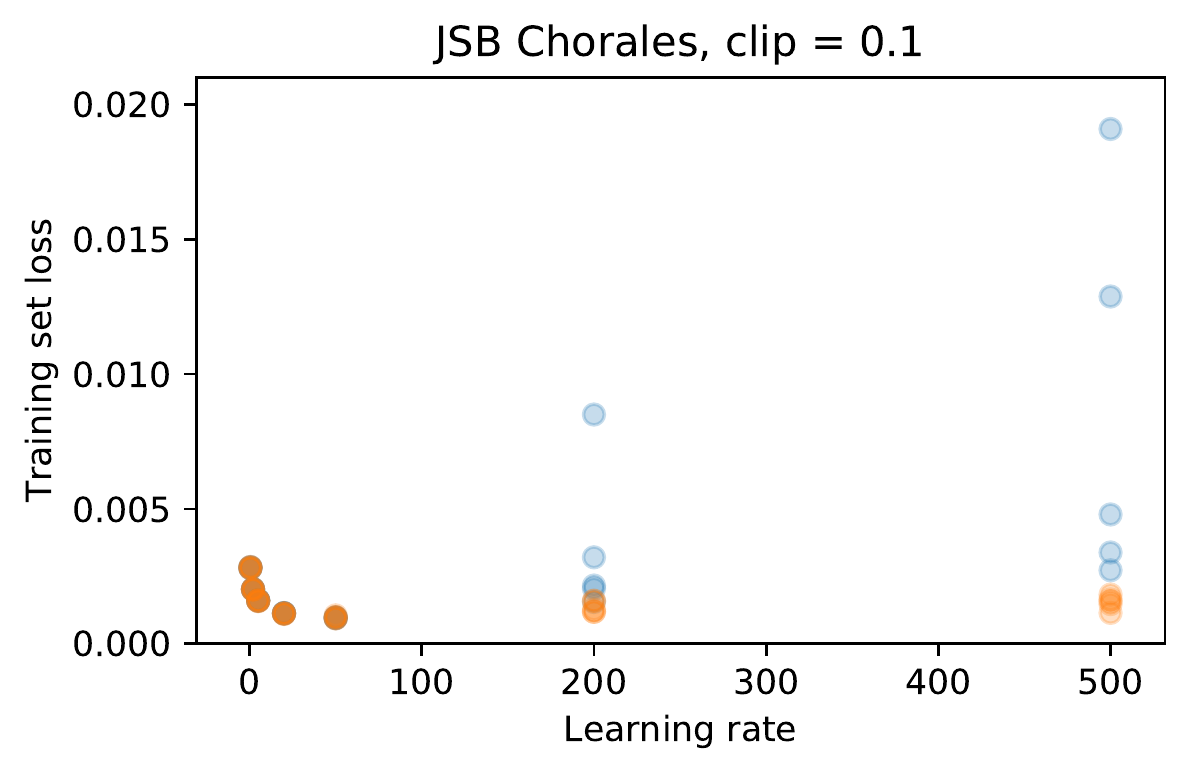}
\end{minipage}
\begin{minipage}{.49\textwidth}
  \centering
  \includegraphics[width=.79\linewidth]{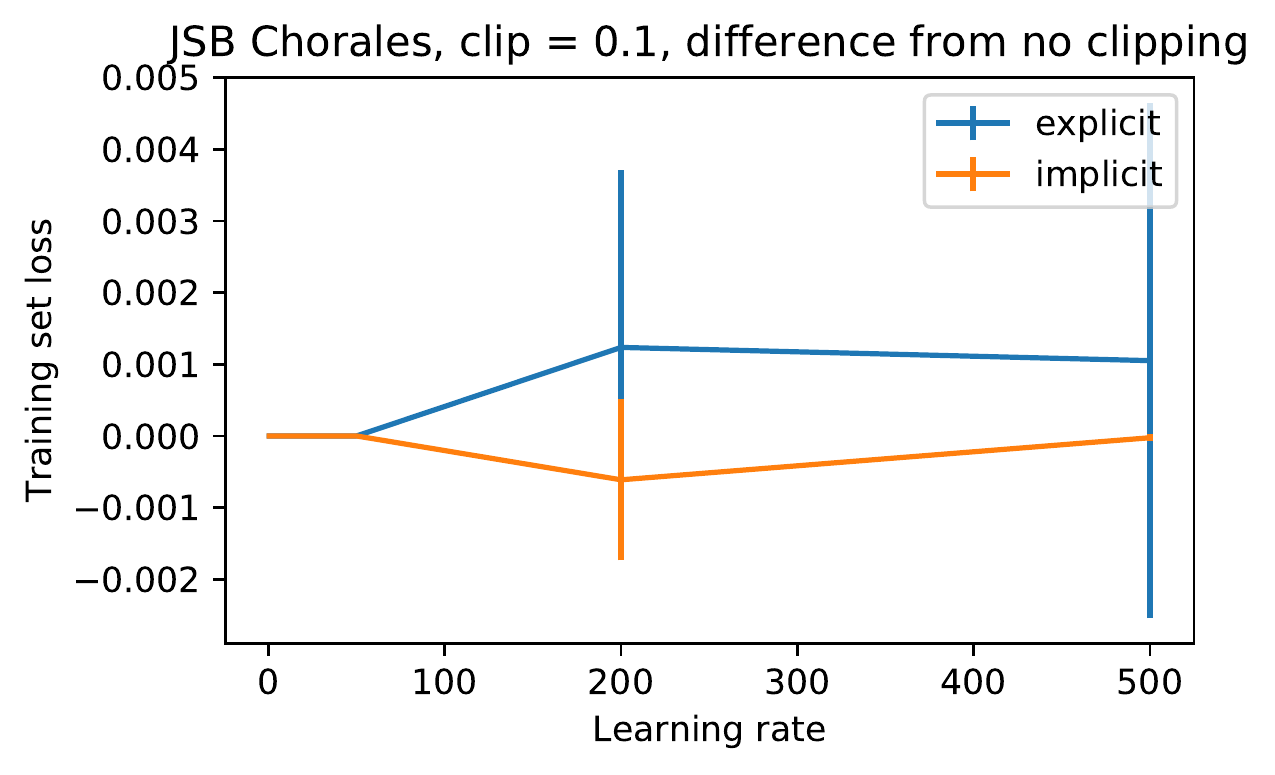}
\end{minipage}%
\hfill
\begin{minipage}{.49\textwidth}
  \centering
  \includegraphics[width=.79\linewidth]{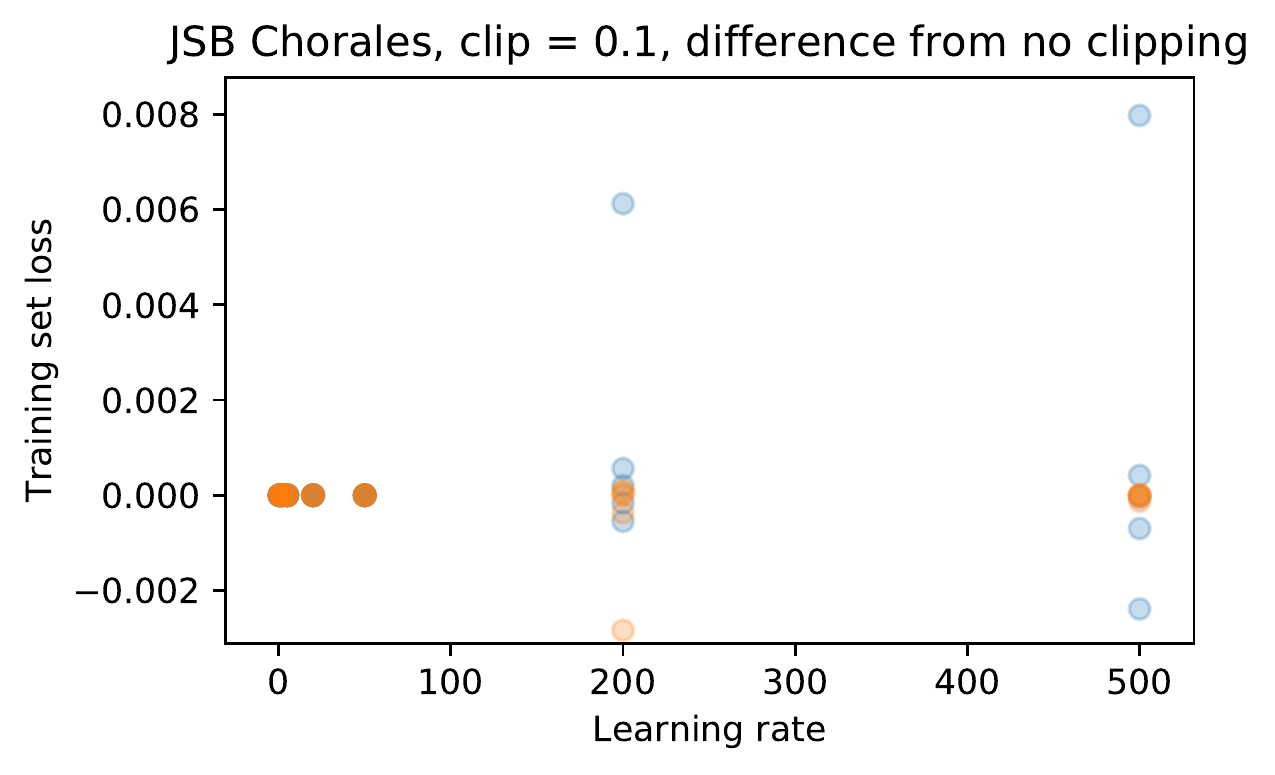}
\end{minipage}
\end{figure}

\vspace{0cm}

\begin{figure}[h]
\begin{minipage}{.49\textwidth}
  \centering
  \includegraphics[width=.79\linewidth]{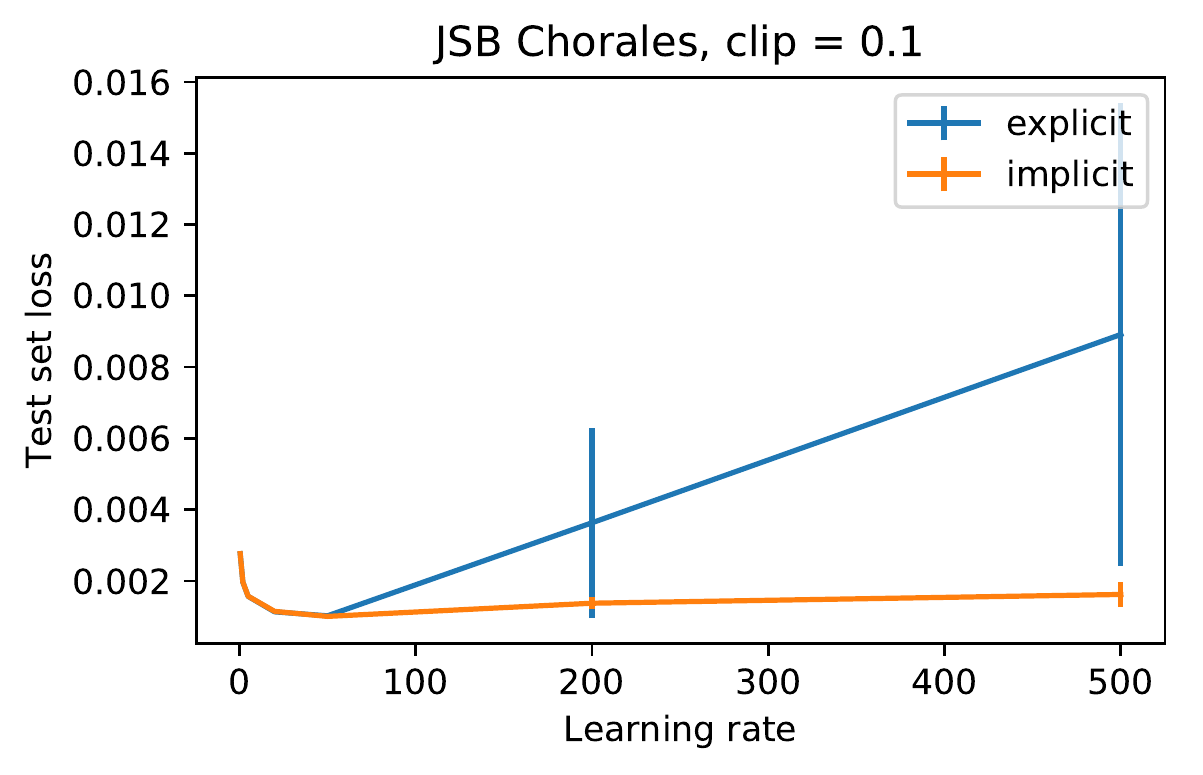}
\end{minipage}%
\hfill
\begin{minipage}{.49\textwidth}
  \centering
  \includegraphics[width=.79\linewidth]{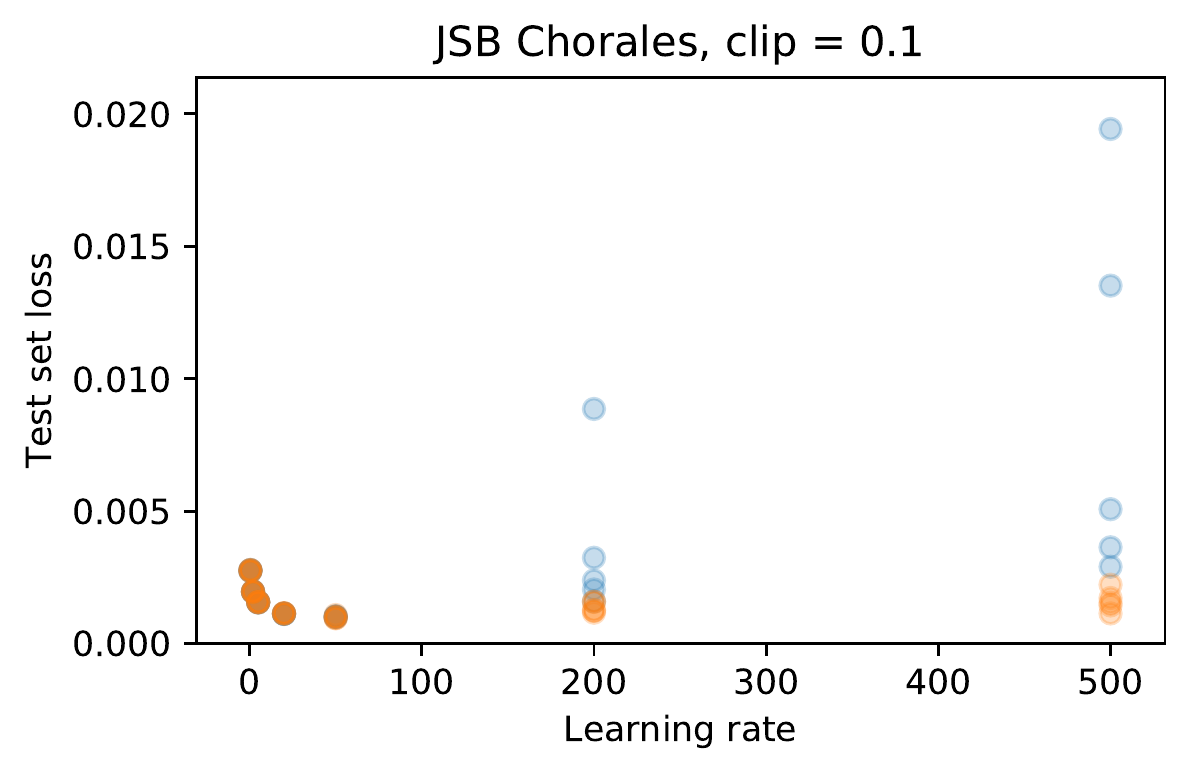}
\end{minipage}
\begin{minipage}{.49\textwidth}
  \centering
  \includegraphics[width=.79\linewidth]{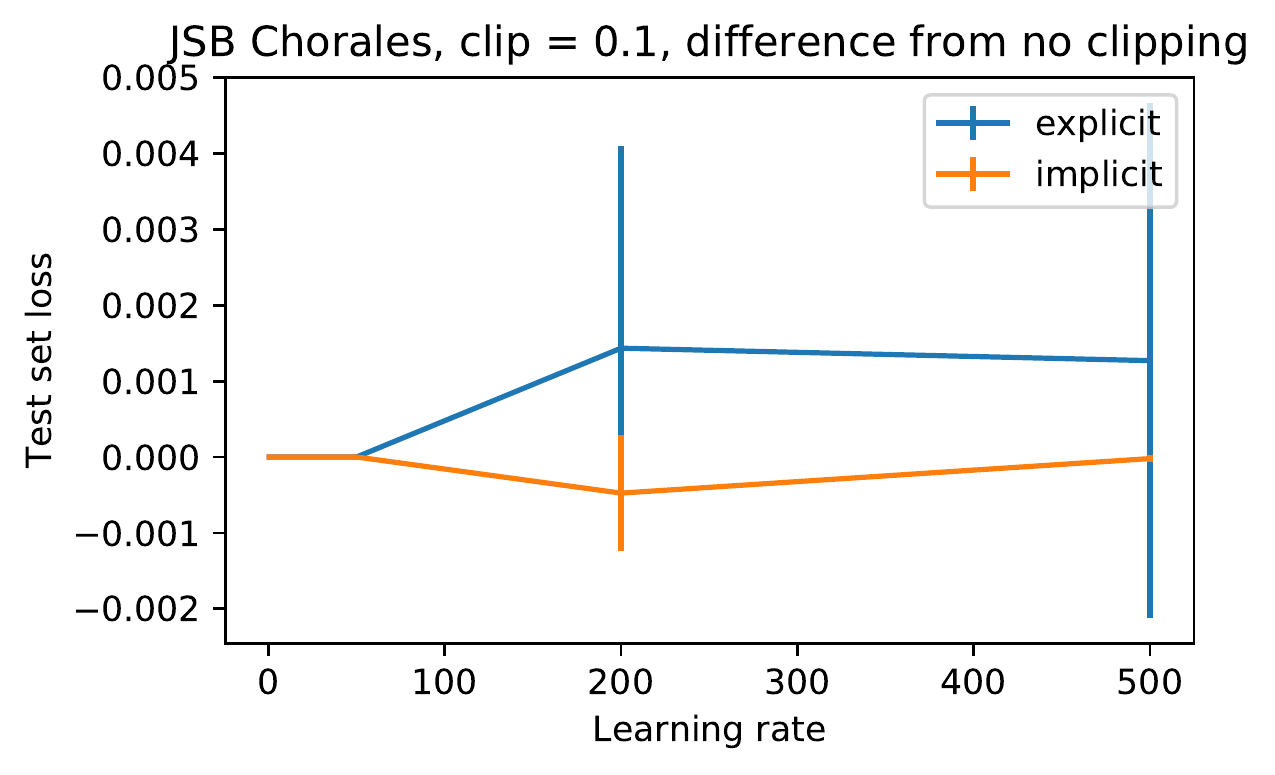}
\end{minipage}%
\hfill
\begin{minipage}{.49\textwidth}
  \centering
  \includegraphics[width=.79\linewidth]{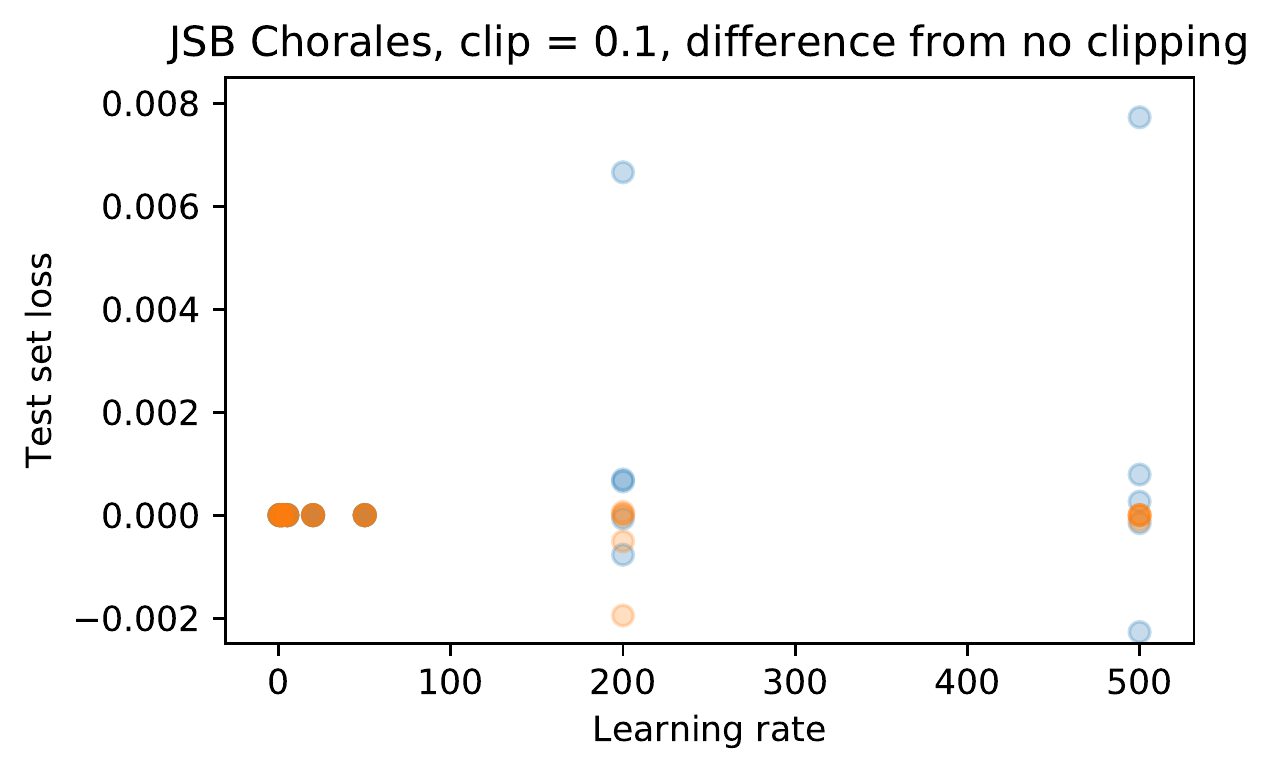}
\end{minipage}
\end{figure}

\clearpage
\subsection{MuseData}

\begin{figure}[h]
\centering
\begin{minipage}{.49\textwidth}
  \centering
  \includegraphics[width=.79\linewidth]{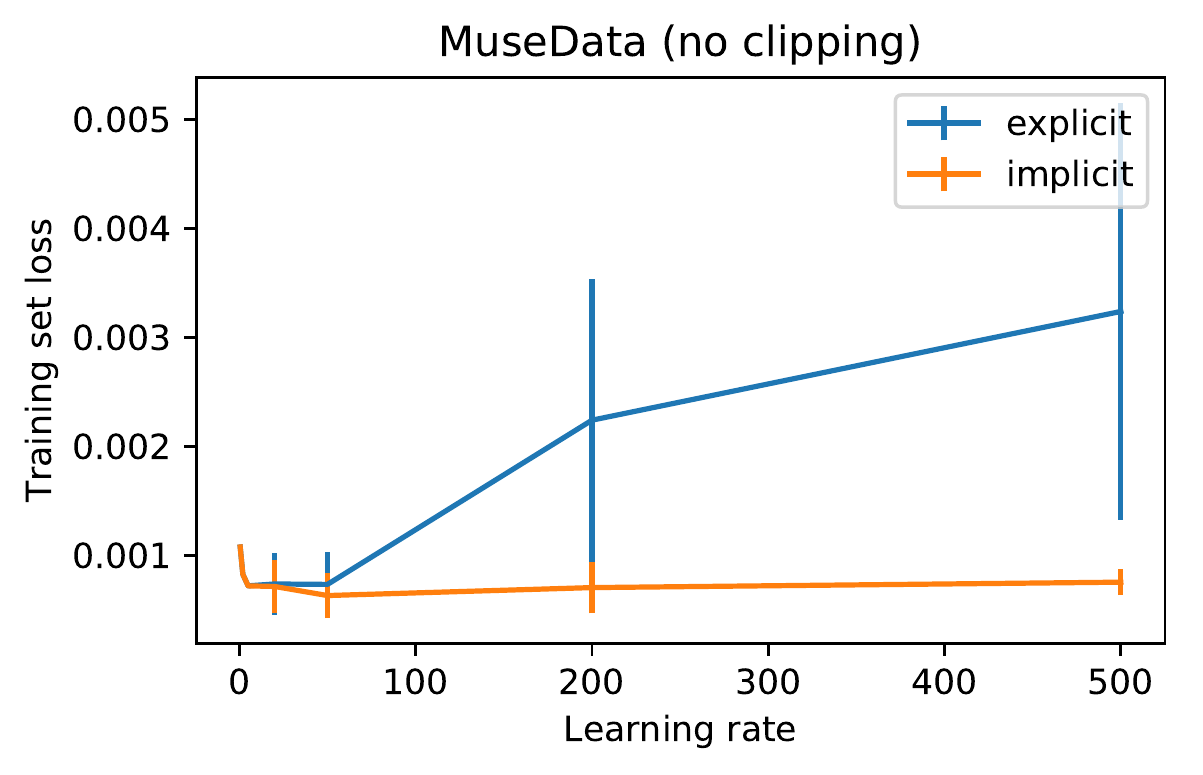}
\end{minipage}%
\hfill
\begin{minipage}{.49\textwidth}
  \centering
  \includegraphics[width=.79\linewidth]{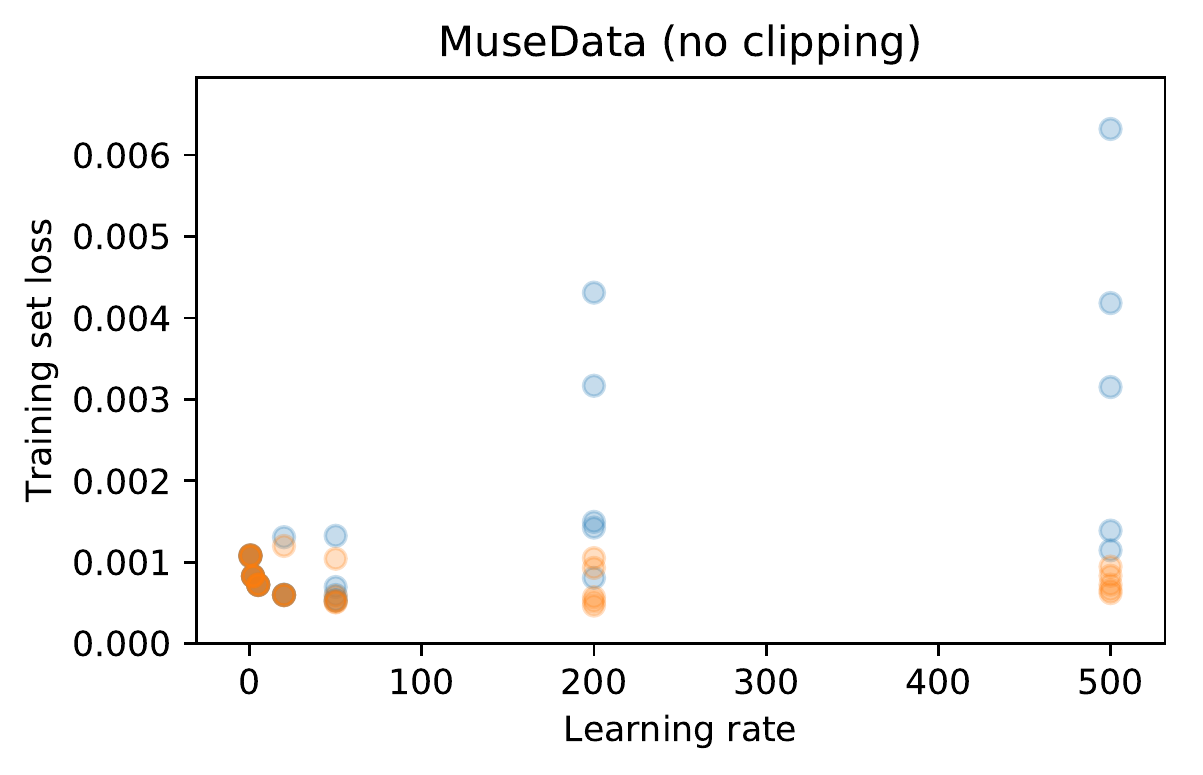}
\end{minipage}
\begin{minipage}{.49\textwidth}
  \centering
  \includegraphics[width=.79\linewidth]{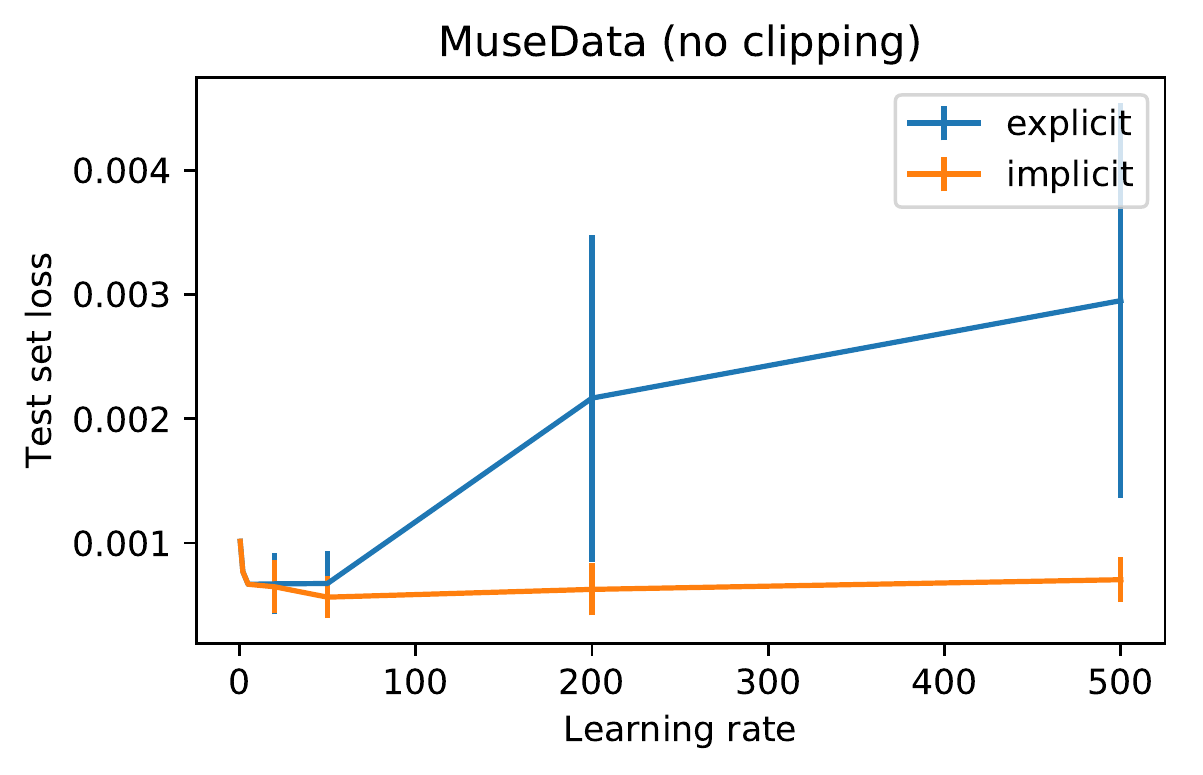}
\end{minipage}%
\hfill
\begin{minipage}{.49\textwidth}
  \centering
  \includegraphics[width=.79\linewidth]{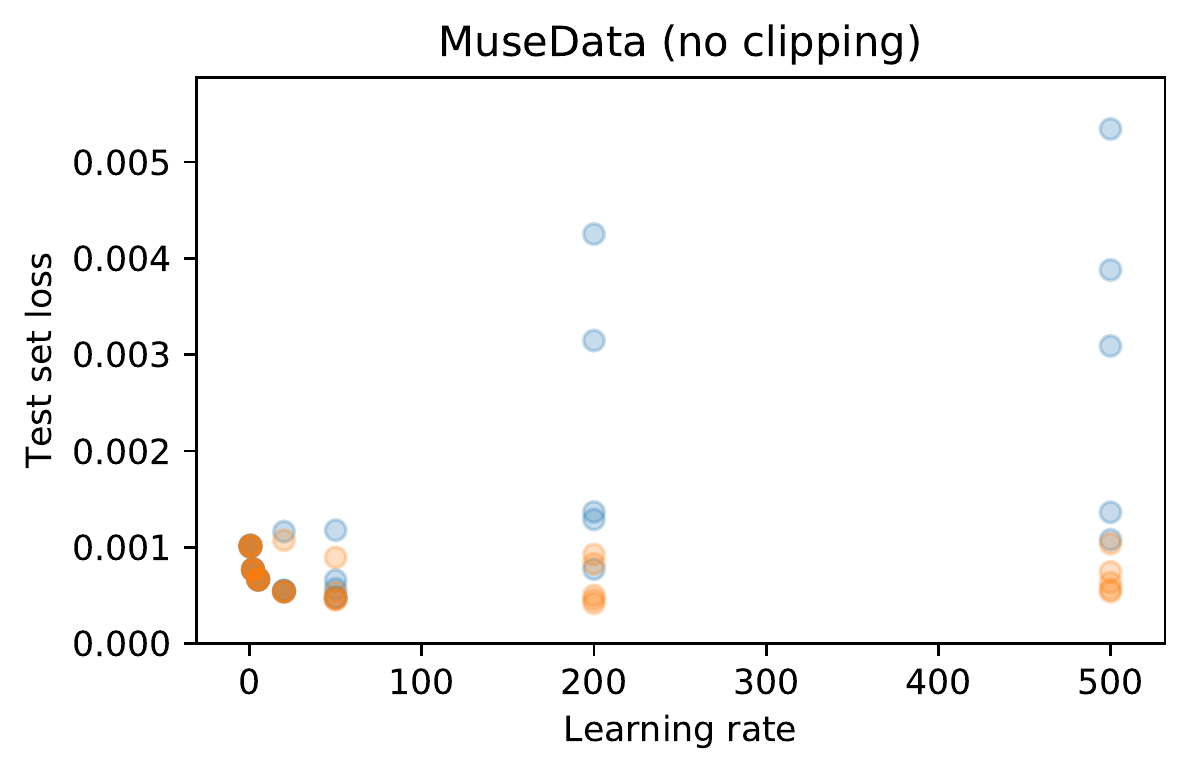}
\end{minipage}
\end{figure}

\begin{figure}[h]
\centering
\begin{minipage}{.49\textwidth}
  \centering
  \includegraphics[width=.79\linewidth]{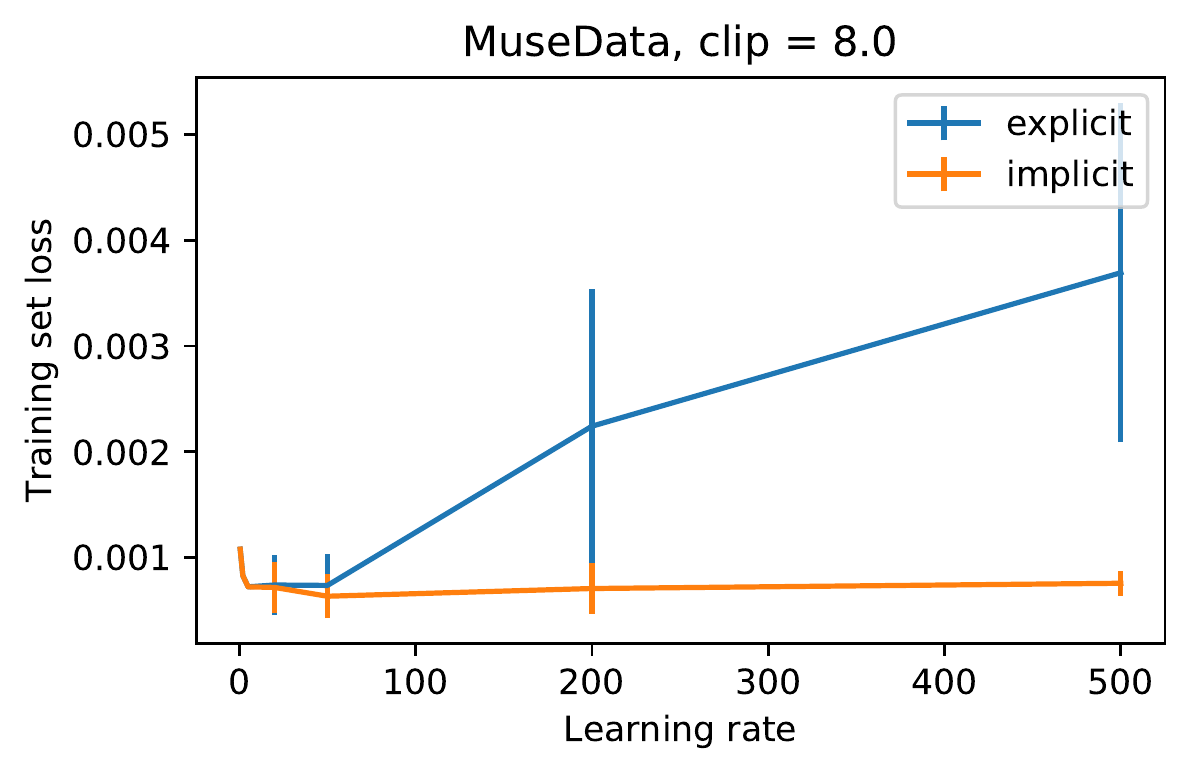}
\end{minipage}%
\hfill
\begin{minipage}{.49\textwidth}
  \centering
  \includegraphics[width=.79\linewidth]{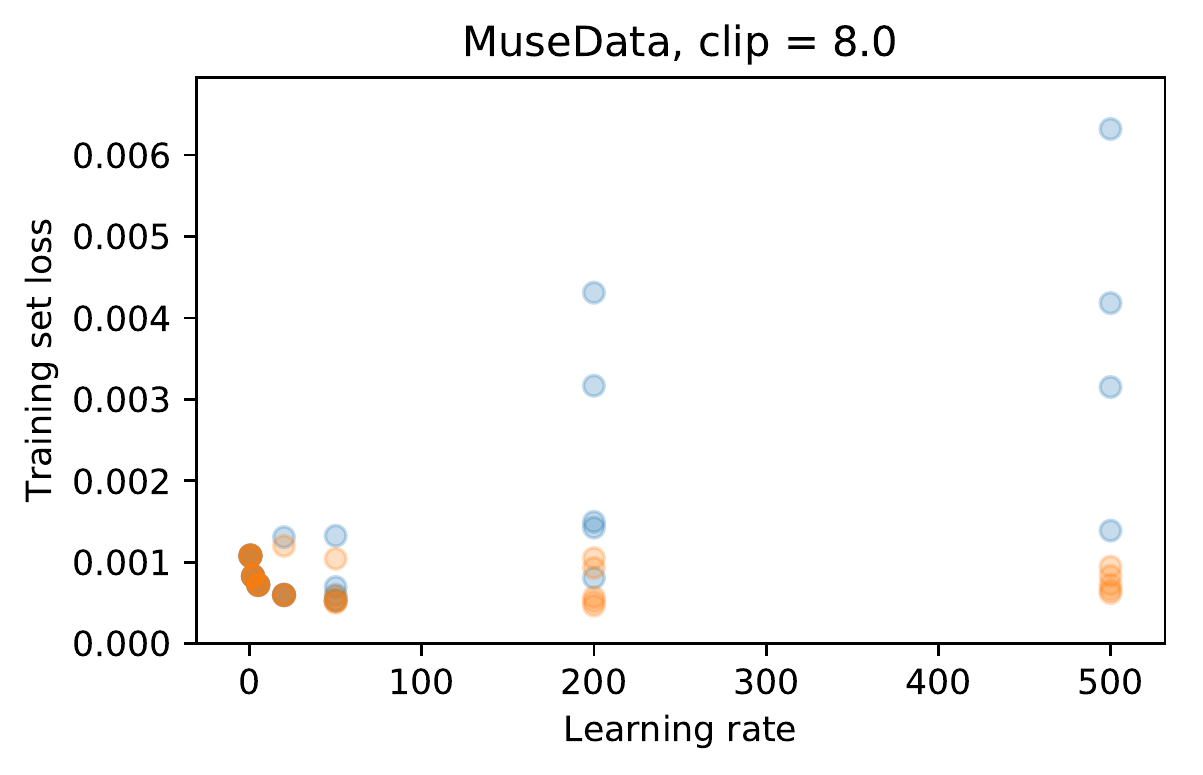}
\end{minipage}
\begin{minipage}{.49\textwidth}
  \centering
  \includegraphics[width=.79\linewidth]{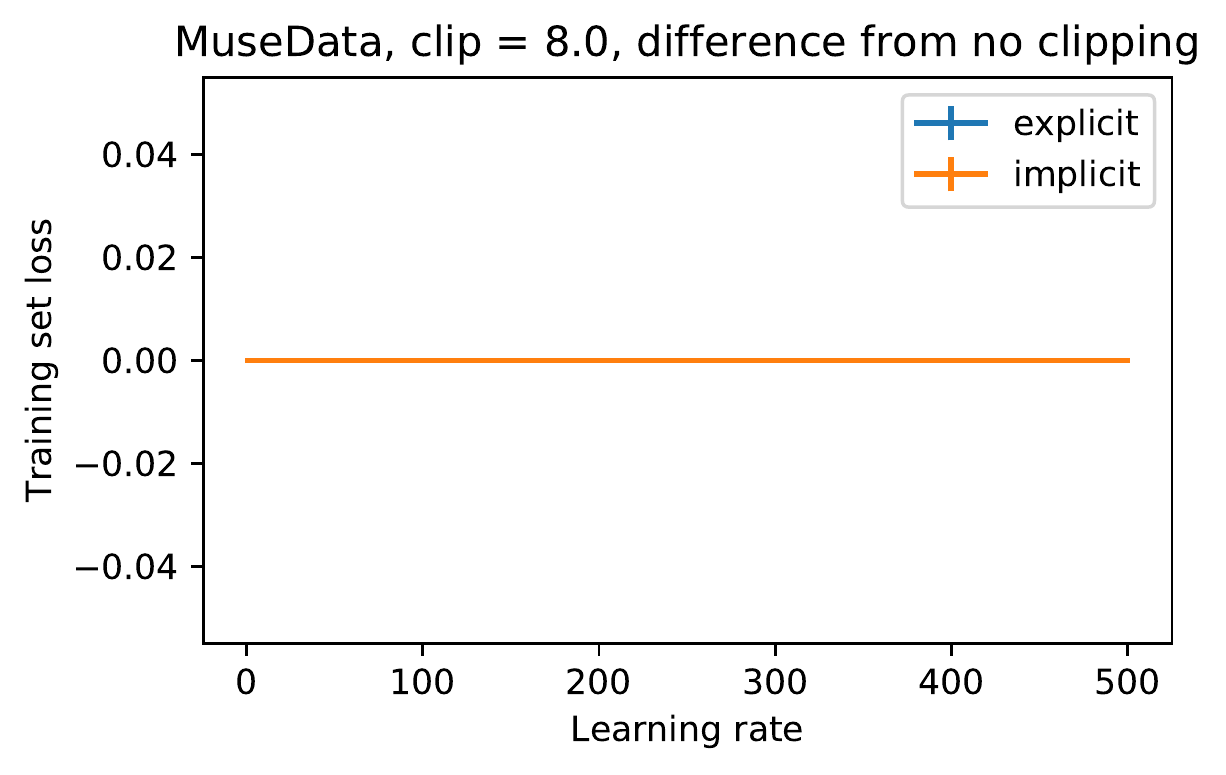}
\end{minipage}%
\hfill
\begin{minipage}{.49\textwidth}
  \centering
  \includegraphics[width=.79\linewidth]{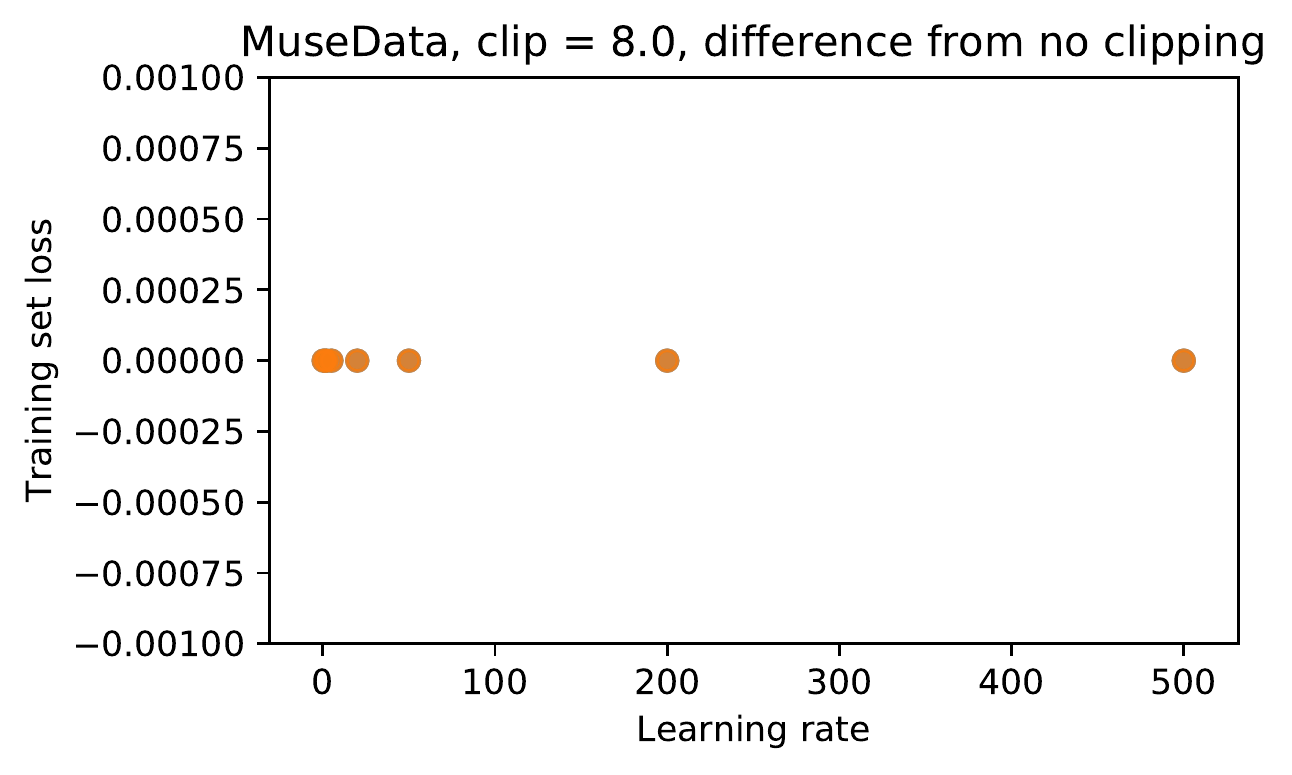}
\end{minipage}
\end{figure}

\vspace{0cm}

\begin{figure}[h]
\begin{minipage}{.49\textwidth}
  \centering
  \includegraphics[width=.79\linewidth]{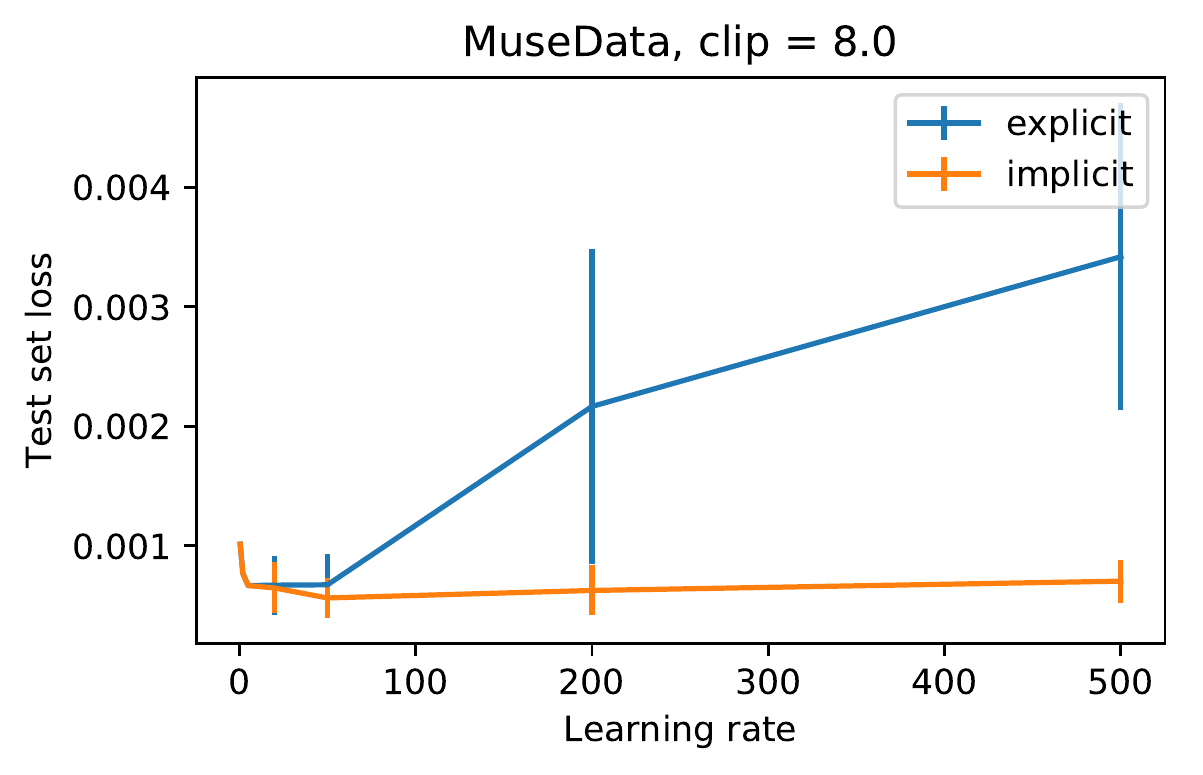}
\end{minipage}%
\hfill
\begin{minipage}{.49\textwidth}
  \centering
  \includegraphics[width=.79\linewidth]{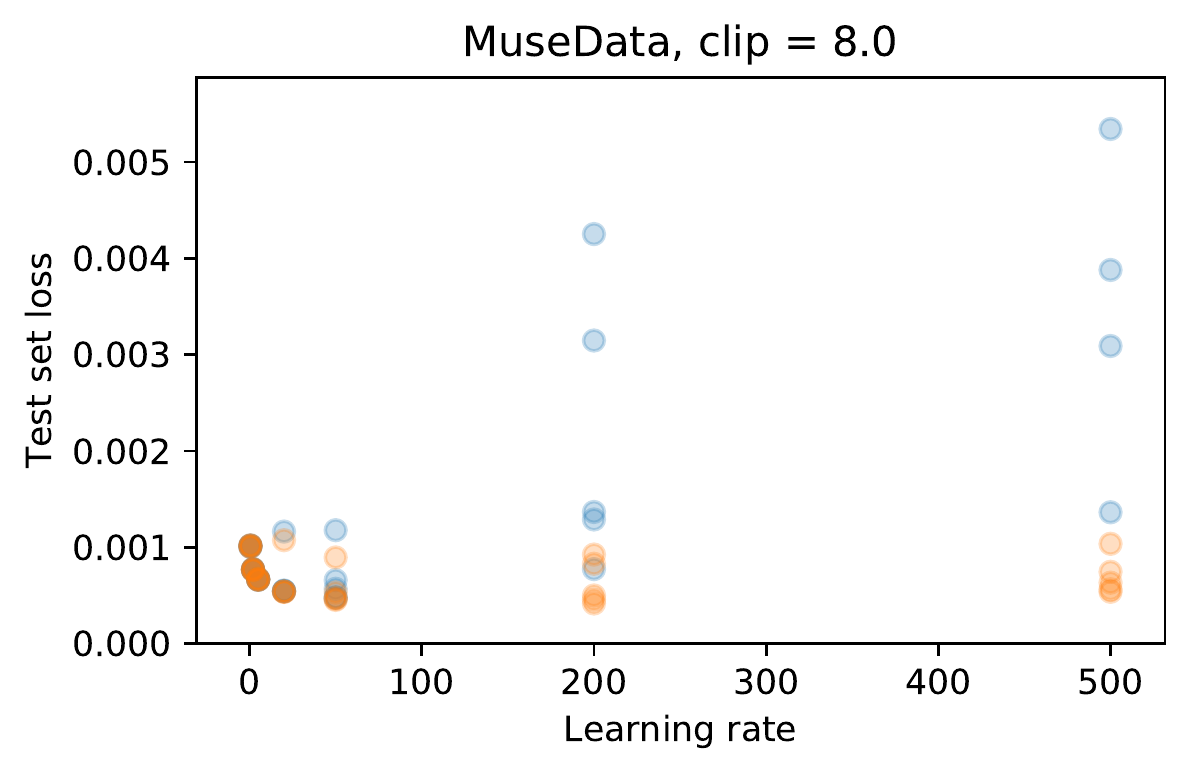}
\end{minipage}
\begin{minipage}{.49\textwidth}
  \centering
  \includegraphics[width=.79\linewidth]{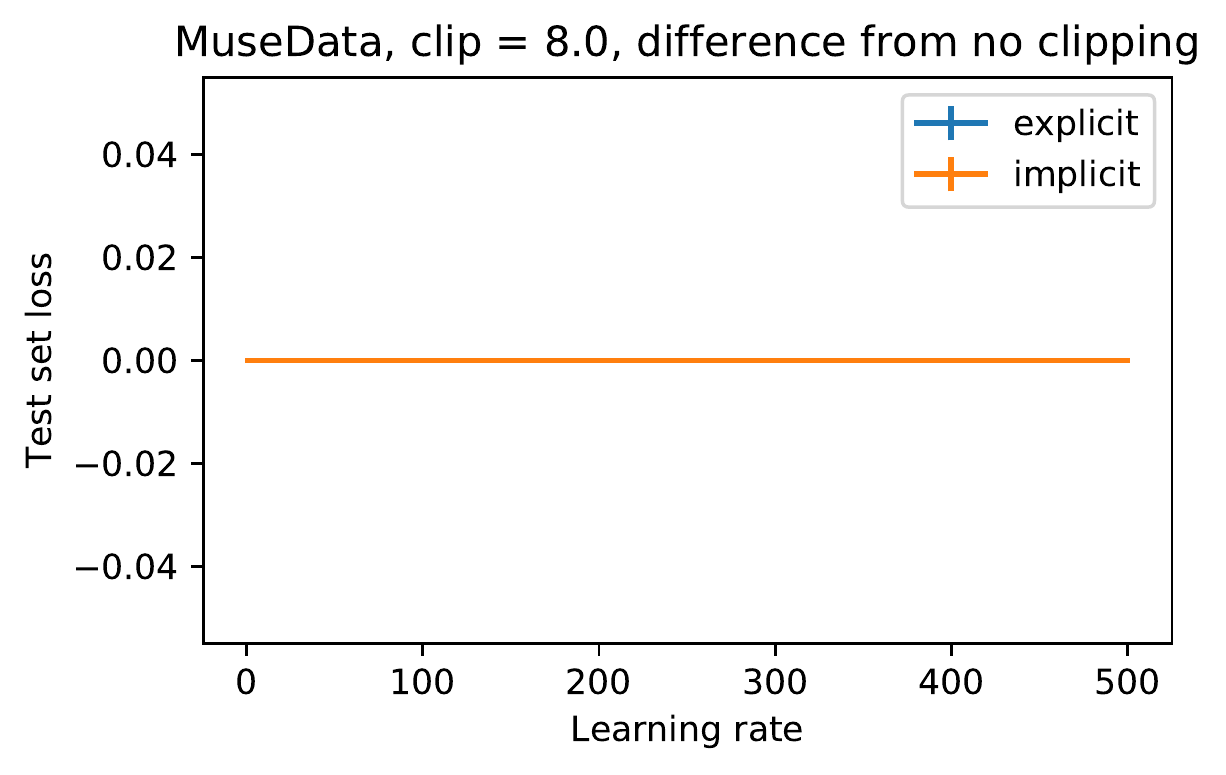}
\end{minipage}%
\hfill
\begin{minipage}{.49\textwidth}
  \centering
  \includegraphics[width=.79\linewidth]{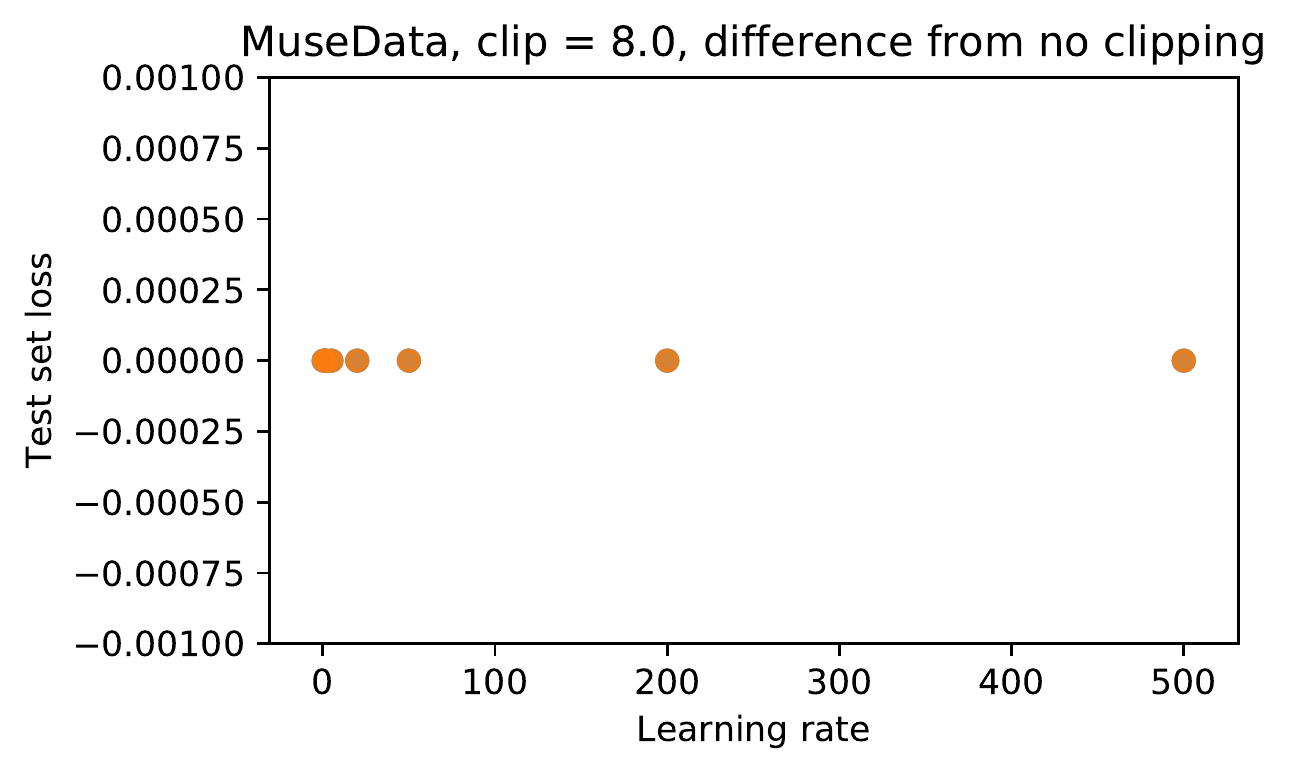}
\end{minipage}
\end{figure}

\clearpage
\subsection{Nottingham}

\begin{figure}[h]
\centering
\begin{minipage}{.49\textwidth}
  \centering
  \includegraphics[width=.79\linewidth]{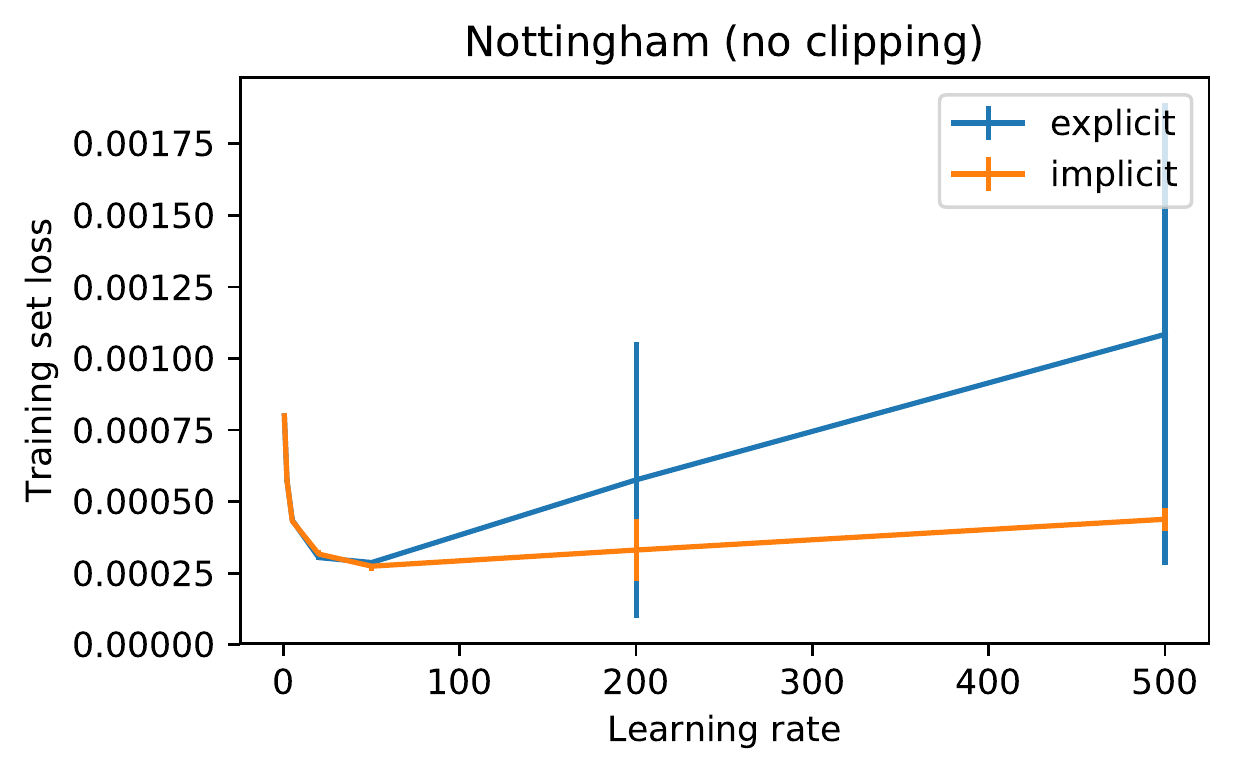}
\end{minipage}%
\hfill
\begin{minipage}{.49\textwidth}
  \centering
  \includegraphics[width=.79\linewidth]{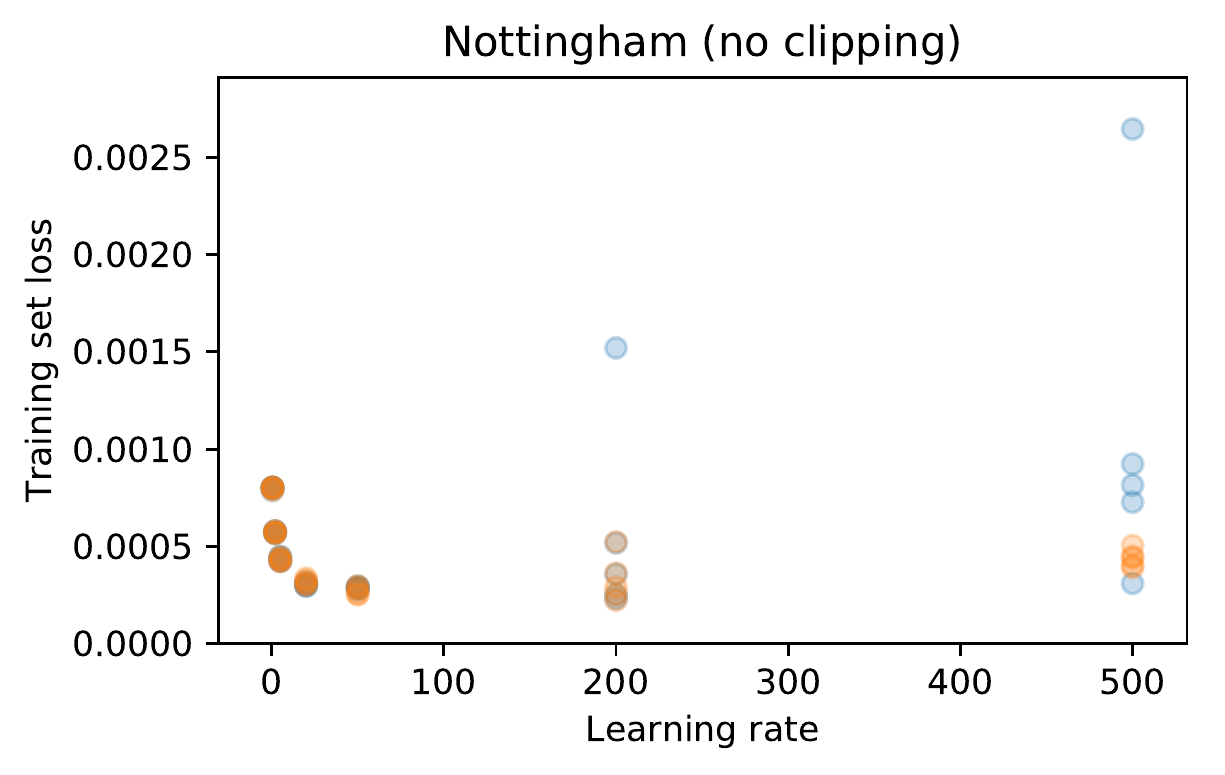}
\end{minipage}
\begin{minipage}{.49\textwidth}
  \centering
  \includegraphics[width=.79\linewidth]{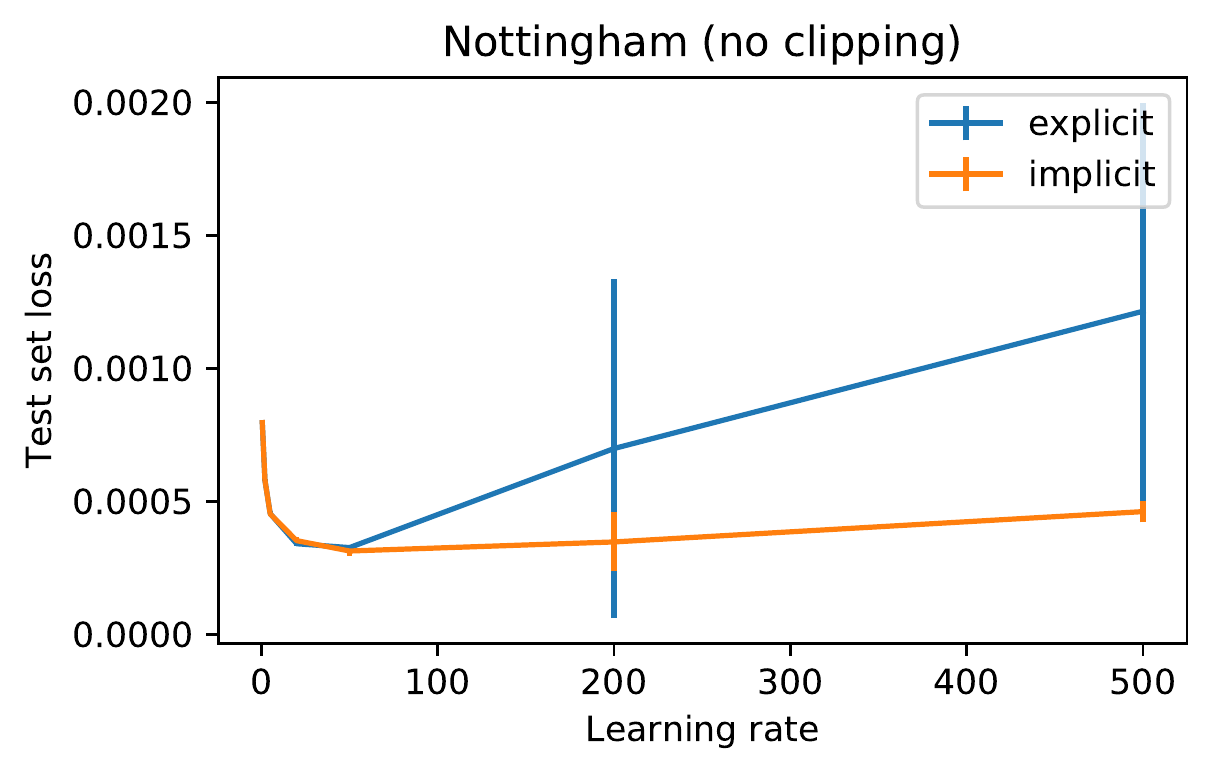}
\end{minipage}%
\hfill
\begin{minipage}{.49\textwidth}
  \centering
  \includegraphics[width=.79\linewidth]{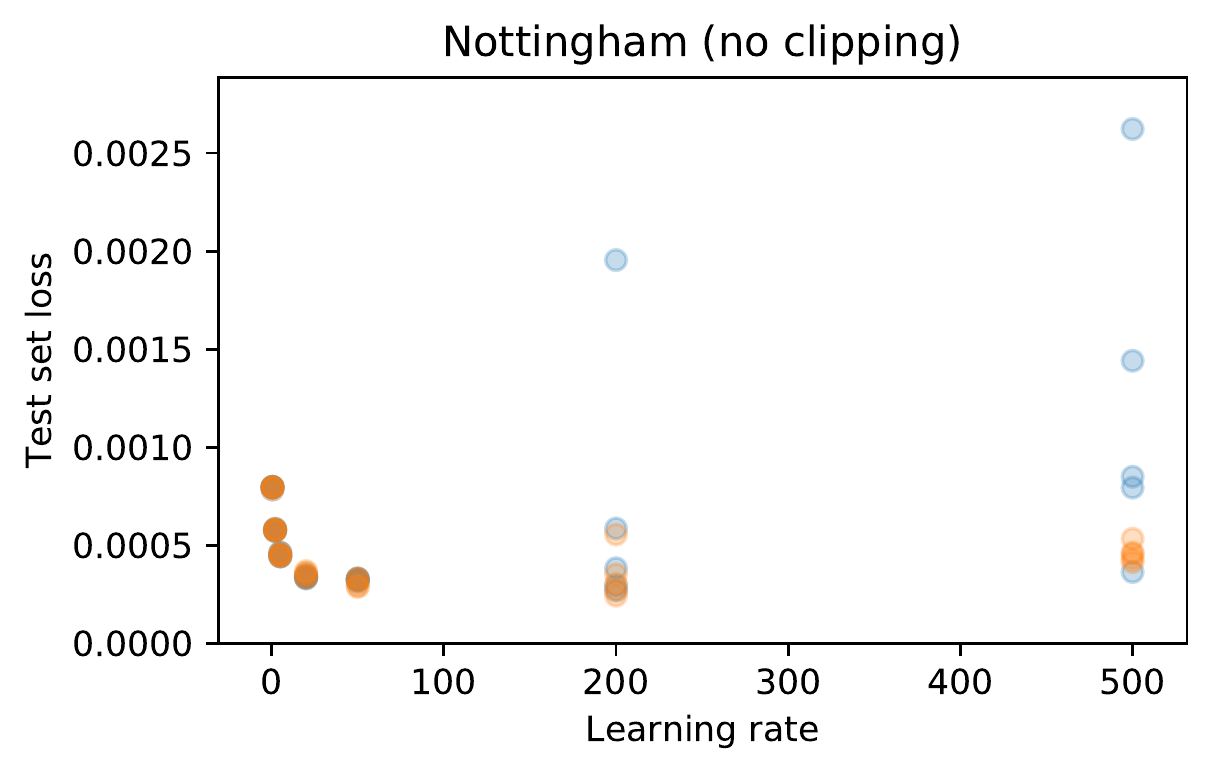}
\end{minipage}
\end{figure}

\begin{figure}[h]
\centering
\begin{minipage}{.49\textwidth}
  \centering
  \includegraphics[width=.79\linewidth]{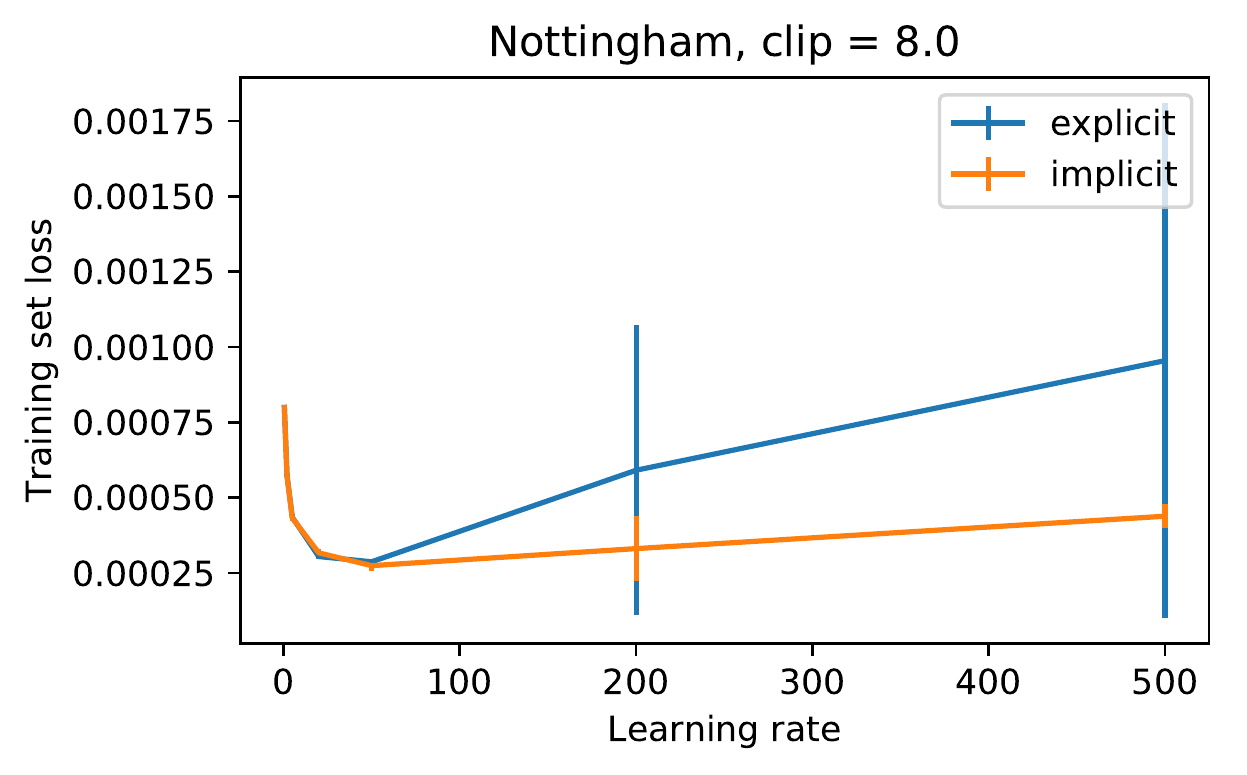}
\end{minipage}%
\hfill
\begin{minipage}{.49\textwidth}
  \centering
  \includegraphics[width=.79\linewidth]{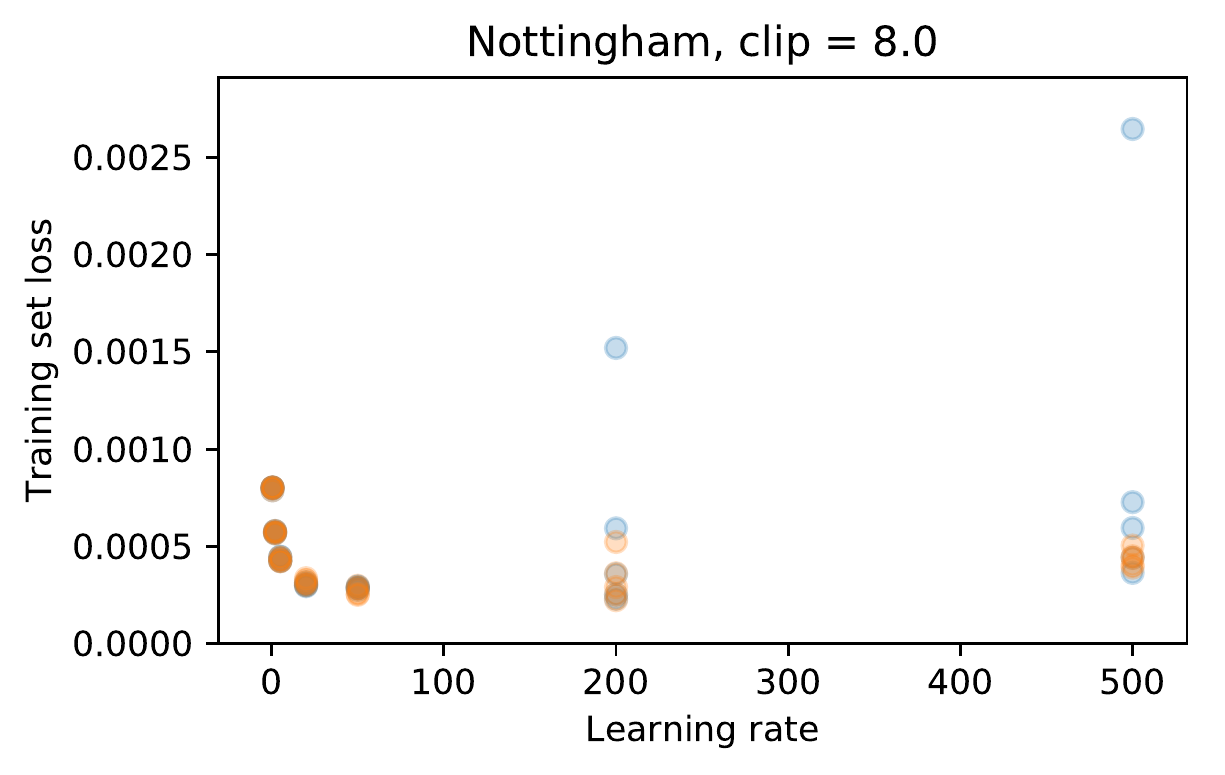}
\end{minipage}
\begin{minipage}{.49\textwidth}
  \centering
  \includegraphics[width=.79\linewidth]{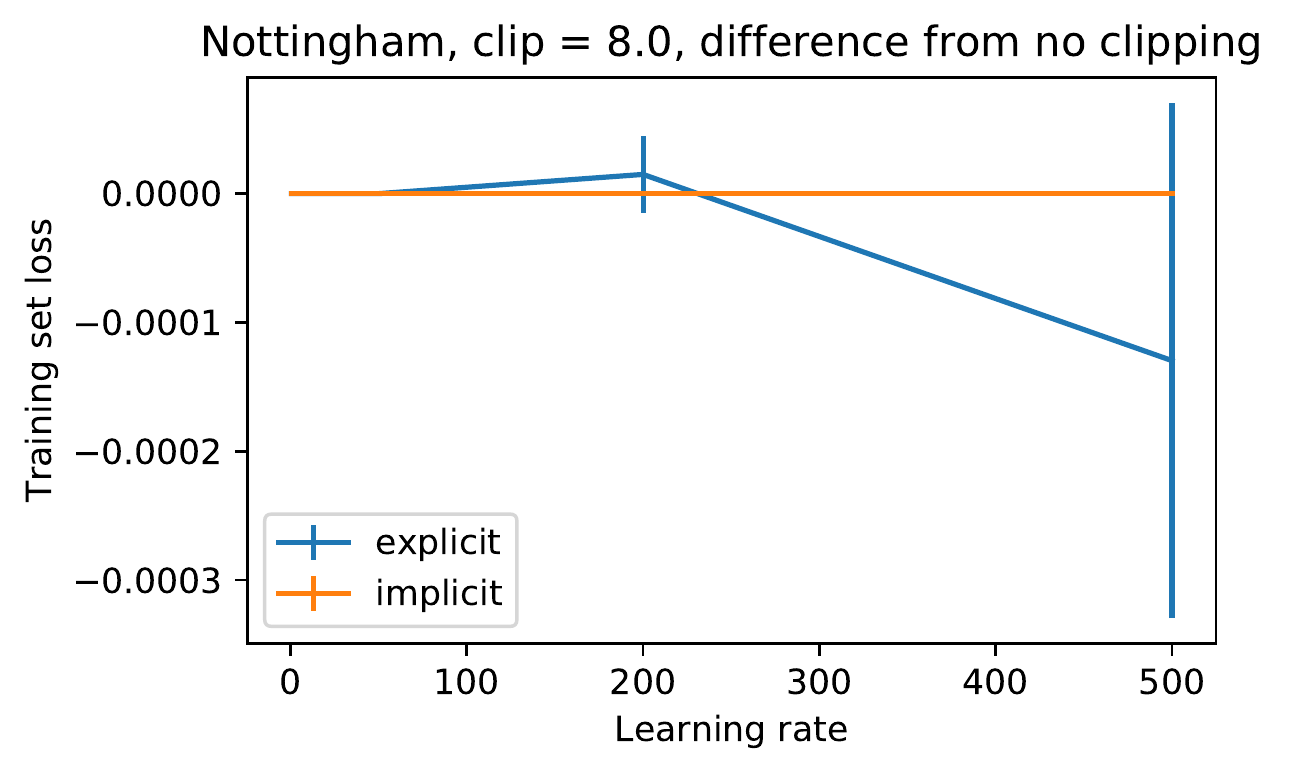}
\end{minipage}%
\hfill
\begin{minipage}{.49\textwidth}
  \centering
  \includegraphics[width=.79\linewidth]{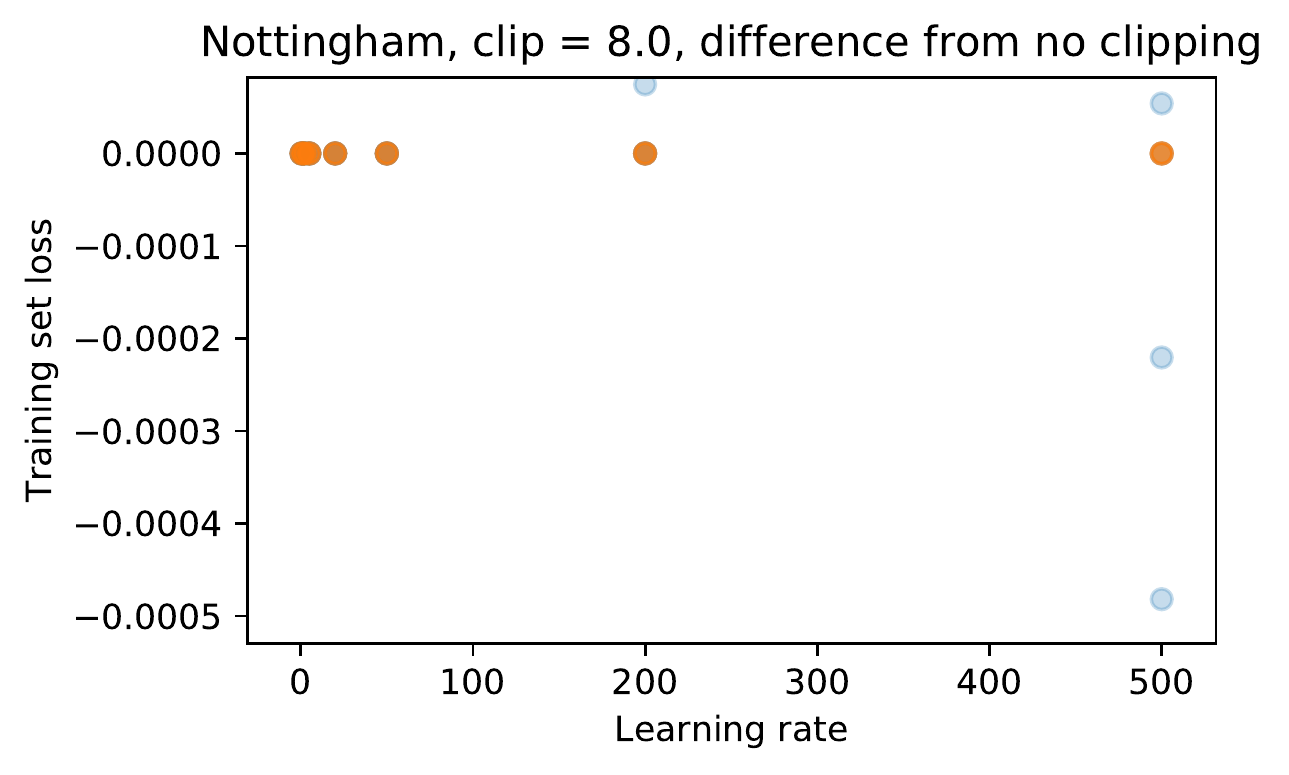}
\end{minipage}
\end{figure}

\vspace{0cm}

\begin{figure}[h]
\begin{minipage}{.49\textwidth}
  \centering
  \includegraphics[width=.79\linewidth]{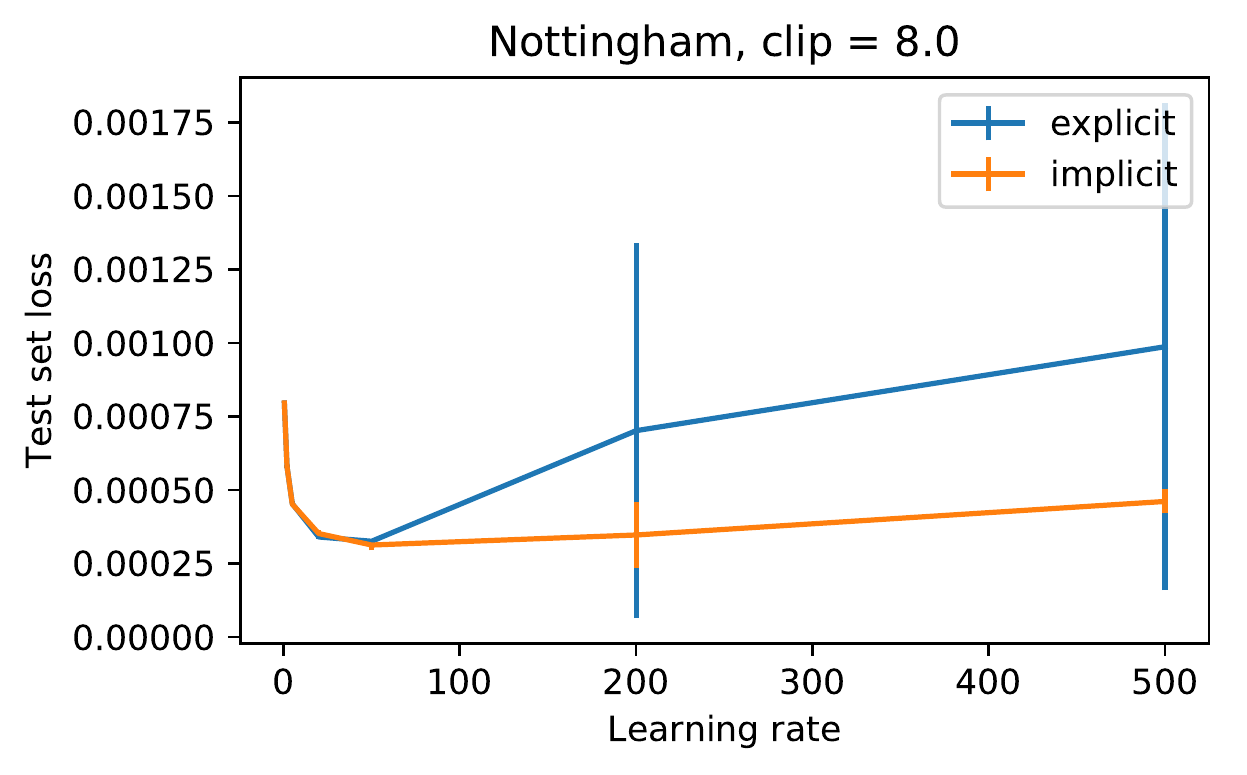}
\end{minipage}%
\hfill
\begin{minipage}{.49\textwidth}
  \centering
  \includegraphics[width=.79\linewidth]{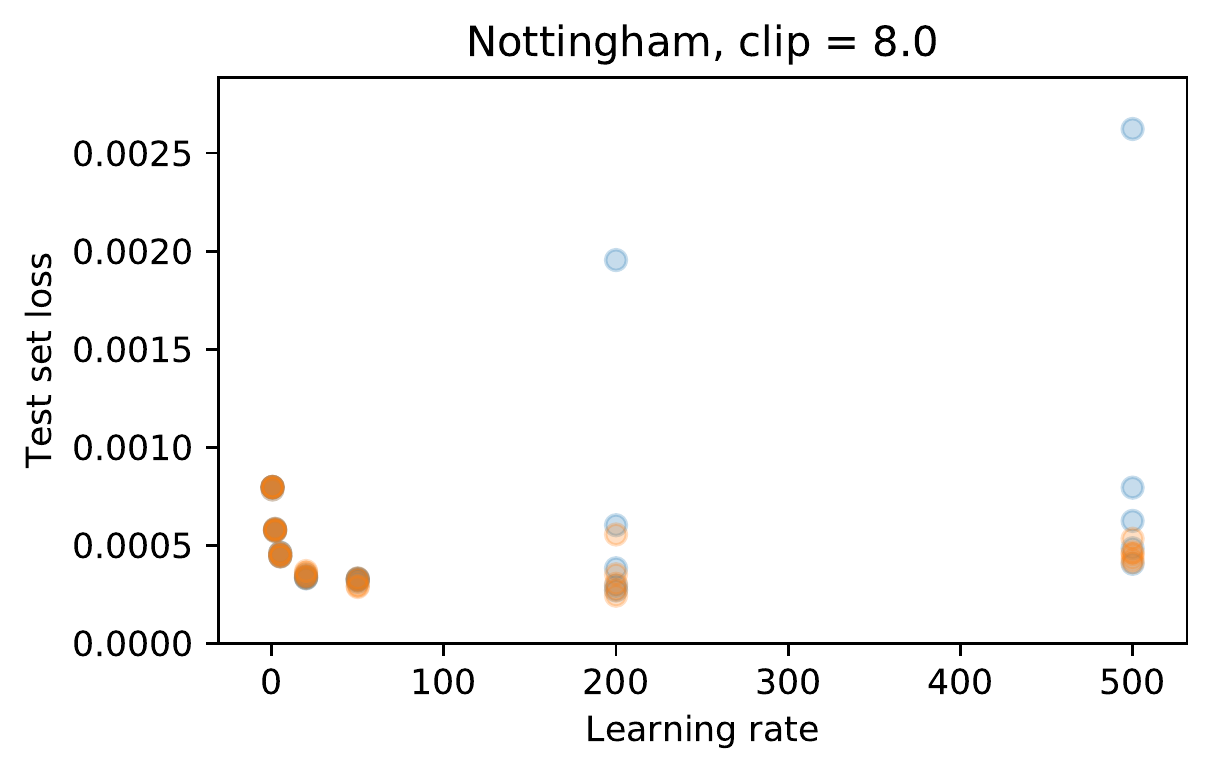}
\end{minipage}
\begin{minipage}{.49\textwidth}
  \centering
  \includegraphics[width=.79\linewidth]{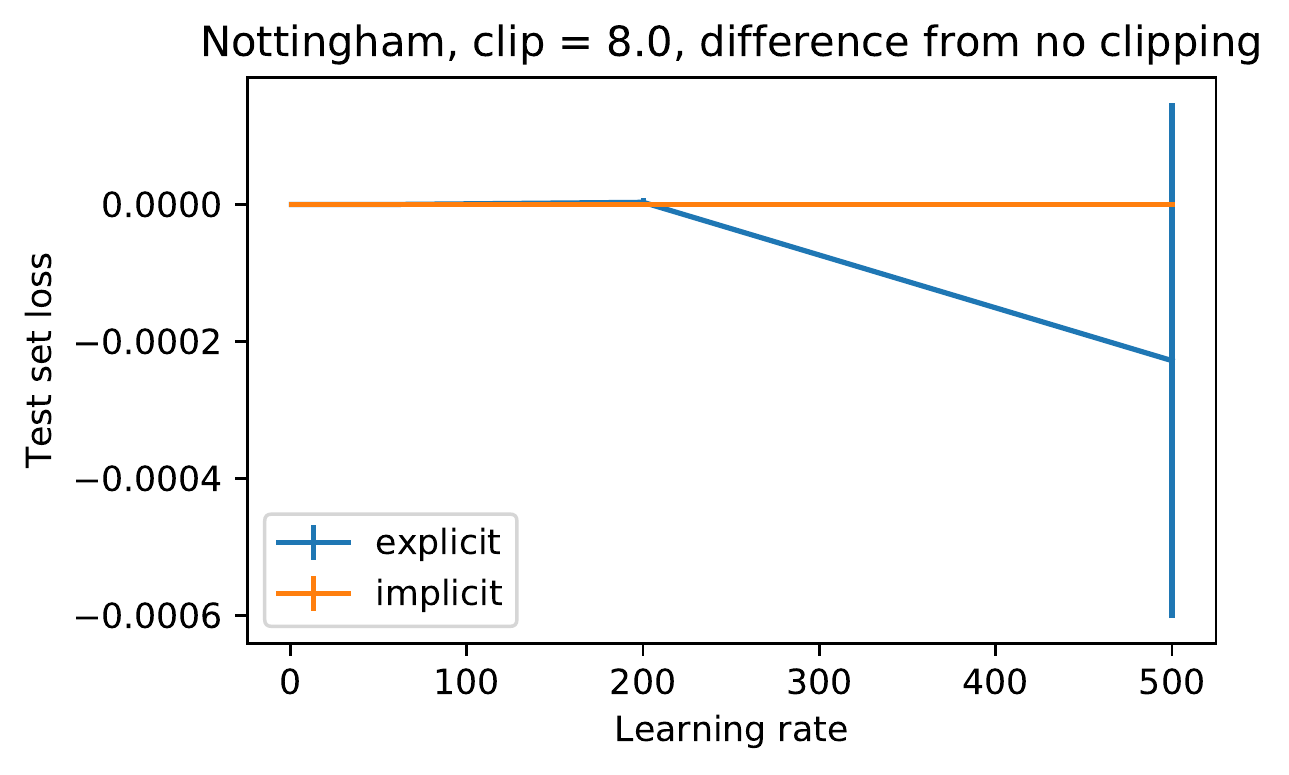}
\end{minipage}%
\hfill
\begin{minipage}{.49\textwidth}
  \centering
  \includegraphics[width=.79\linewidth]{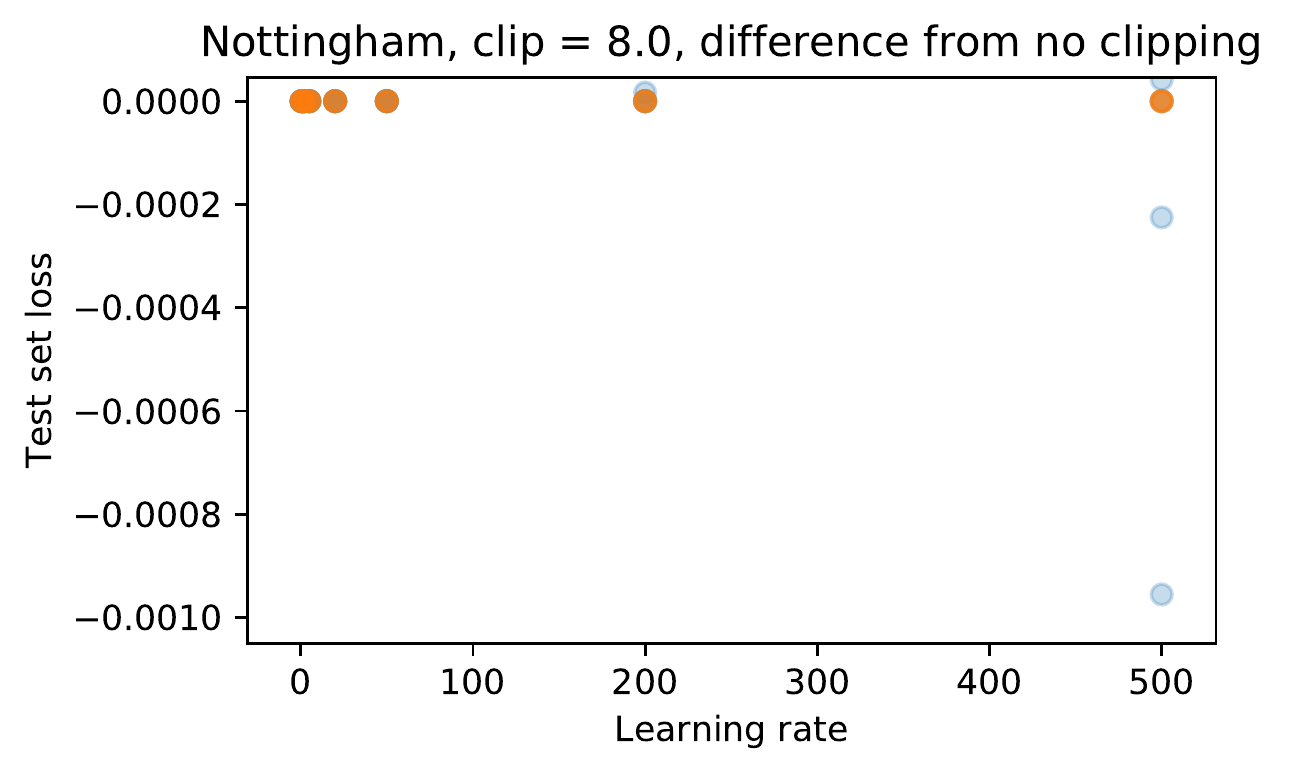}
\end{minipage}
\end{figure}

\clearpage
\subsection{Piano-midi.de}

\begin{figure}[h]
\centering
\begin{minipage}{.49\textwidth}
  \centering
  \includegraphics[width=.79\linewidth]{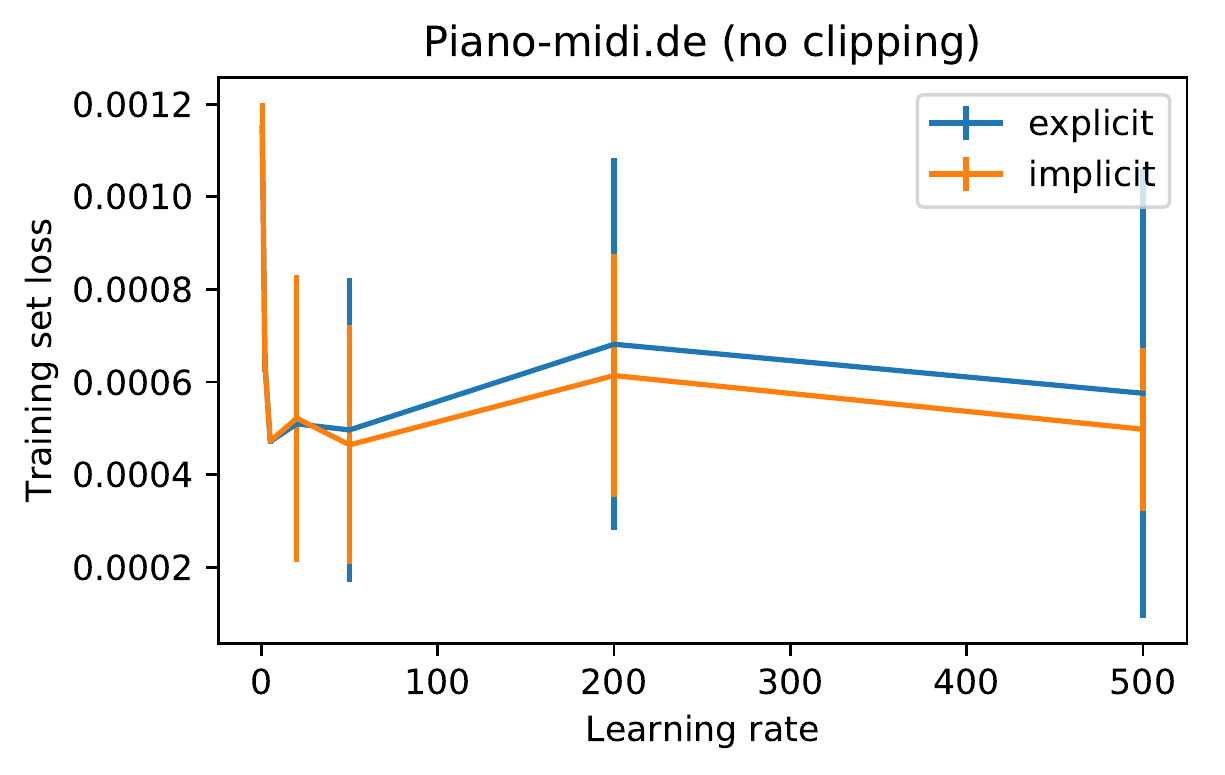}
\end{minipage}%
\hfill
\begin{minipage}{.49\textwidth}
  \centering
  \includegraphics[width=.79\linewidth]{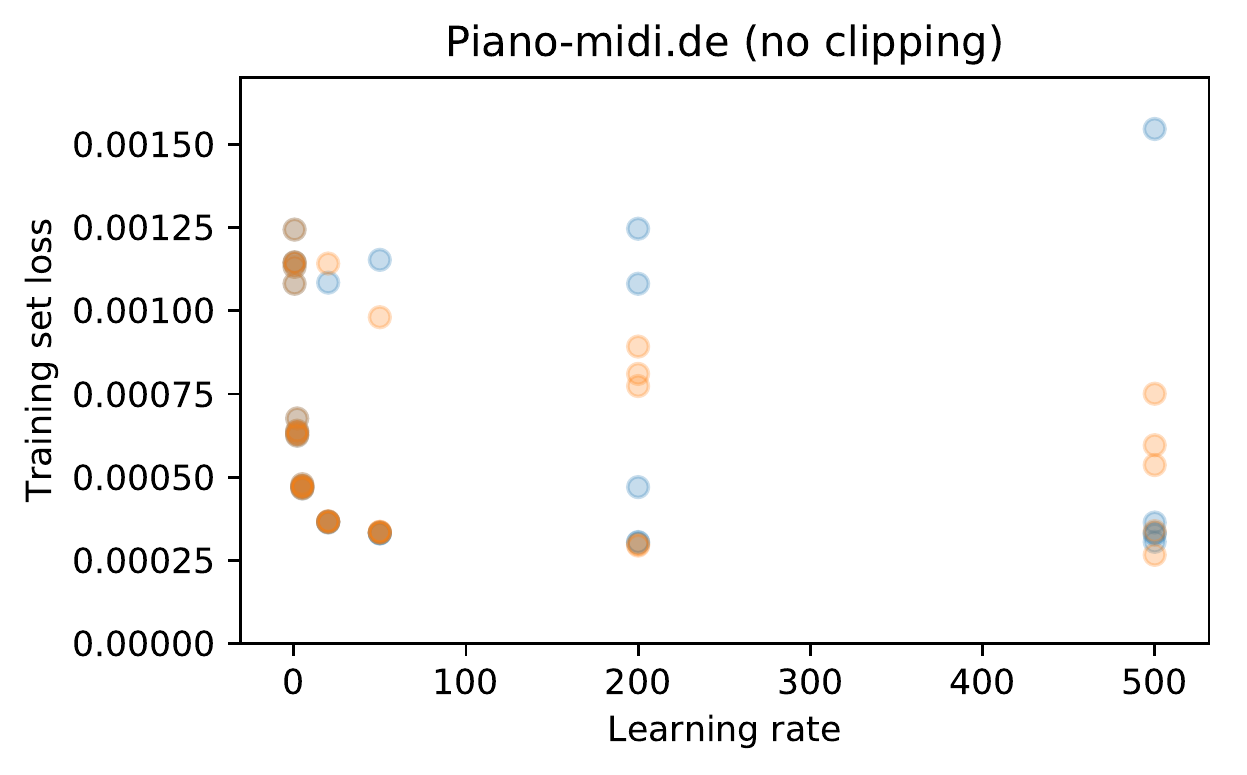}
\end{minipage}
\begin{minipage}{.49\textwidth}
  \centering
  \includegraphics[width=.79\linewidth]{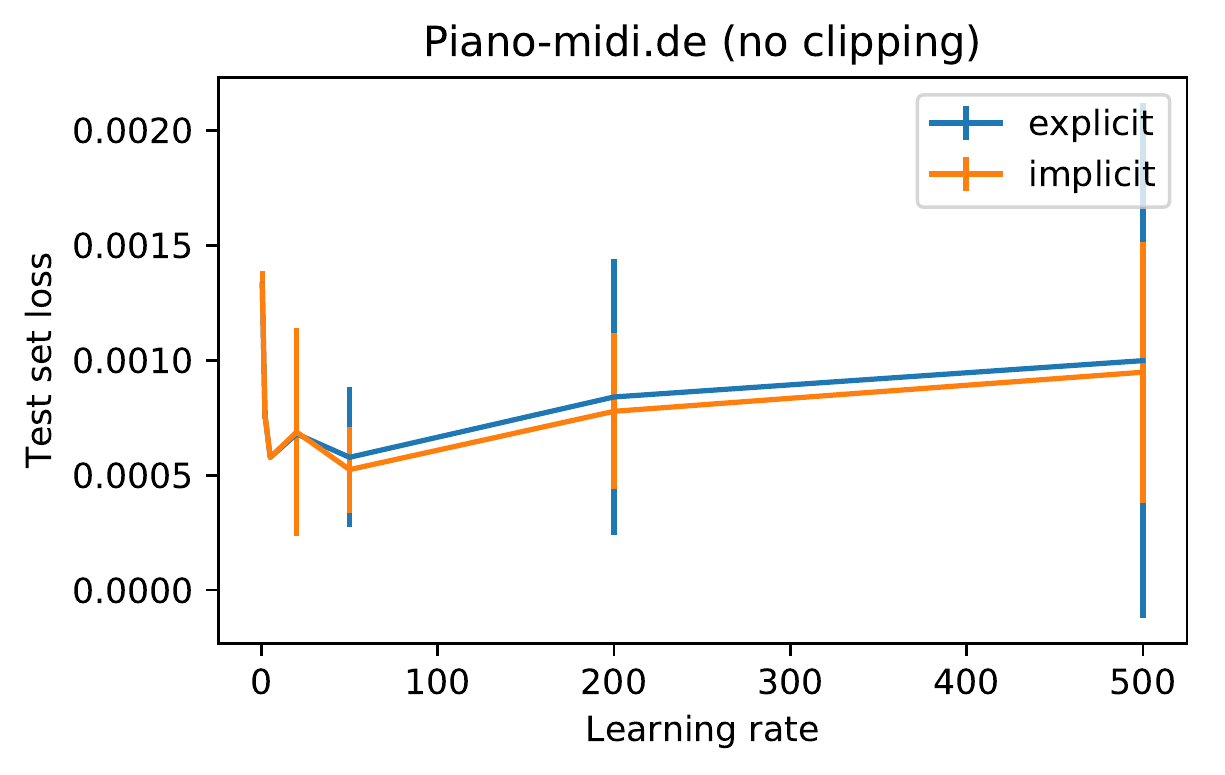}
\end{minipage}%
\hfill
\begin{minipage}{.49\textwidth}
  \centering
  \includegraphics[width=.79\linewidth]{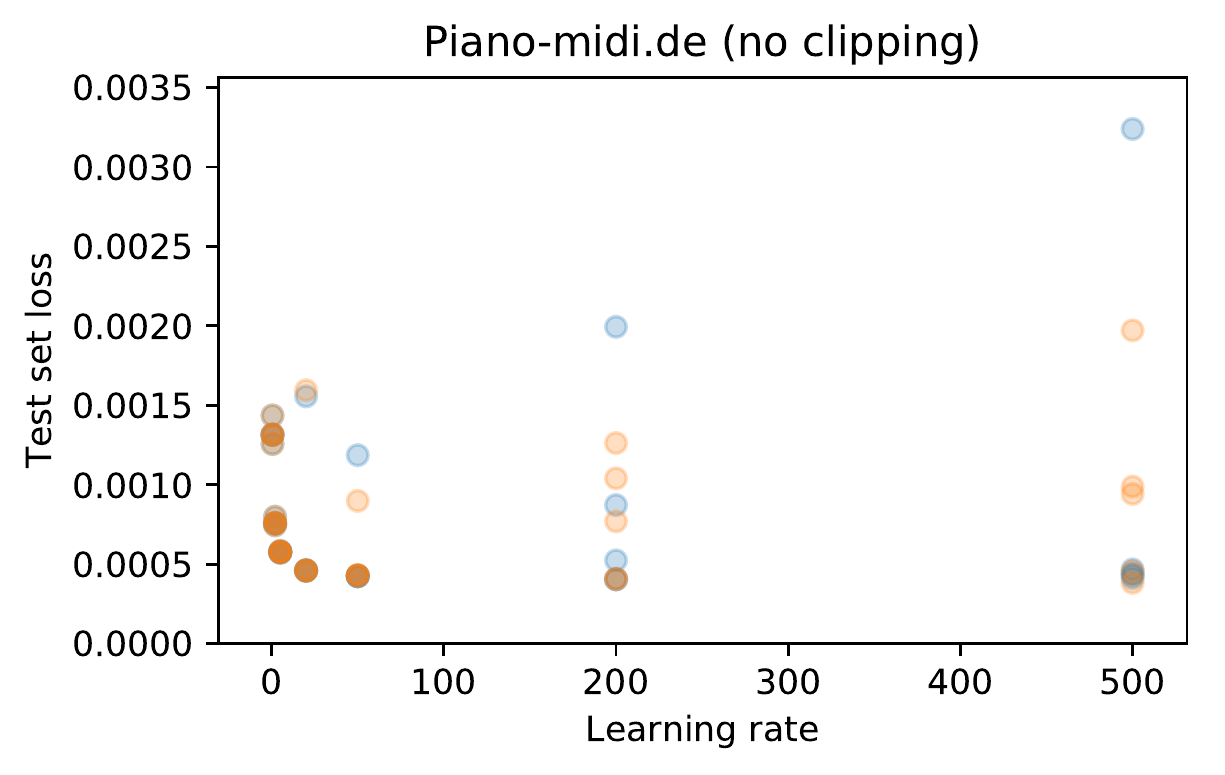}
\end{minipage}
\end{figure}

\begin{figure}[h]
\centering
\begin{minipage}{.49\textwidth}
  \centering
  \includegraphics[width=.79\linewidth]{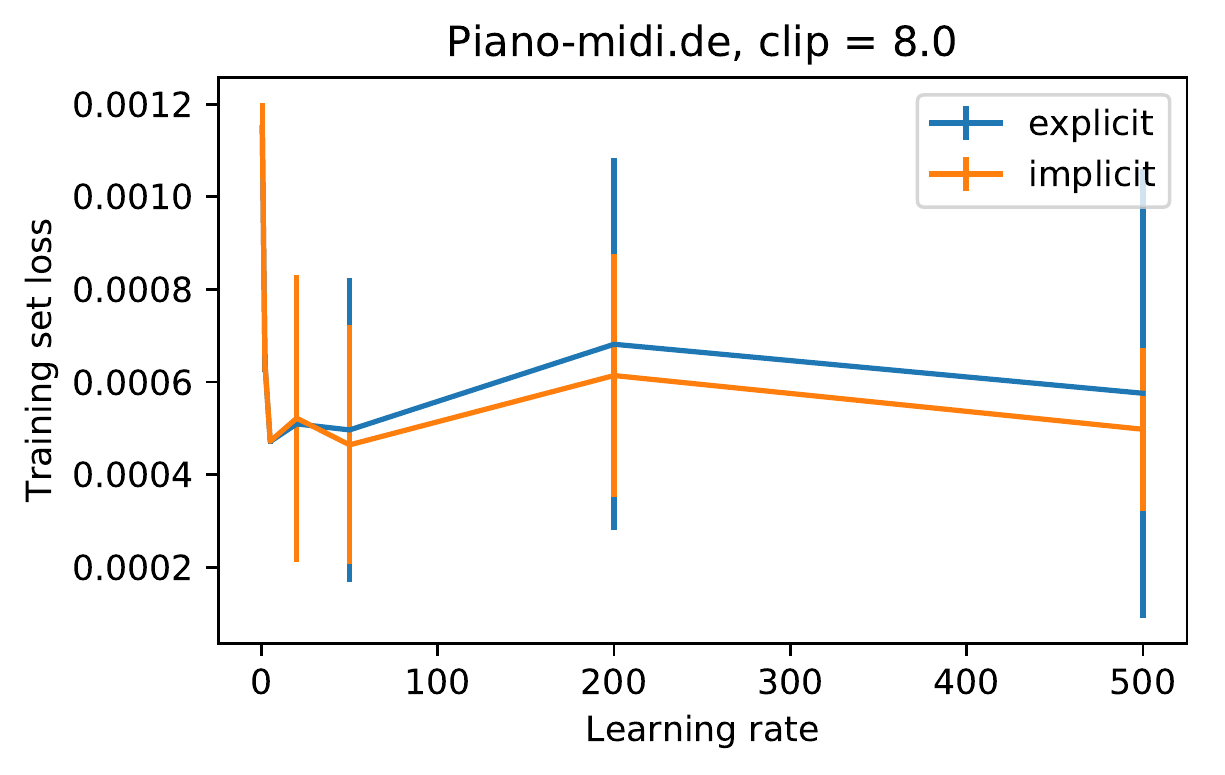}
\end{minipage}%
\hfill
\begin{minipage}{.49\textwidth}
  \centering
  \includegraphics[width=.79\linewidth]{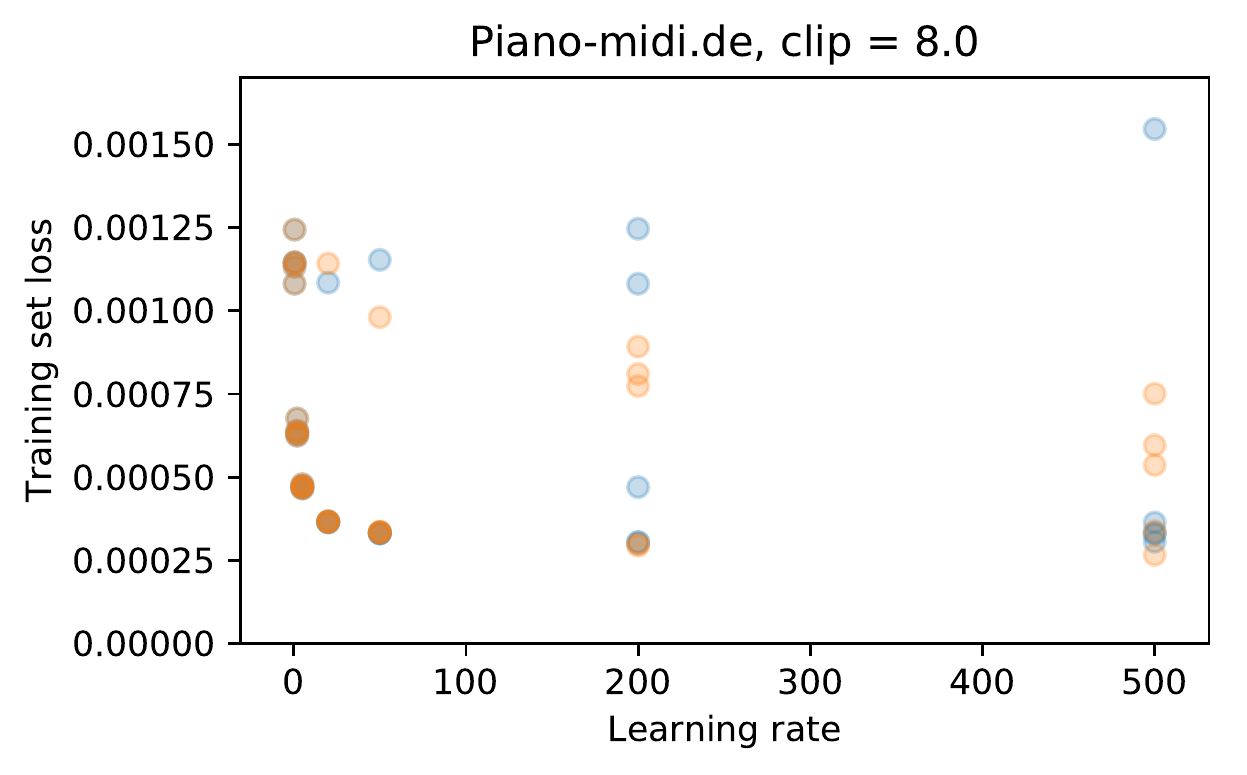}
\end{minipage}
\begin{minipage}{.49\textwidth}
  \centering
  \includegraphics[width=.79\linewidth]{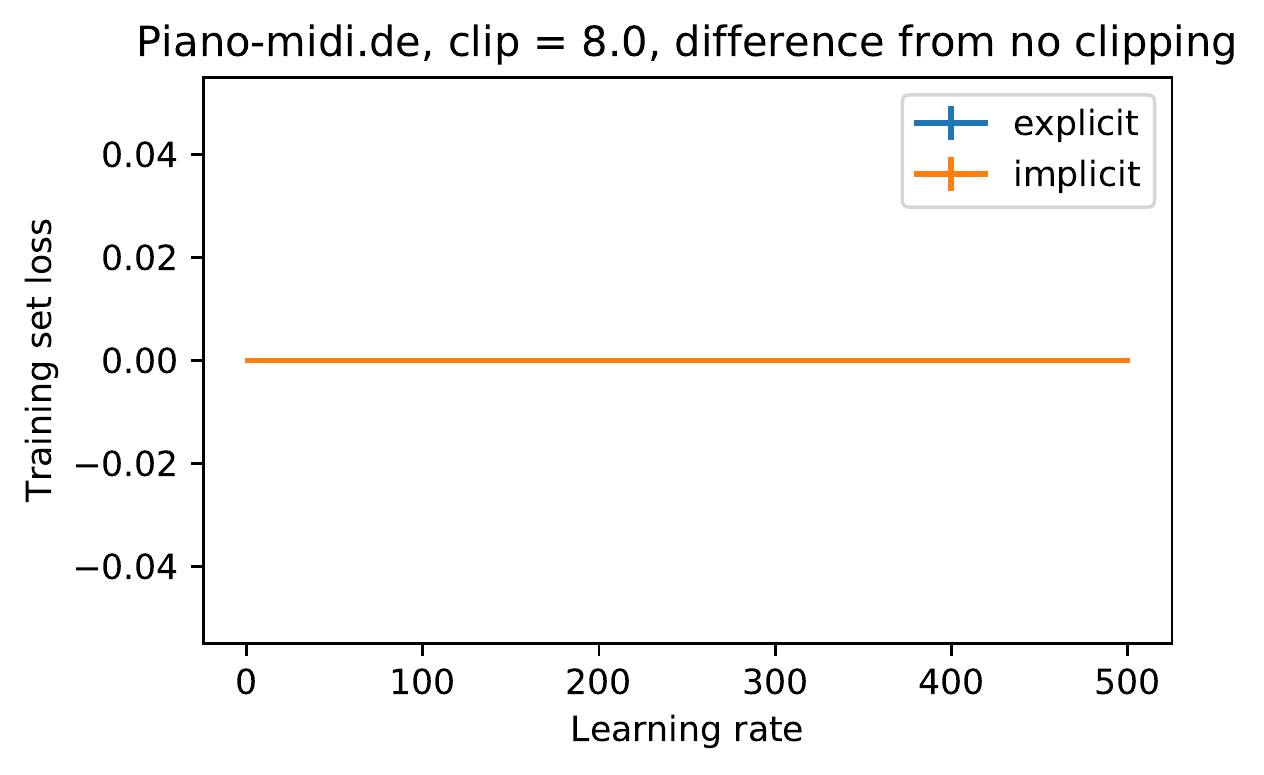}
\end{minipage}%
\hfill
\begin{minipage}{.49\textwidth}
  \centering
  \includegraphics[width=.79\linewidth]{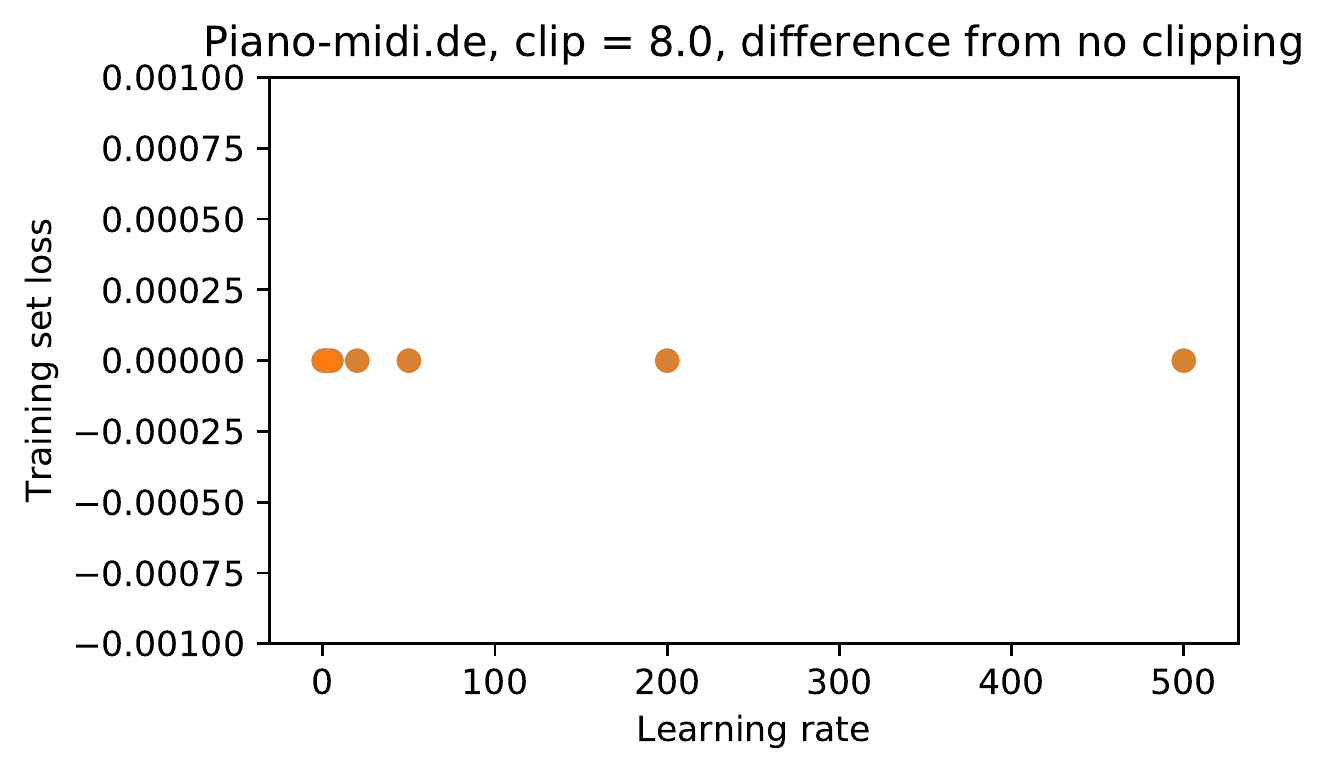}
\end{minipage}
\end{figure}

\vspace{0cm}

\begin{figure}[h]
\begin{minipage}{.49\textwidth}
  \centering
  \includegraphics[width=.79\linewidth]{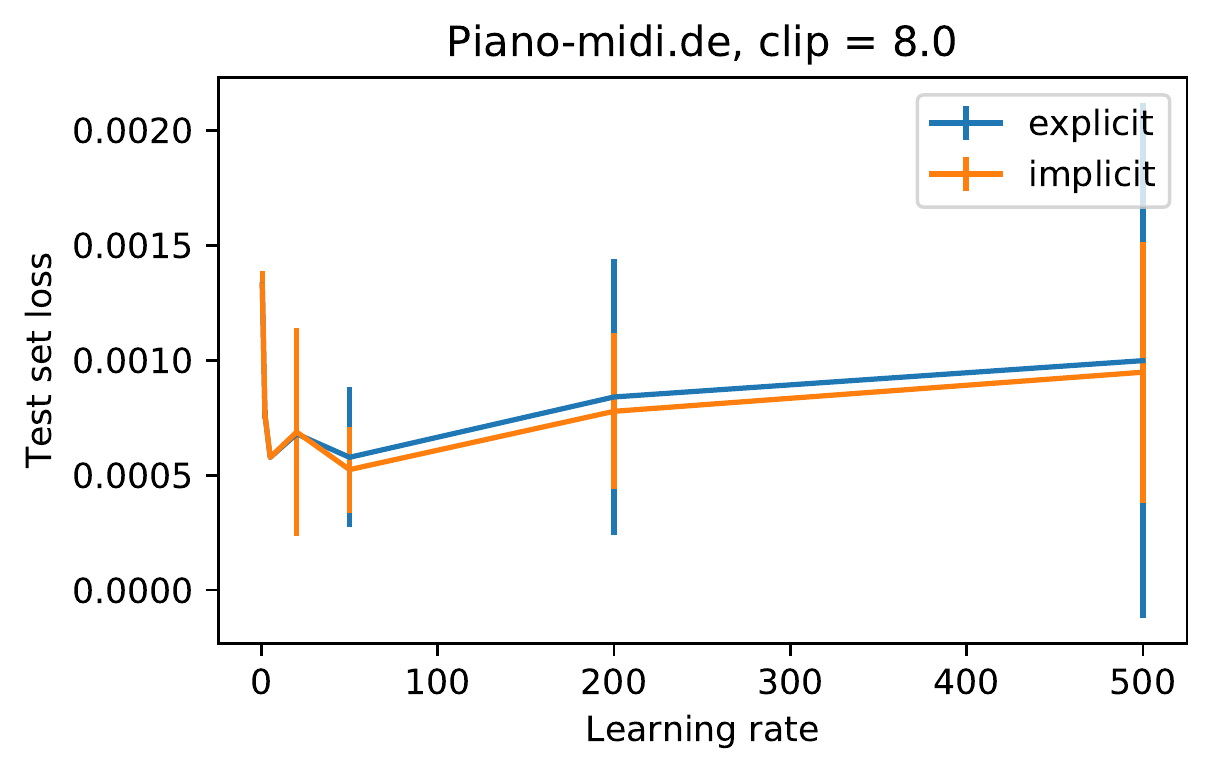}
\end{minipage}%
\hfill
\begin{minipage}{.49\textwidth}
  \centering
  \includegraphics[width=.79\linewidth]{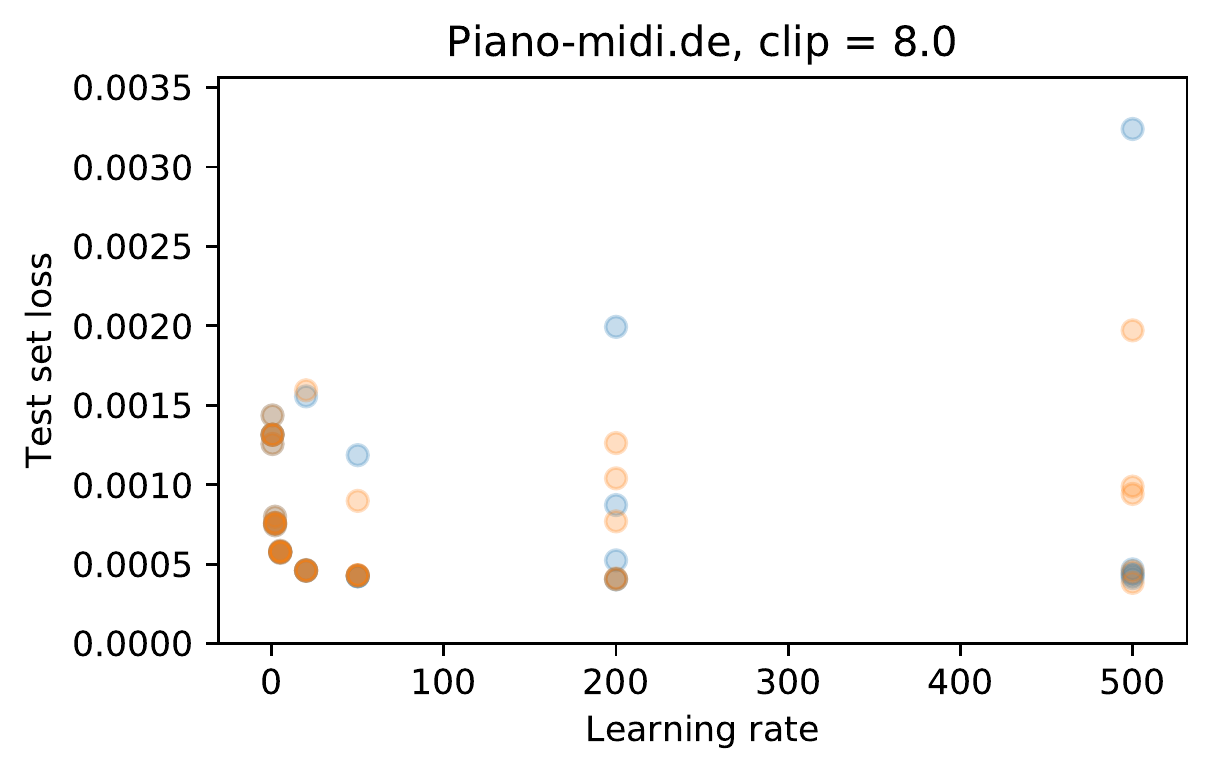}
\end{minipage}
\begin{minipage}{.49\textwidth}
  \centering
  \includegraphics[width=.79\linewidth]{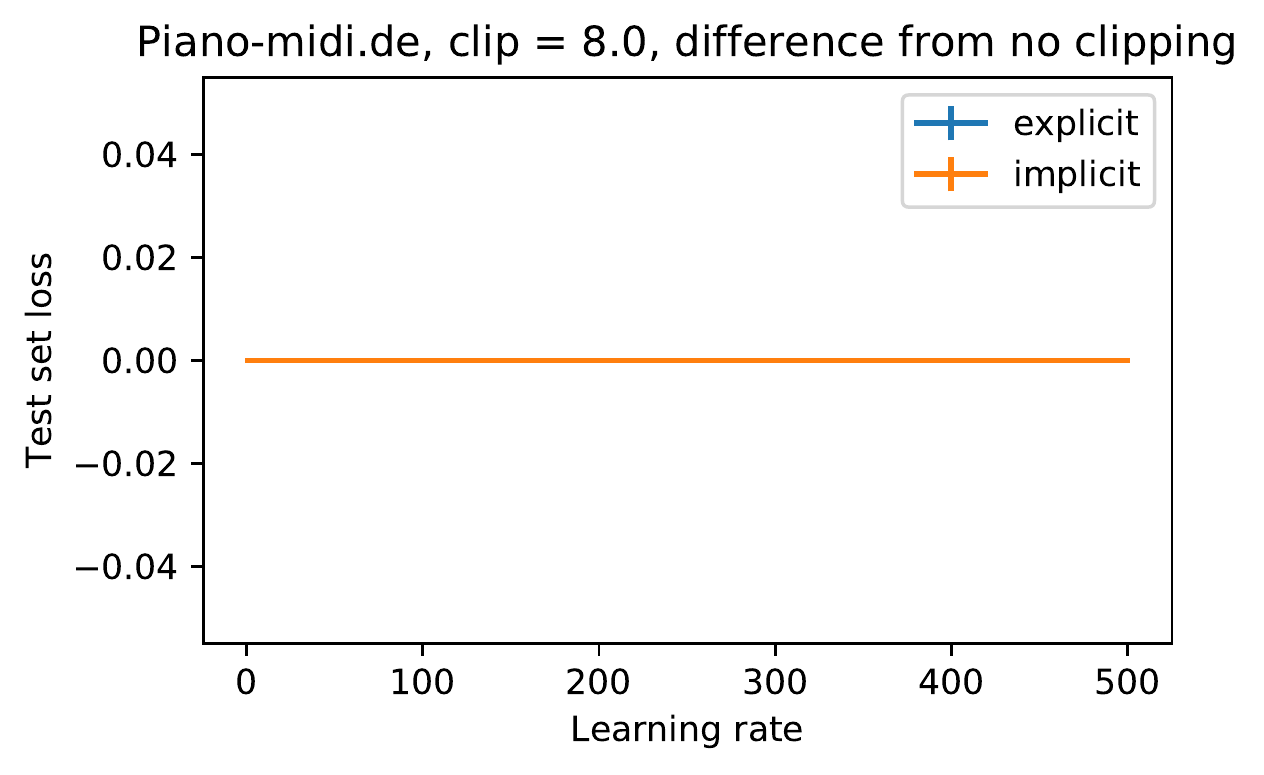}
\end{minipage}%
\hfill
\begin{minipage}{.49\textwidth}
  \centering
  \includegraphics[width=.79\linewidth]{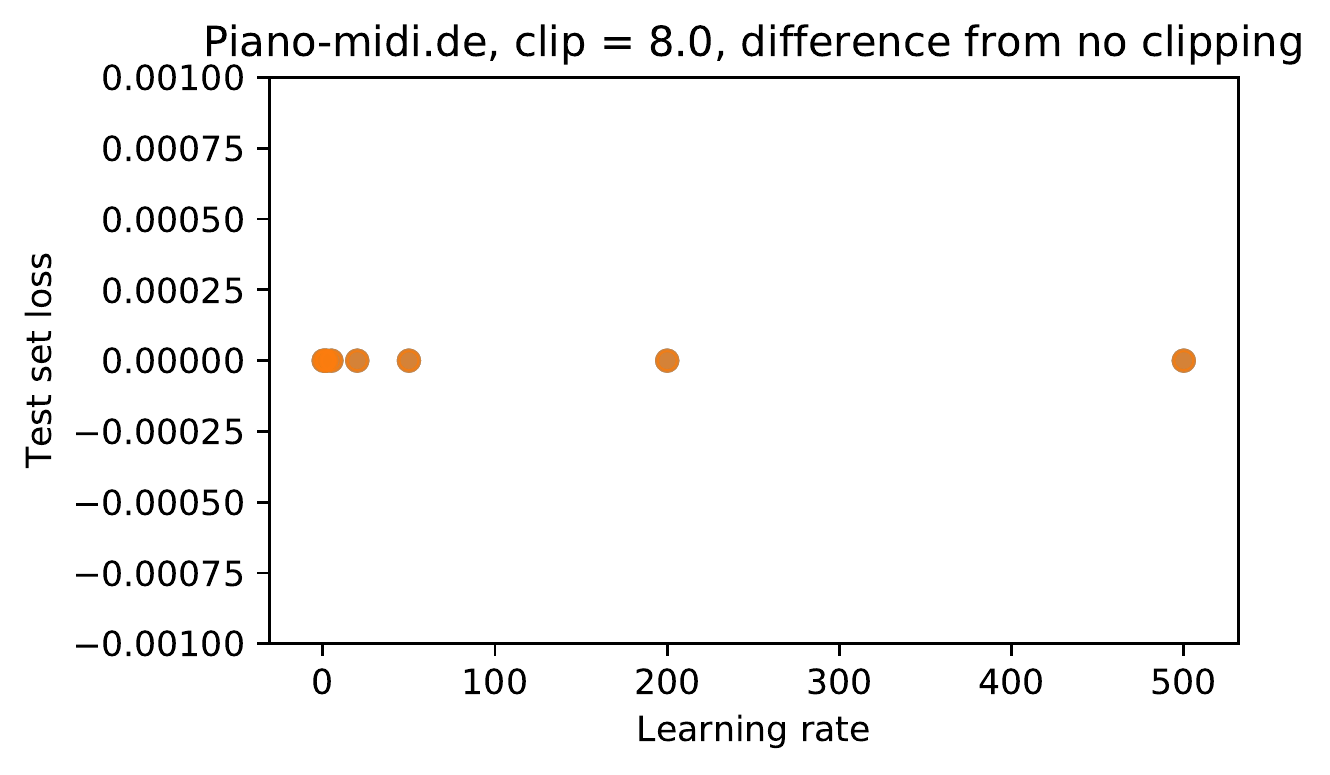}
\end{minipage}
\end{figure}

\begin{figure}[h]
\centering
\begin{minipage}{.49\textwidth}
  \centering
  \includegraphics[width=.79\linewidth]{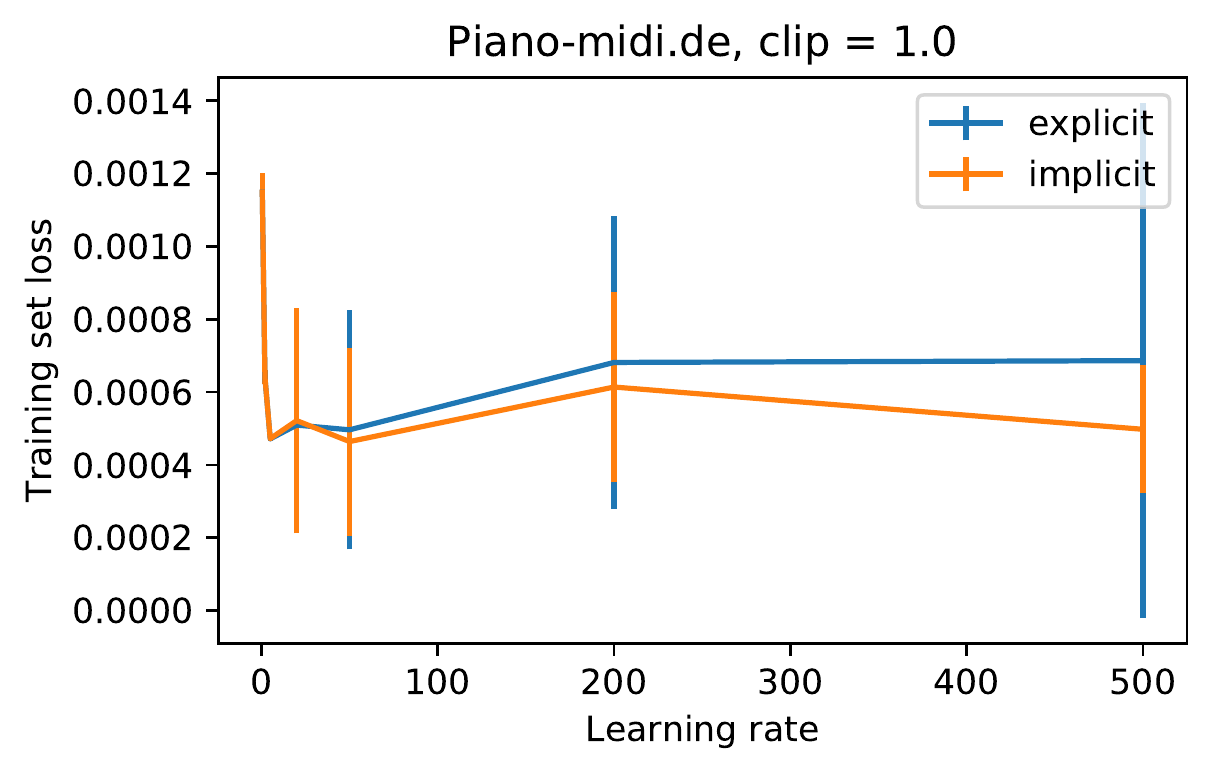}
\end{minipage}%
\hfill
\begin{minipage}{.49\textwidth}
  \centering
  \includegraphics[width=.79\linewidth]{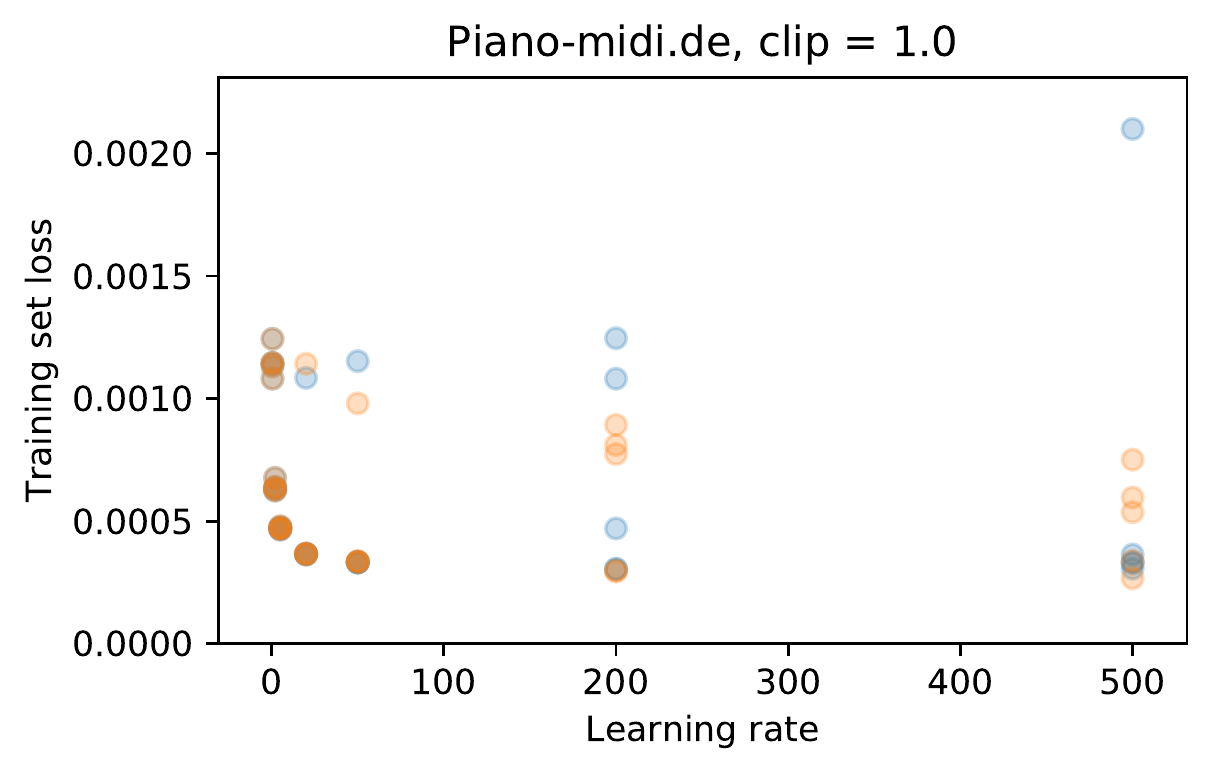}
\end{minipage}
\begin{minipage}{.49\textwidth}
  \centering
  \includegraphics[width=.79\linewidth]{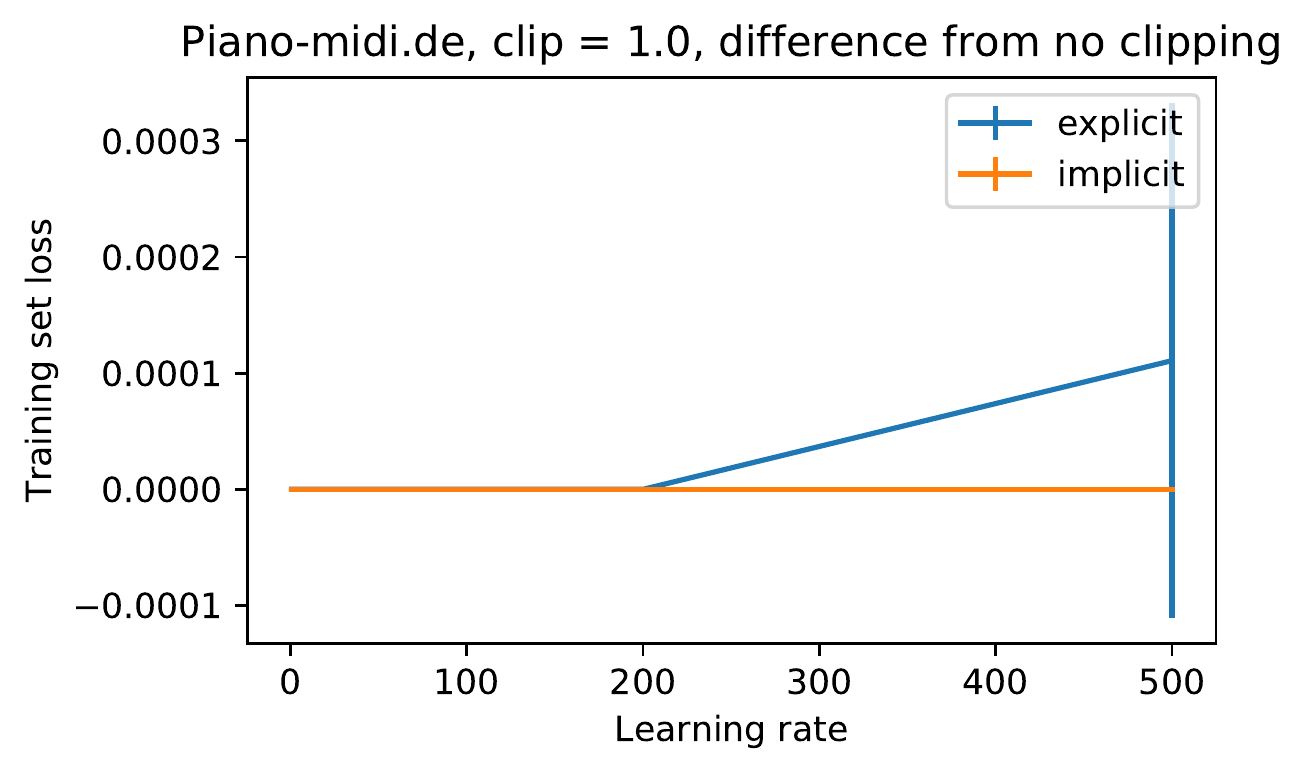}
\end{minipage}%
\hfill
\begin{minipage}{.49\textwidth}
  \centering
  \includegraphics[width=.79\linewidth]{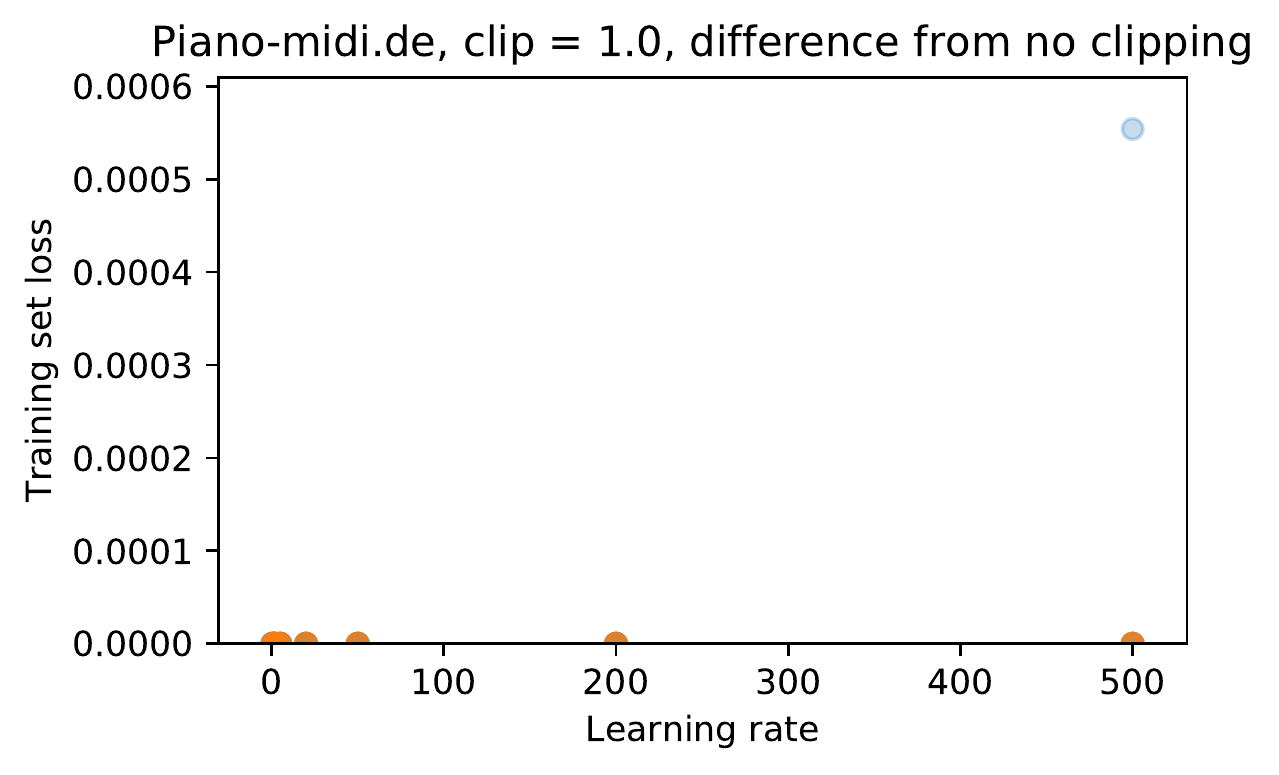}
\end{minipage}
\end{figure}

\vspace{0cm}

\begin{figure}[h]
\begin{minipage}{.49\textwidth}
  \centering
  \includegraphics[width=.79\linewidth]{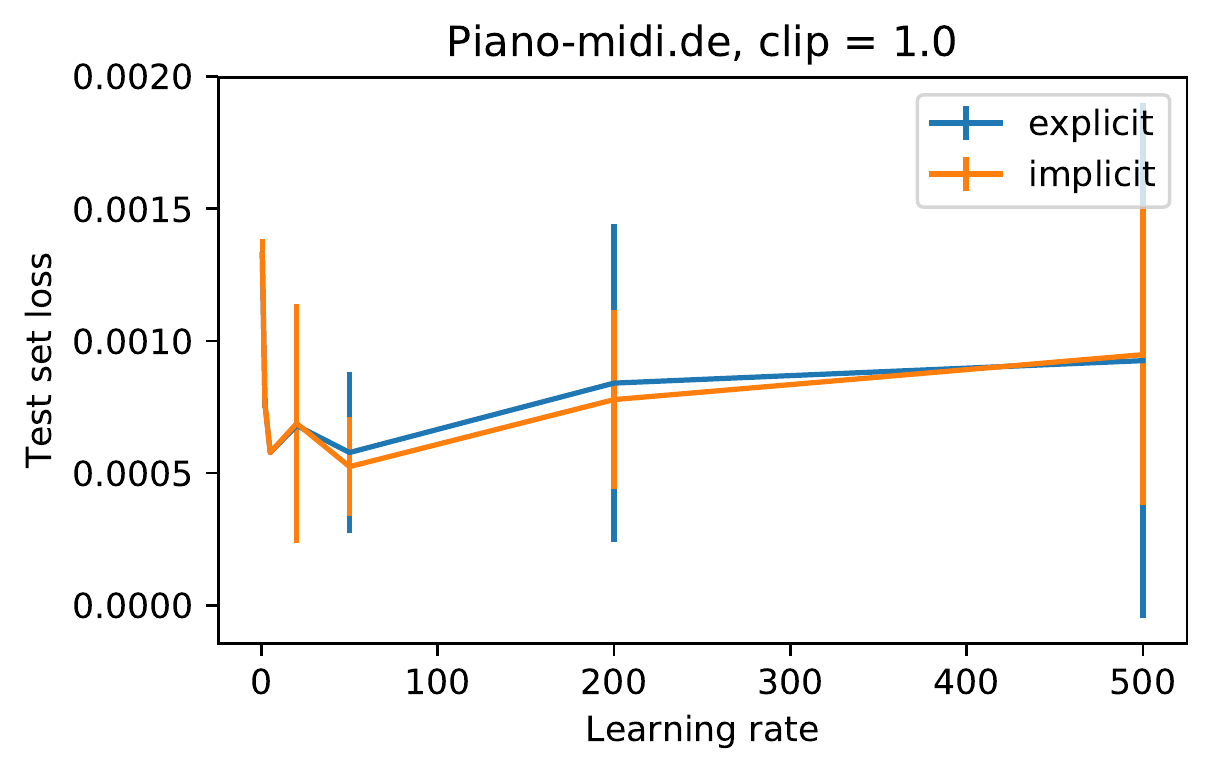}
\end{minipage}%
\hfill
\begin{minipage}{.49\textwidth}
  \centering
  \includegraphics[width=.79\linewidth]{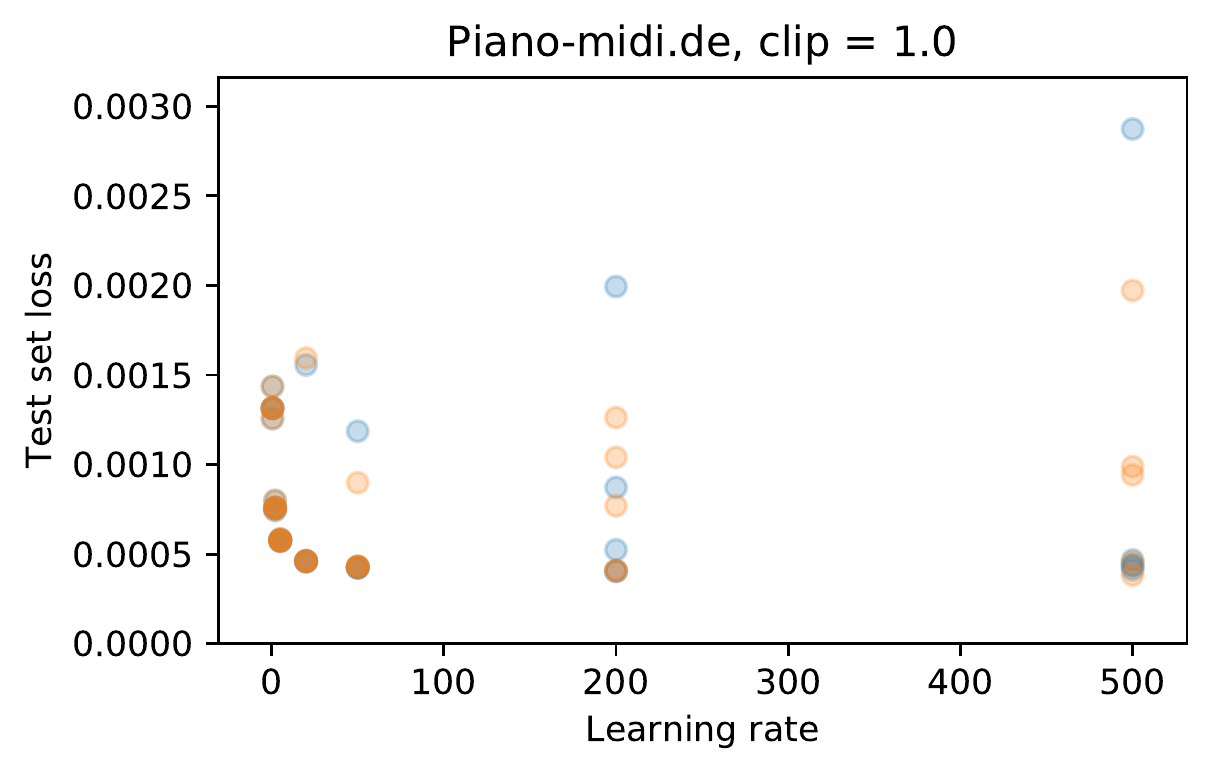}
\end{minipage}
\begin{minipage}{.49\textwidth}
  \centering
  \includegraphics[width=.79\linewidth]{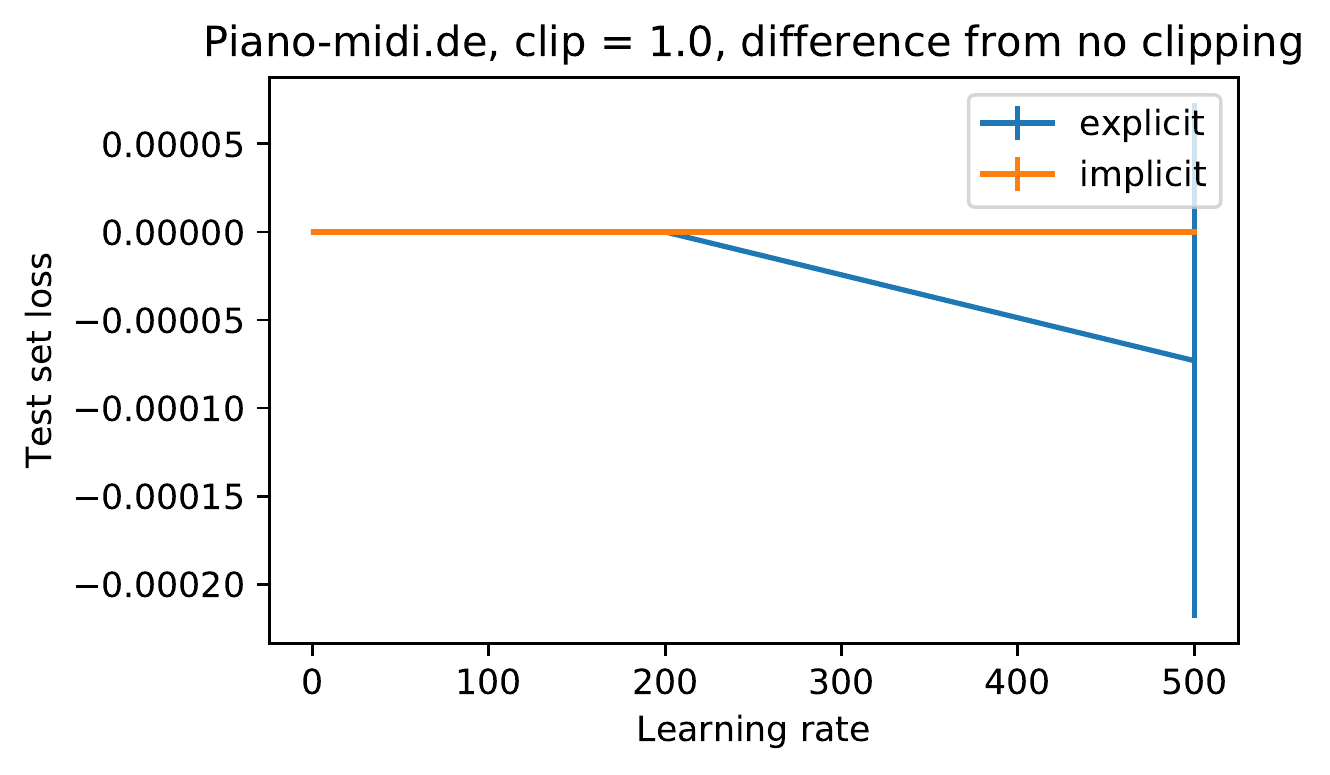}
\end{minipage}%
\hfill
\begin{minipage}{.49\textwidth}
  \centering
  \includegraphics[width=.79\linewidth]{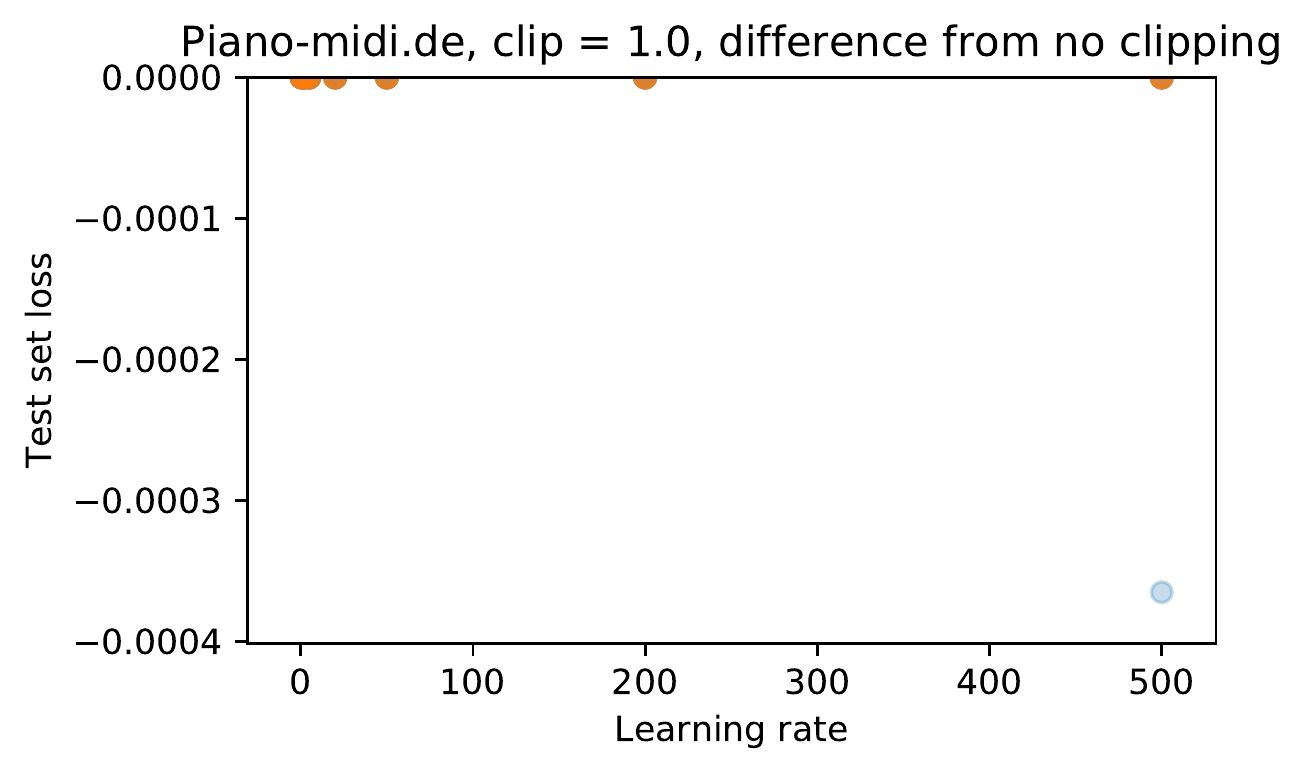}
\end{minipage}
\end{figure}

\begin{figure}[h]
\centering
\begin{minipage}{.49\textwidth}
  \centering
  \includegraphics[width=.79\linewidth]{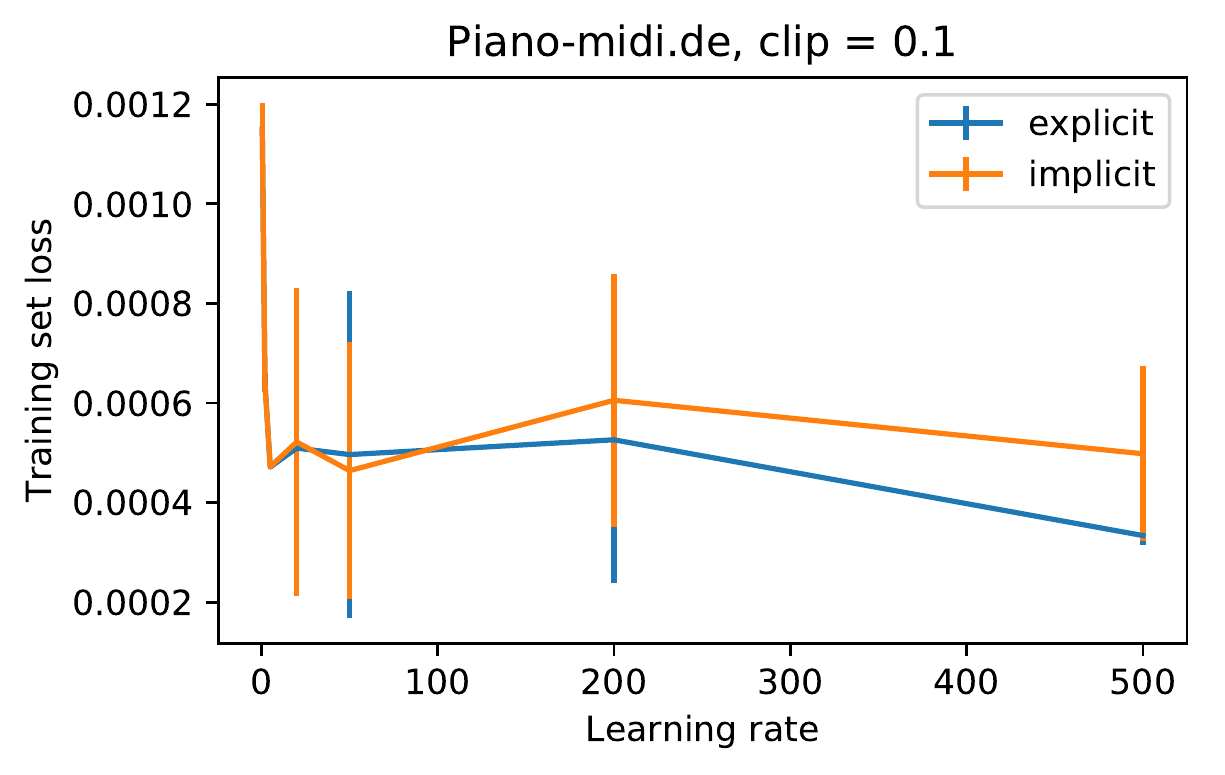}
\end{minipage}%
\hfill
\begin{minipage}{.49\textwidth}
  \centering
  \includegraphics[width=.79\linewidth]{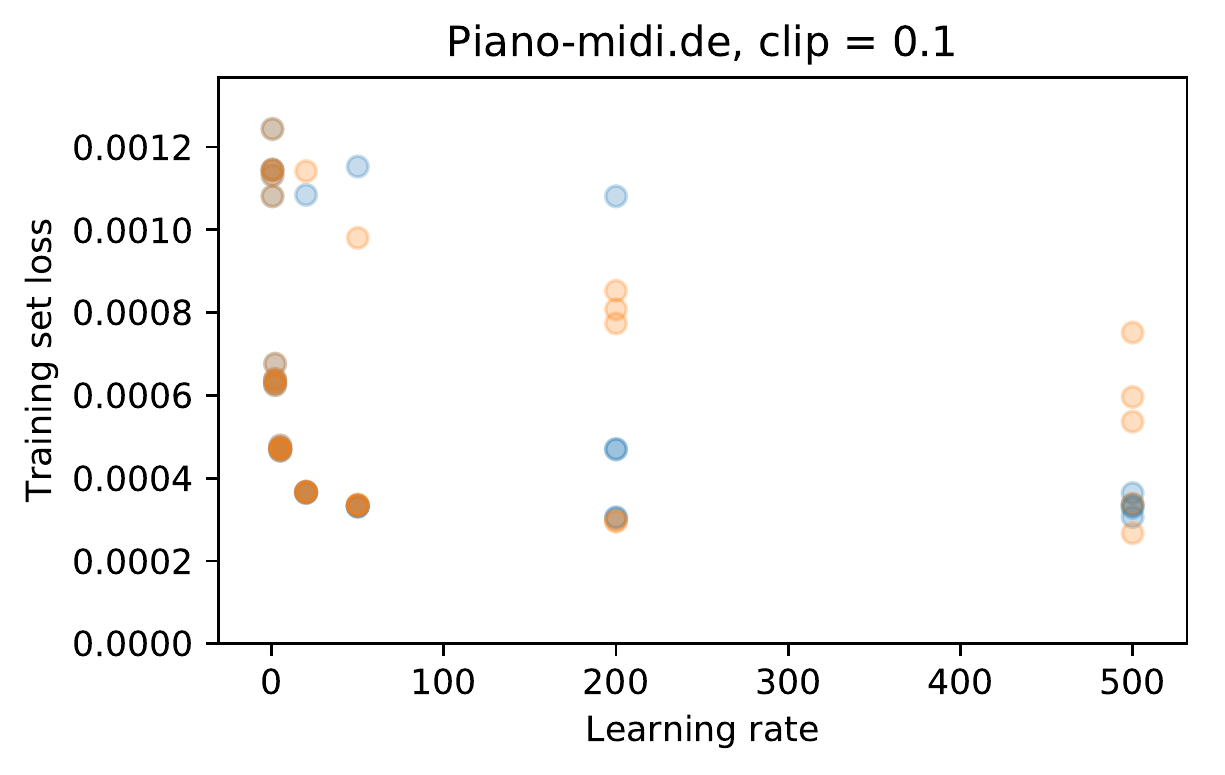}
\end{minipage}
\begin{minipage}{.49\textwidth}
  \centering
  \includegraphics[width=.79\linewidth]{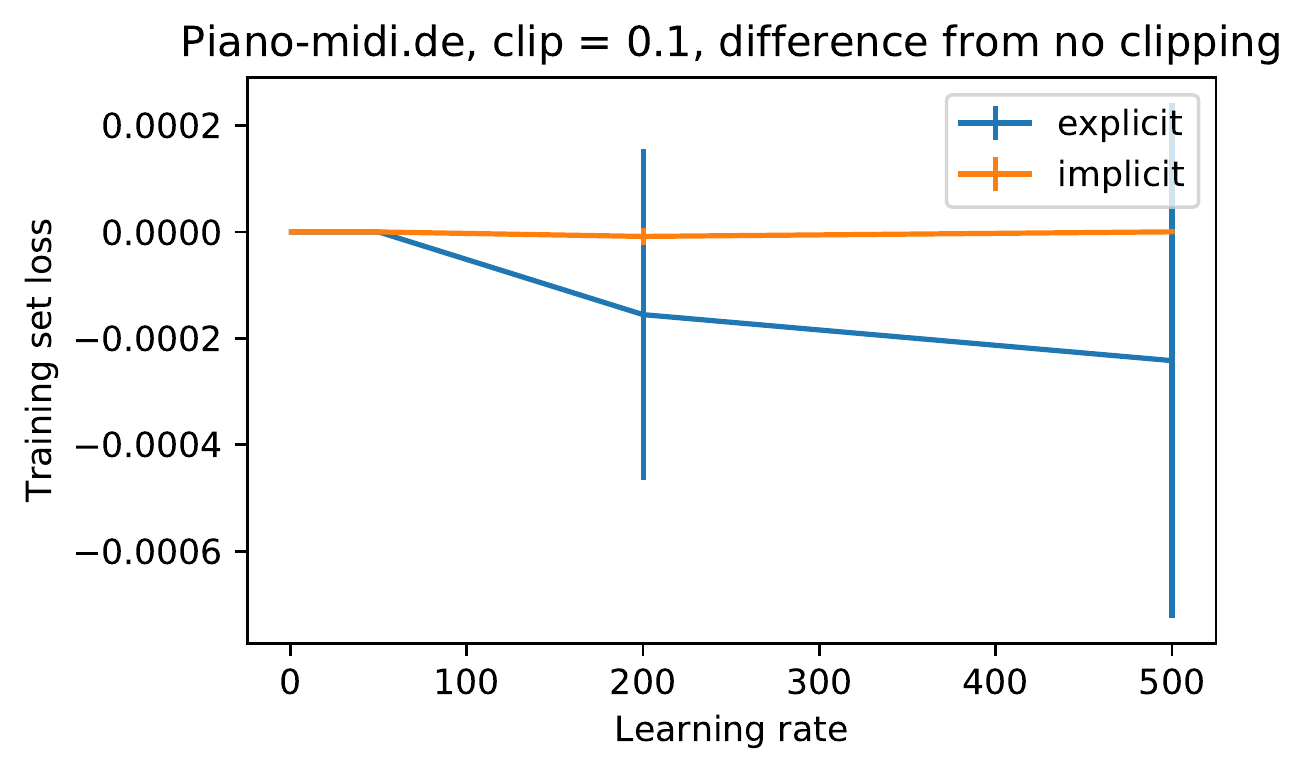}
\end{minipage}%
\hfill
\begin{minipage}{.49\textwidth}
  \centering
  \includegraphics[width=.79\linewidth]{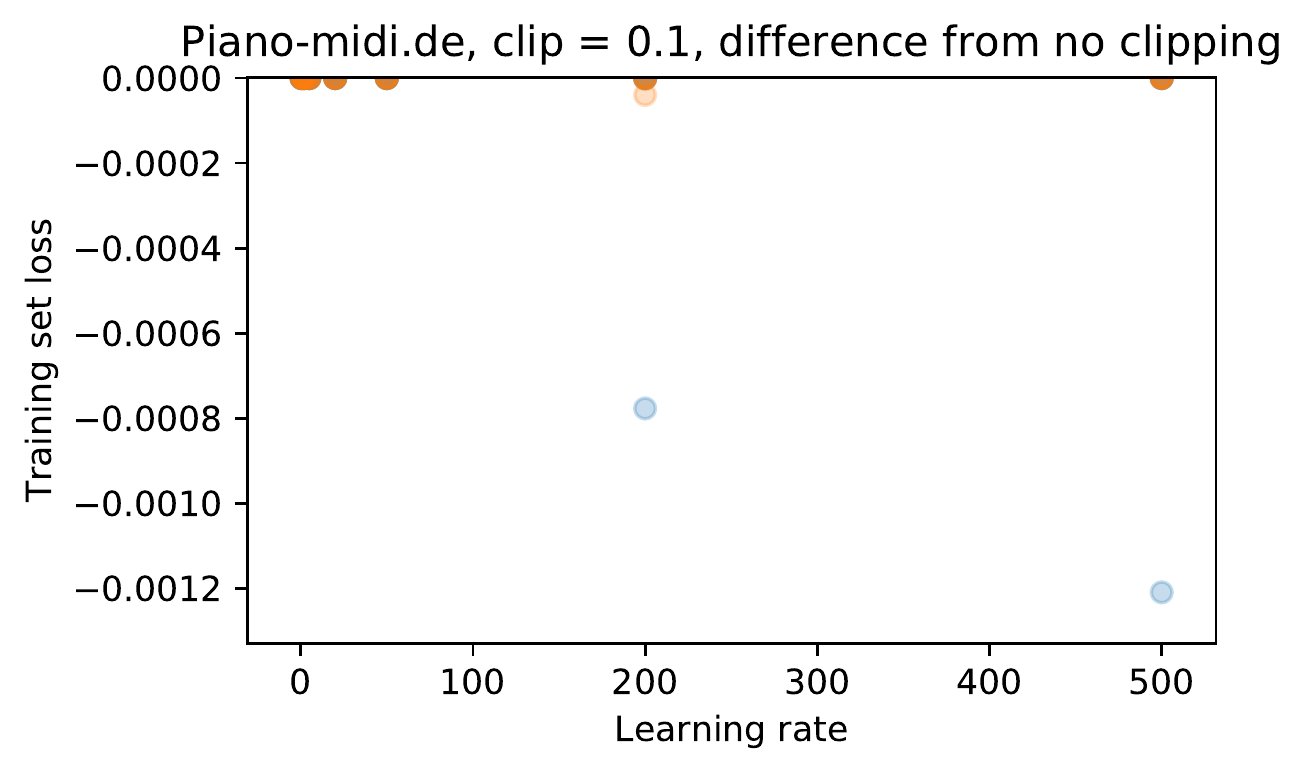}
\end{minipage}
\end{figure}

\vspace{0cm}

\begin{figure}[h]
\begin{minipage}{.49\textwidth}
  \centering
  \includegraphics[width=.79\linewidth]{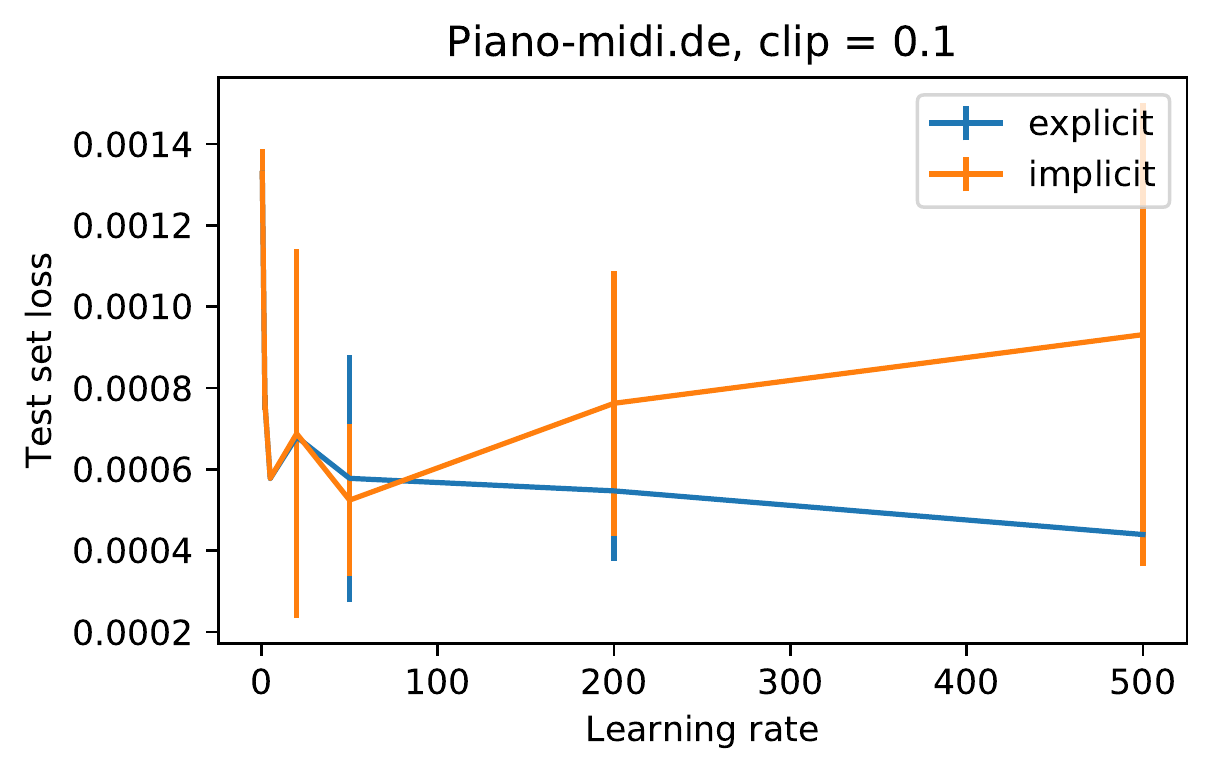}
\end{minipage}%
\hfill
\begin{minipage}{.49\textwidth}
  \centering
  \includegraphics[width=.79\linewidth]{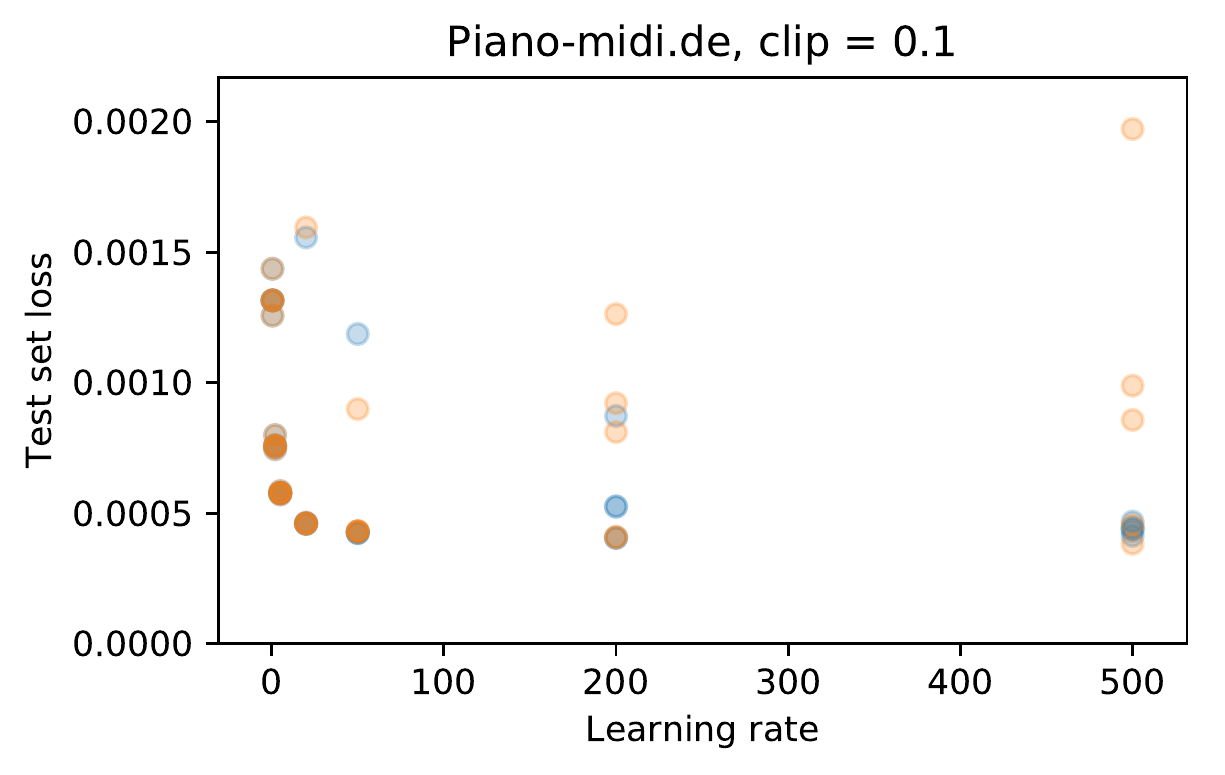}
\end{minipage}
\begin{minipage}{.49\textwidth}
  \centering
  \includegraphics[width=.79\linewidth]{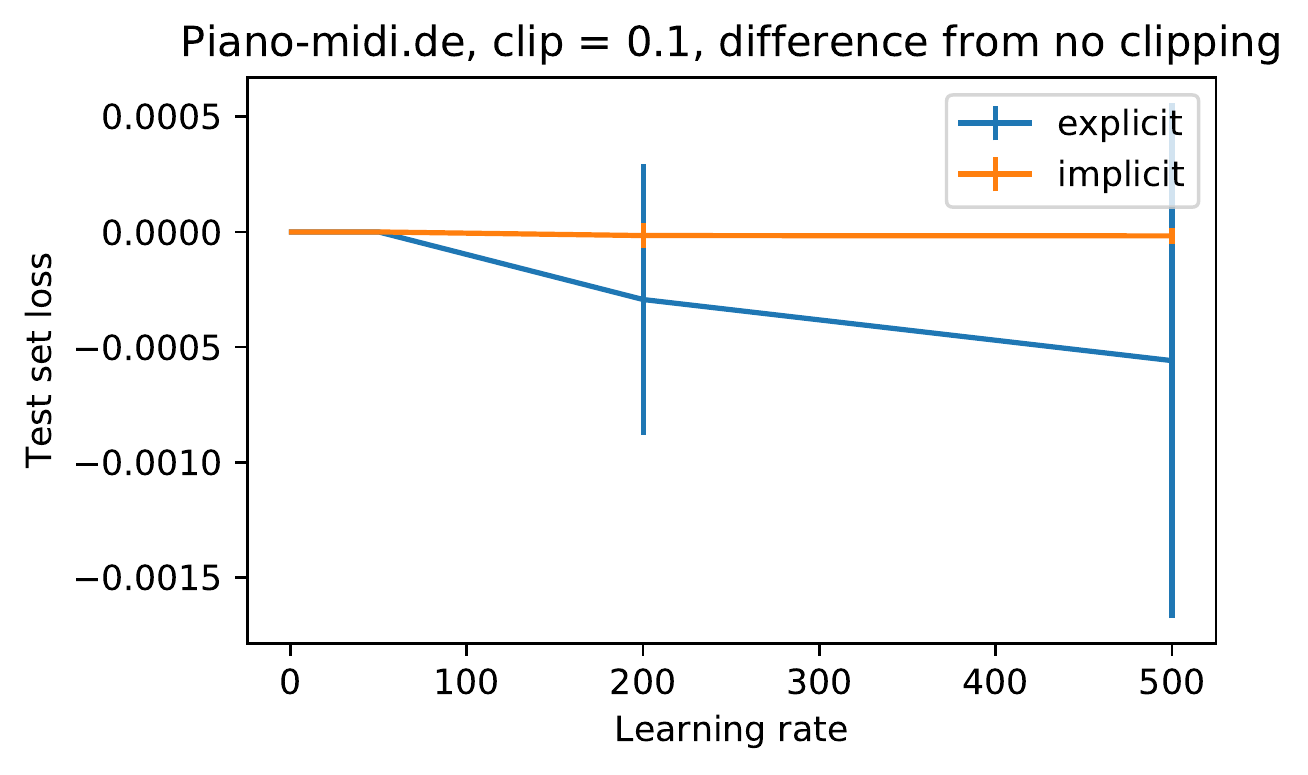}
\end{minipage}%
\hfill
\begin{minipage}{.49\textwidth}
  \centering
  \includegraphics[width=.79\linewidth]{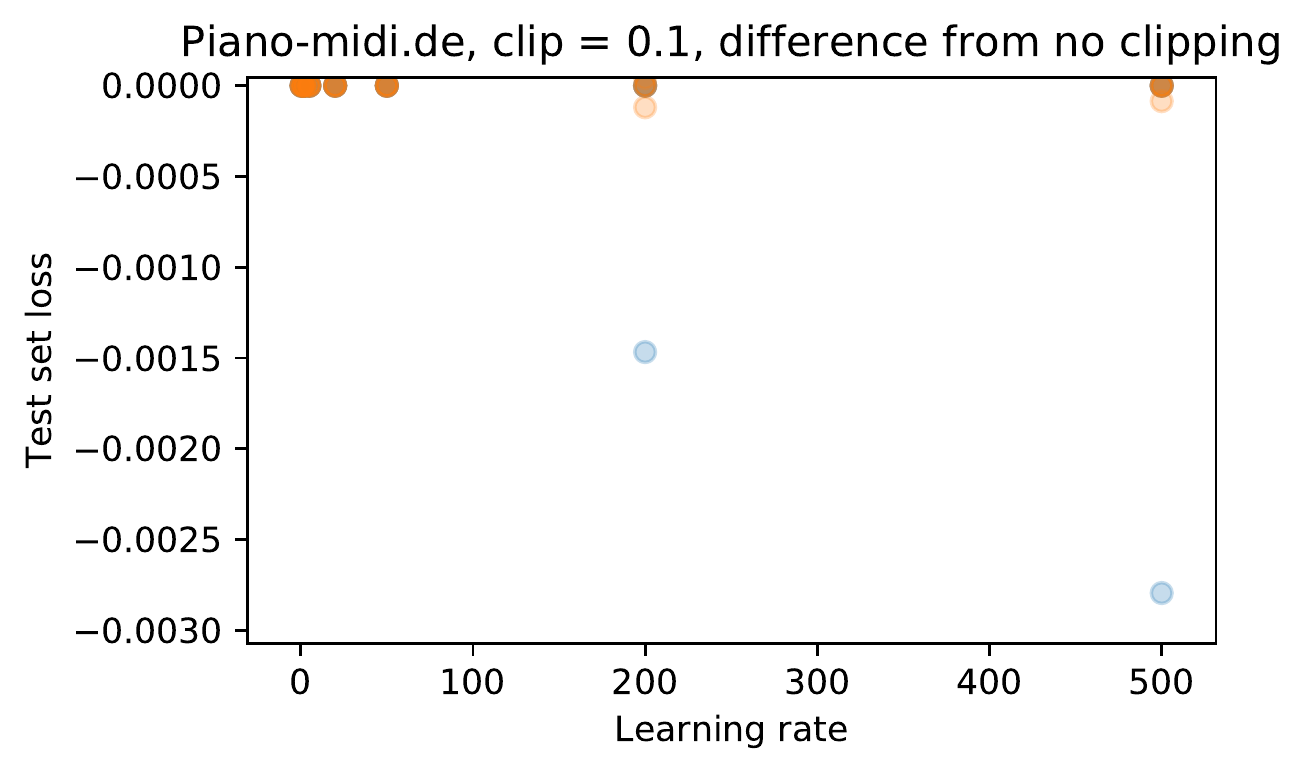}
\end{minipage}
\end{figure}

\begin{figure}[h]
\centering
\begin{minipage}{.49\textwidth}
  \centering
  \includegraphics[width=.79\linewidth]{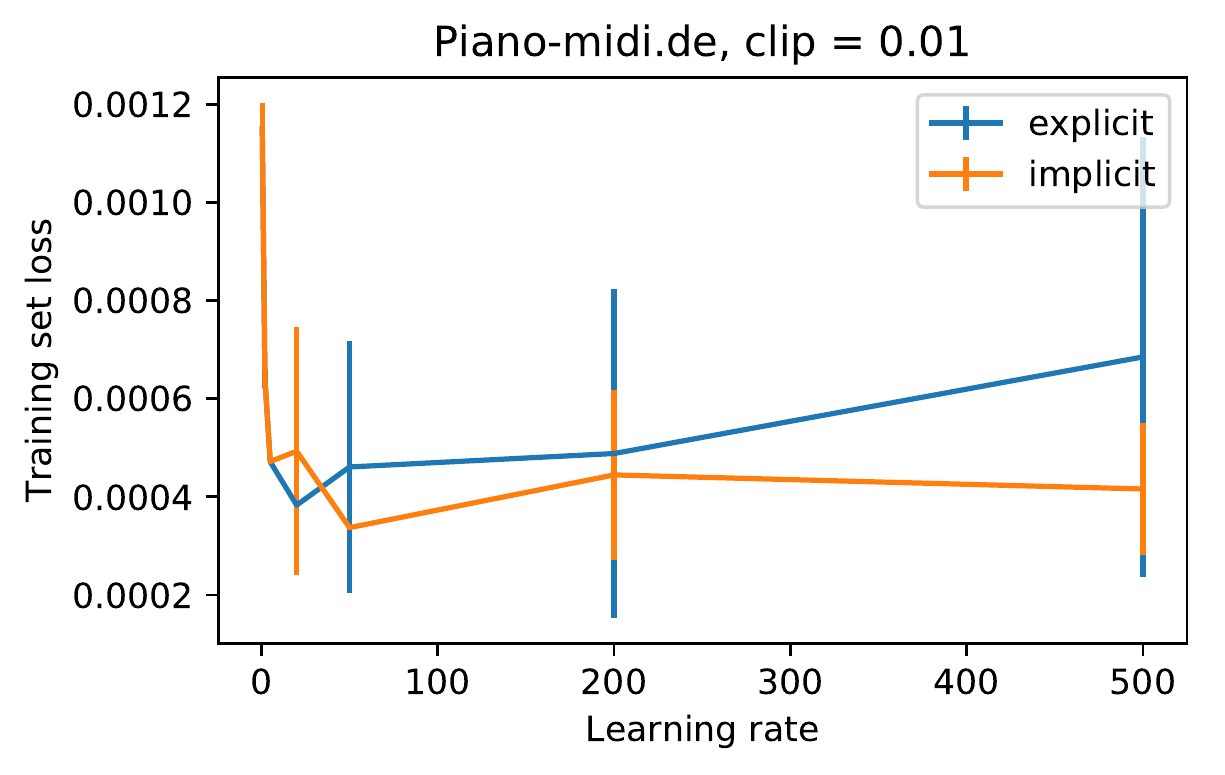}
\end{minipage}%
\hfill
\begin{minipage}{.49\textwidth}
  \centering
  \includegraphics[width=.79\linewidth]{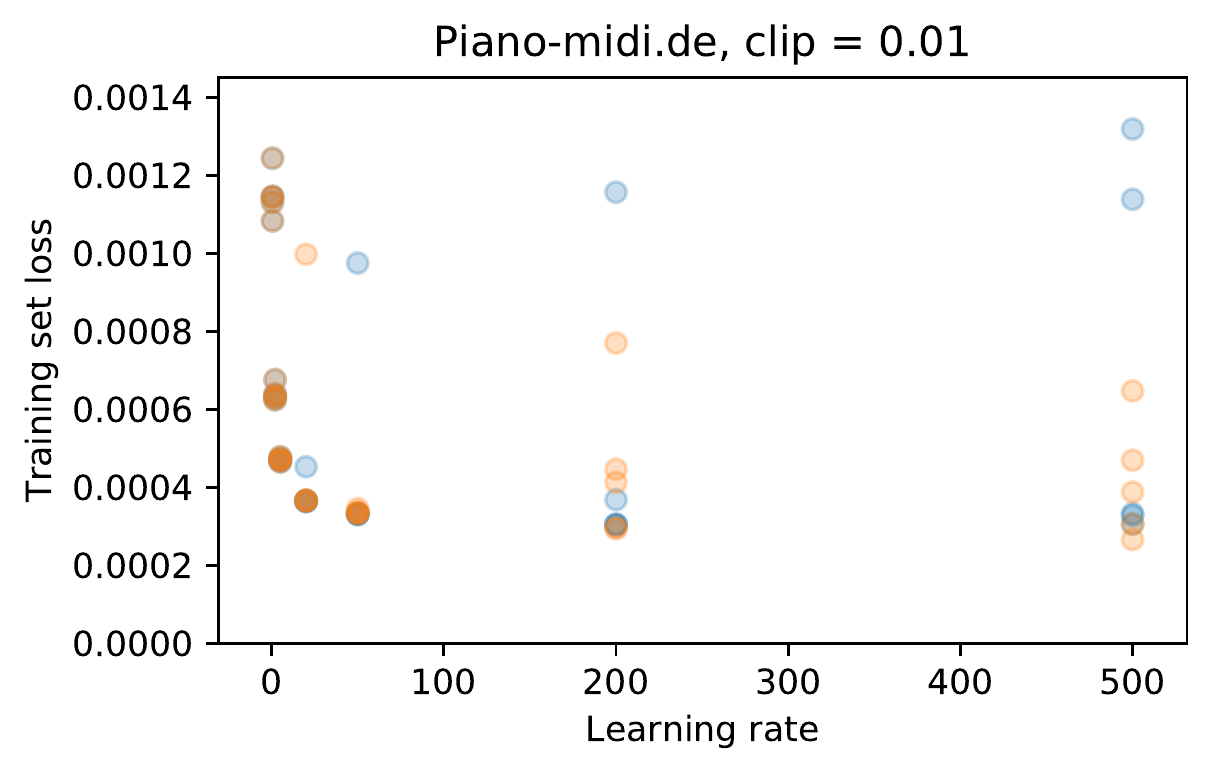}
\end{minipage}
\begin{minipage}{.49\textwidth}
  \centering
  \includegraphics[width=.79\linewidth]{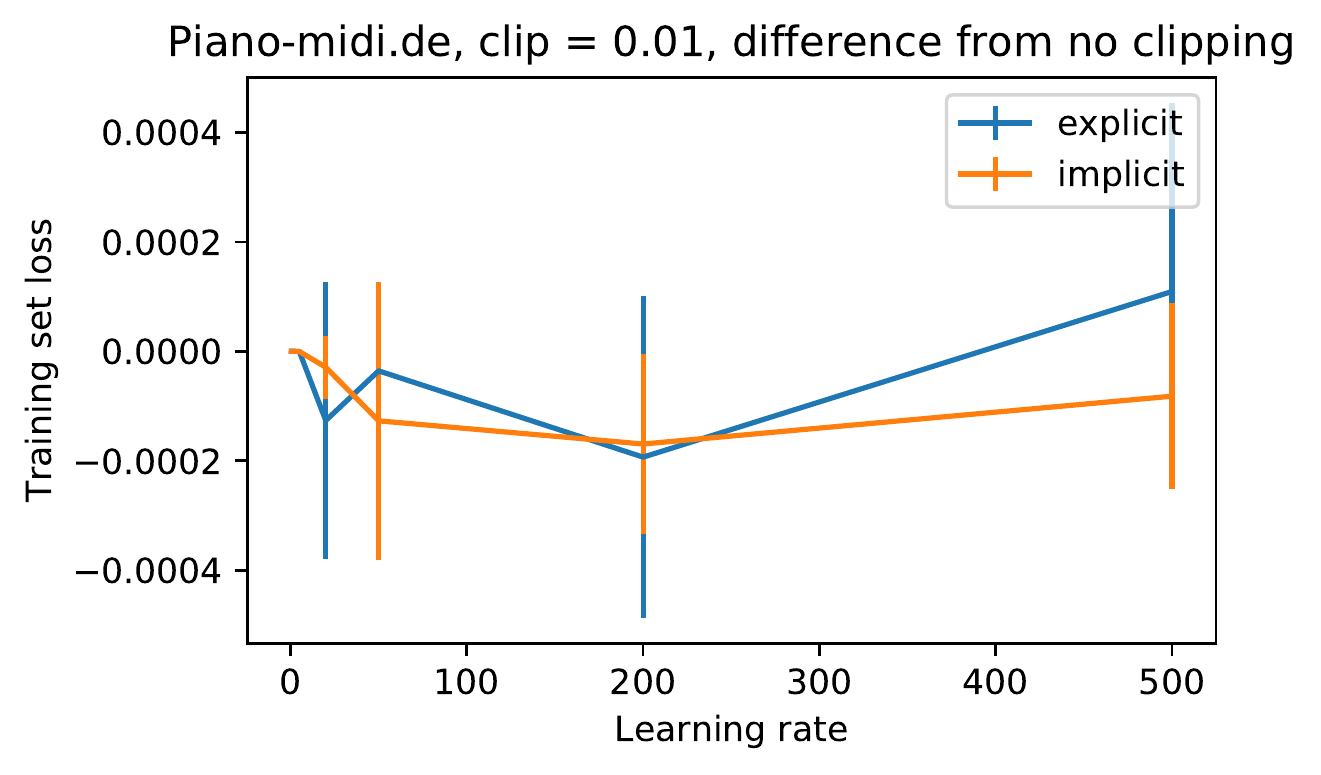}
\end{minipage}%
\hfill
\begin{minipage}{.49\textwidth}
  \centering
  \includegraphics[width=.79\linewidth]{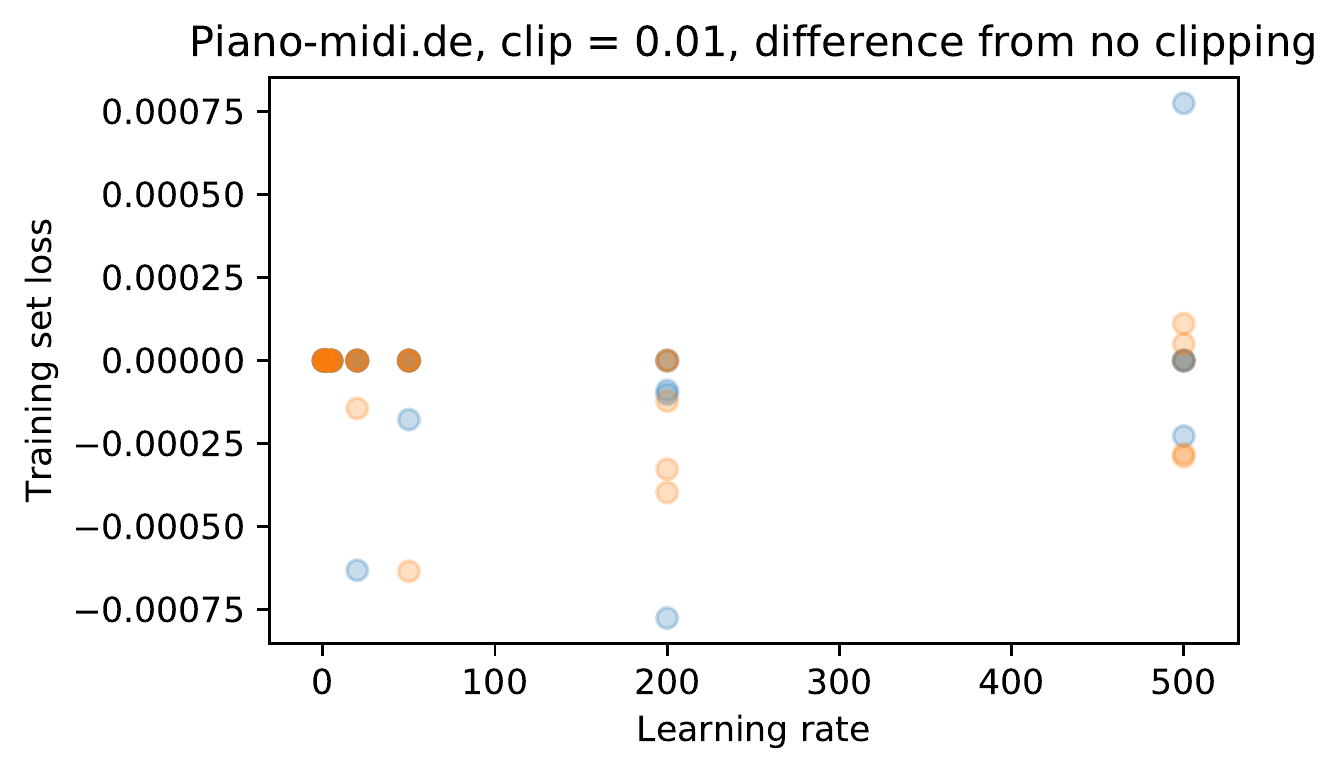}
\end{minipage}
\end{figure}

\vspace{2cm}

\begin{figure}[h]
\begin{minipage}{.49\textwidth}
  \centering
  \includegraphics[width=.79\linewidth]{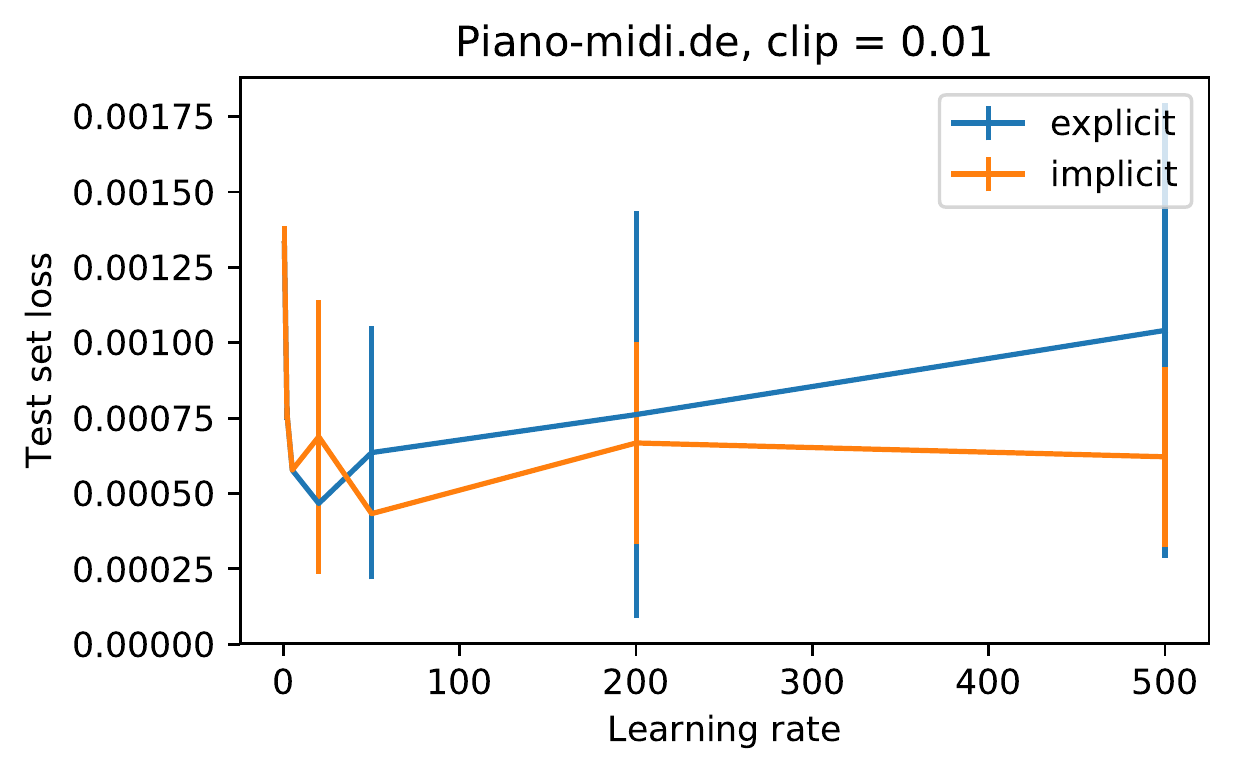}
\end{minipage}%
\hfill
\begin{minipage}{.49\textwidth}
  \centering
  \includegraphics[width=.79\linewidth]{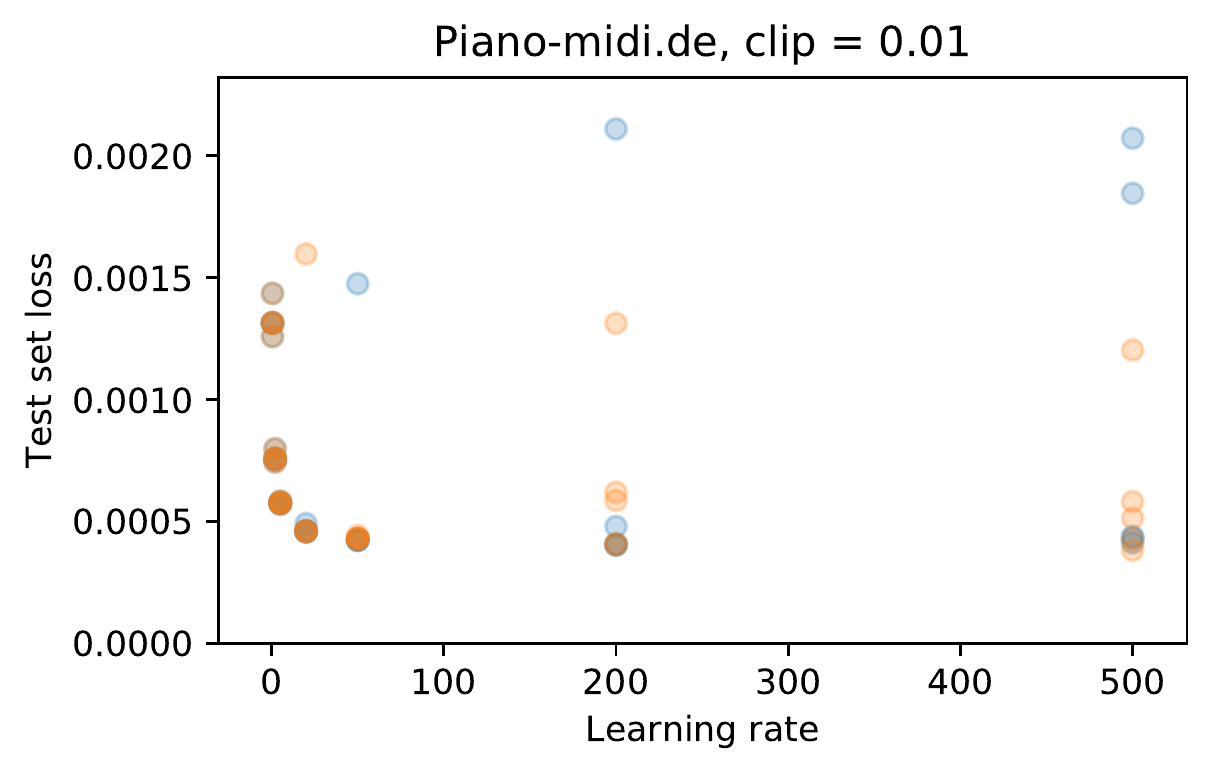}
\end{minipage}
\begin{minipage}{.49\textwidth}
  \centering
  \includegraphics[width=.79\linewidth]{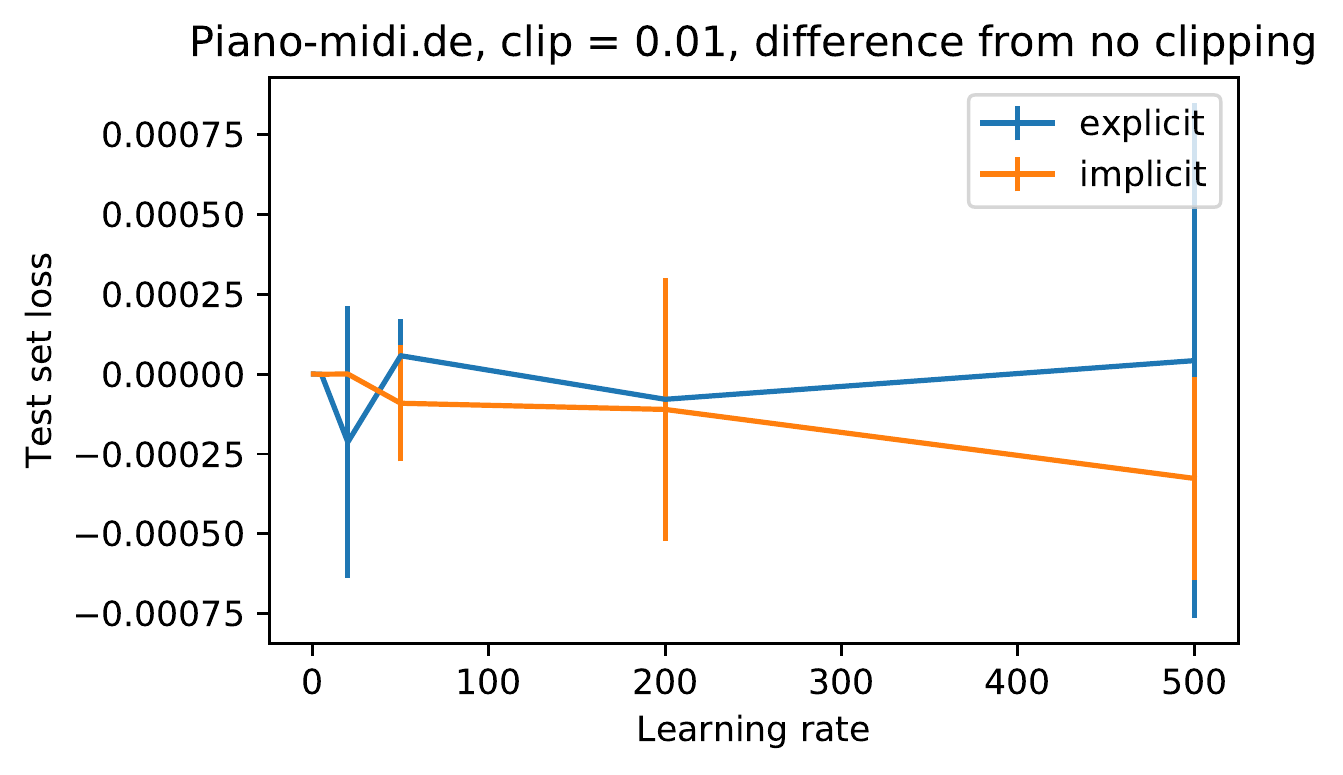}
\end{minipage}%
\hfill
\begin{minipage}{.49\textwidth}
  \centering
  \includegraphics[width=.79\linewidth]{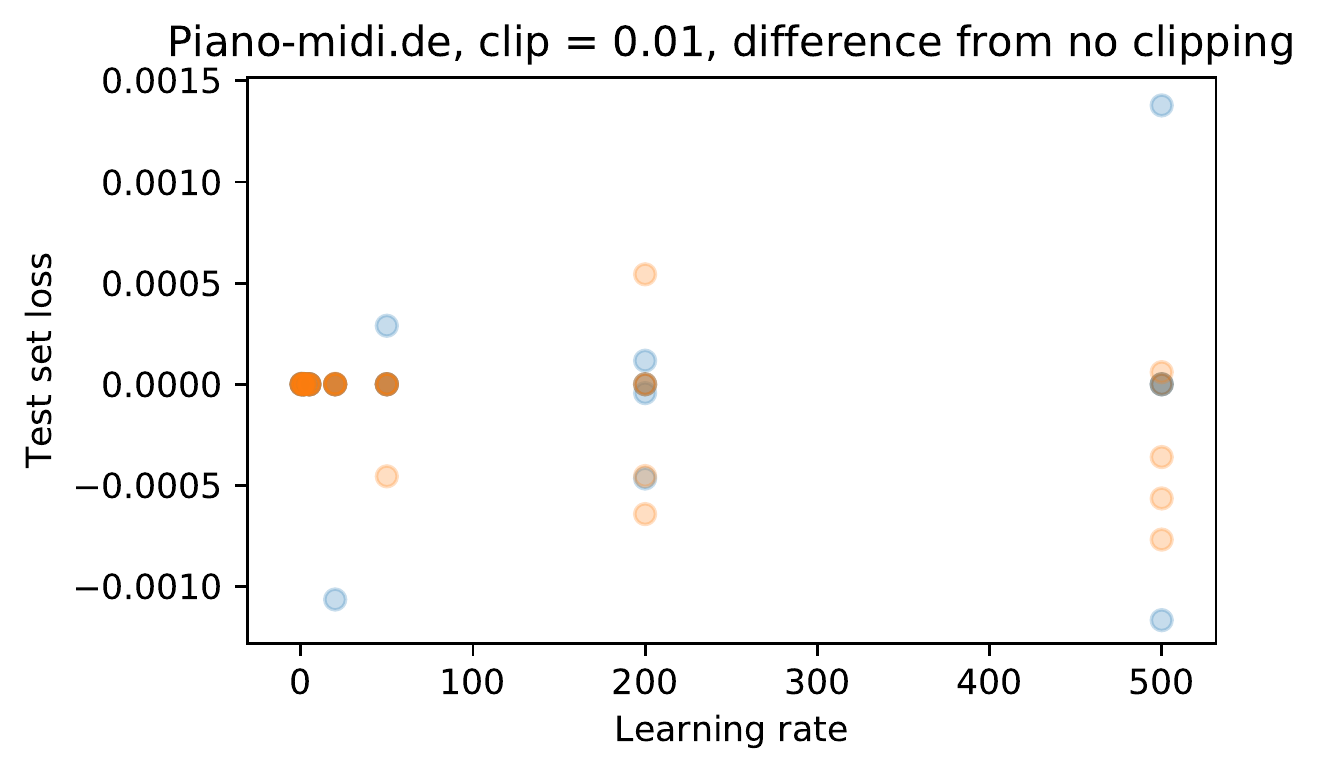}
\end{minipage}
\end{figure}


\clearpage
\subsection{MNIST classification}
\begin{figure}[h]
\centering
\begin{minipage}{.49\textwidth}
  \centering
  \includegraphics[width=.99\linewidth]{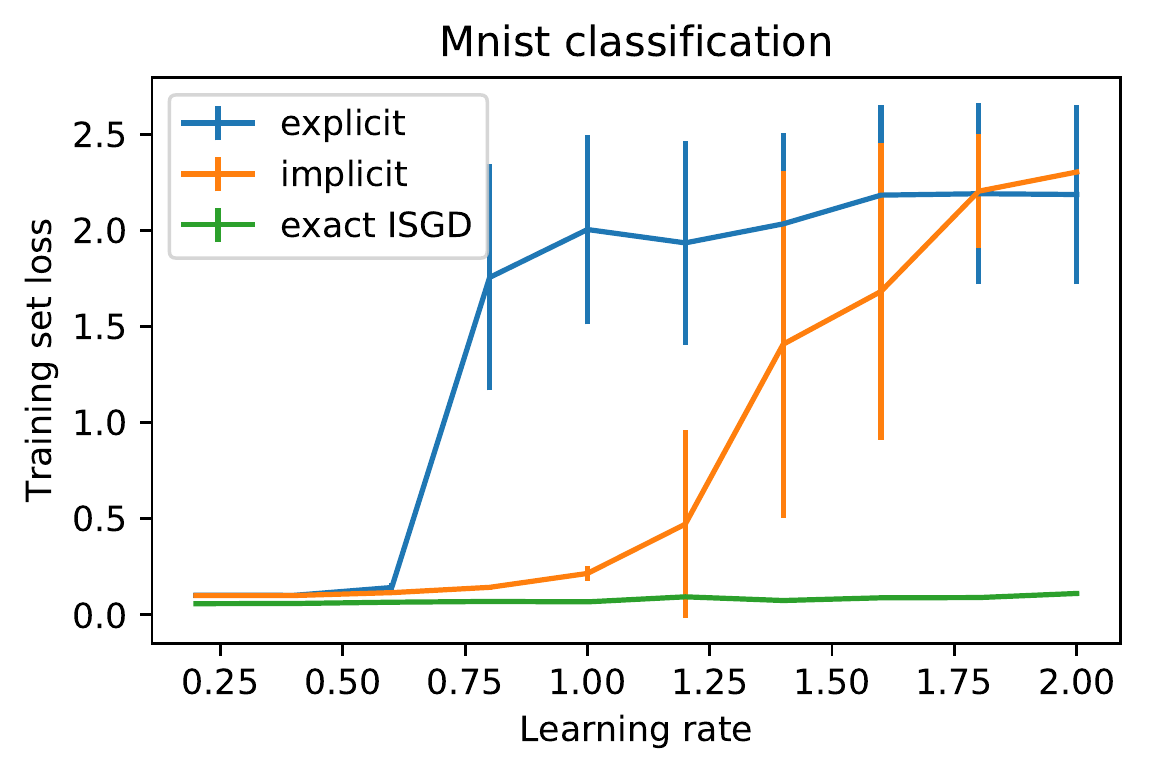}
\end{minipage}%
\hfill
\begin{minipage}{.49\textwidth}
  \centering
  \includegraphics[width=.99\linewidth]{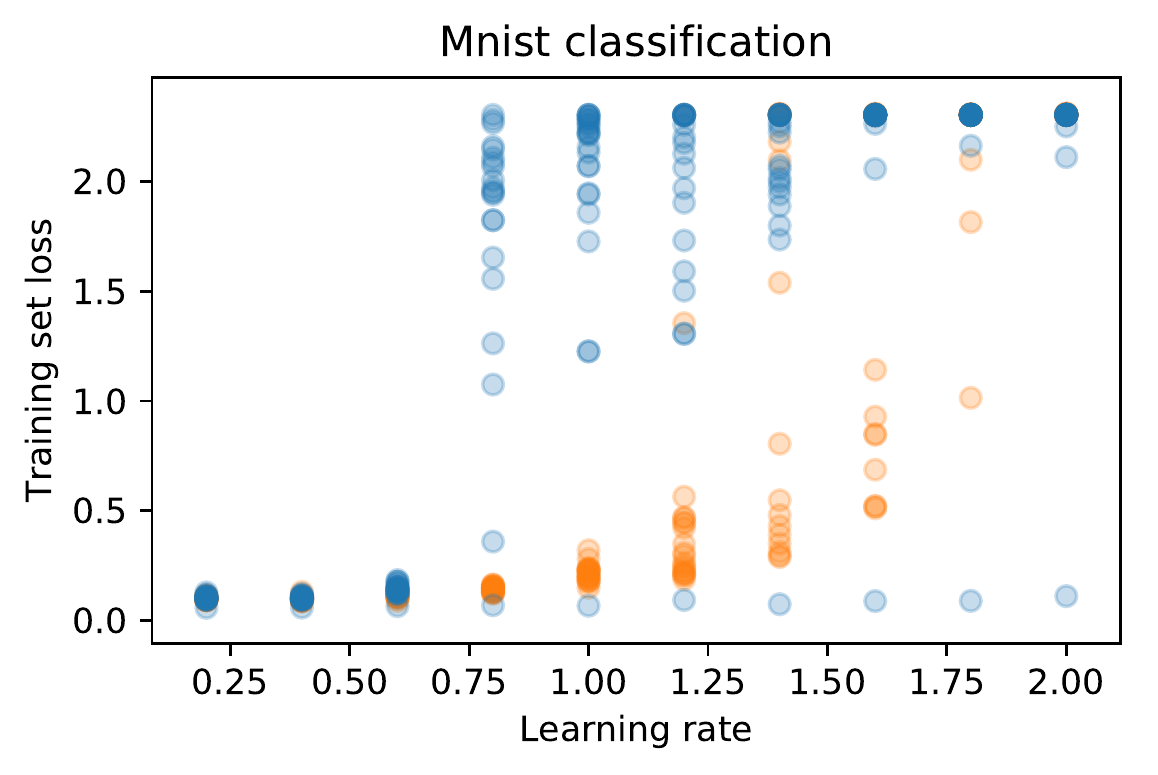}
\end{minipage}
\begin{minipage}{.49\textwidth}
  \centering
  \includegraphics[width=.99\linewidth]{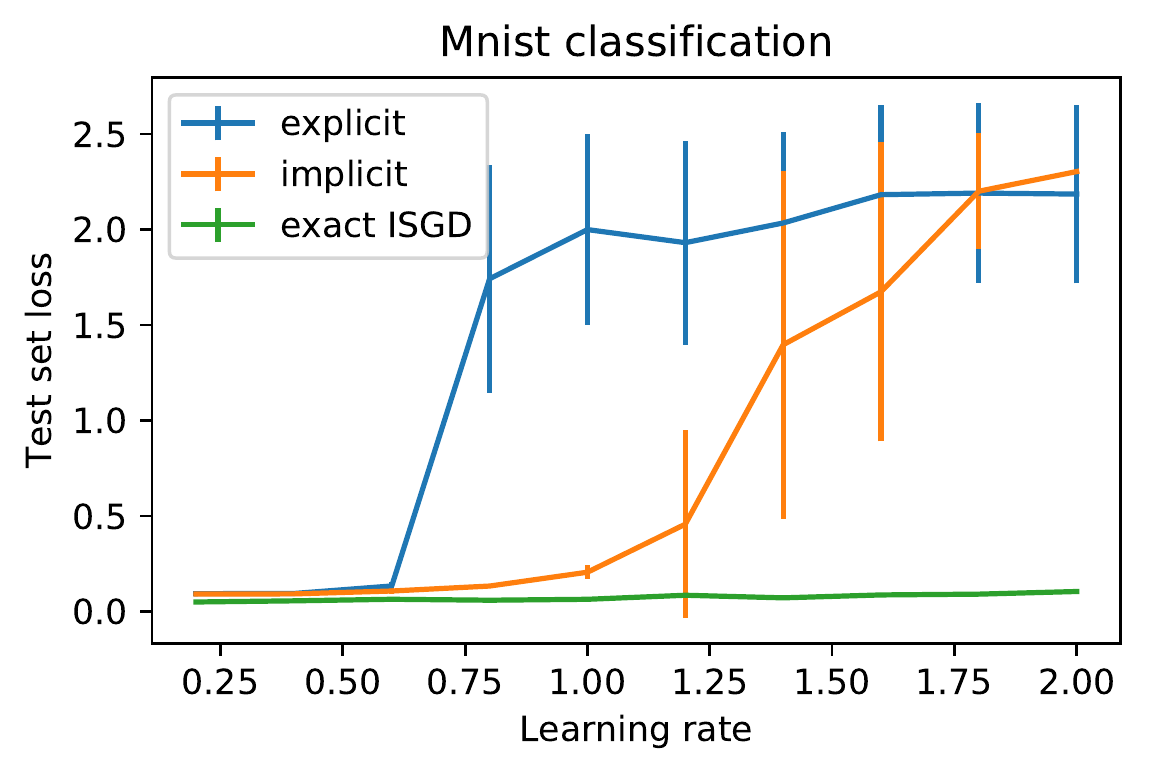}
\end{minipage}%
\hfill
\begin{minipage}{.49\textwidth}
  \centering
  \includegraphics[width=.99\linewidth]{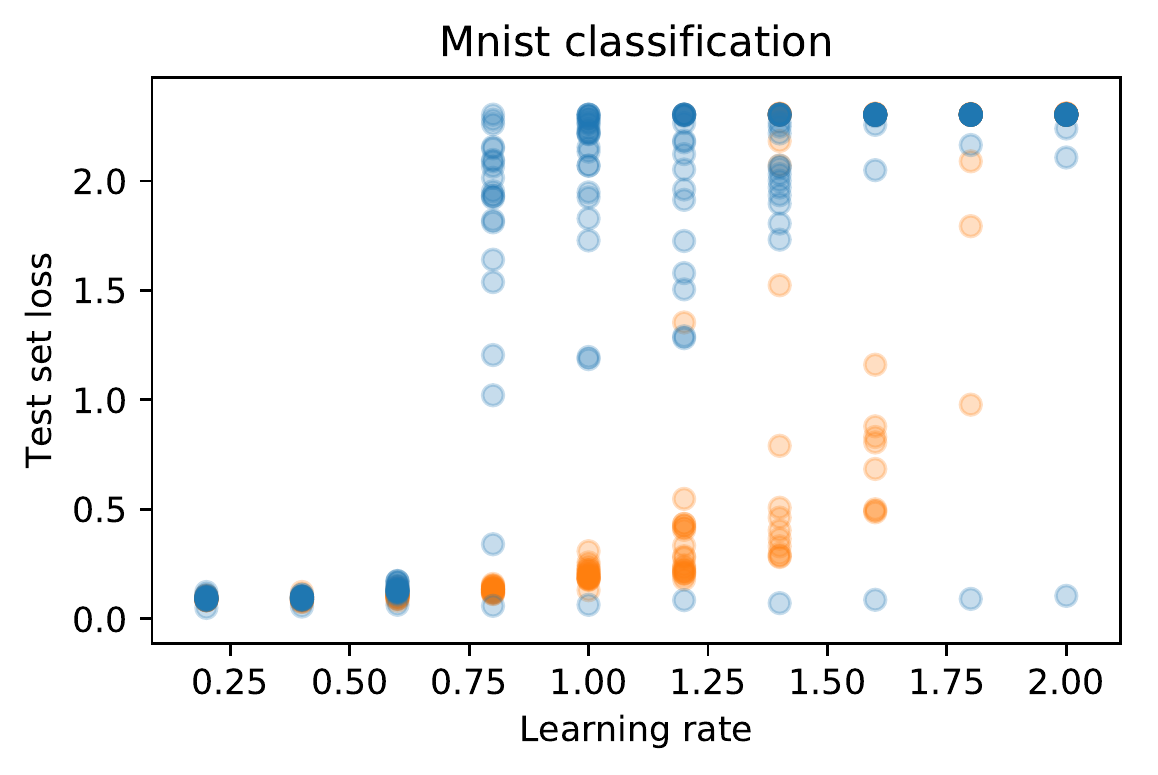}
\end{minipage}
\begin{minipage}{.49\textwidth}
  \centering
  \includegraphics[width=.99\linewidth]{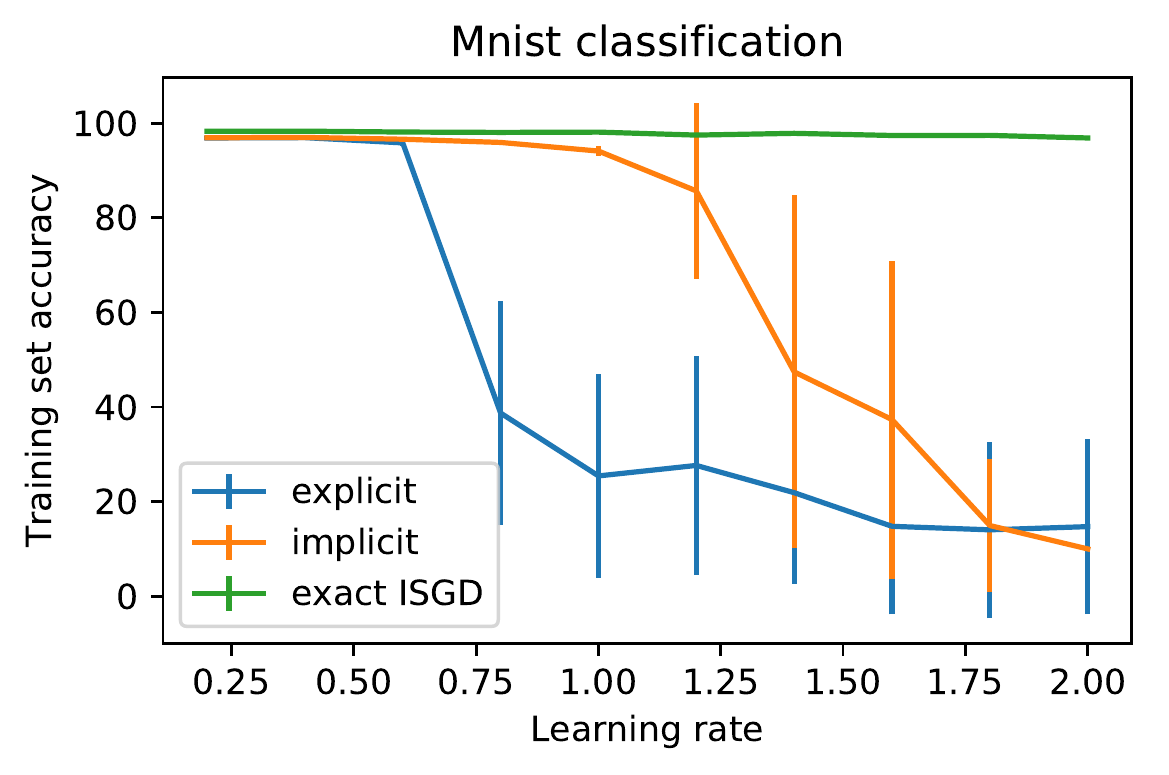}
\end{minipage}%
\hfill
\begin{minipage}{.49\textwidth}
  \centering
  \includegraphics[width=.99\linewidth]{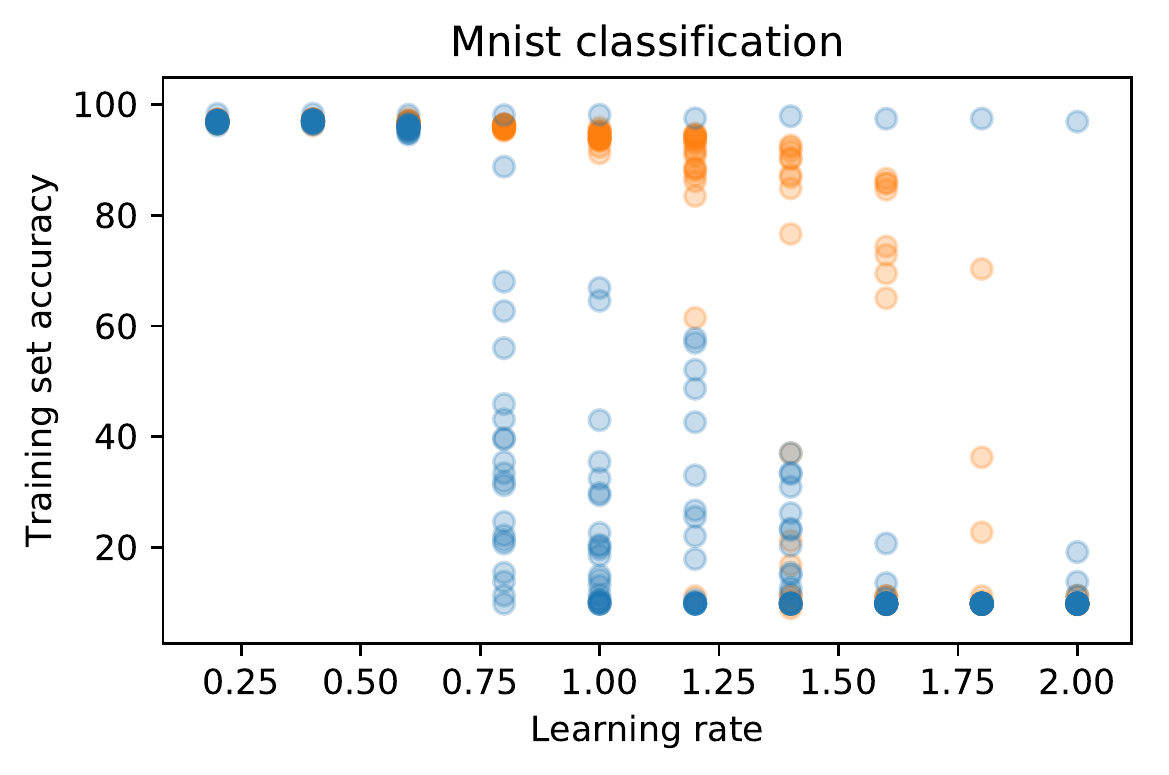}
\end{minipage}
\begin{minipage}{.49\textwidth}
  \centering
  \includegraphics[width=.99\linewidth]{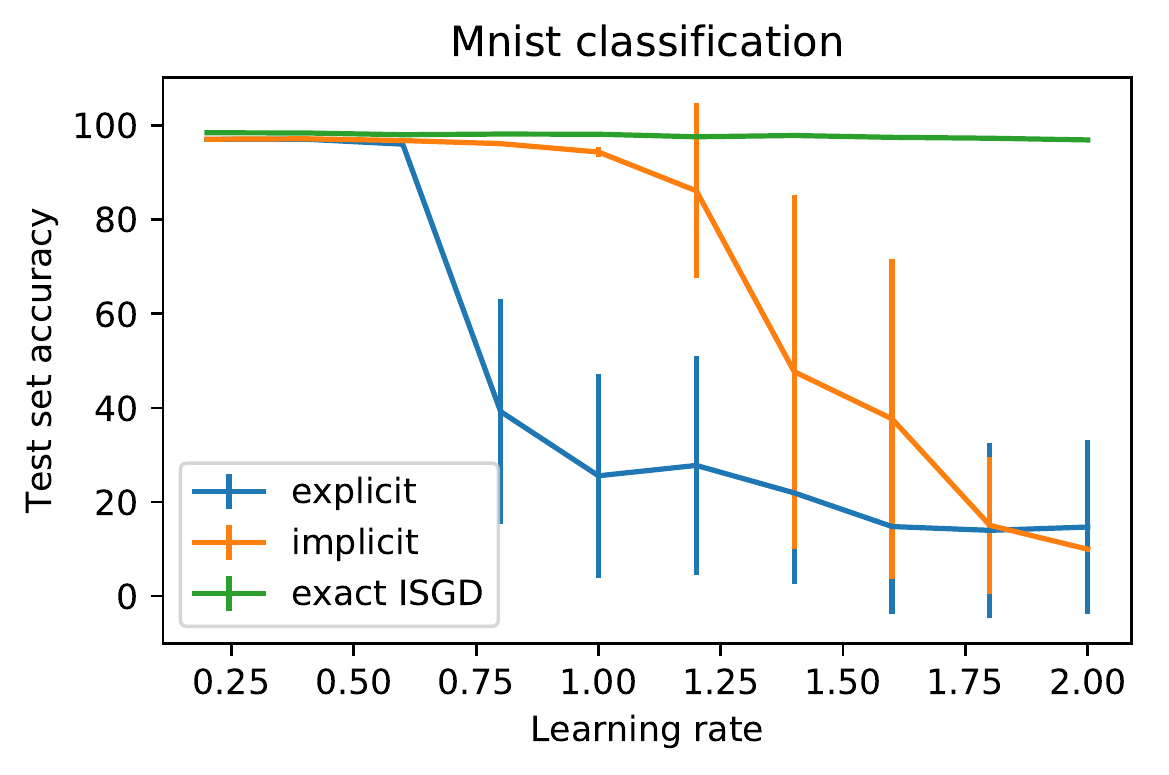}
\end{minipage}%
\hfill
\begin{minipage}{.49\textwidth}
  \centering
  \includegraphics[width=.99\linewidth]{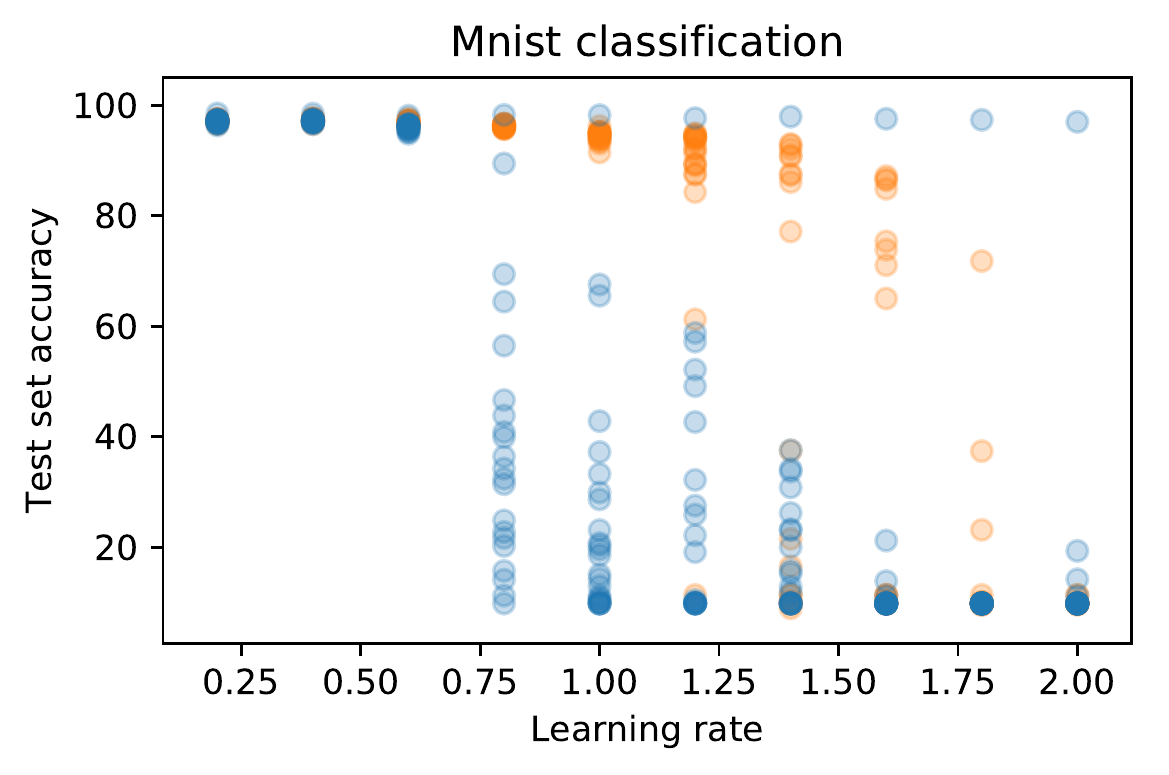}
\end{minipage}
\end{figure}

%

\begin{figure}[h]
\centering
\begin{minipage}{.49\textwidth}
  \centering
  \includegraphics[width=.99\linewidth]{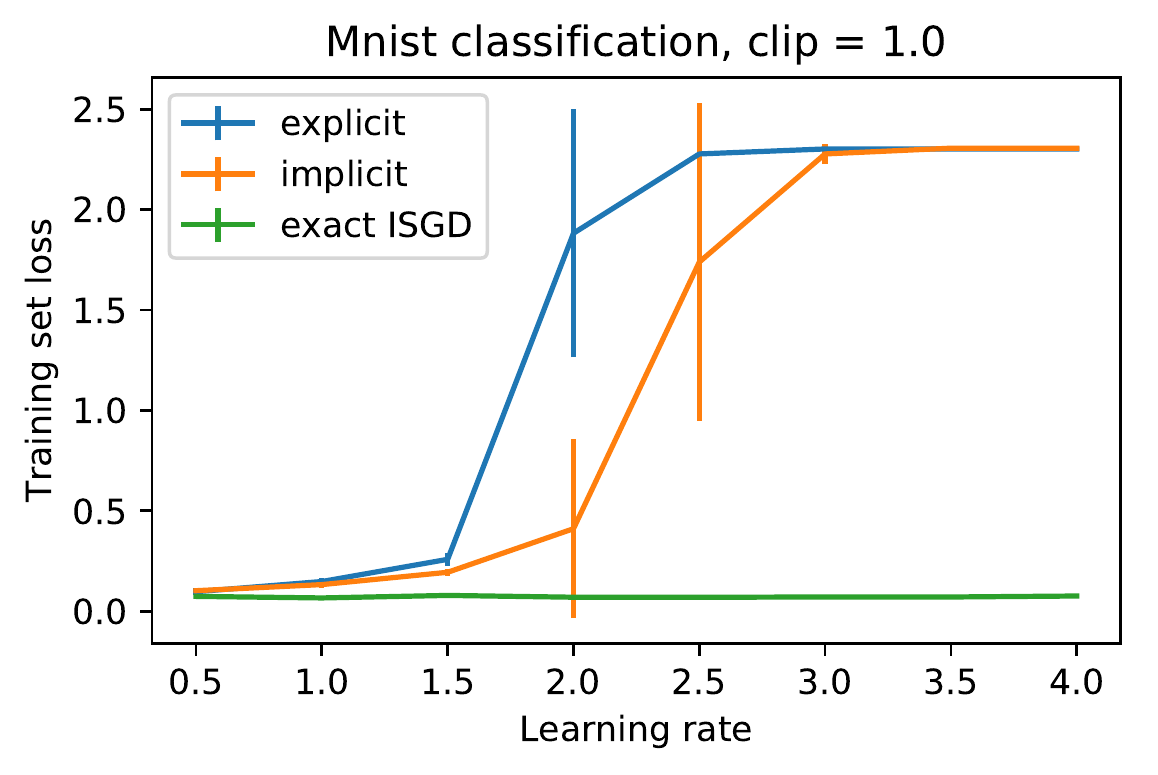}
\end{minipage}%
\hfill
\begin{minipage}{.49\textwidth}
  \centering
  \includegraphics[width=.99\linewidth]{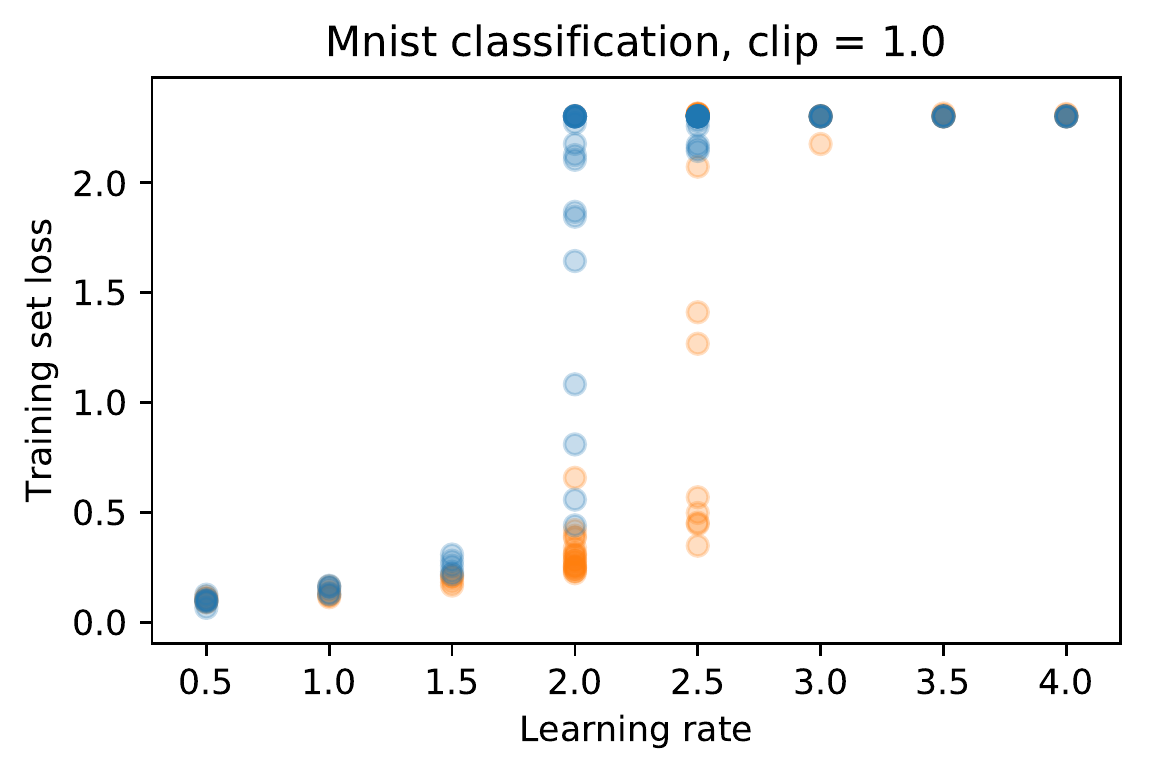}
\end{minipage}
\begin{minipage}{.49\textwidth}
  \centering
  \includegraphics[width=.99\linewidth]{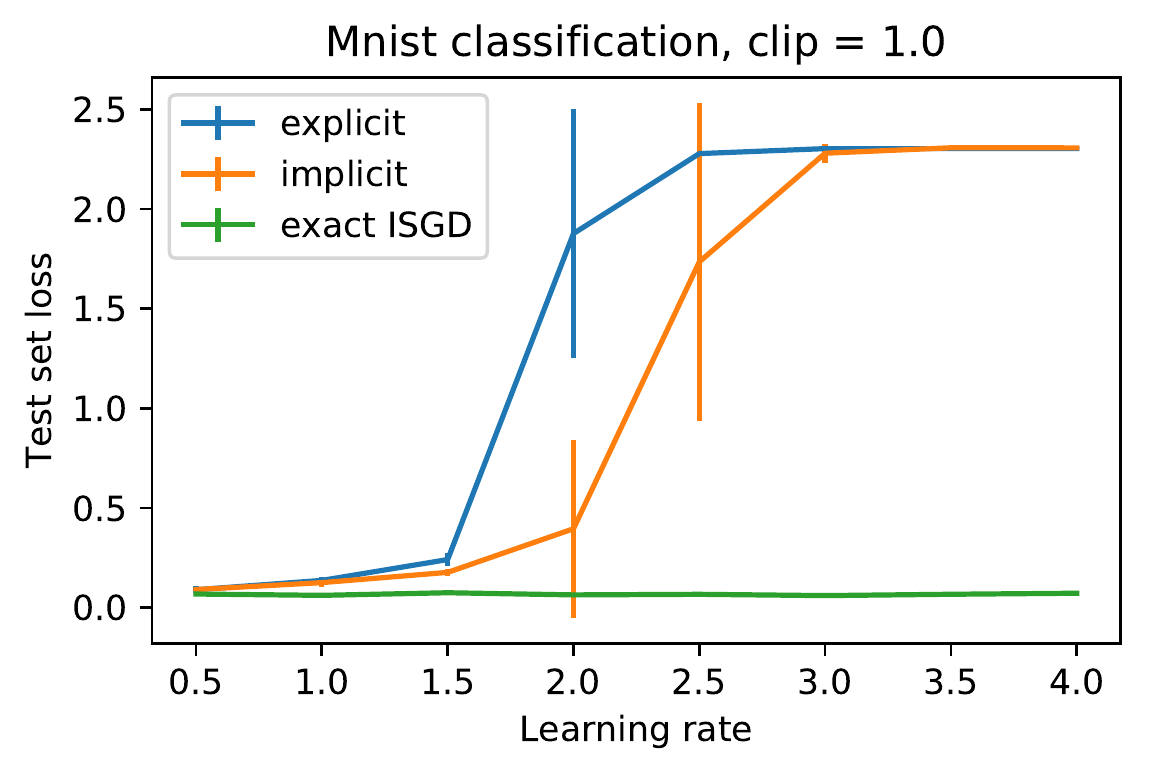}
\end{minipage}%
\hfill
\begin{minipage}{.49\textwidth}
  \centering
  \includegraphics[width=.99\linewidth]{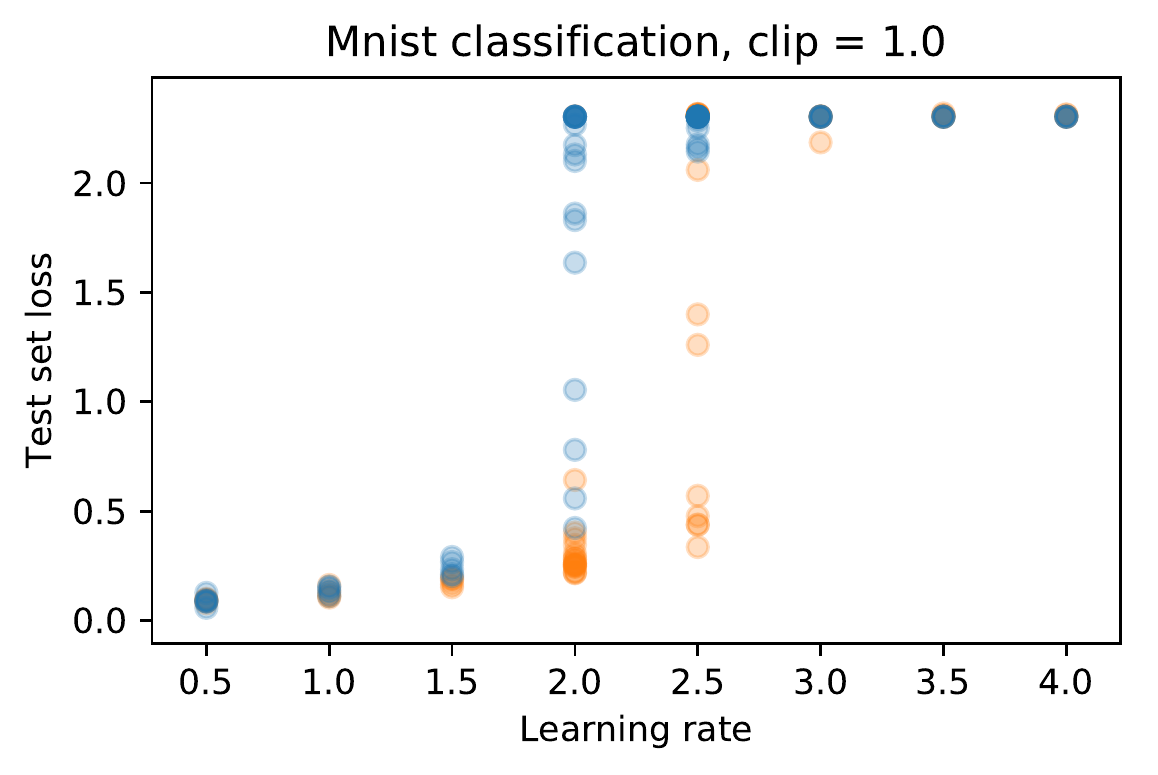}
\end{minipage}

\vspace{1cm}

\begin{minipage}{.49\textwidth}
  \centering
  \includegraphics[width=.99\linewidth]{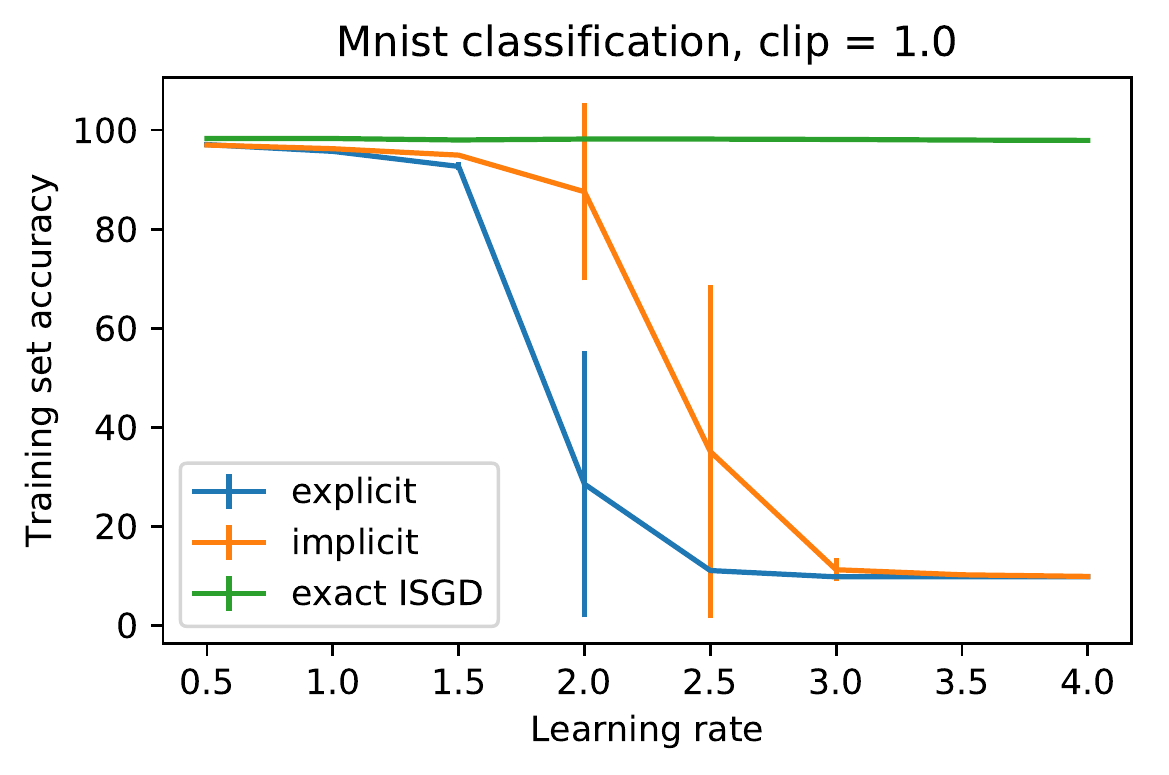}
\end{minipage}%
\hfill
\begin{minipage}{.49\textwidth}
  \centering
  \includegraphics[width=.99\linewidth]{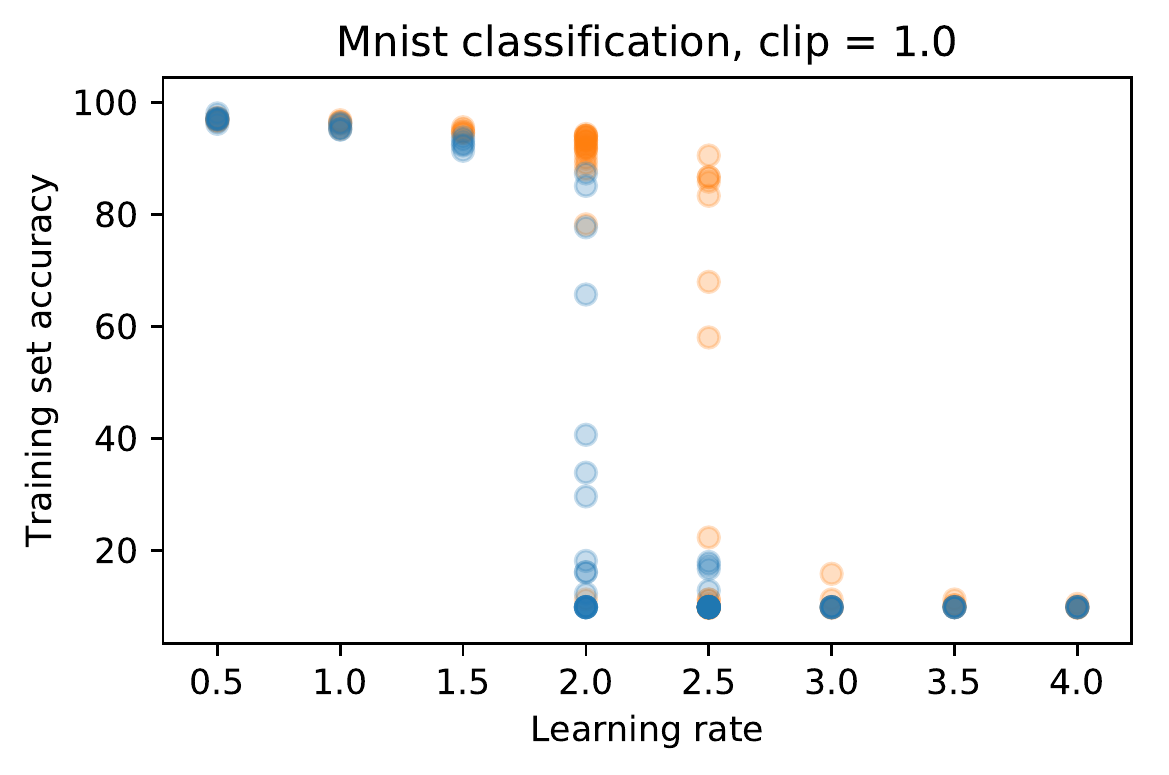}
\end{minipage}
\begin{minipage}{.49\textwidth}
  \centering
  \includegraphics[width=.99\linewidth]{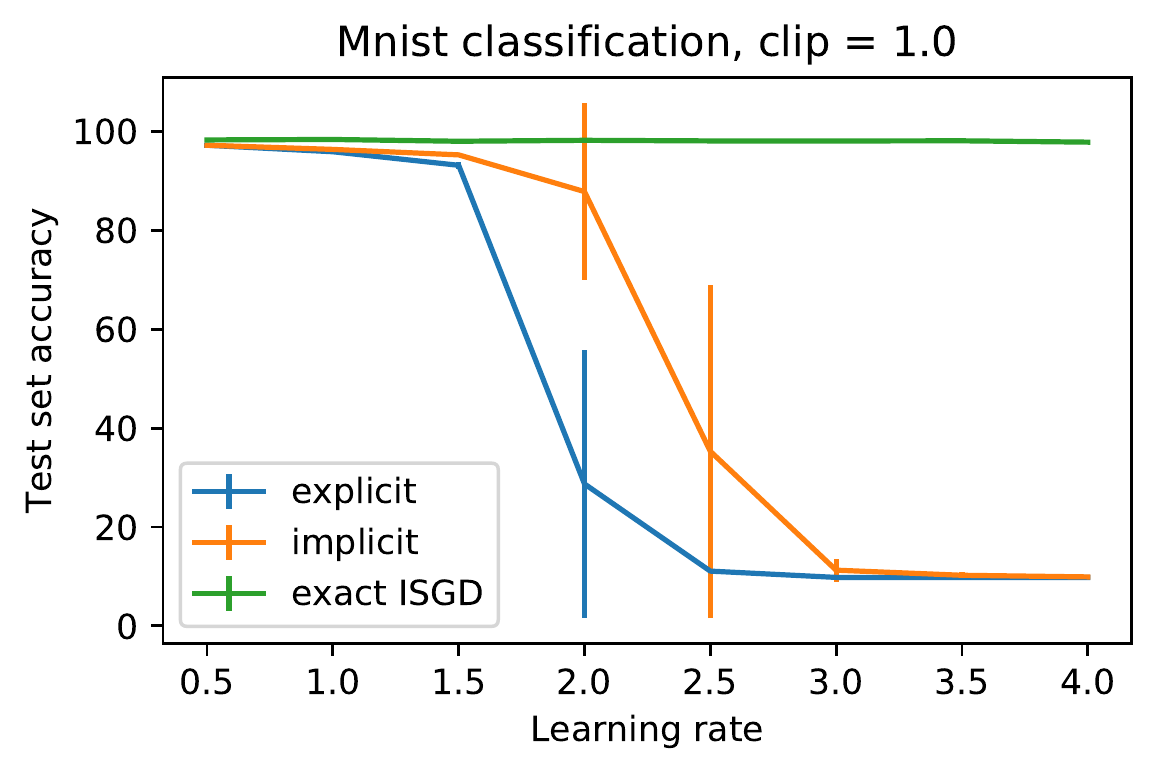}
\end{minipage}%
\hfill
\begin{minipage}{.49\textwidth}
  \centering
  \includegraphics[width=.99\linewidth]{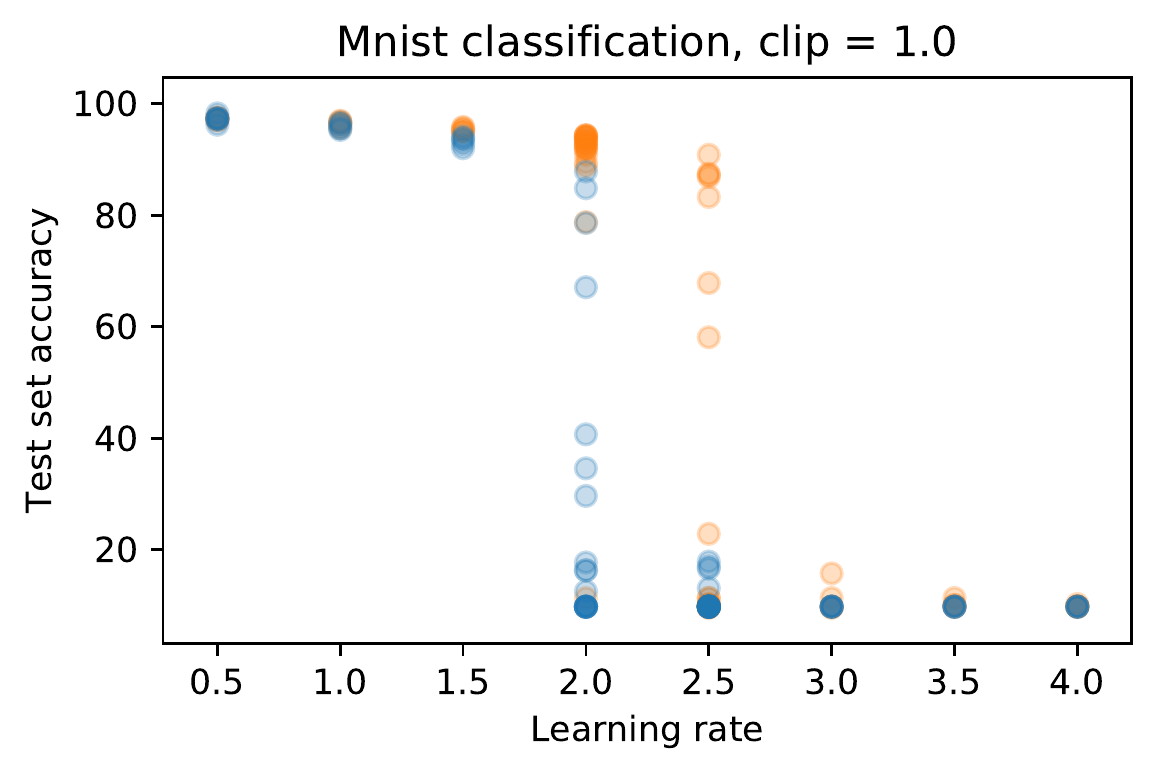}
\end{minipage}
\end{figure}

\clearpage
\subsection{MNIST autoencoder}

\begin{figure}[h]
\centering
\begin{minipage}{.49\textwidth}
  \centering
  \includegraphics[width=.99\linewidth]{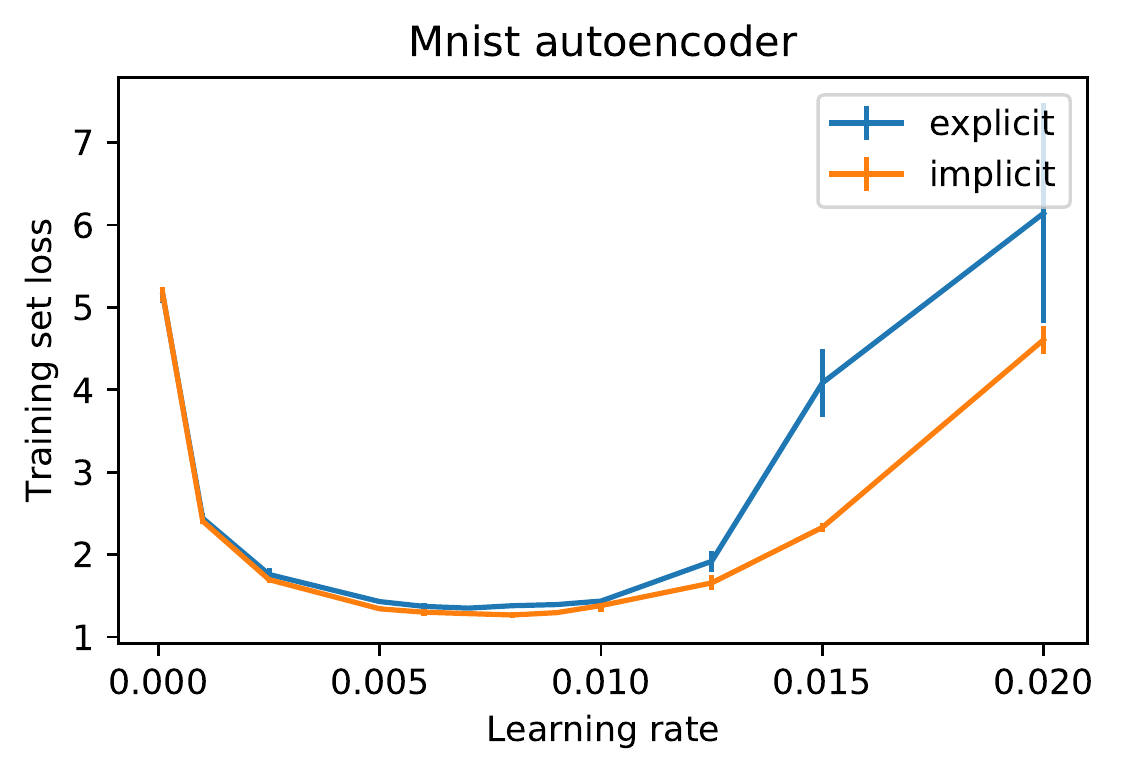}
\end{minipage}%
\hfill
\begin{minipage}{.49\textwidth}
  \centering
  \includegraphics[width=.99\linewidth]{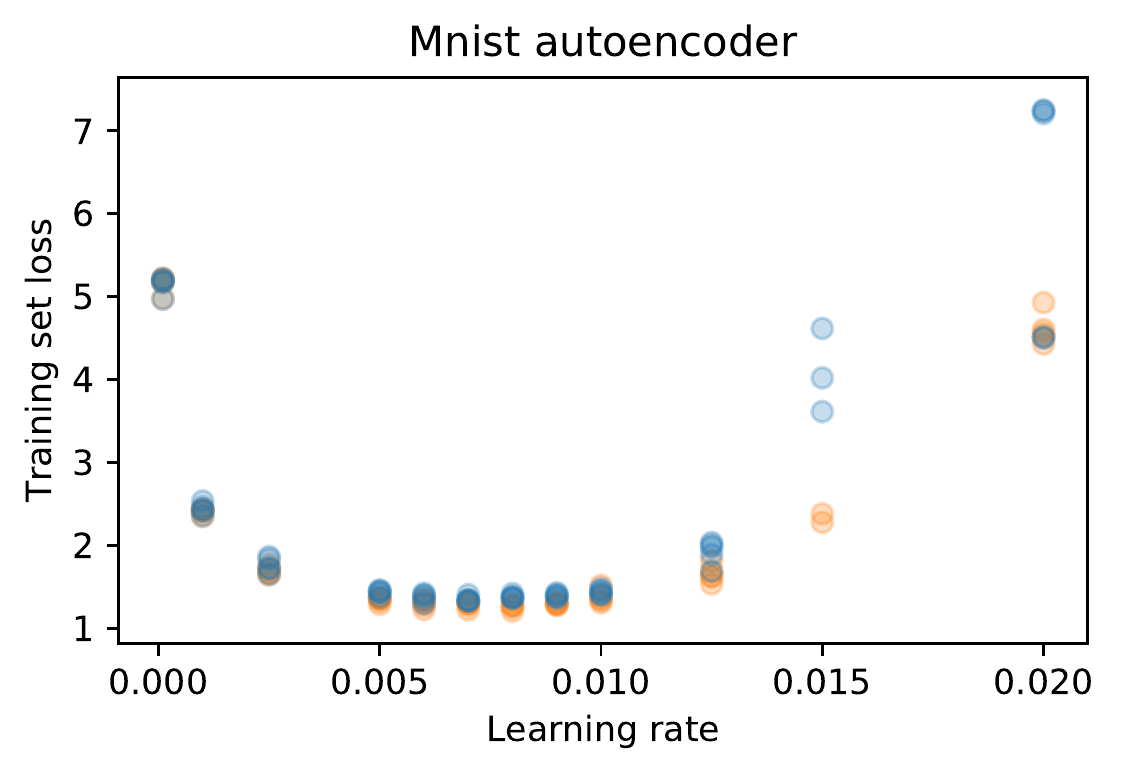}
\end{minipage}
\begin{minipage}{.49\textwidth}
  \centering
  \includegraphics[width=.99\linewidth]{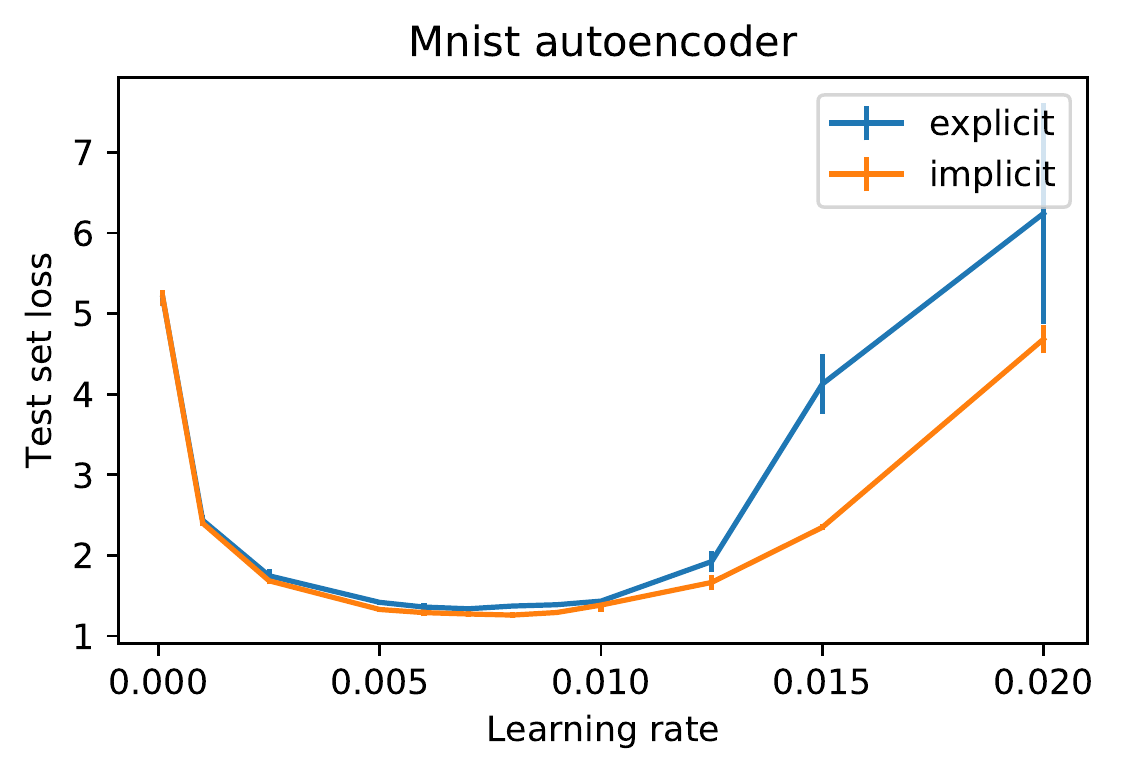}
\end{minipage}%
\hfill
\begin{minipage}{.49\textwidth}
  \centering
  \includegraphics[width=.99\linewidth]{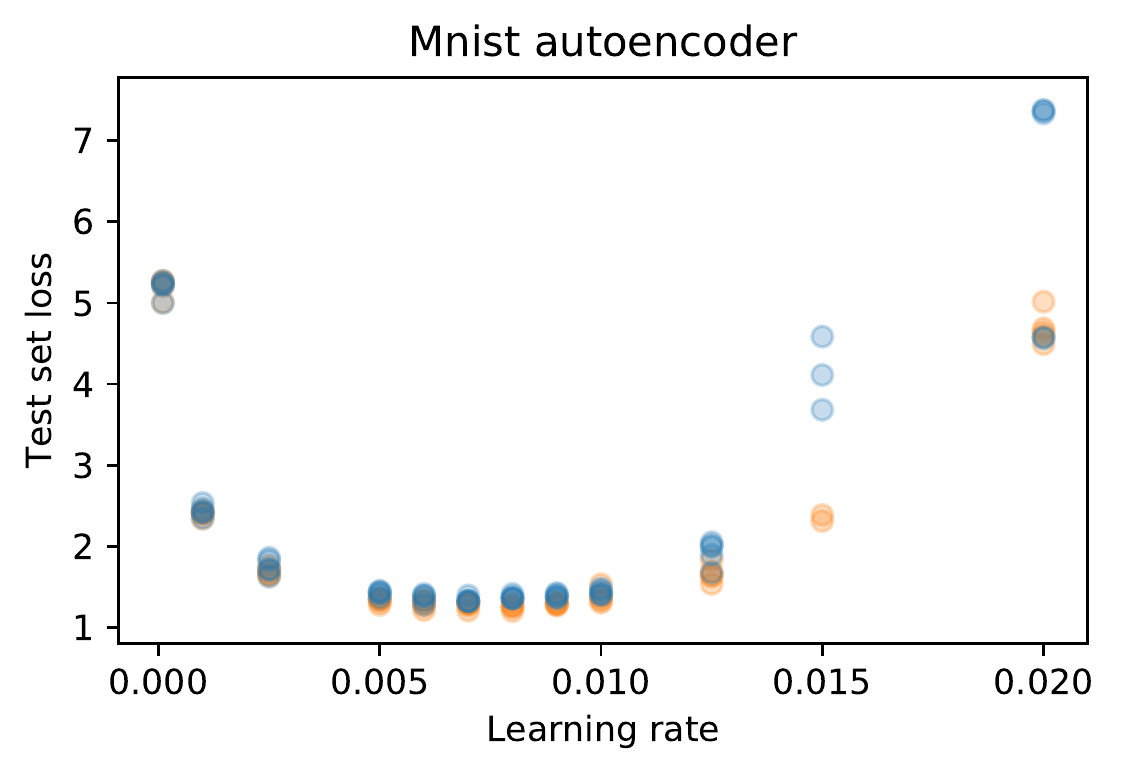}
\end{minipage}

\vspace{1cm}

\begin{minipage}{.49\textwidth}
  \centering
  \includegraphics[width=.99\linewidth]{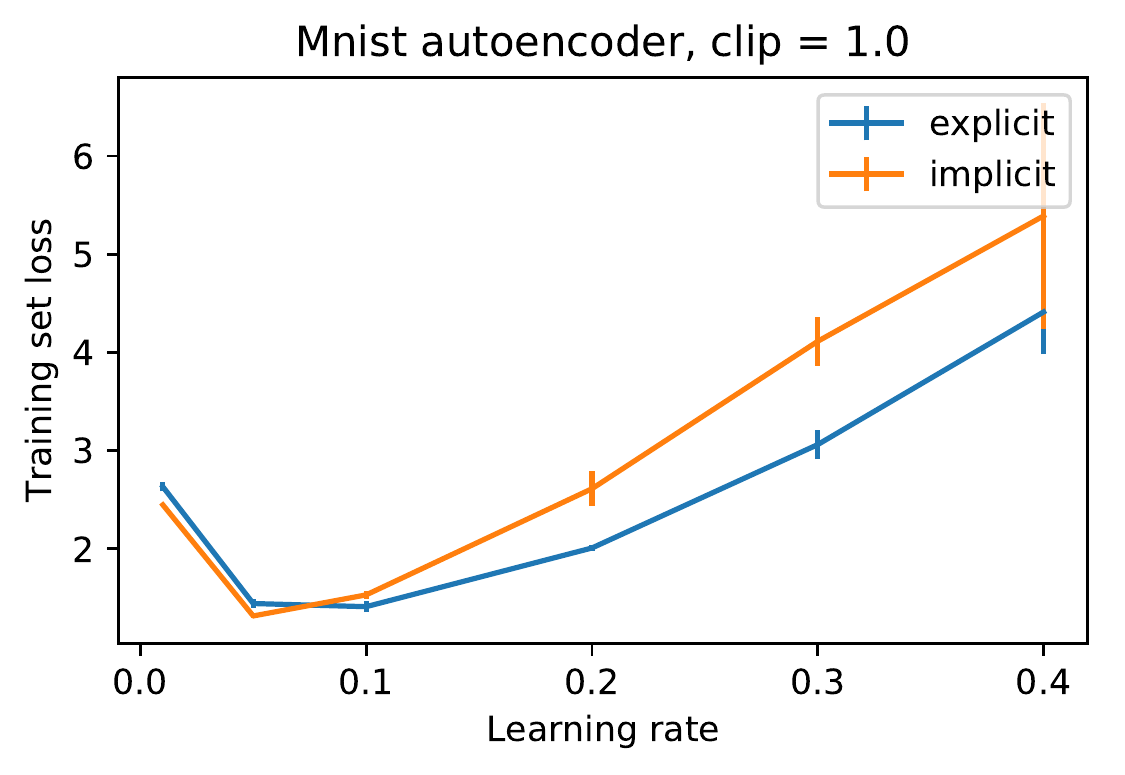}
\end{minipage}%
\hfill
\begin{minipage}{.49\textwidth}
  \centering
  \includegraphics[width=.99\linewidth]{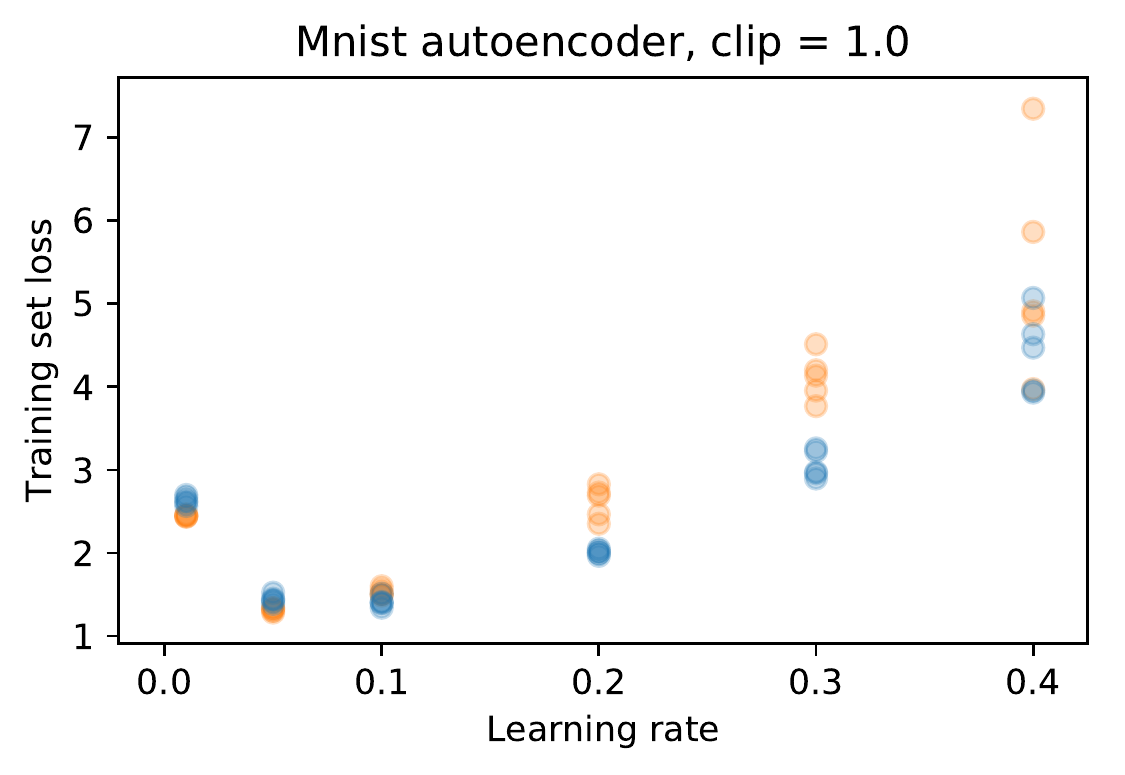}
\end{minipage}
\begin{minipage}{.49\textwidth}
  \centering
  \includegraphics[width=.99\linewidth]{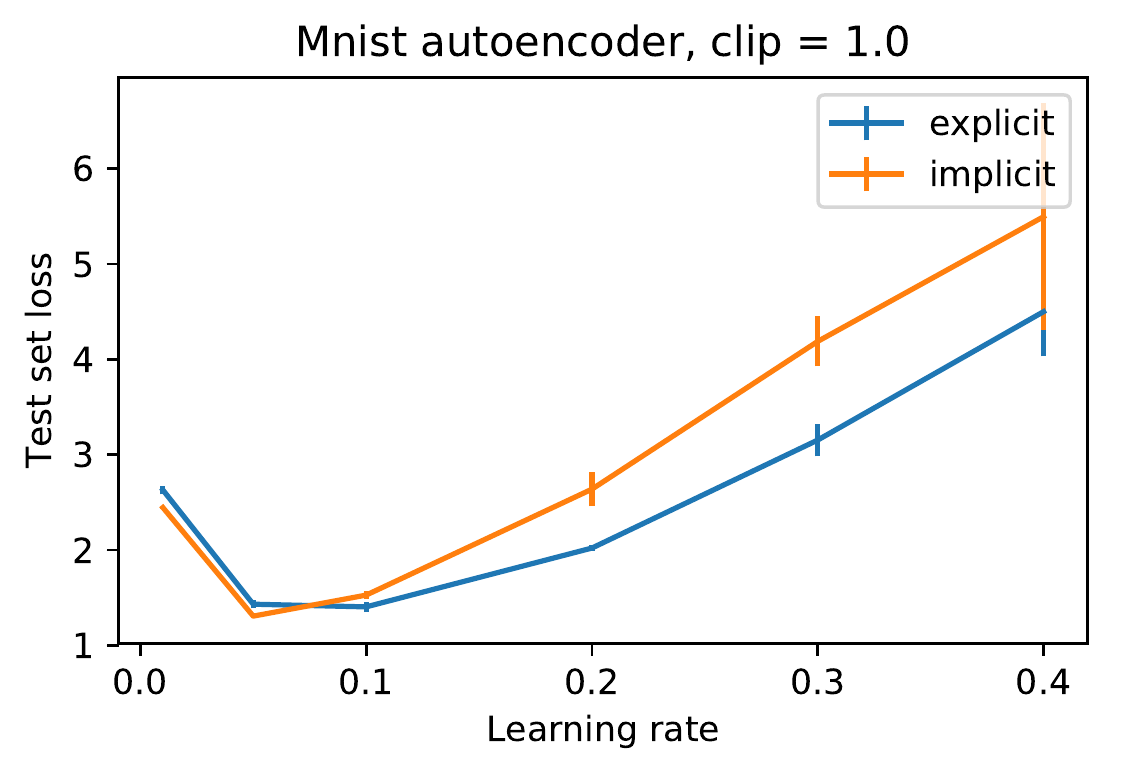}
\end{minipage}%
\hfill
\begin{minipage}{.49\textwidth}
  \centering
  \includegraphics[width=.99\linewidth]{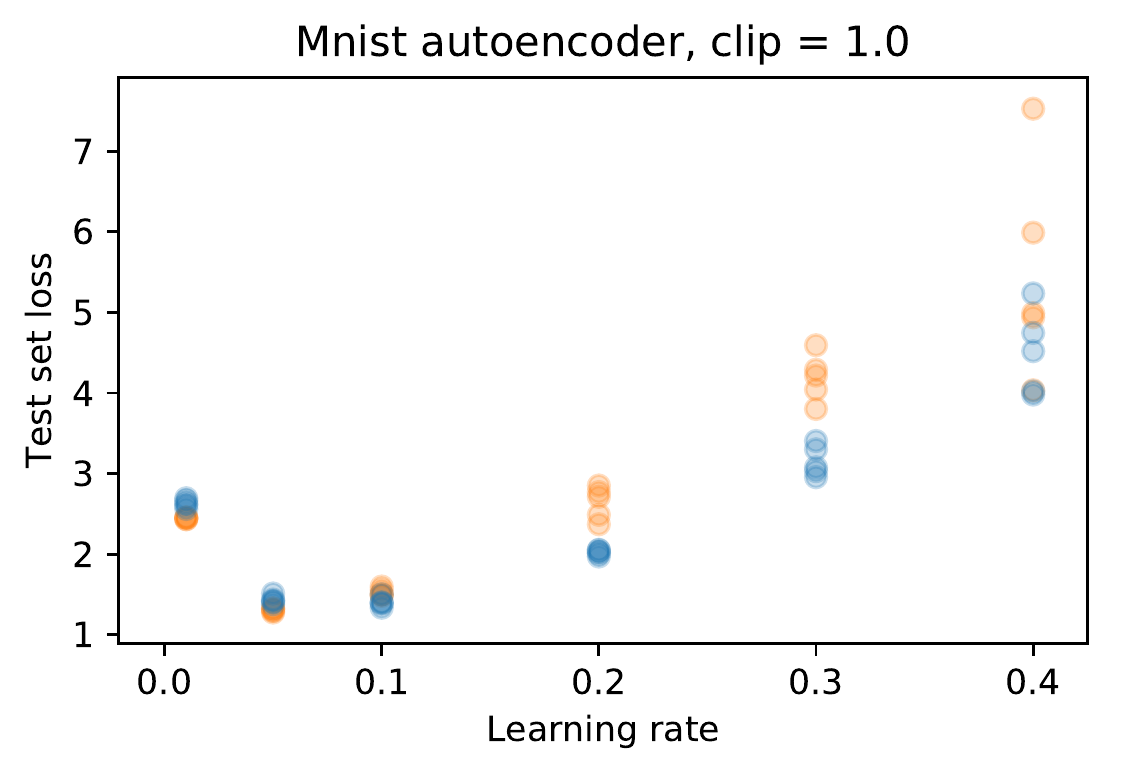}
\end{minipage}
\end{figure}

\clearpage
\subsection{JSB Chorales}

\begin{figure}[h]
\centering
\begin{minipage}{.49\textwidth}
  \centering
  \includegraphics[width=.99\linewidth]{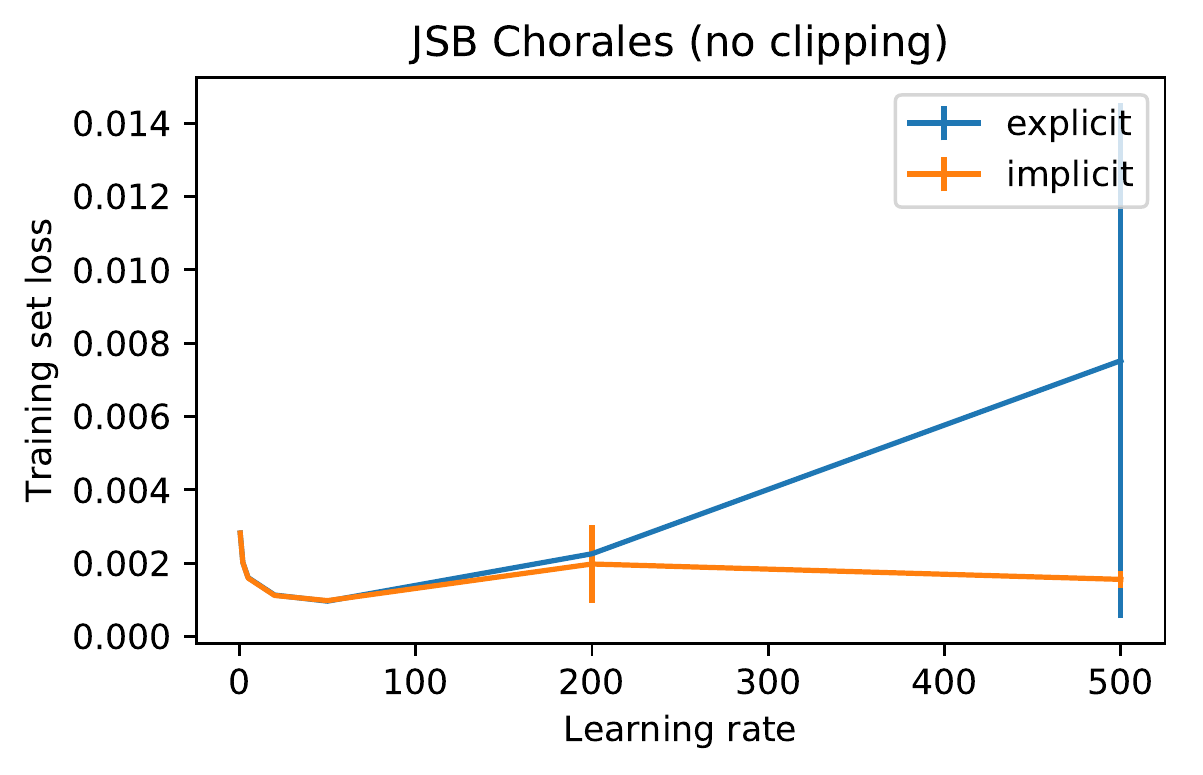}
\end{minipage}%
\hfill
\begin{minipage}{.49\textwidth}
  \centering
  \includegraphics[width=.99\linewidth]{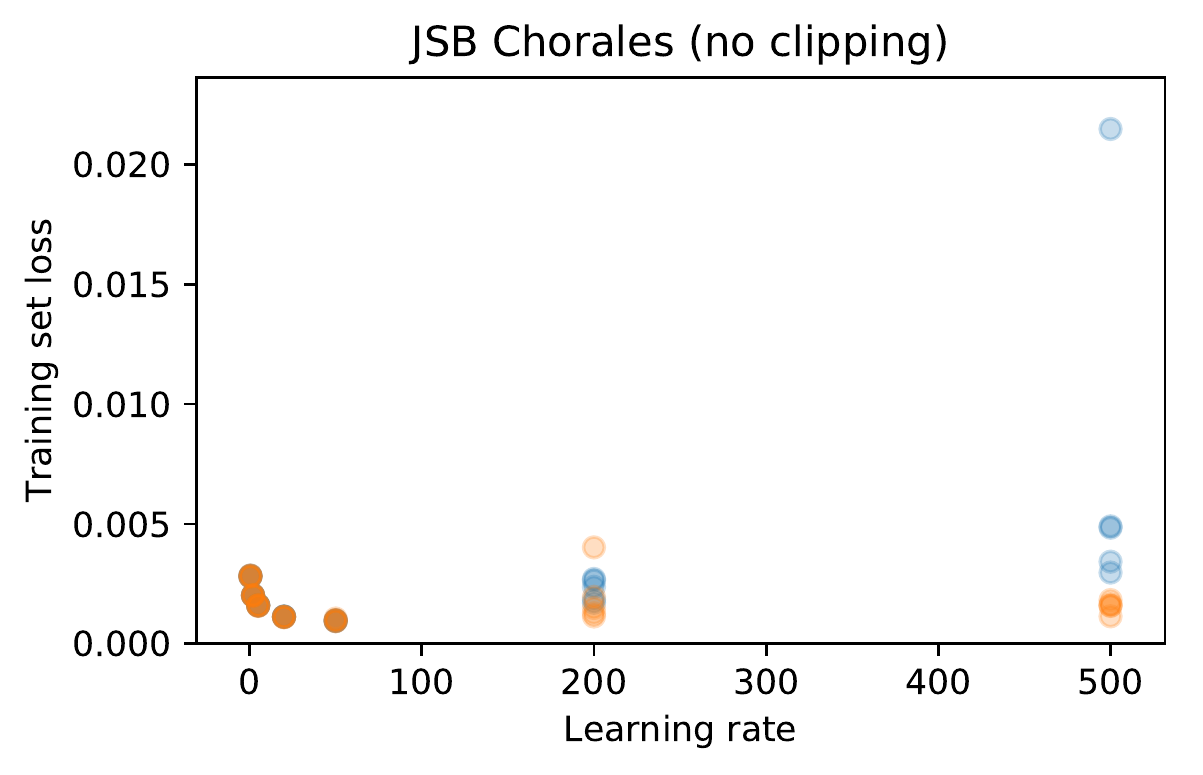}
\end{minipage}
\begin{minipage}{.49\textwidth}
  \centering
  \includegraphics[width=.99\linewidth]{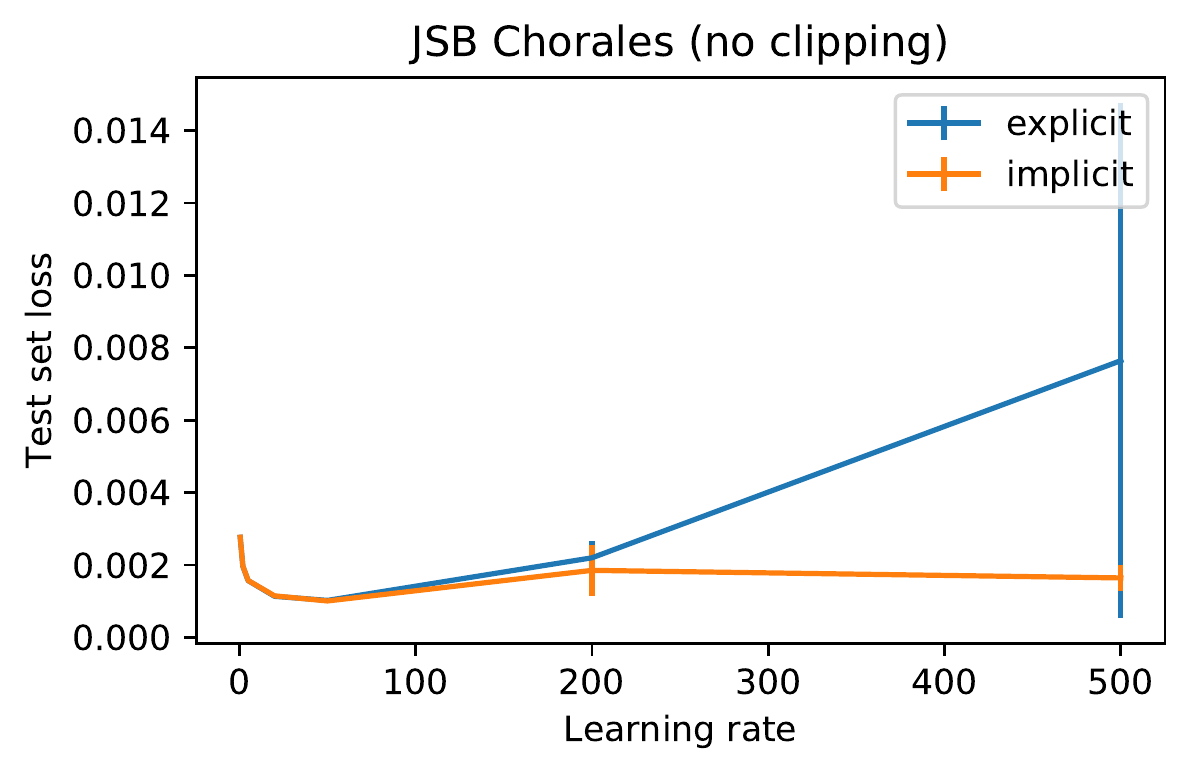}
\end{minipage}%
\hfill
\begin{minipage}{.49\textwidth}
  \centering
  \includegraphics[width=.99\linewidth]{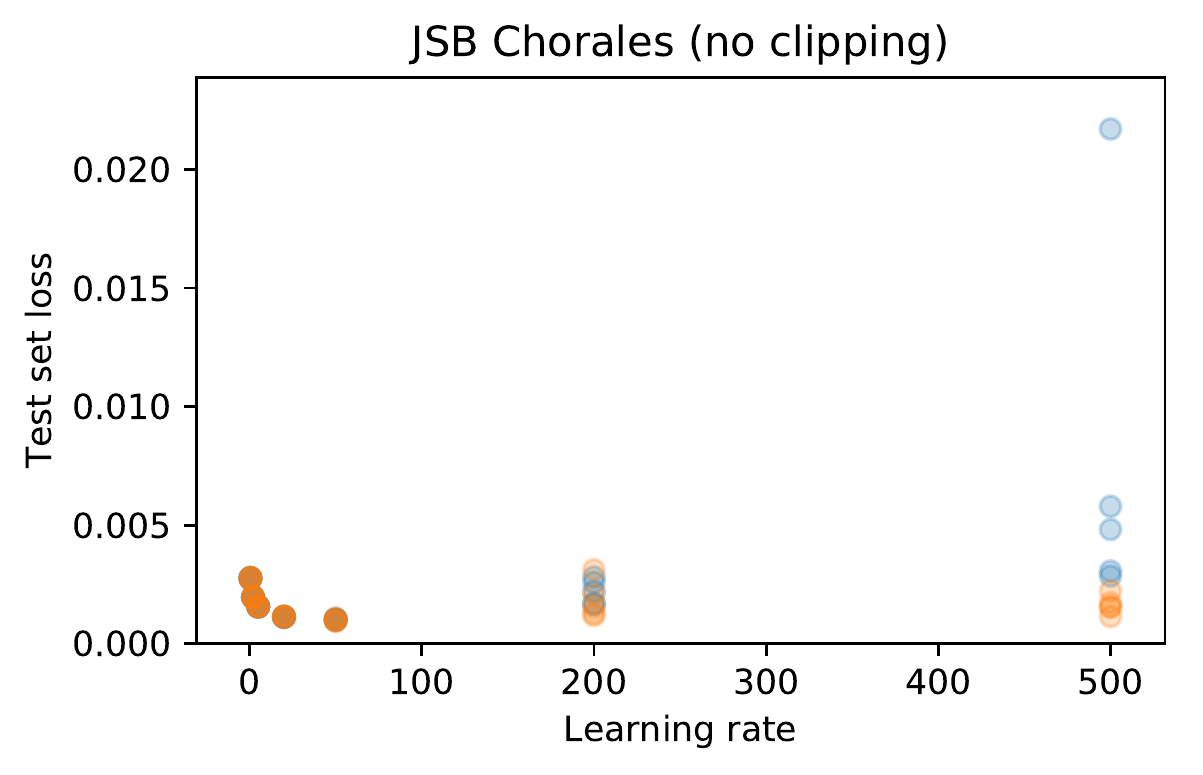}
\end{minipage}
\end{figure}

\begin{figure}[h]
\centering
\begin{minipage}{.49\textwidth}
  \centering
  \includegraphics[width=.99\linewidth]{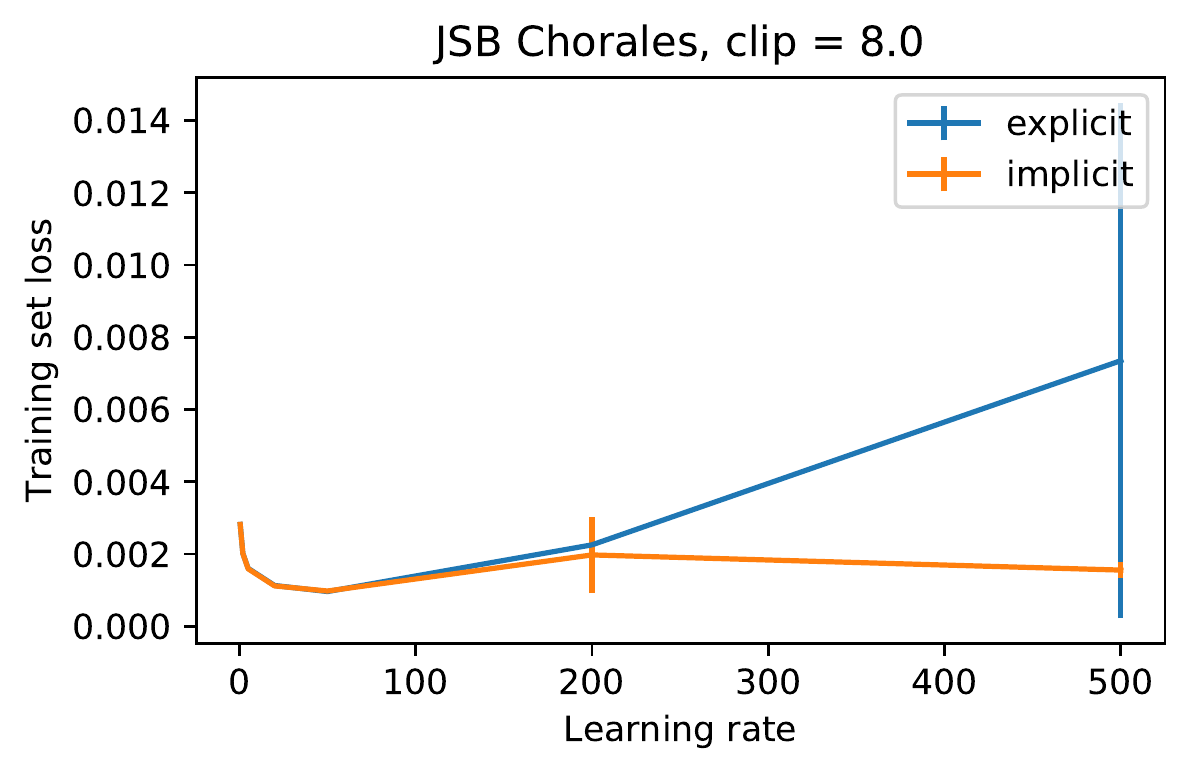}
\end{minipage}%
\hfill
\begin{minipage}{.49\textwidth}
  \centering
  \includegraphics[width=.99\linewidth]{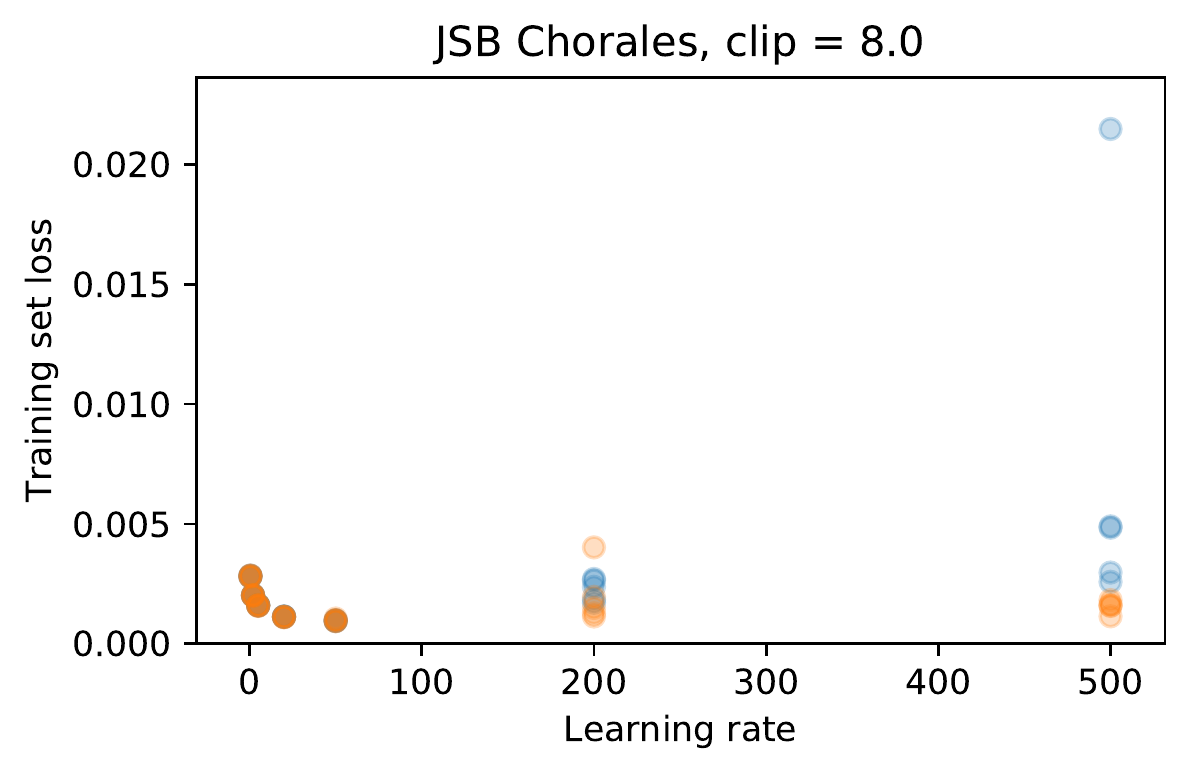}
\end{minipage}
\begin{minipage}{.49\textwidth}
  \centering
  \includegraphics[width=.99\linewidth]{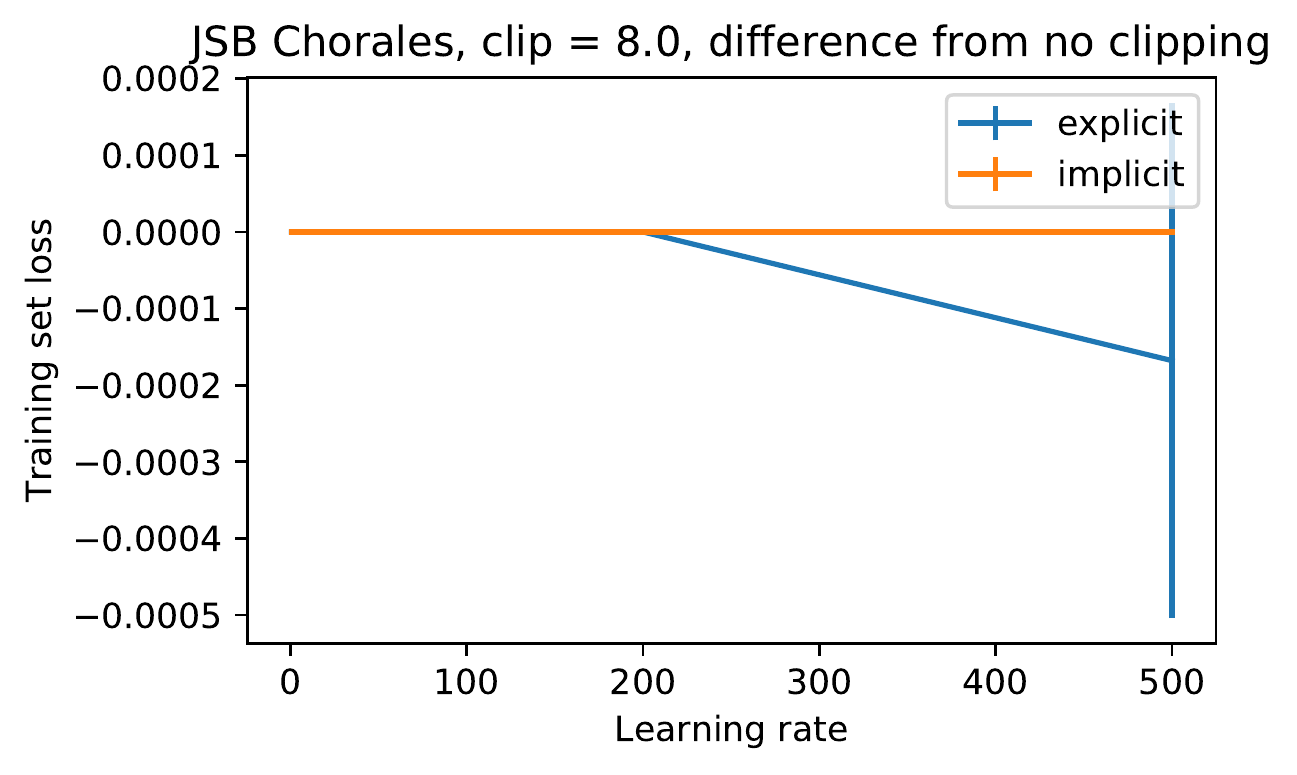}
\end{minipage}%
\hfill
\begin{minipage}{.49\textwidth}
  \centering
  \includegraphics[width=.99\linewidth]{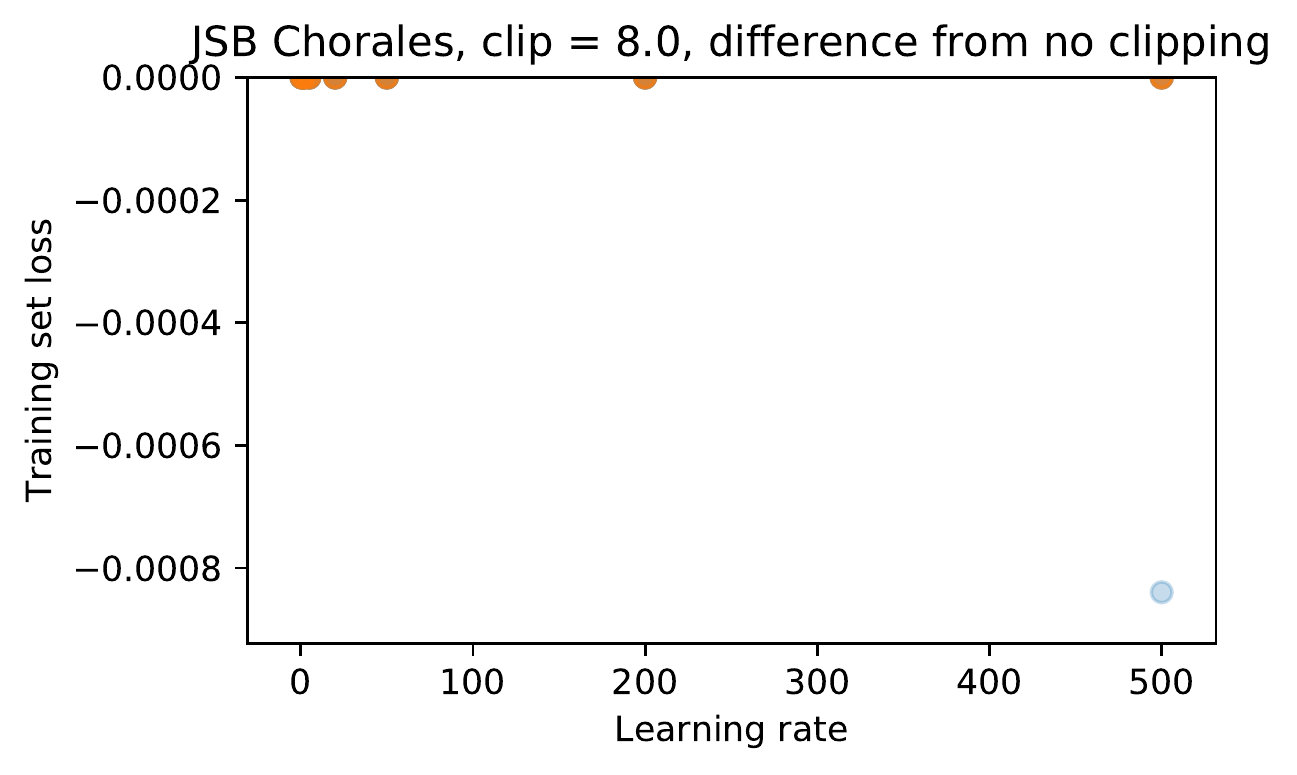}
\end{minipage}

\vspace{1cm}

\begin{minipage}{.49\textwidth}
  \centering
  \includegraphics[width=.99\linewidth]{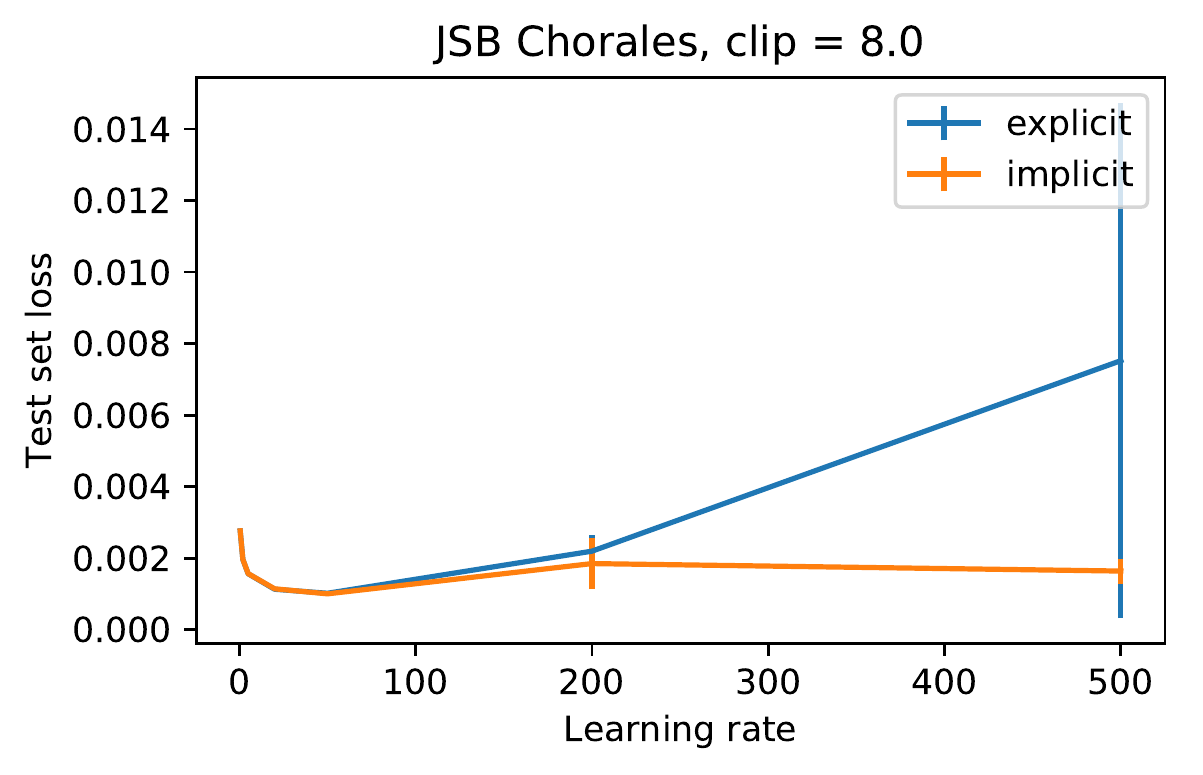}
\end{minipage}%
\hfill
\begin{minipage}{.49\textwidth}
  \centering
  \includegraphics[width=.99\linewidth]{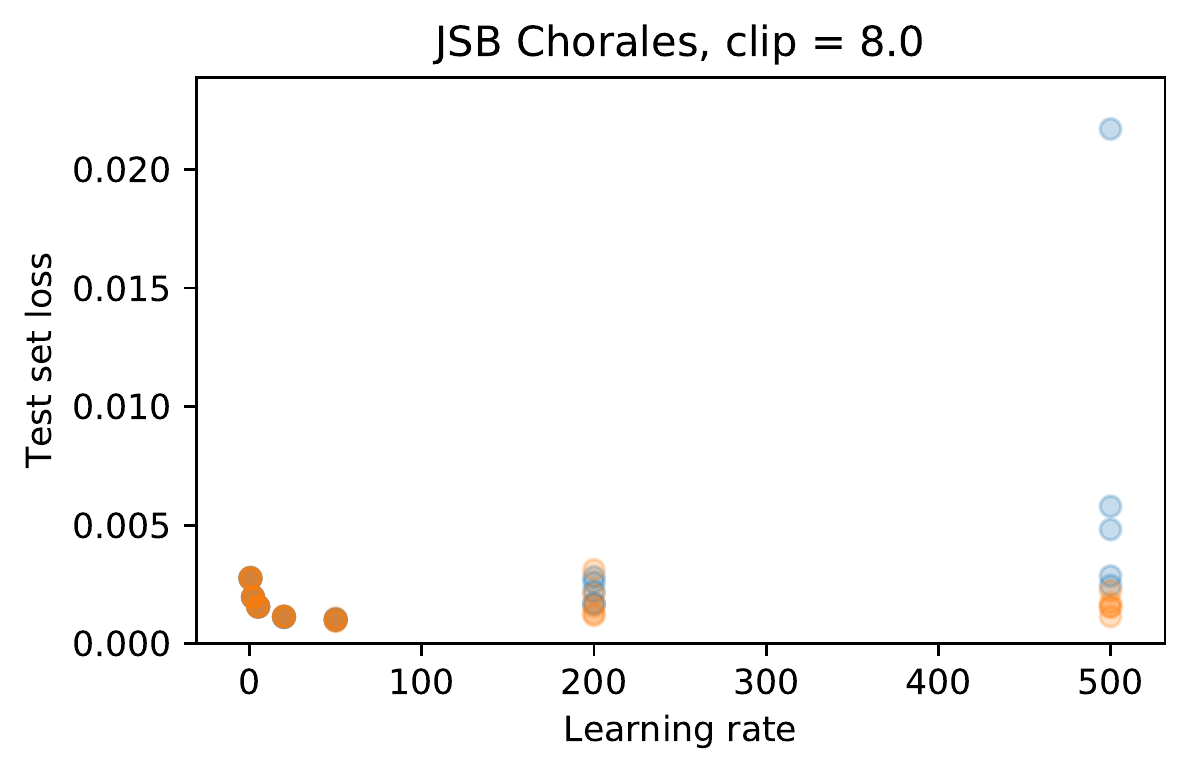}
\end{minipage}
\begin{minipage}{.49\textwidth}
  \centering
  \includegraphics[width=.99\linewidth]{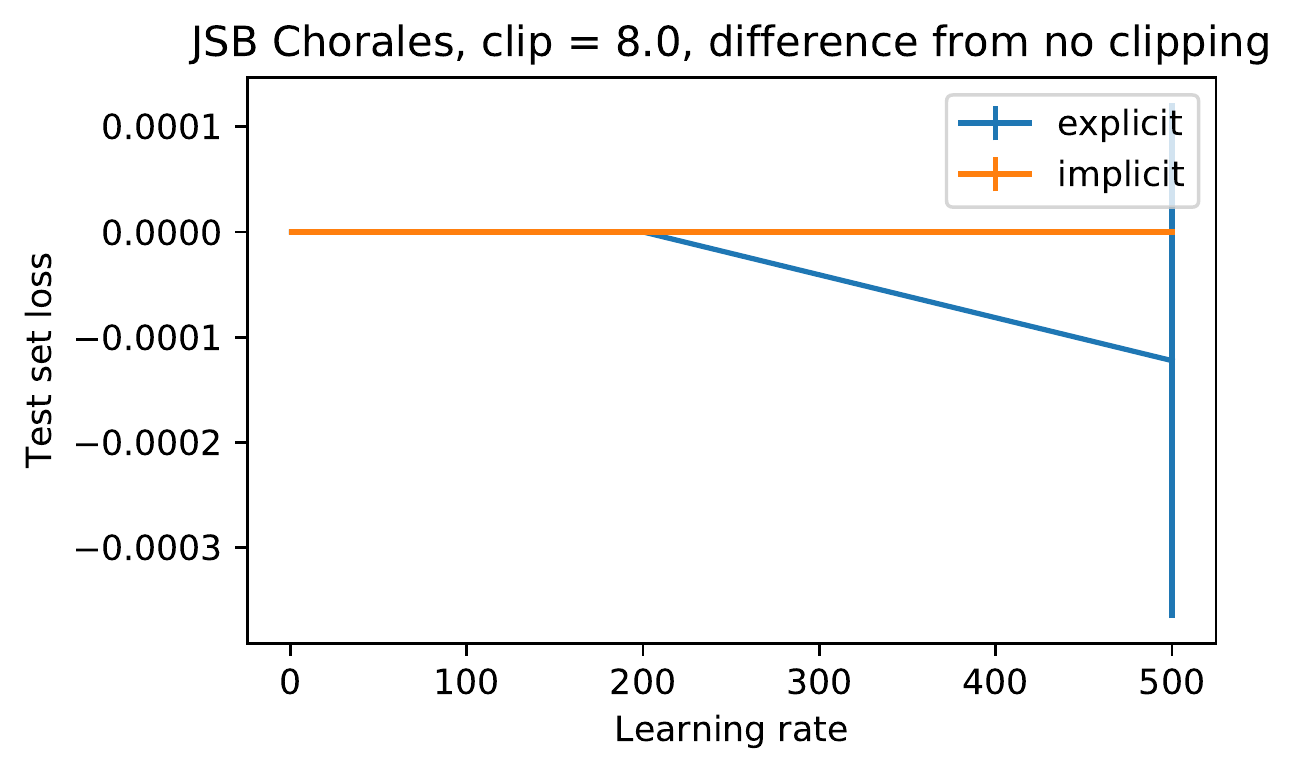}
\end{minipage}%
\hfill
\begin{minipage}{.49\textwidth}
  \centering
  \includegraphics[width=.99\linewidth]{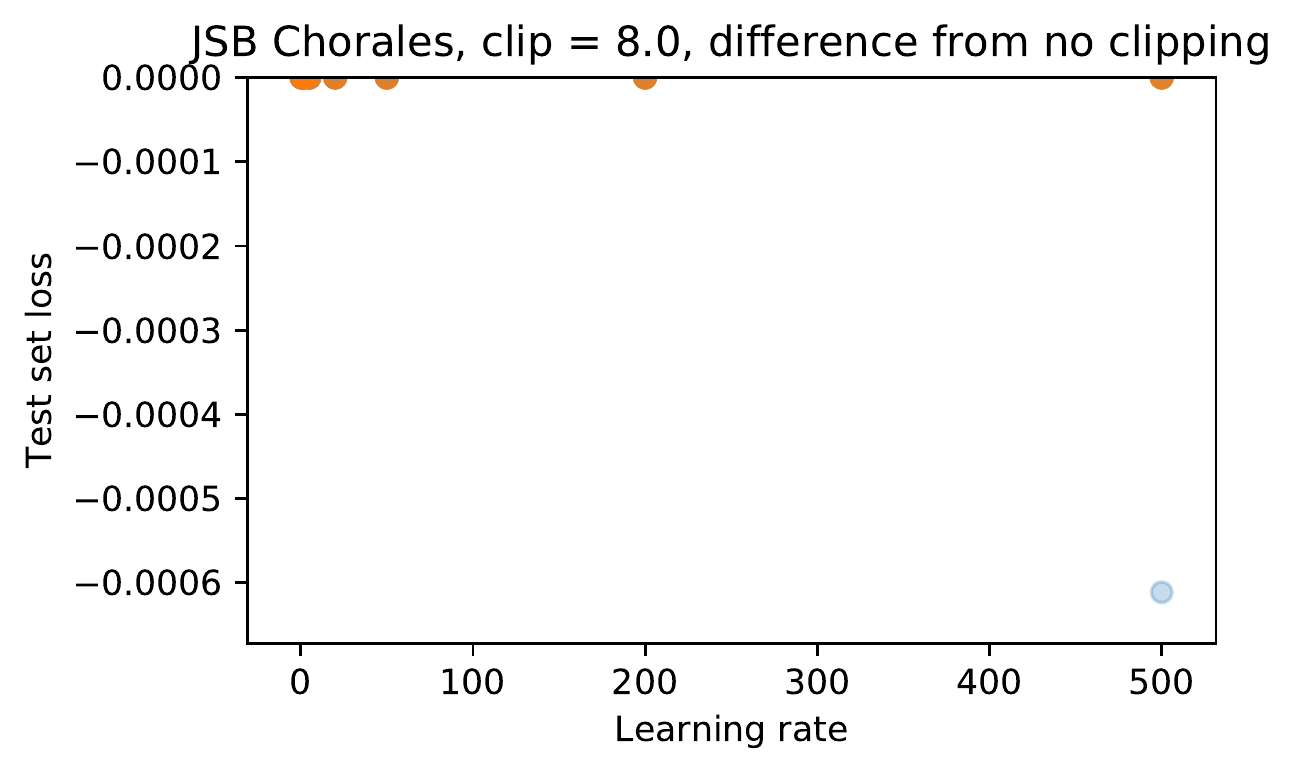}
\end{minipage}
\end{figure}

\begin{figure}[h]
\centering
\begin{minipage}{.49\textwidth}
  \centering
  \includegraphics[width=.99\linewidth]{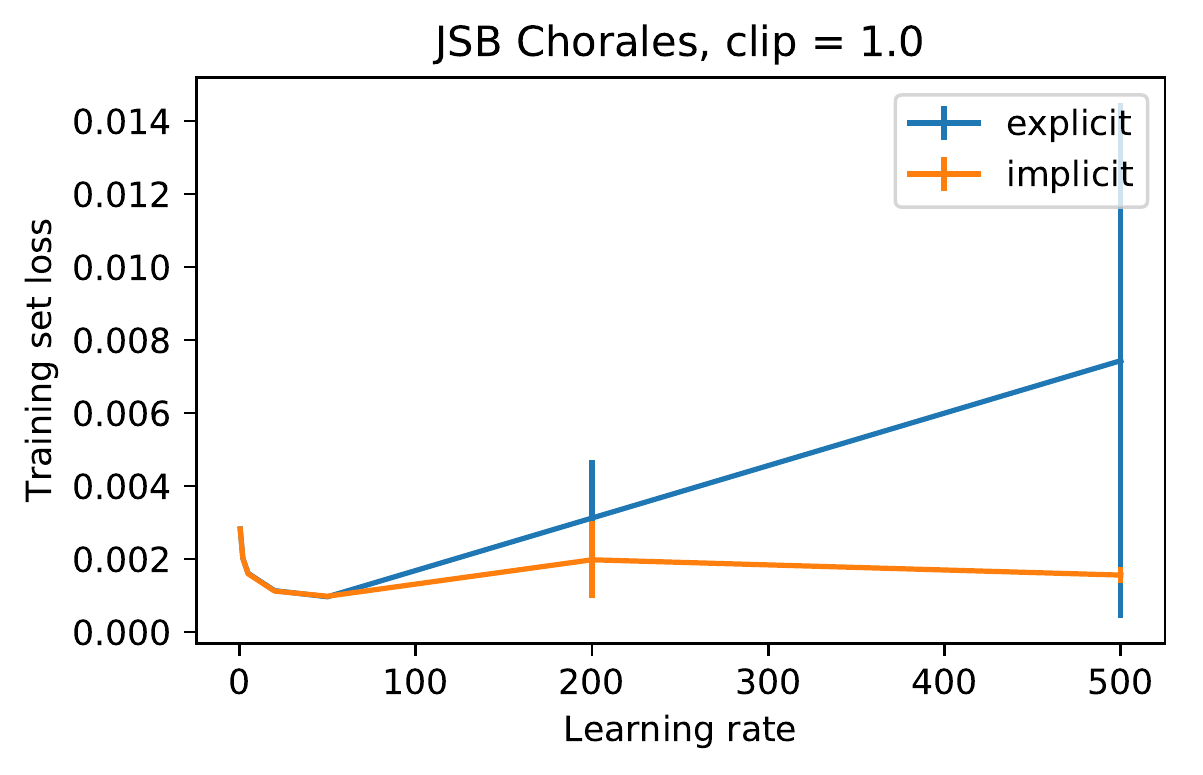}
\end{minipage}%
\hfill
\begin{minipage}{.49\textwidth}
  \centering
  \includegraphics[width=.99\linewidth]{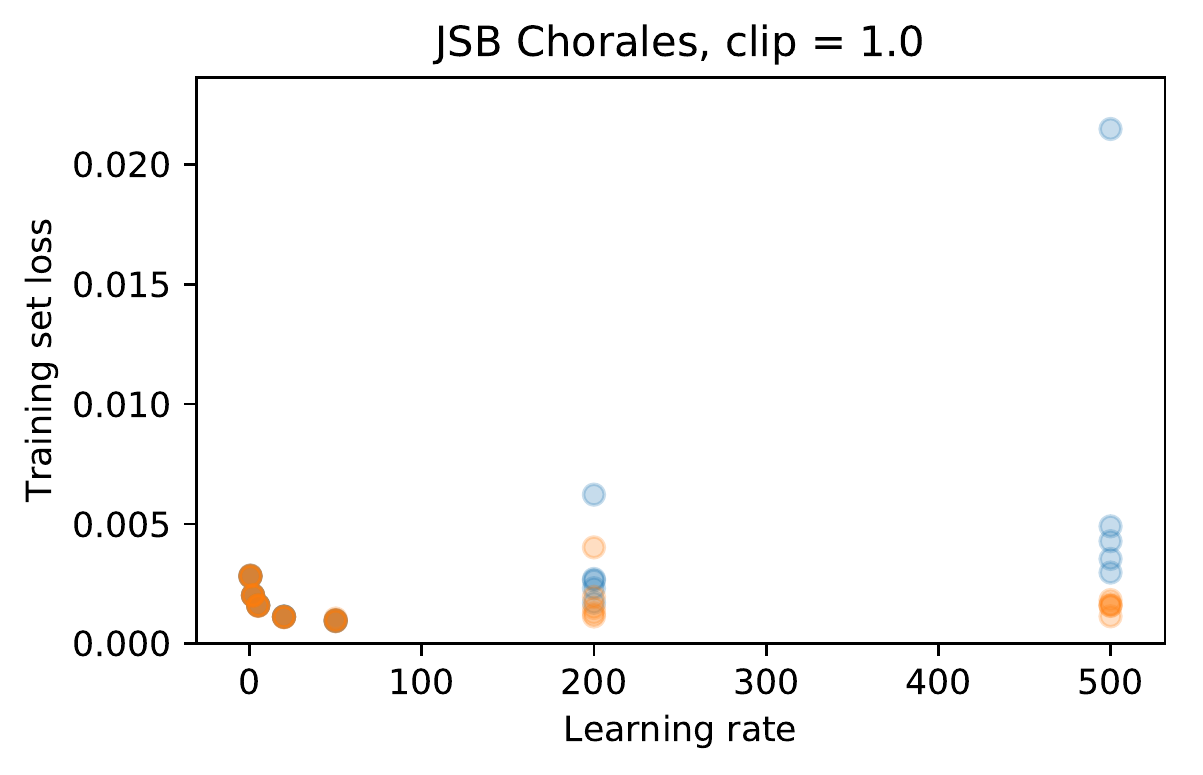}
\end{minipage}
\begin{minipage}{.49\textwidth}
  \centering
  \includegraphics[width=.99\linewidth]{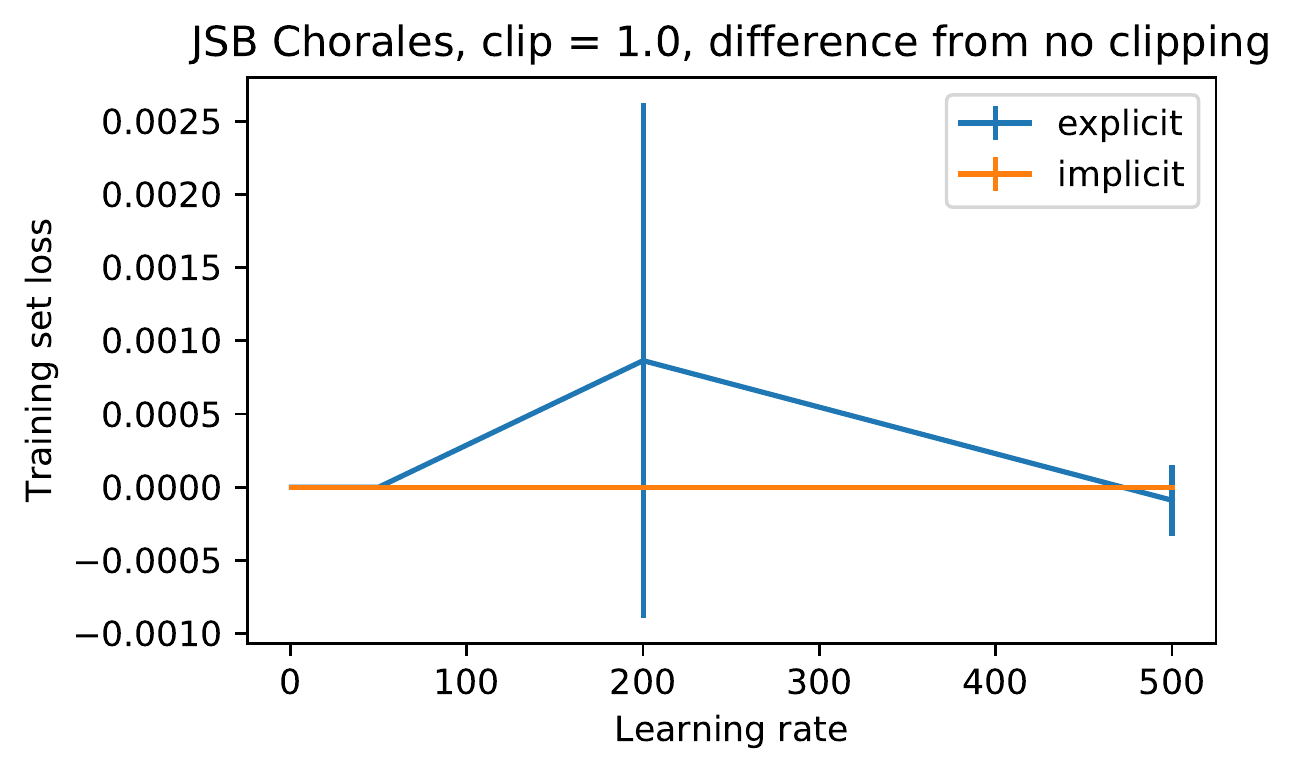}
\end{minipage}%
\hfill
\begin{minipage}{.49\textwidth}
  \centering
  \includegraphics[width=.99\linewidth]{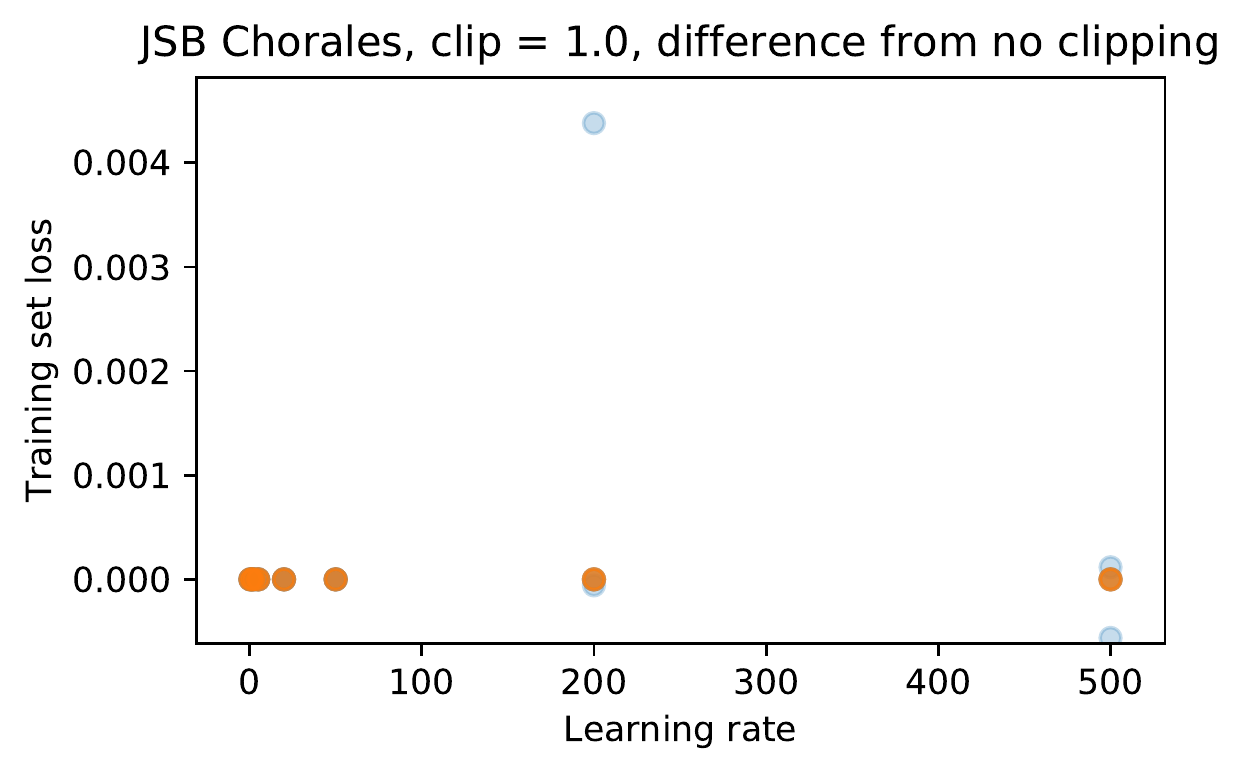}
\end{minipage}

\vspace{1cm}

\begin{minipage}{.49\textwidth}
  \centering
  \includegraphics[width=.99\linewidth]{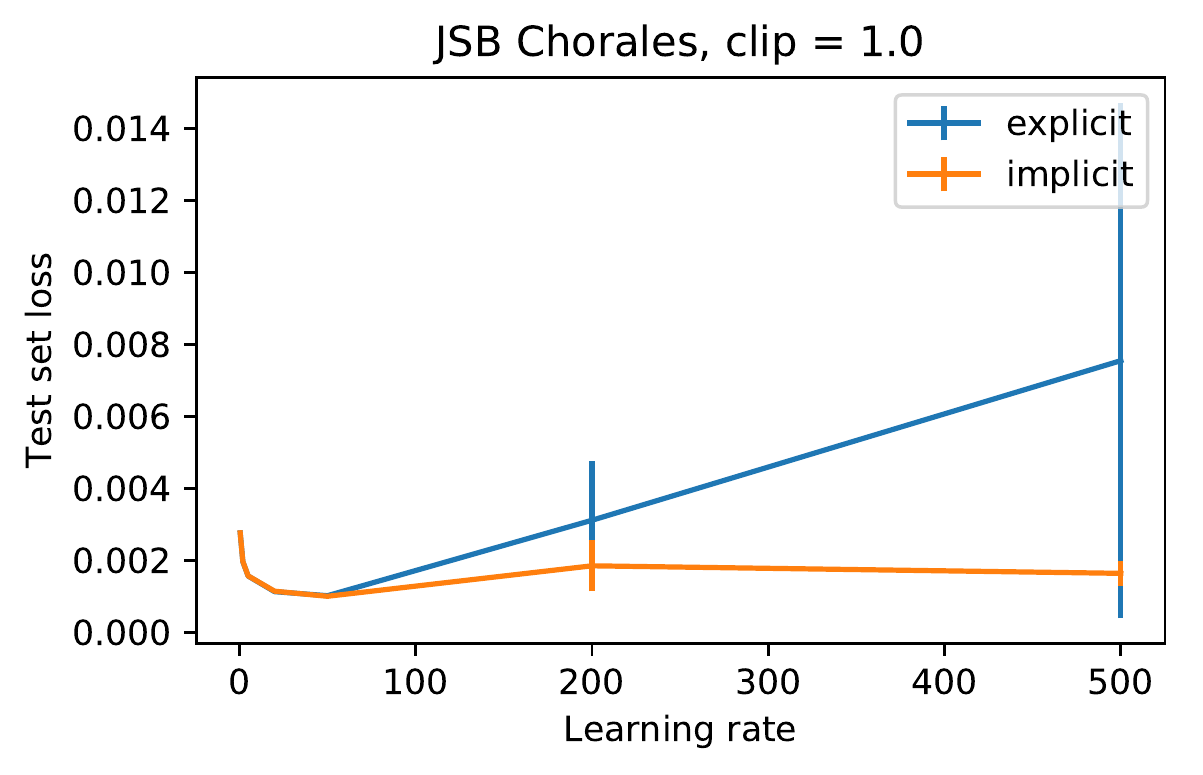}
\end{minipage}%
\hfill
\begin{minipage}{.49\textwidth}
  \centering
  \includegraphics[width=.99\linewidth]{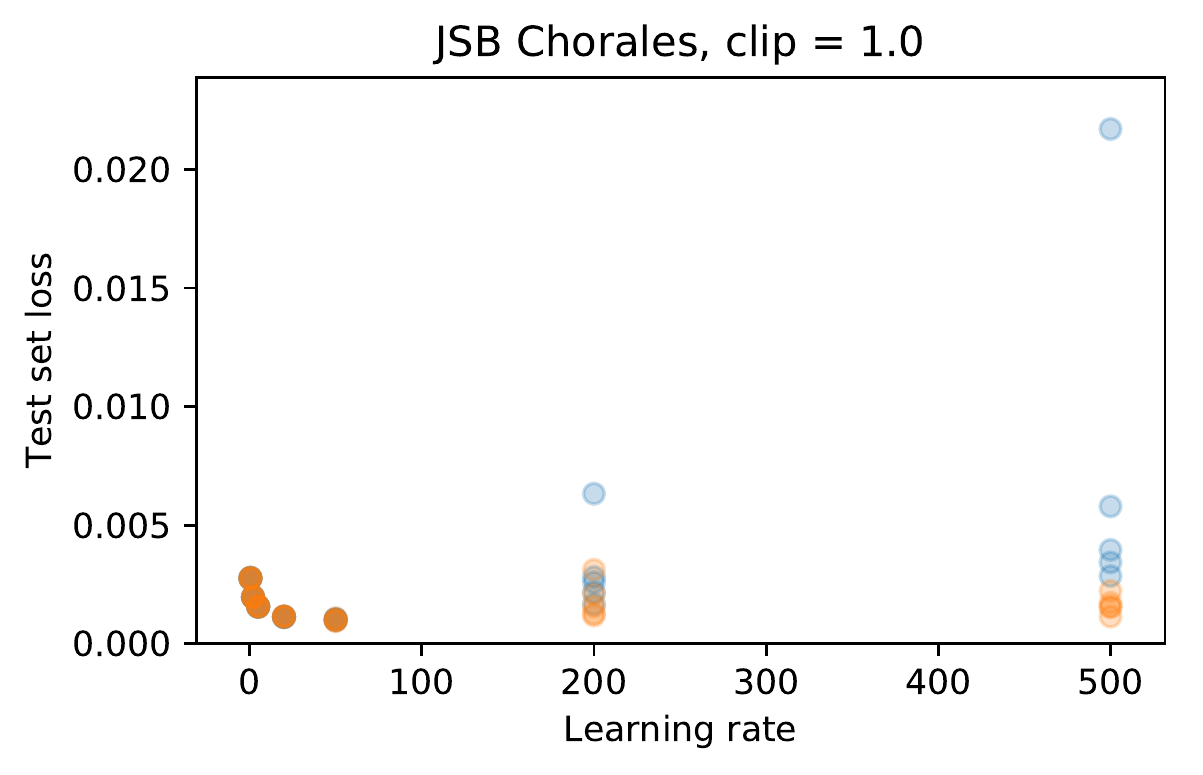}
\end{minipage}
\begin{minipage}{.49\textwidth}
  \centering
  \includegraphics[width=.99\linewidth]{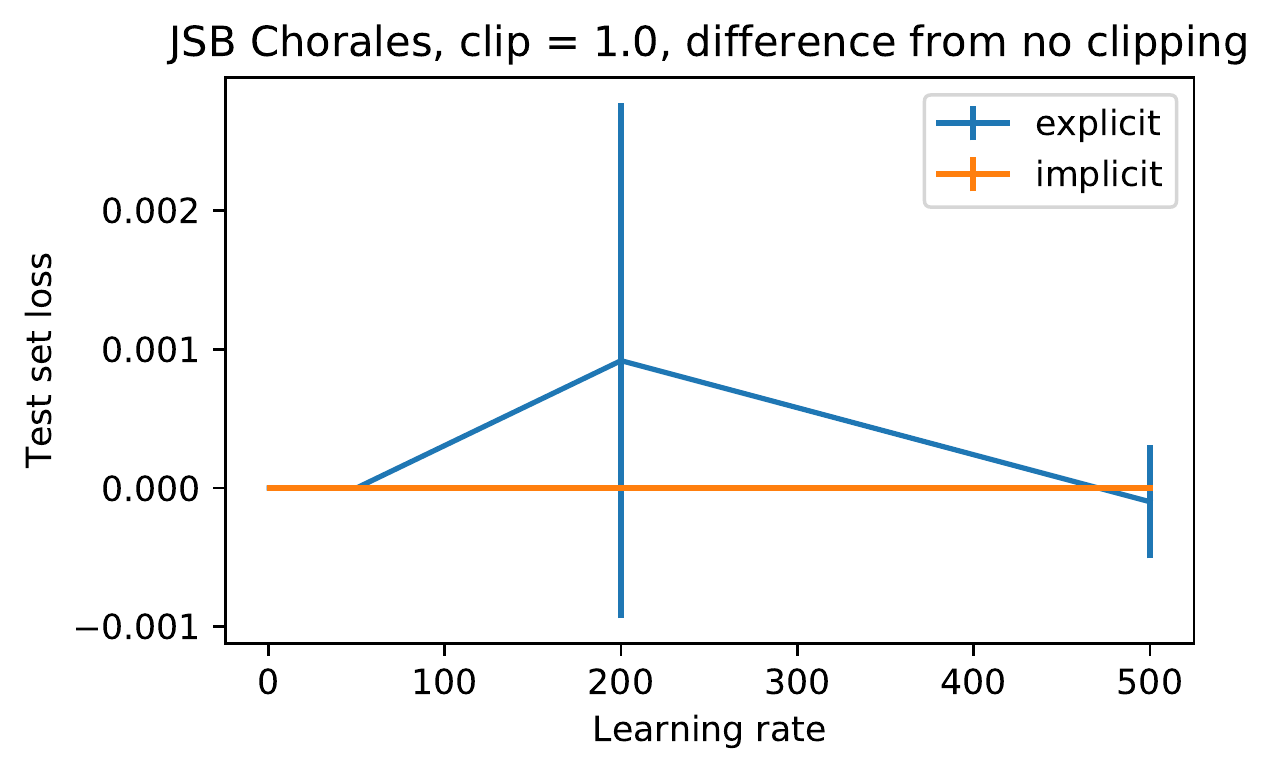}
\end{minipage}%
\hfill
\begin{minipage}{.49\textwidth}
  \centering
  \includegraphics[width=.99\linewidth]{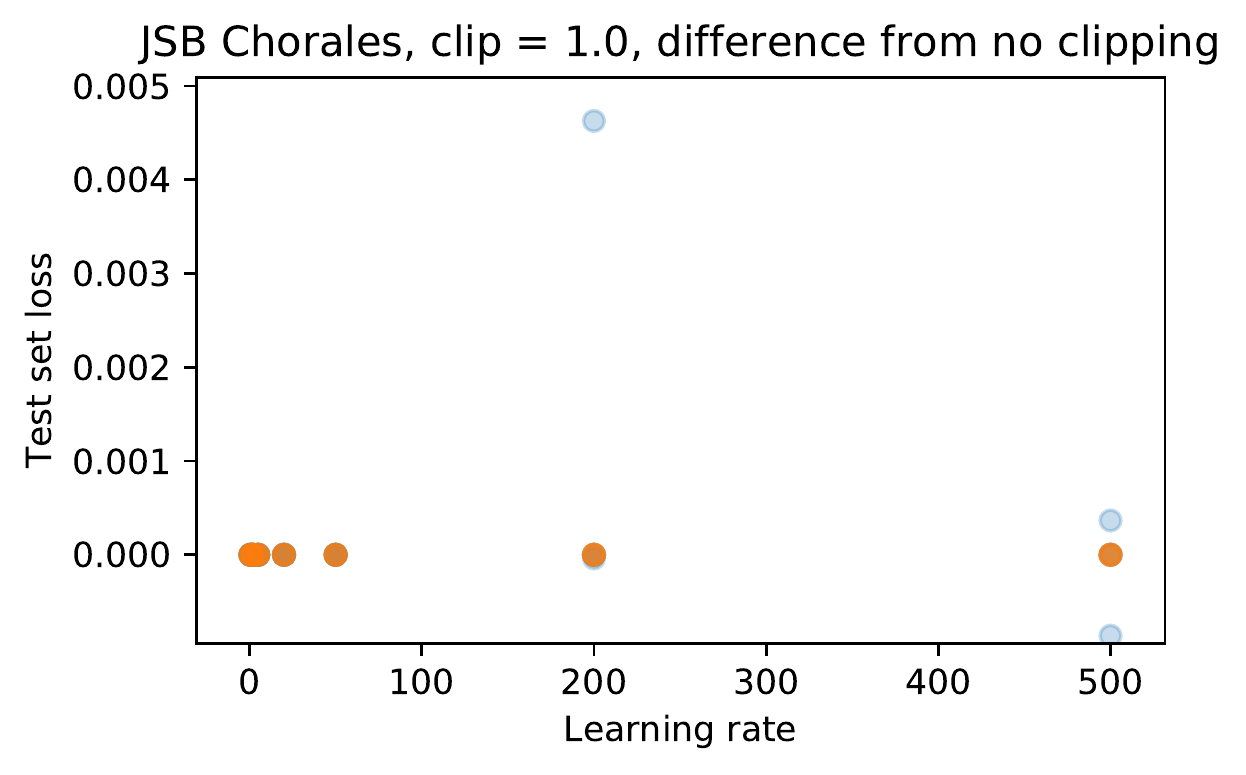}
\end{minipage}
\end{figure}

\begin{figure}[h]
\centering
\begin{minipage}{.49\textwidth}
  \centering
  \includegraphics[width=.99\linewidth]{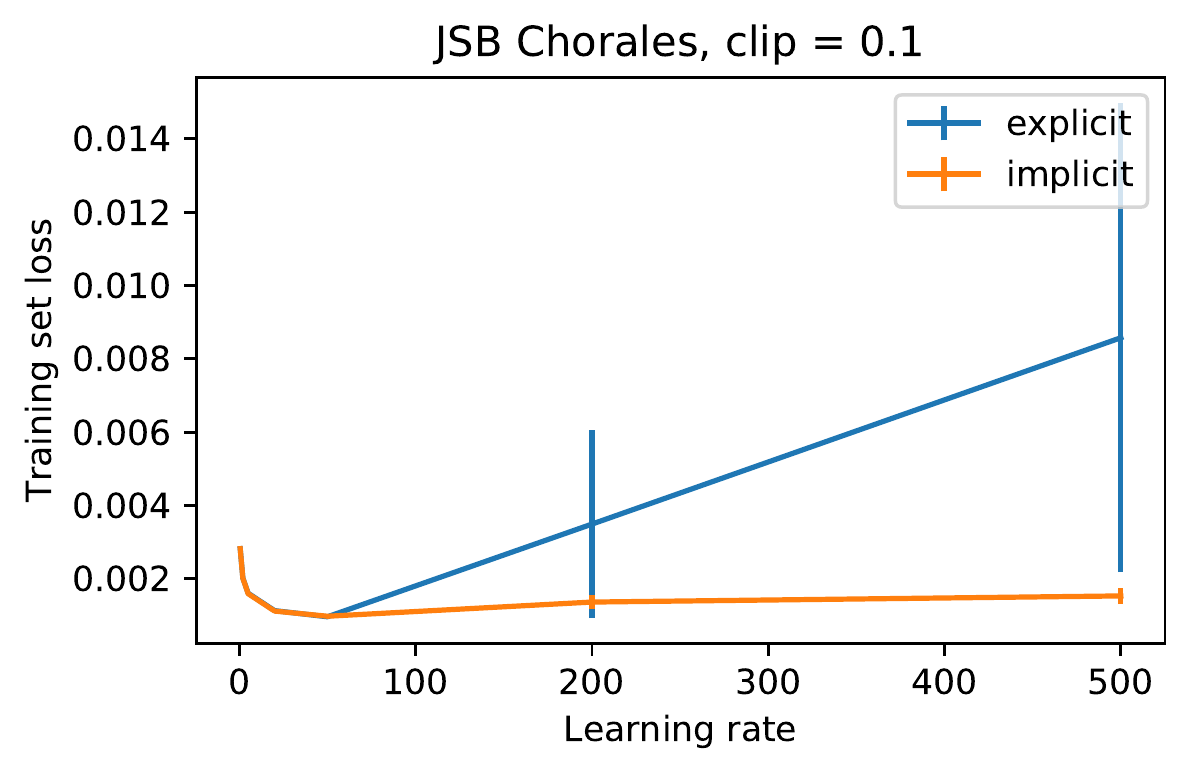}
\end{minipage}%
\hfill
\begin{minipage}{.49\textwidth}
  \centering
  \includegraphics[width=.99\linewidth]{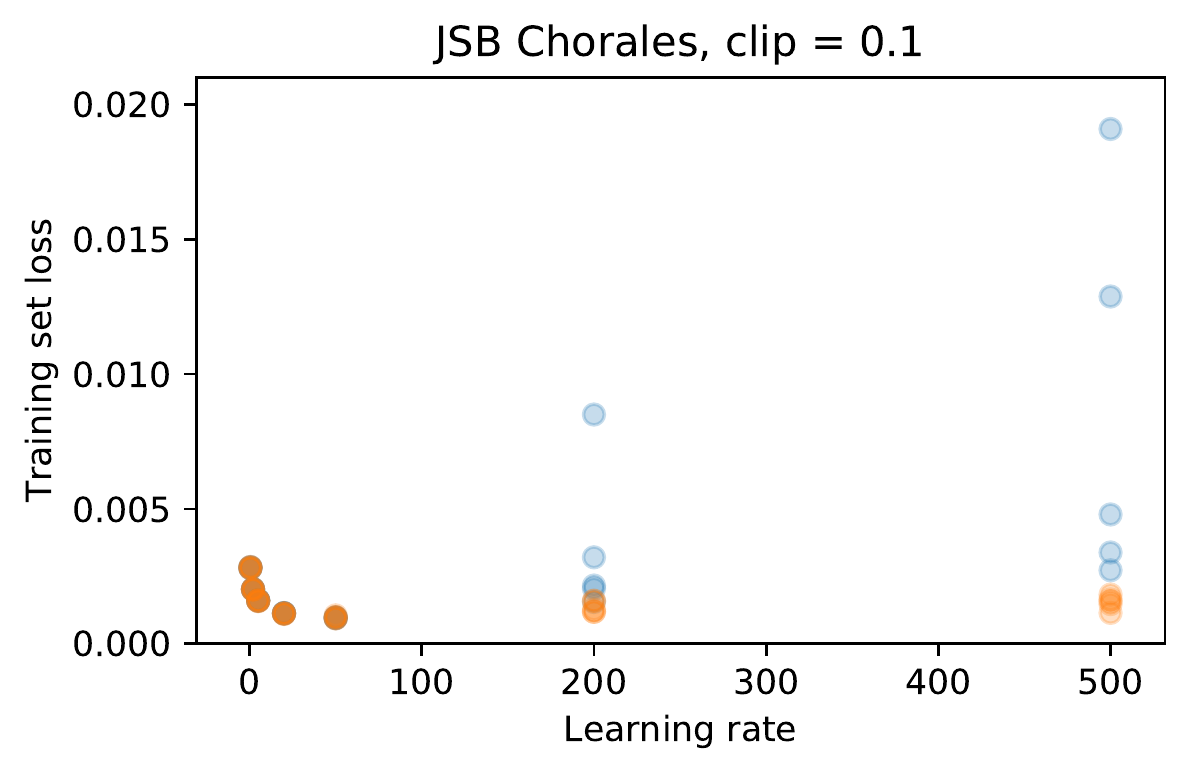}
\end{minipage}
\begin{minipage}{.49\textwidth}
  \centering
  \includegraphics[width=.99\linewidth]{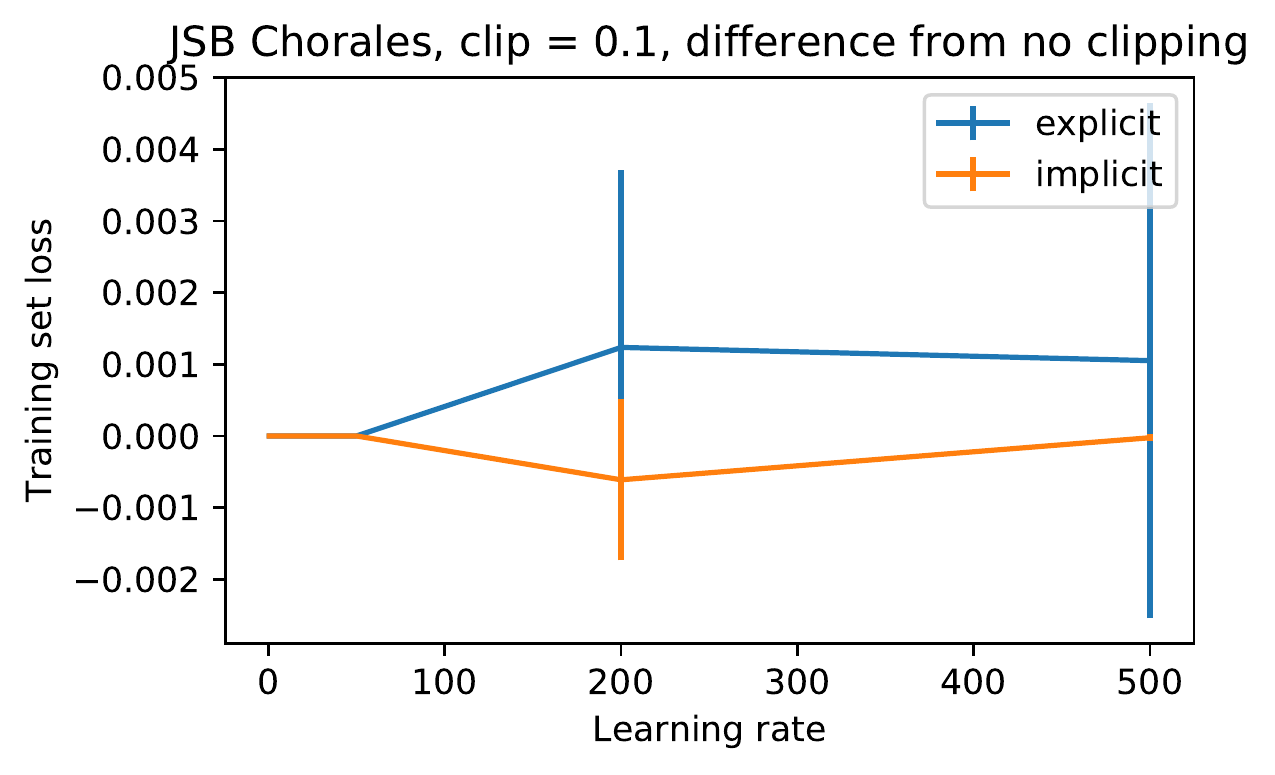}
\end{minipage}%
\hfill
\begin{minipage}{.49\textwidth}
  \centering
  \includegraphics[width=.99\linewidth]{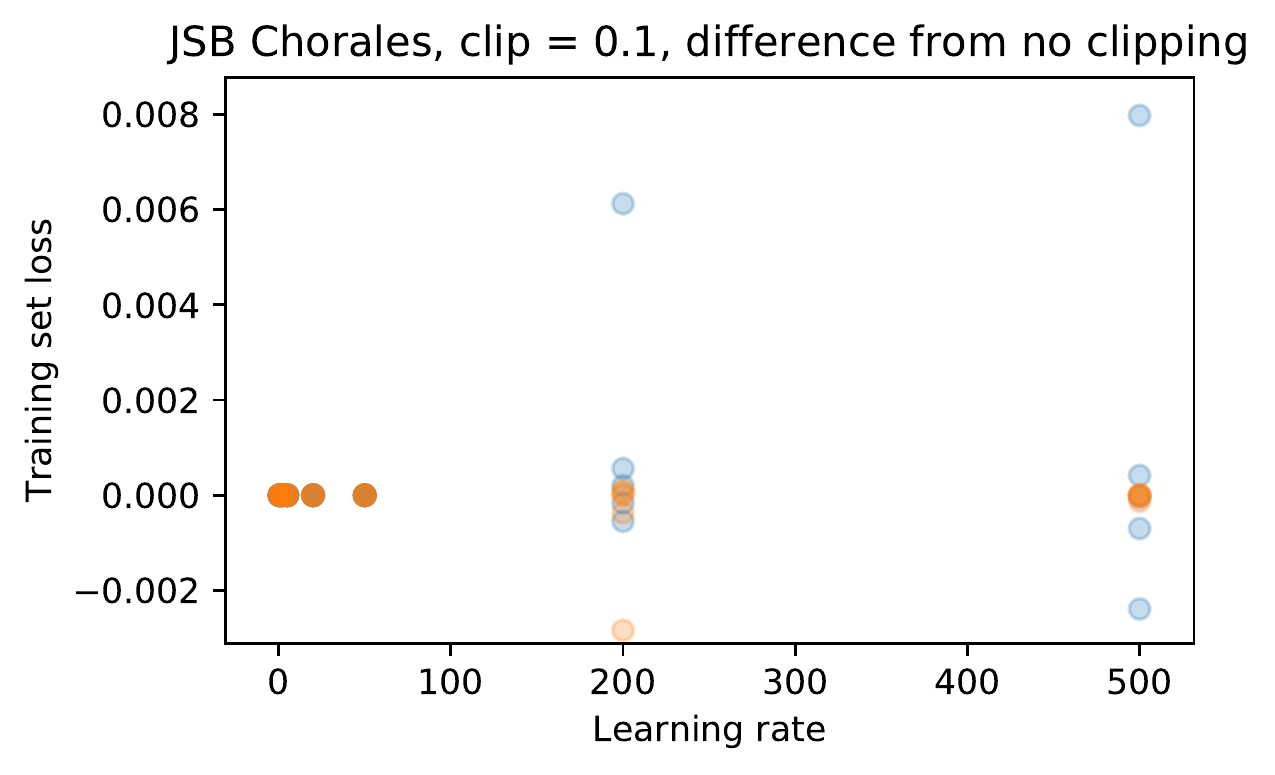}
\end{minipage}

\vspace{1cm}

\begin{minipage}{.49\textwidth}
  \centering
  \includegraphics[width=.99\linewidth]{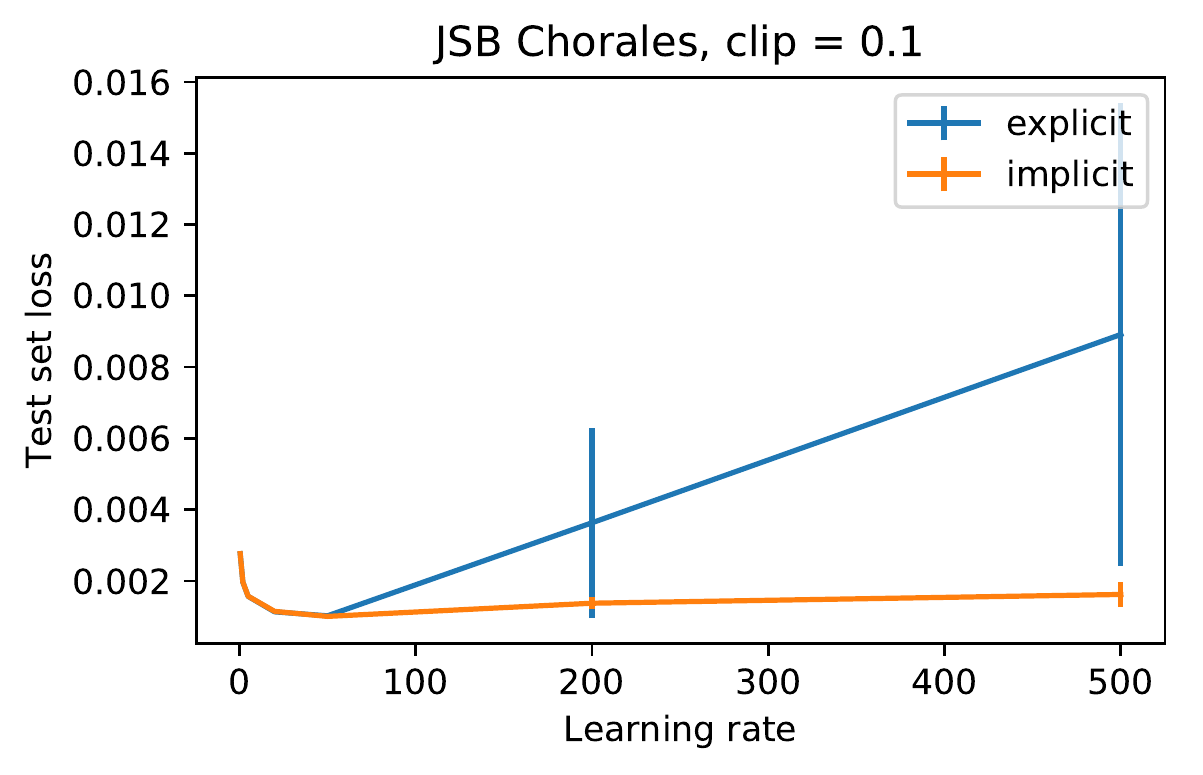}
\end{minipage}%
\hfill
\begin{minipage}{.49\textwidth}
  \centering
  \includegraphics[width=.99\linewidth]{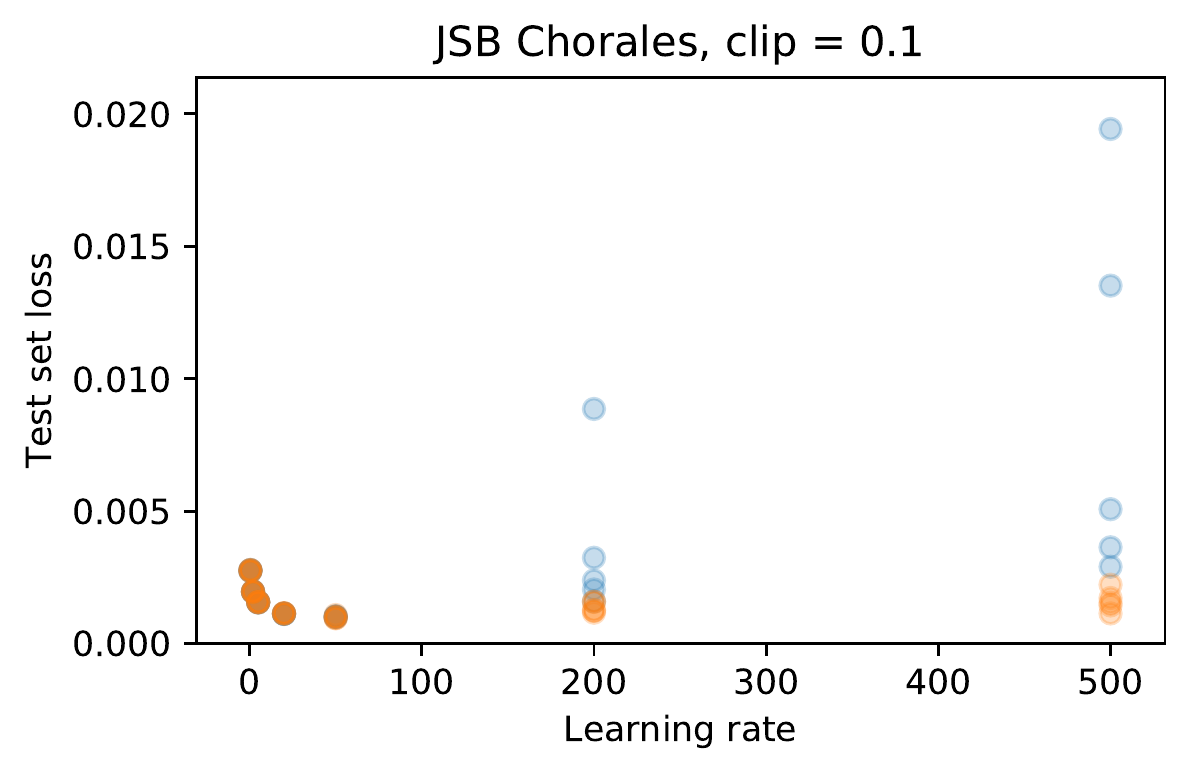}
\end{minipage}
\begin{minipage}{.49\textwidth}
  \centering
  \includegraphics[width=.99\linewidth]{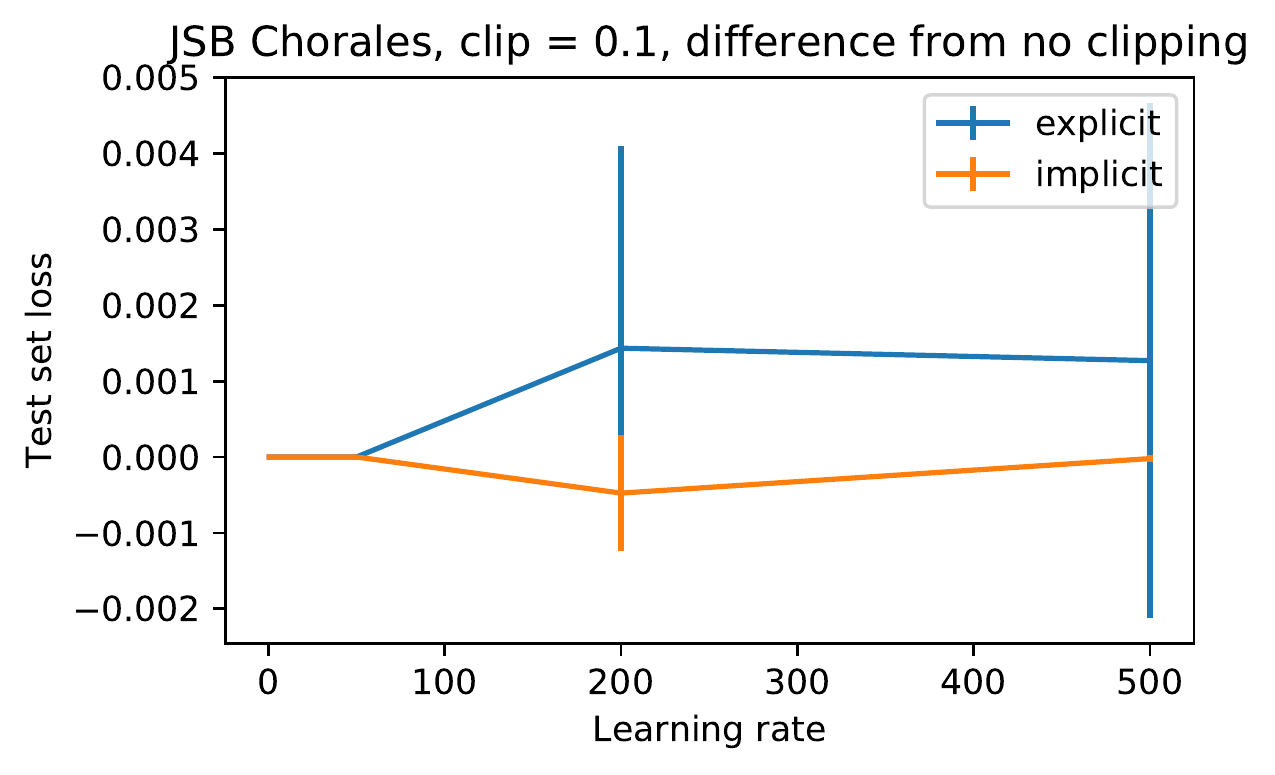}
\end{minipage}%
\hfill
\begin{minipage}{.49\textwidth}
  \centering
  \includegraphics[width=.99\linewidth]{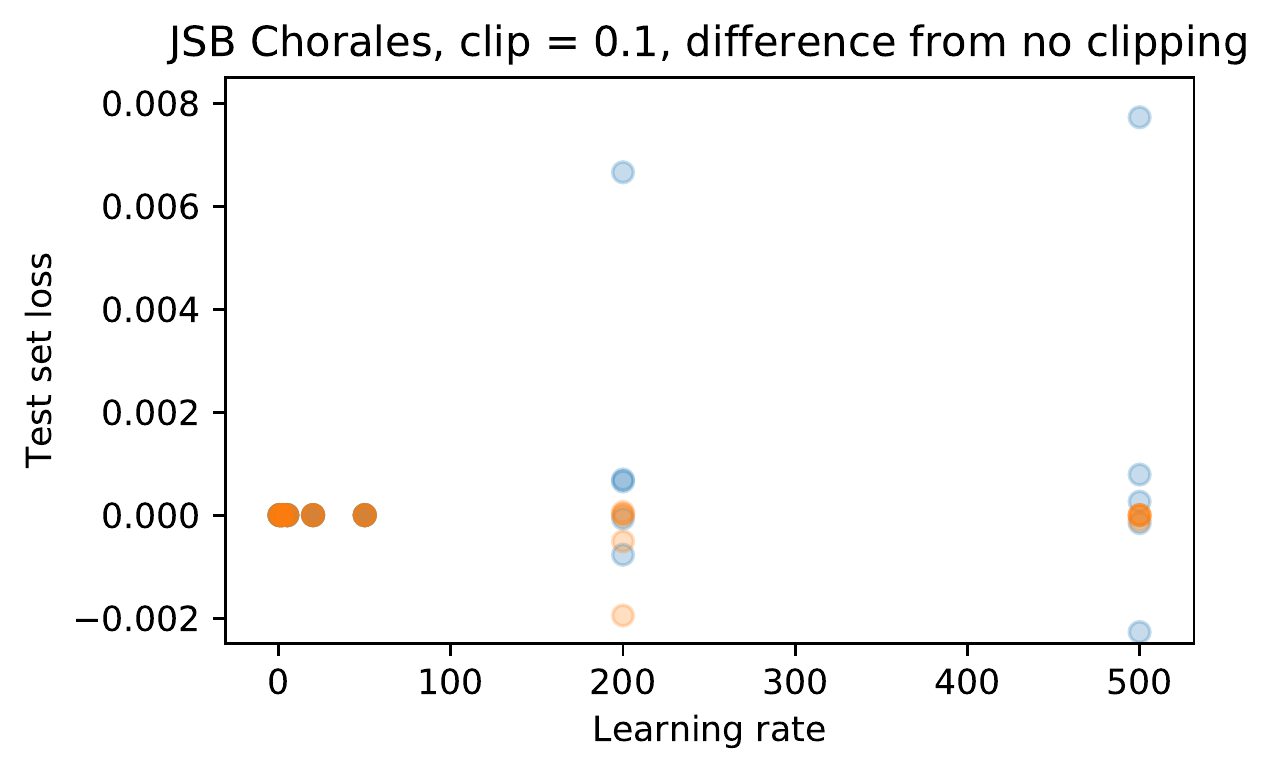}
\end{minipage}
\end{figure}

\clearpage
\subsection{MuseData}

\begin{figure}[h]
\centering
\begin{minipage}{.49\textwidth}
  \centering
  \includegraphics[width=.99\linewidth]{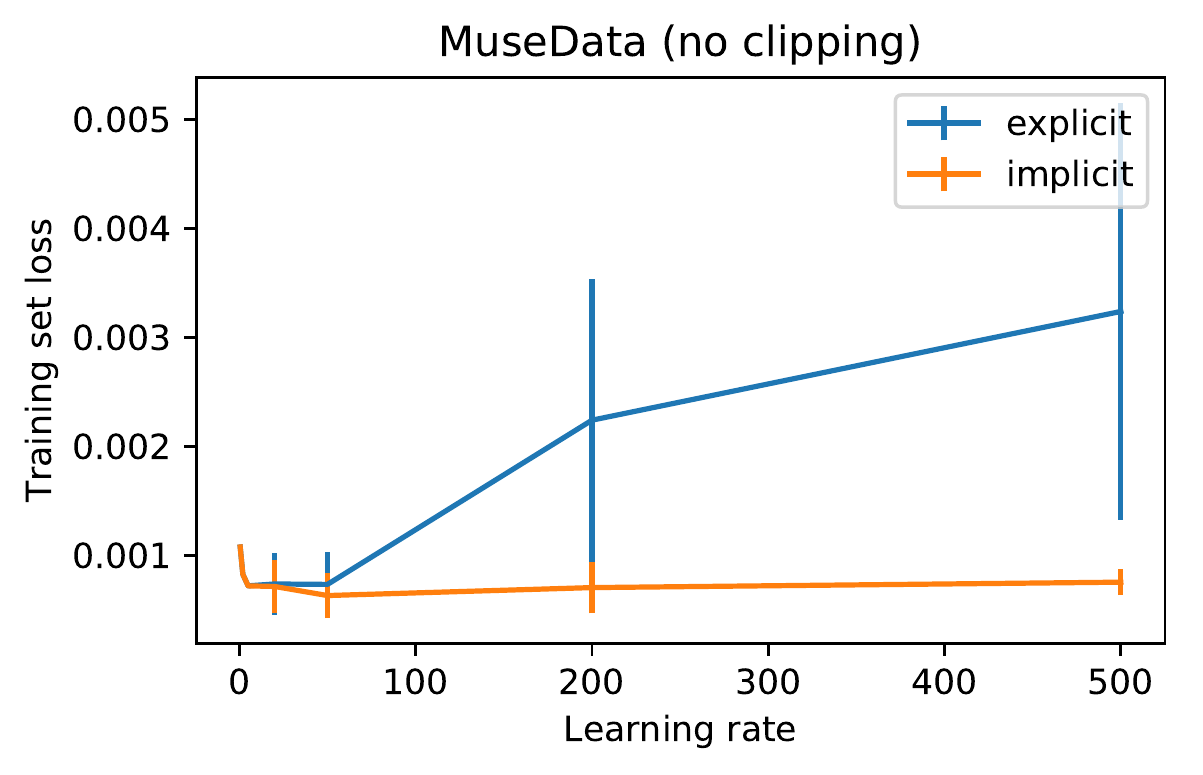}
\end{minipage}%
\hfill
\begin{minipage}{.49\textwidth}
  \centering
  \includegraphics[width=.99\linewidth]{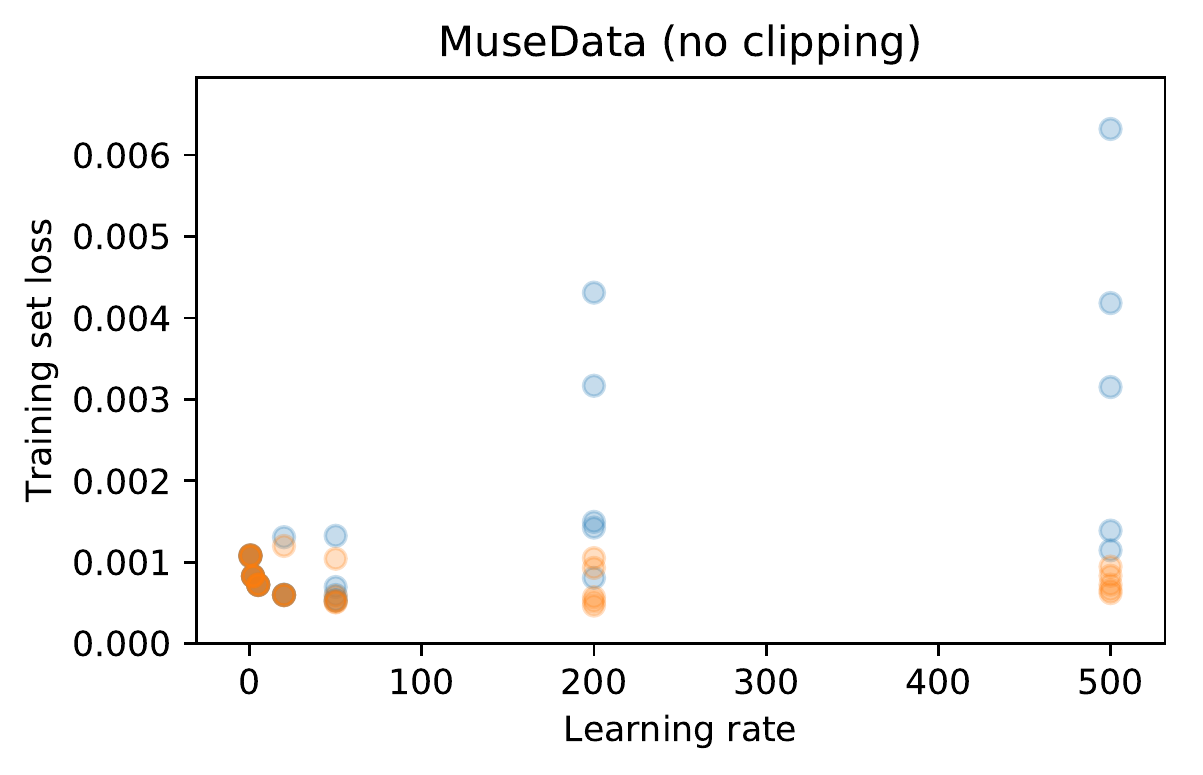}
\end{minipage}
\begin{minipage}{.49\textwidth}
  \centering
  \includegraphics[width=.99\linewidth]{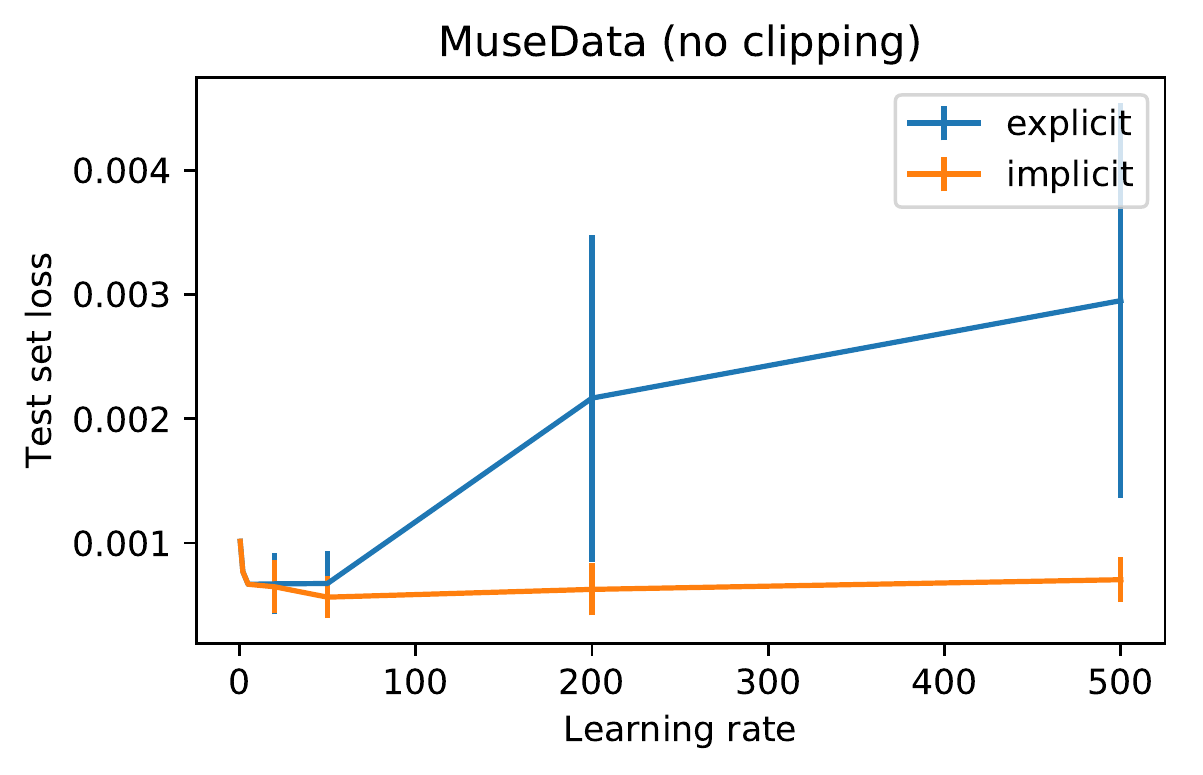}
\end{minipage}%
\hfill
\begin{minipage}{.49\textwidth}
  \centering
  \includegraphics[width=.99\linewidth]{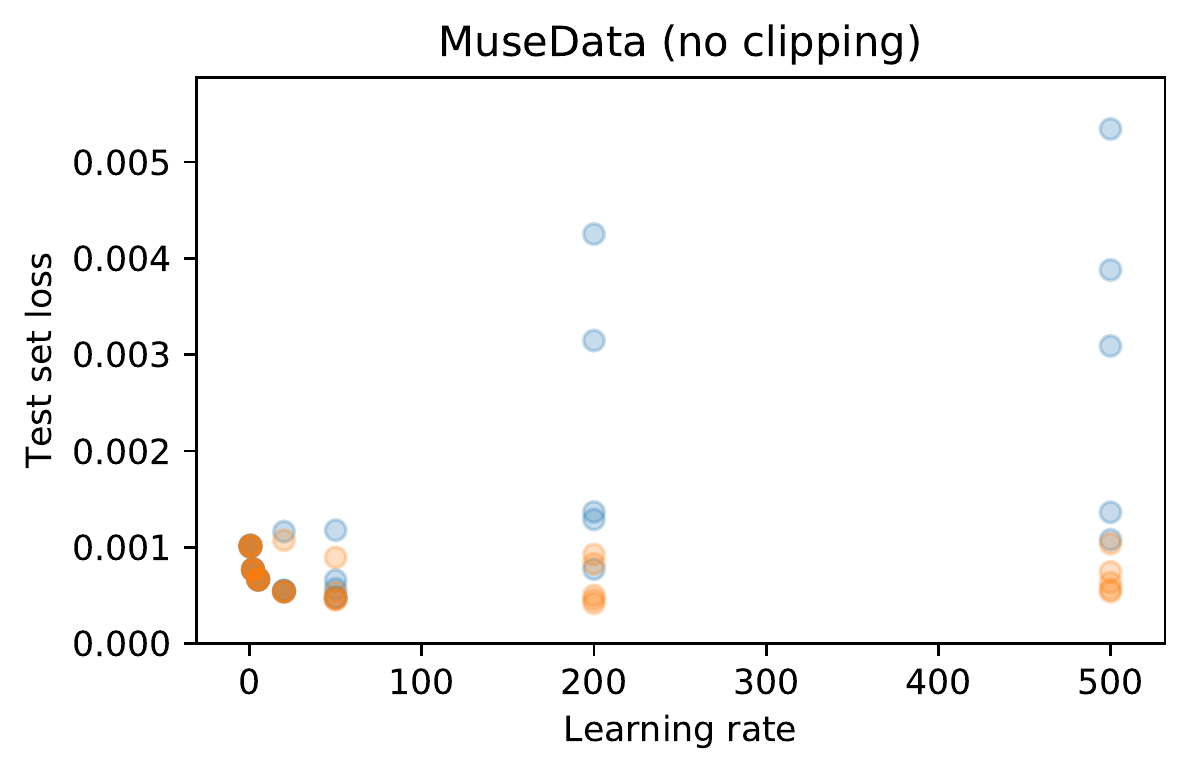}
\end{minipage}
\end{figure}

\begin{figure}[h]
\centering
\begin{minipage}{.49\textwidth}
  \centering
  \includegraphics[width=.99\linewidth]{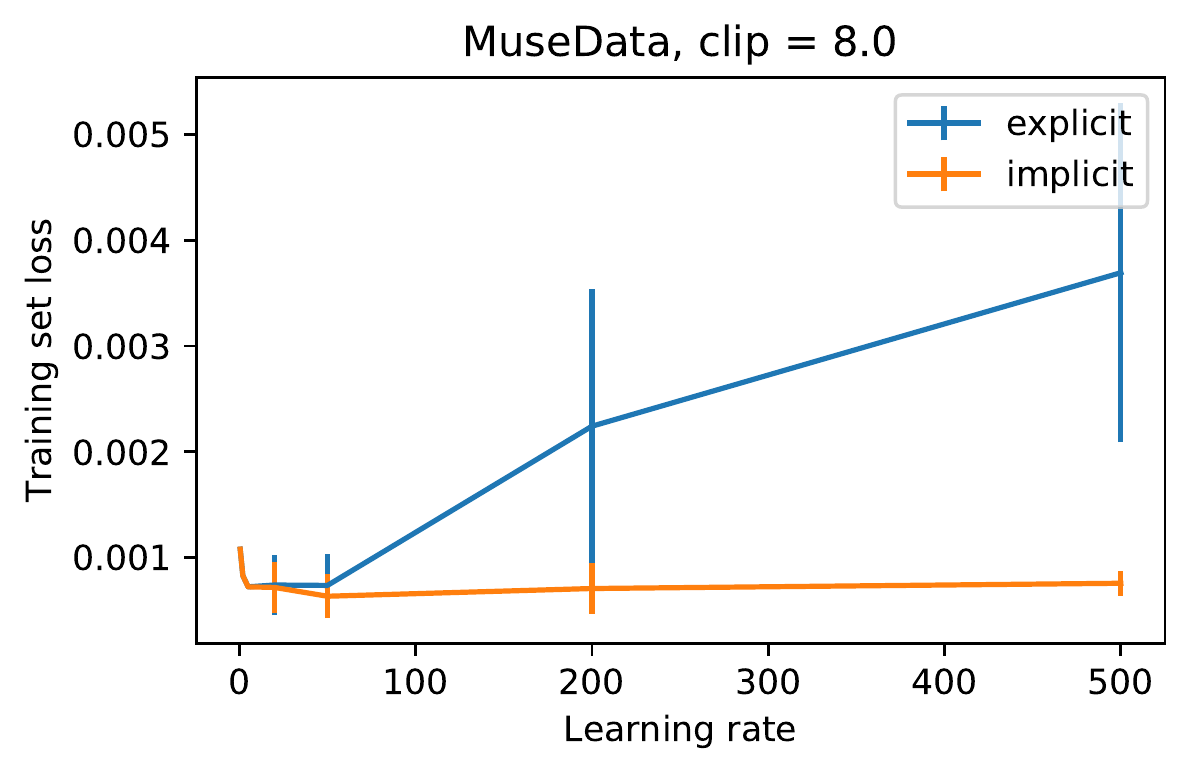}
\end{minipage}%
\hfill
\begin{minipage}{.49\textwidth}
  \centering
  \includegraphics[width=.99\linewidth]{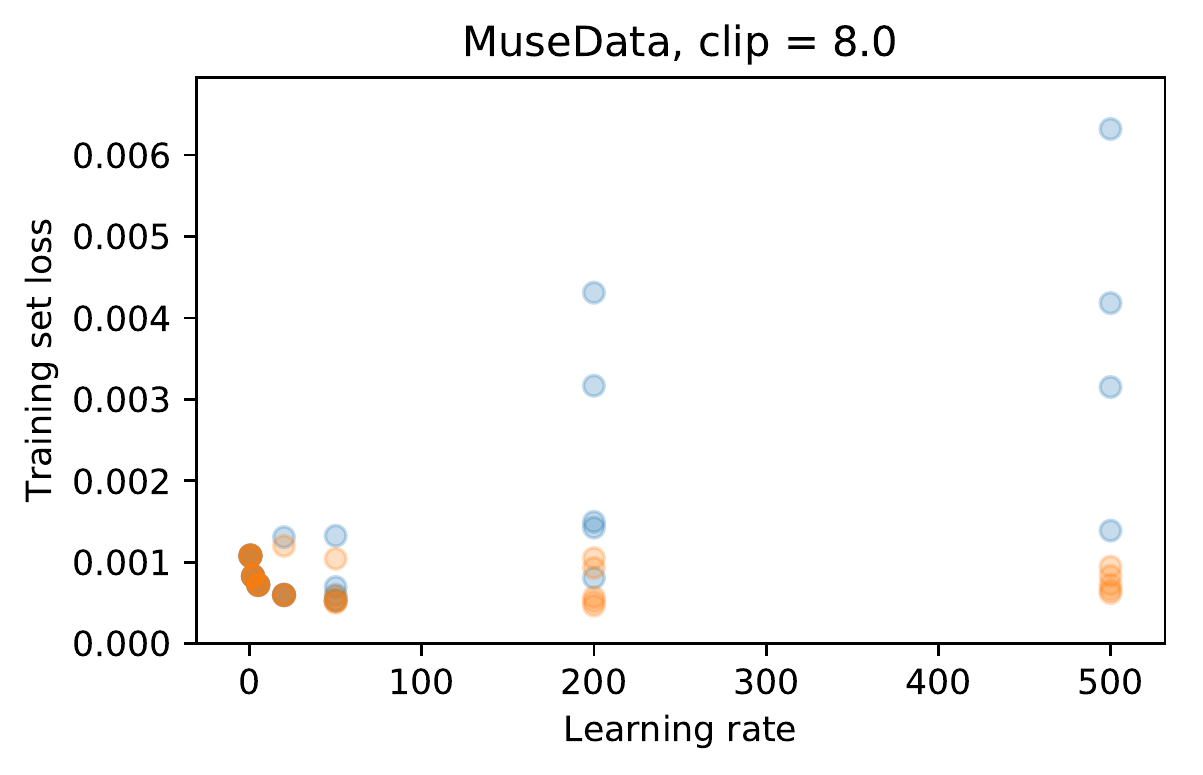}
\end{minipage}
\begin{minipage}{.49\textwidth}
  \centering
  \includegraphics[width=.99\linewidth]{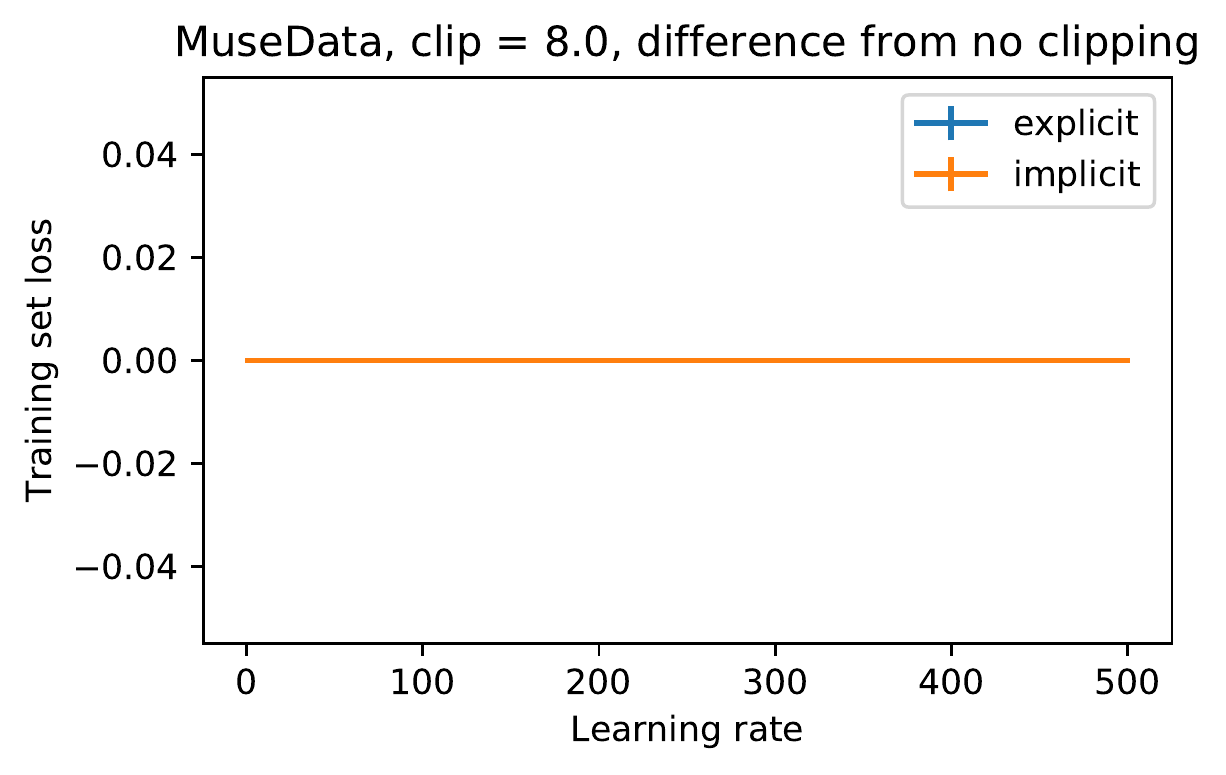}
\end{minipage}%
\hfill
\begin{minipage}{.49\textwidth}
  \centering
  \includegraphics[width=.99\linewidth]{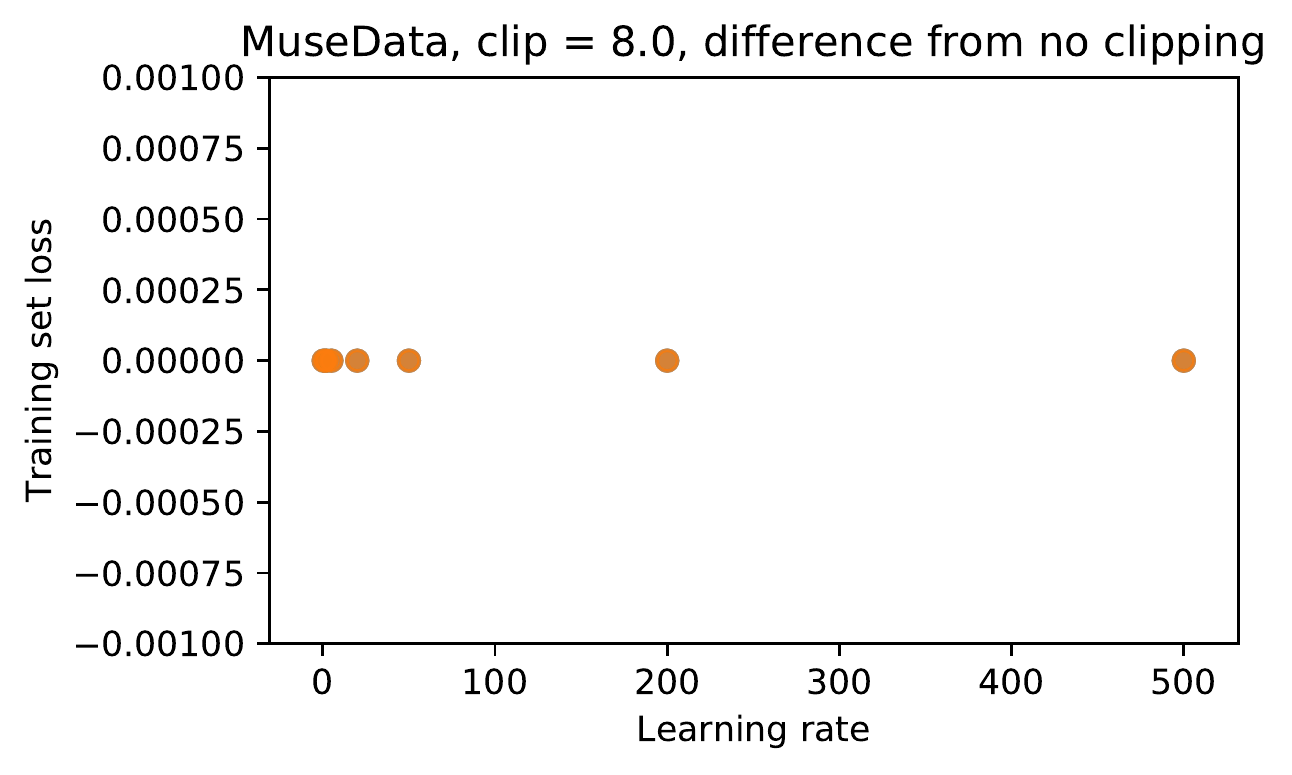}
\end{minipage}

\vspace{1cm}

\begin{minipage}{.49\textwidth}
  \centering
  \includegraphics[width=.99\linewidth]{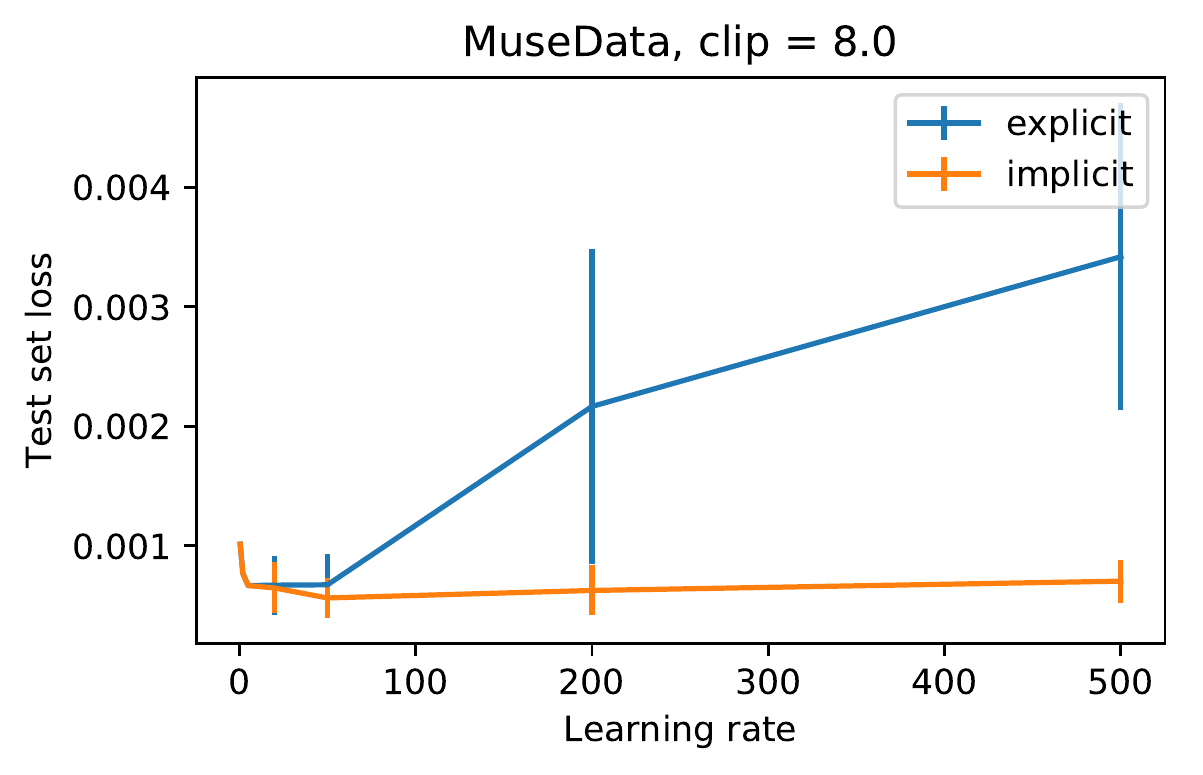}
\end{minipage}%
\hfill
\begin{minipage}{.49\textwidth}
  \centering
  \includegraphics[width=.99\linewidth]{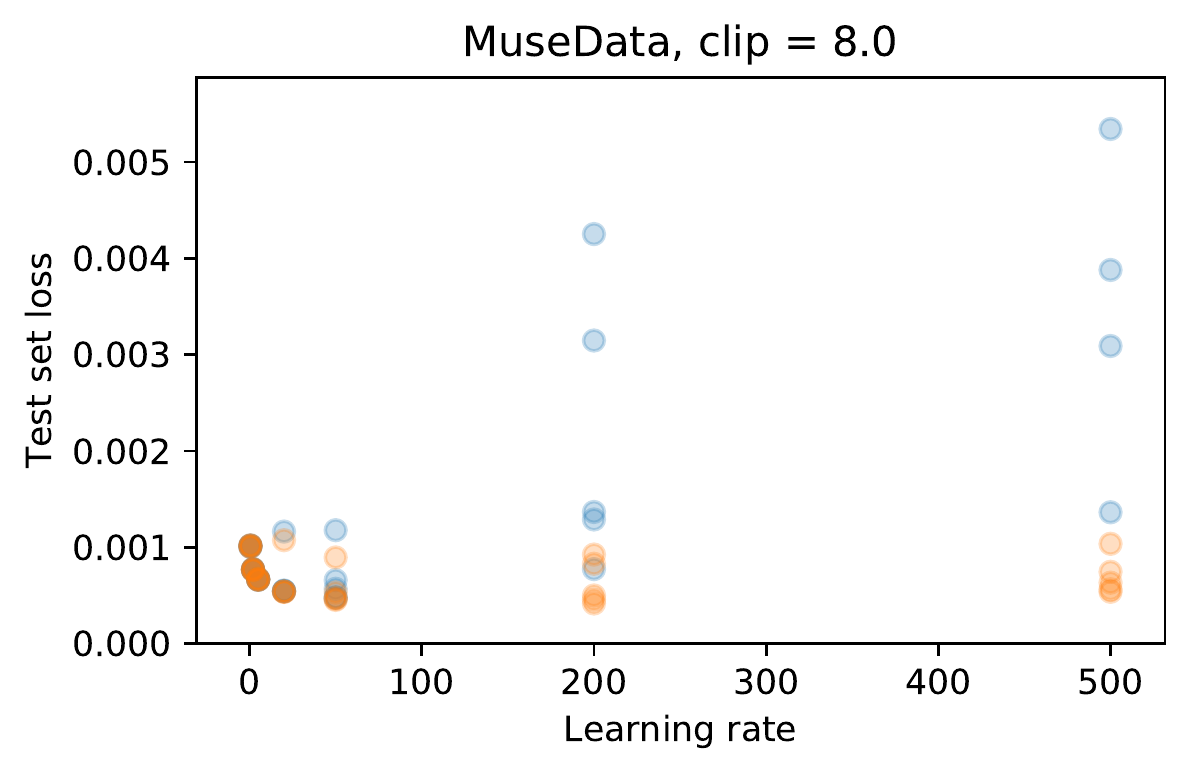}
\end{minipage}
\begin{minipage}{.49\textwidth}
  \centering
  \includegraphics[width=.99\linewidth]{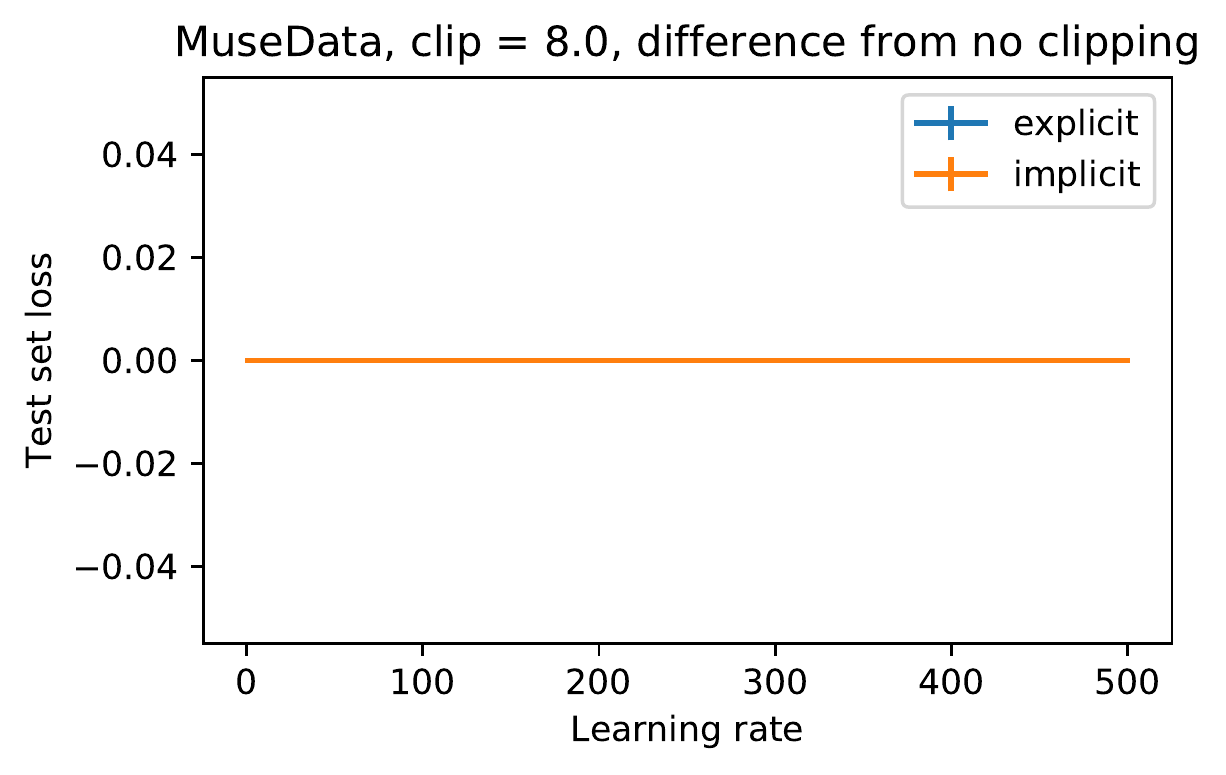}
\end{minipage}%
\hfill
\begin{minipage}{.49\textwidth}
  \centering
  \includegraphics[width=.99\linewidth]{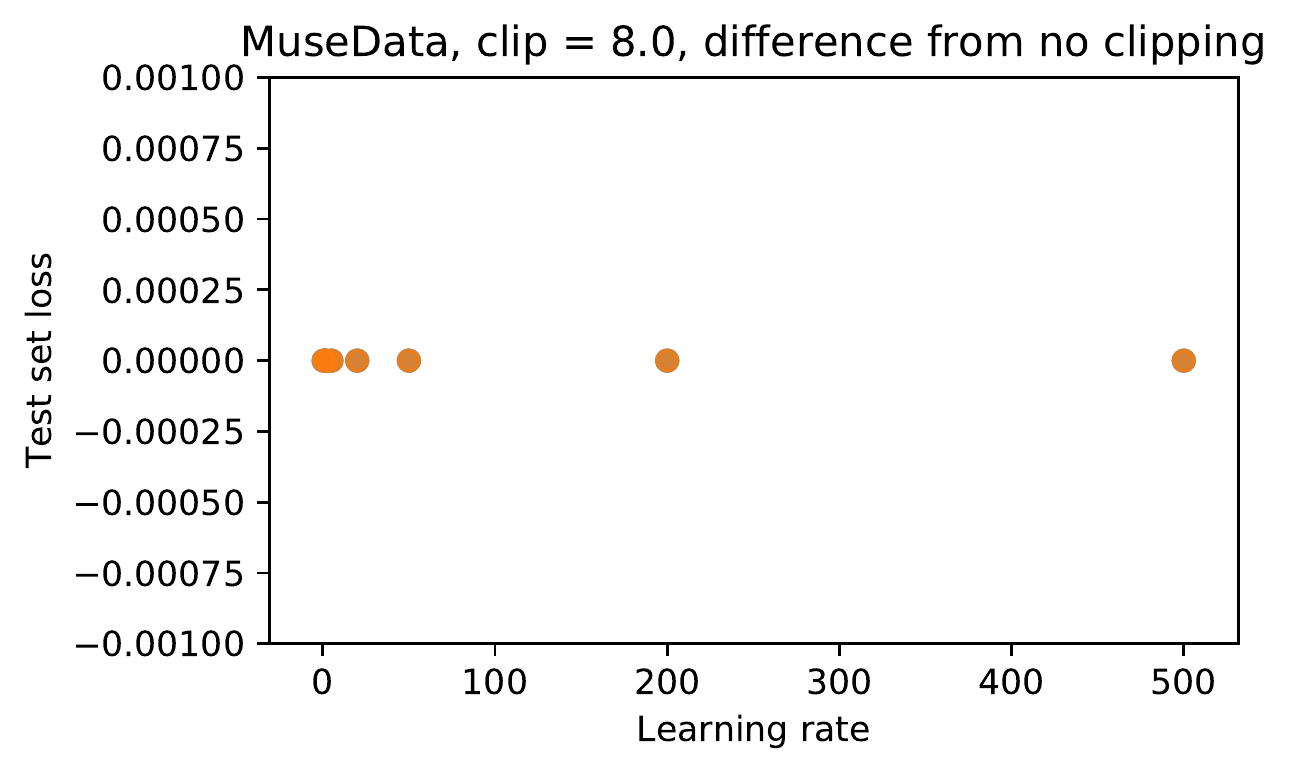}
\end{minipage}
\end{figure}

\clearpage
\subsection{Nottingham}

\begin{figure}[h]
\centering
\begin{minipage}{.49\textwidth}
  \centering
  \includegraphics[width=.99\linewidth]{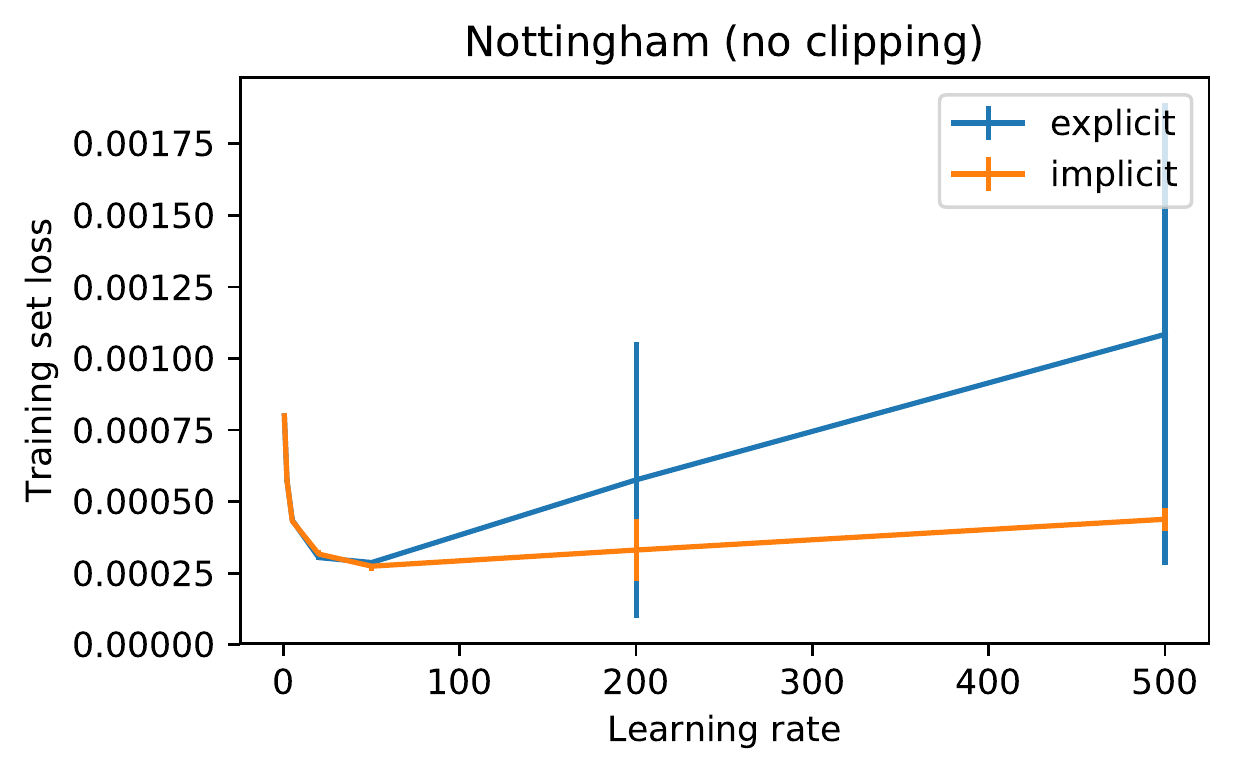}
\end{minipage}%
\hfill
\begin{minipage}{.49\textwidth}
  \centering
  \includegraphics[width=.99\linewidth]{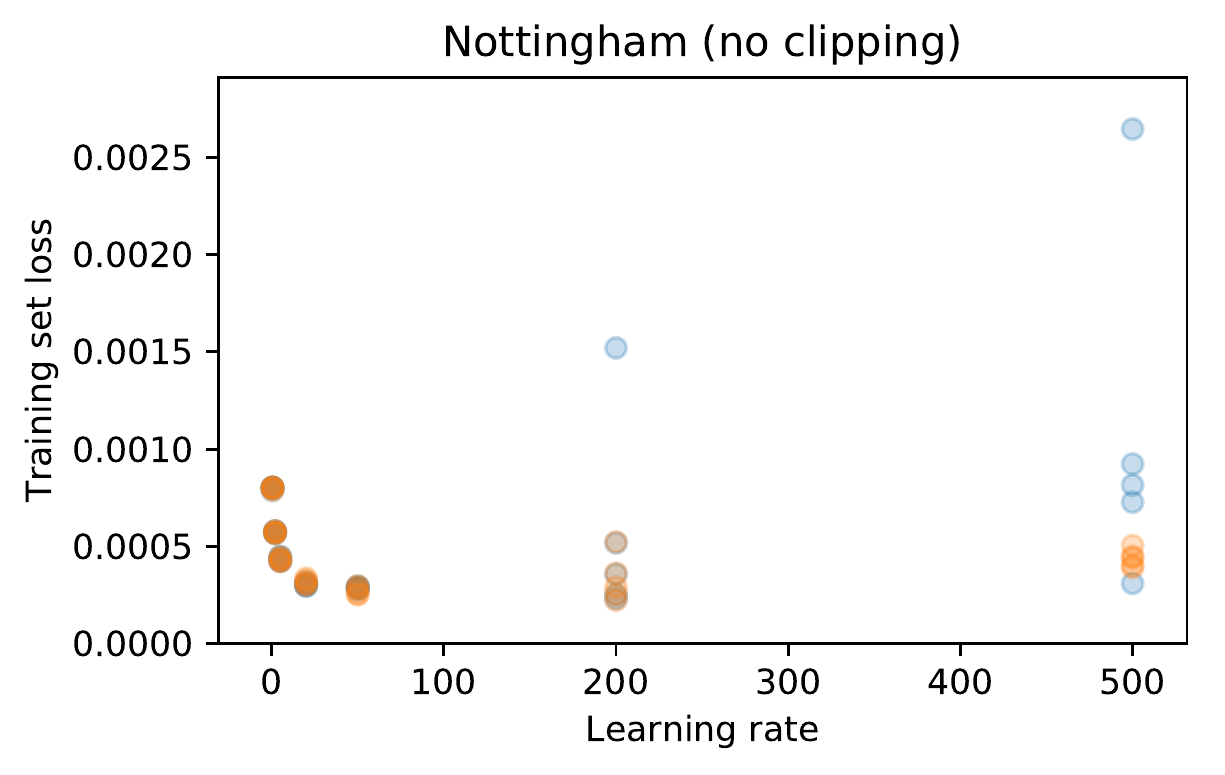}
\end{minipage}
\begin{minipage}{.49\textwidth}
  \centering
  \includegraphics[width=.99\linewidth]{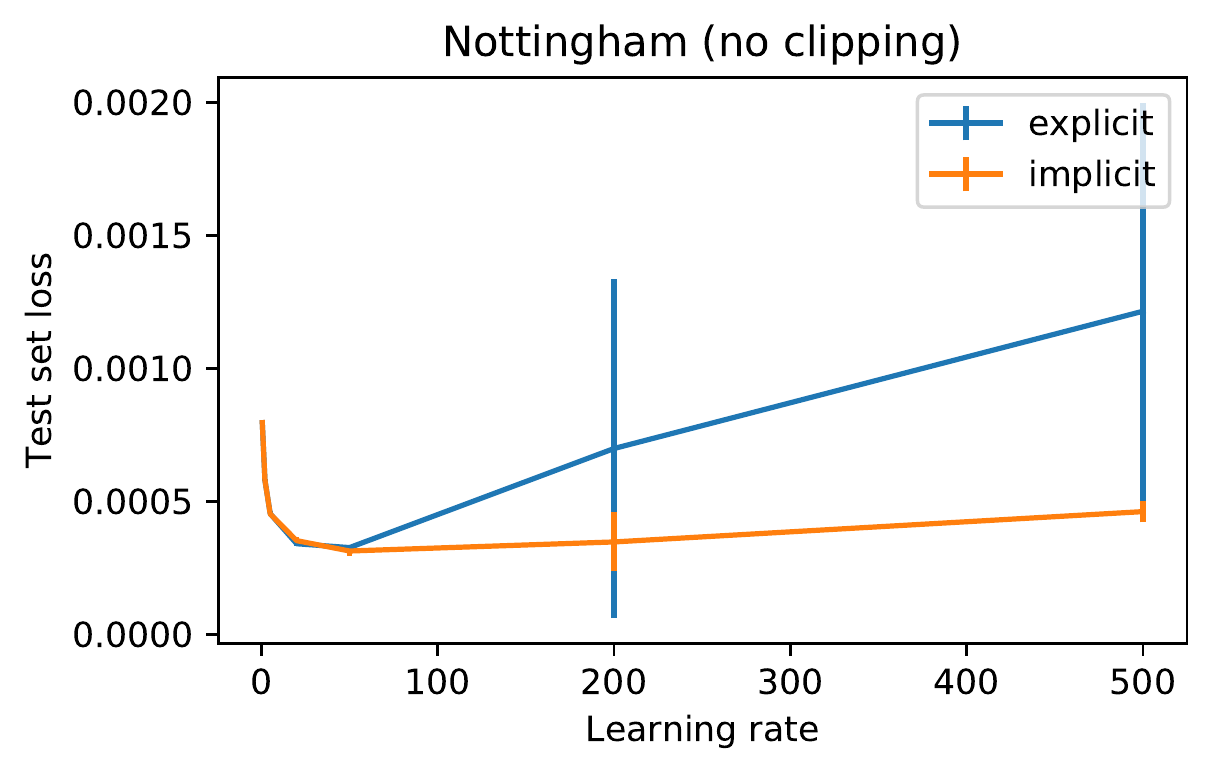}
\end{minipage}%
\hfill
\begin{minipage}{.49\textwidth}
  \centering
  \includegraphics[width=.99\linewidth]{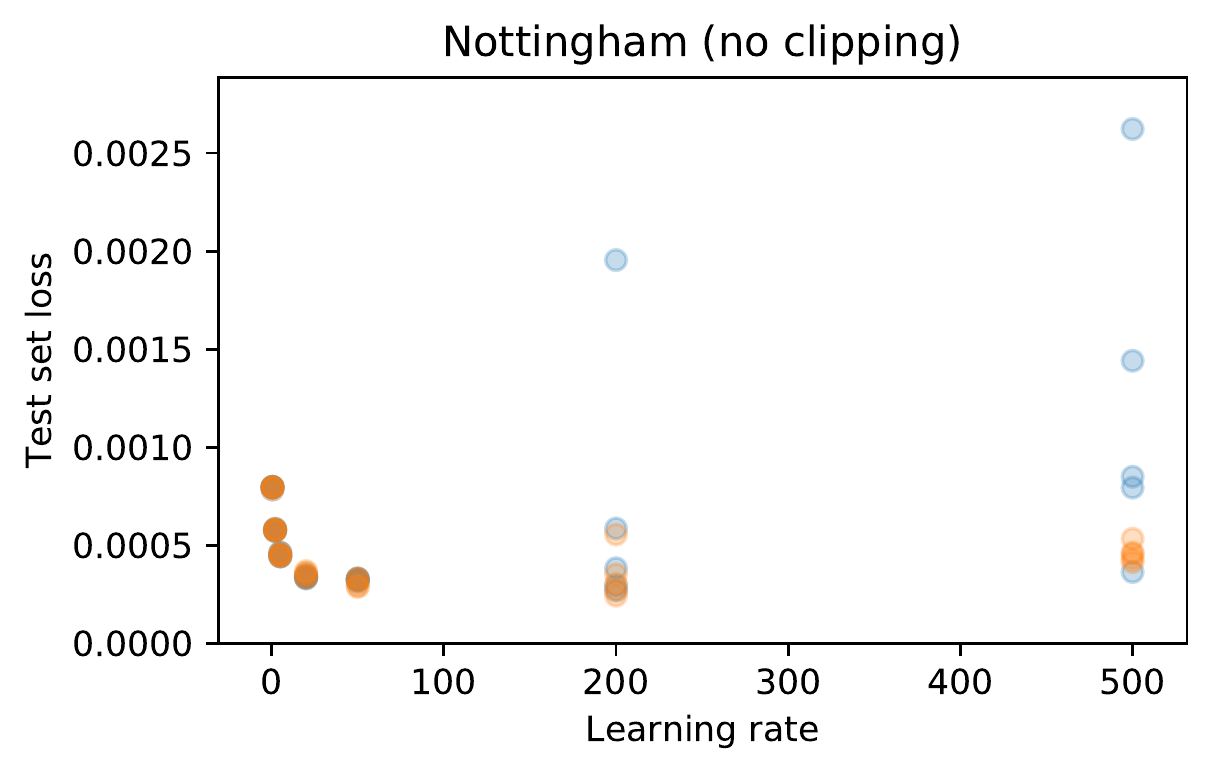}
\end{minipage}
\end{figure}

\begin{figure}[h]
\centering
\begin{minipage}{.49\textwidth}
  \centering
  \includegraphics[width=.99\linewidth]{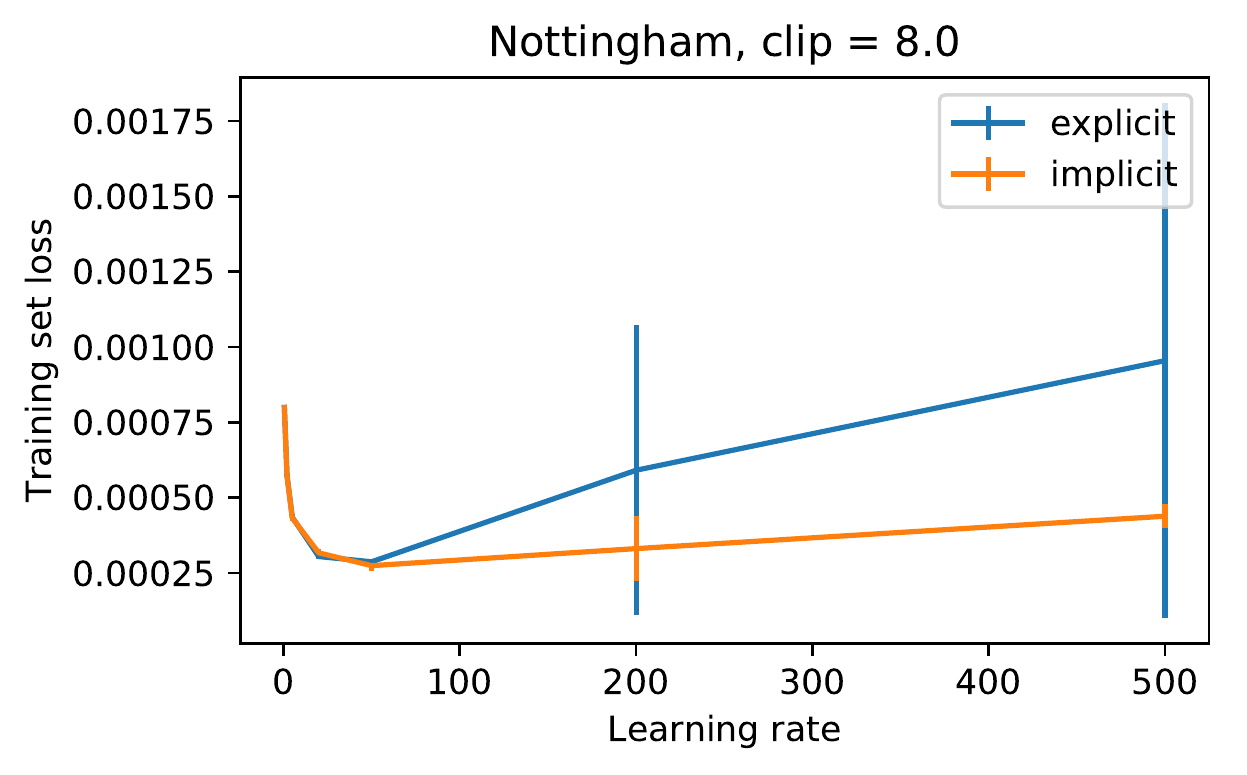}
\end{minipage}%
\hfill
\begin{minipage}{.49\textwidth}
  \centering
  \includegraphics[width=.99\linewidth]{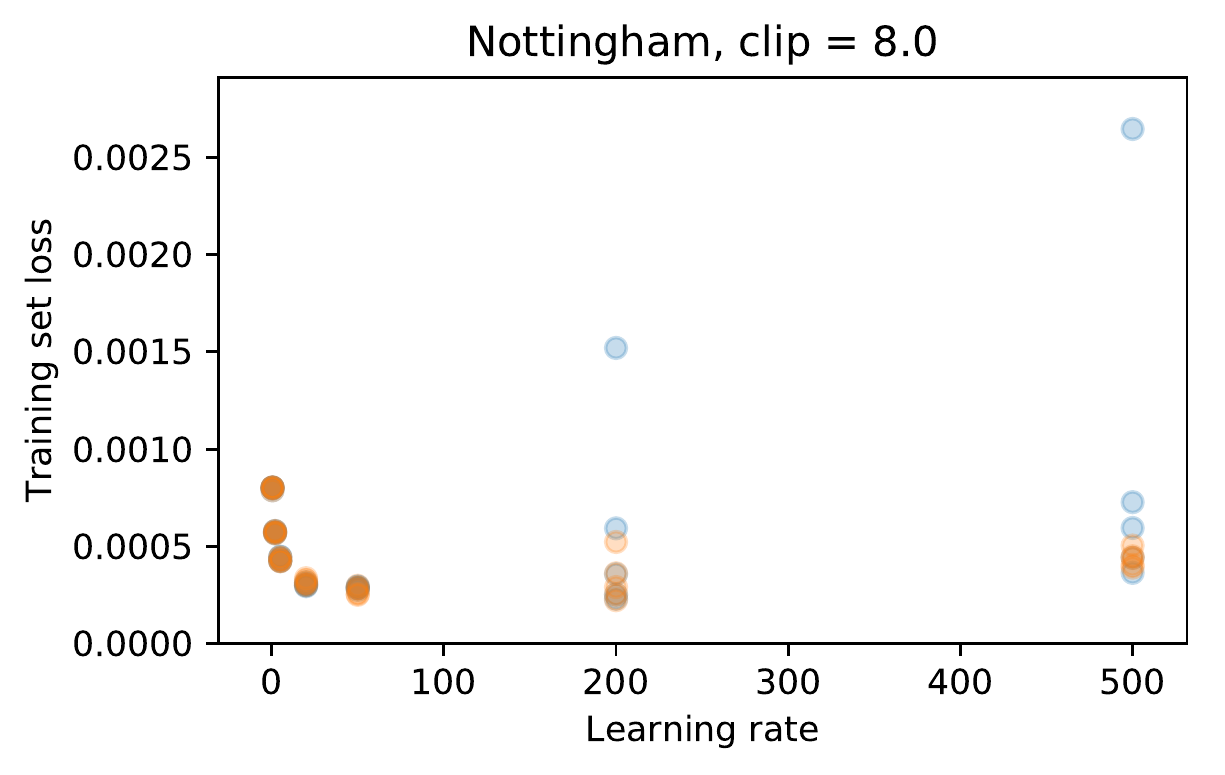}
\end{minipage}
\begin{minipage}{.49\textwidth}
  \centering
  \includegraphics[width=.99\linewidth]{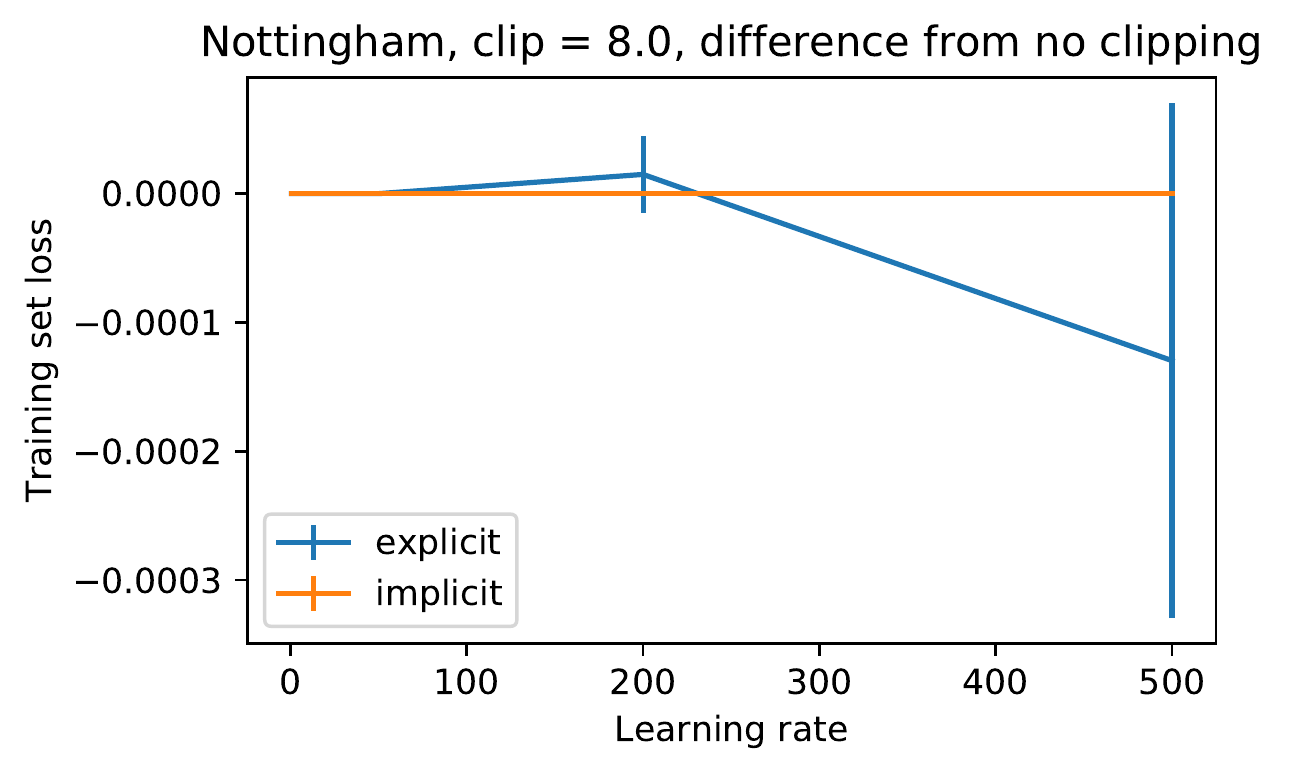}
\end{minipage}%
\hfill
\begin{minipage}{.49\textwidth}
  \centering
  \includegraphics[width=.99\linewidth]{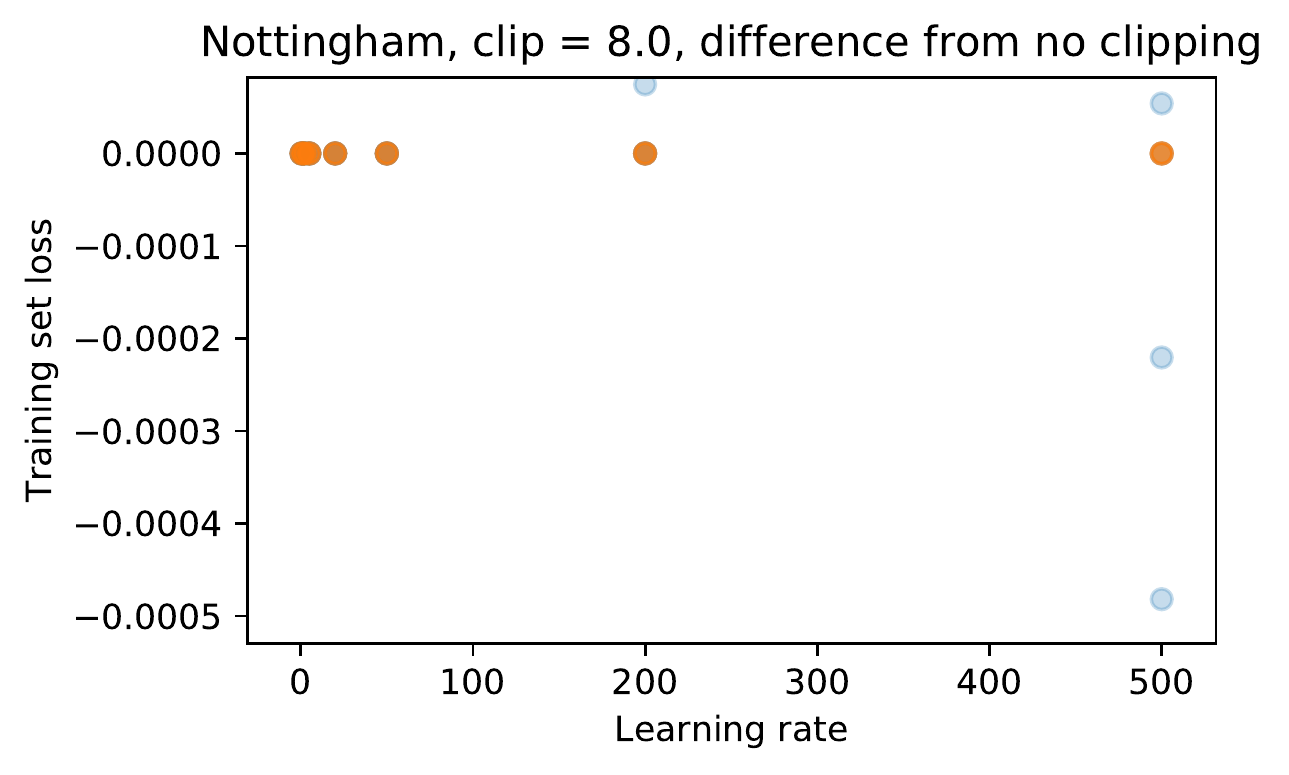}
\end{minipage}

\vspace{1cm}

\begin{minipage}{.49\textwidth}
  \centering
  \includegraphics[width=.99\linewidth]{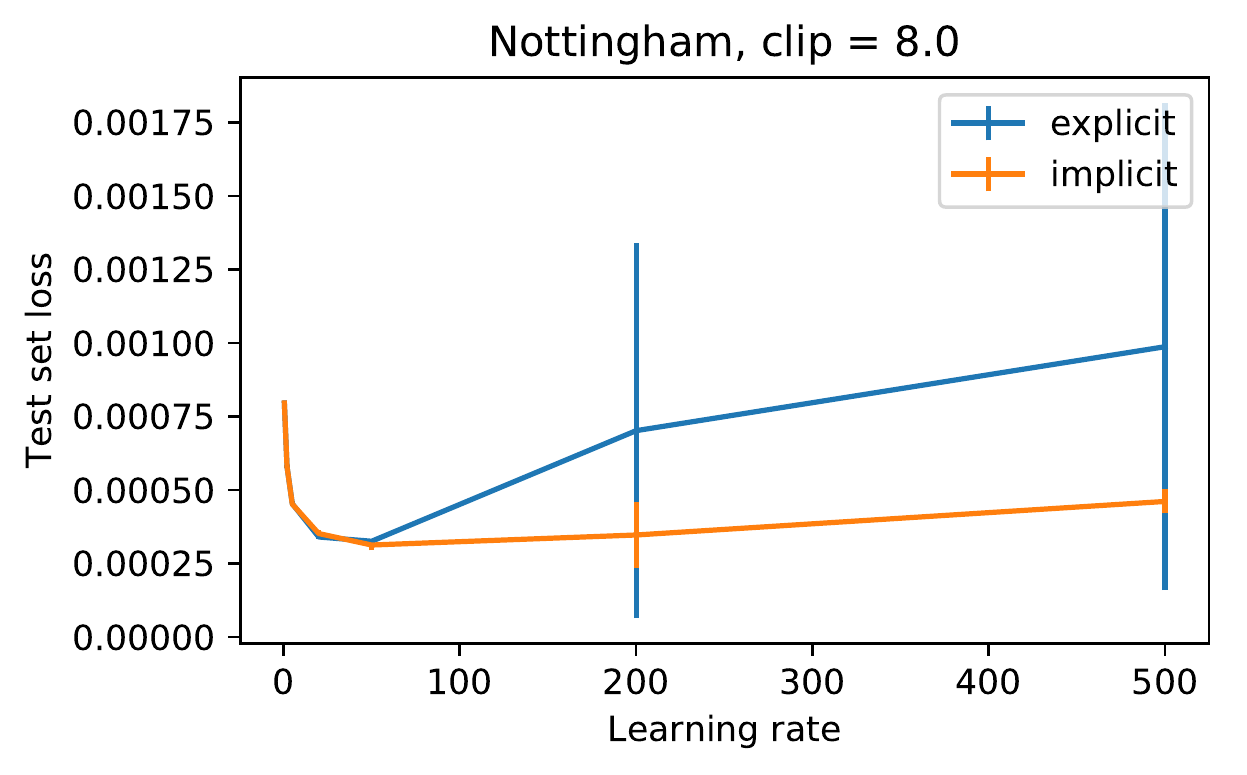}
\end{minipage}%
\hfill
\begin{minipage}{.49\textwidth}
  \centering
  \includegraphics[width=.99\linewidth]{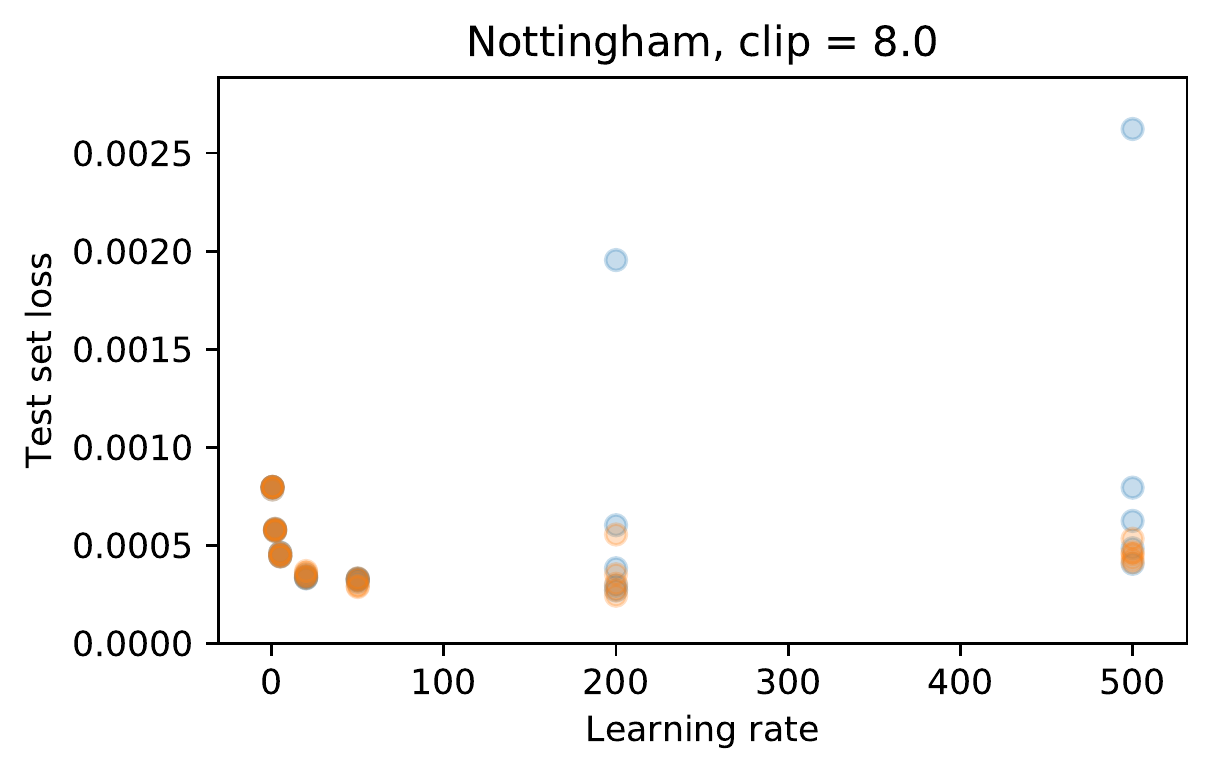}
\end{minipage}
\begin{minipage}{.49\textwidth}
  \centering
  \includegraphics[width=.99\linewidth]{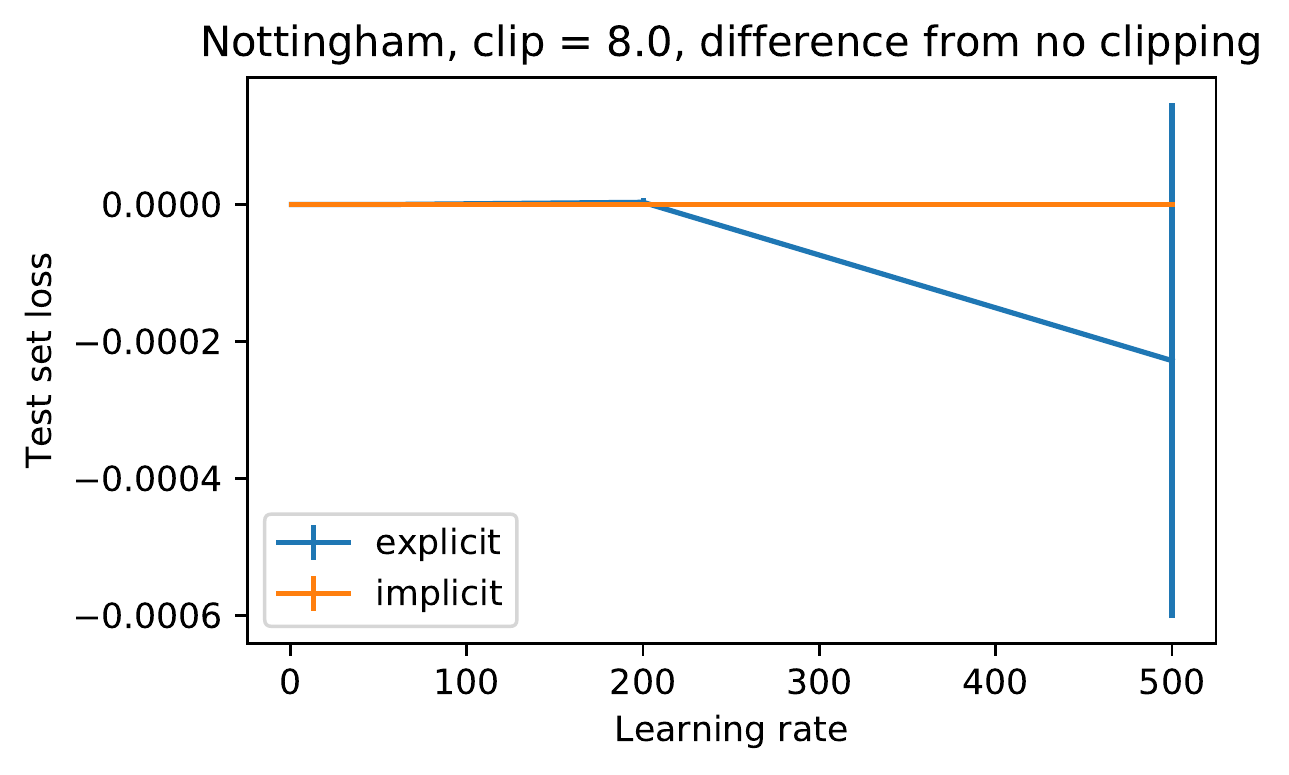}
\end{minipage}%
\hfill
\begin{minipage}{.49\textwidth}
  \centering
  \includegraphics[width=.99\linewidth]{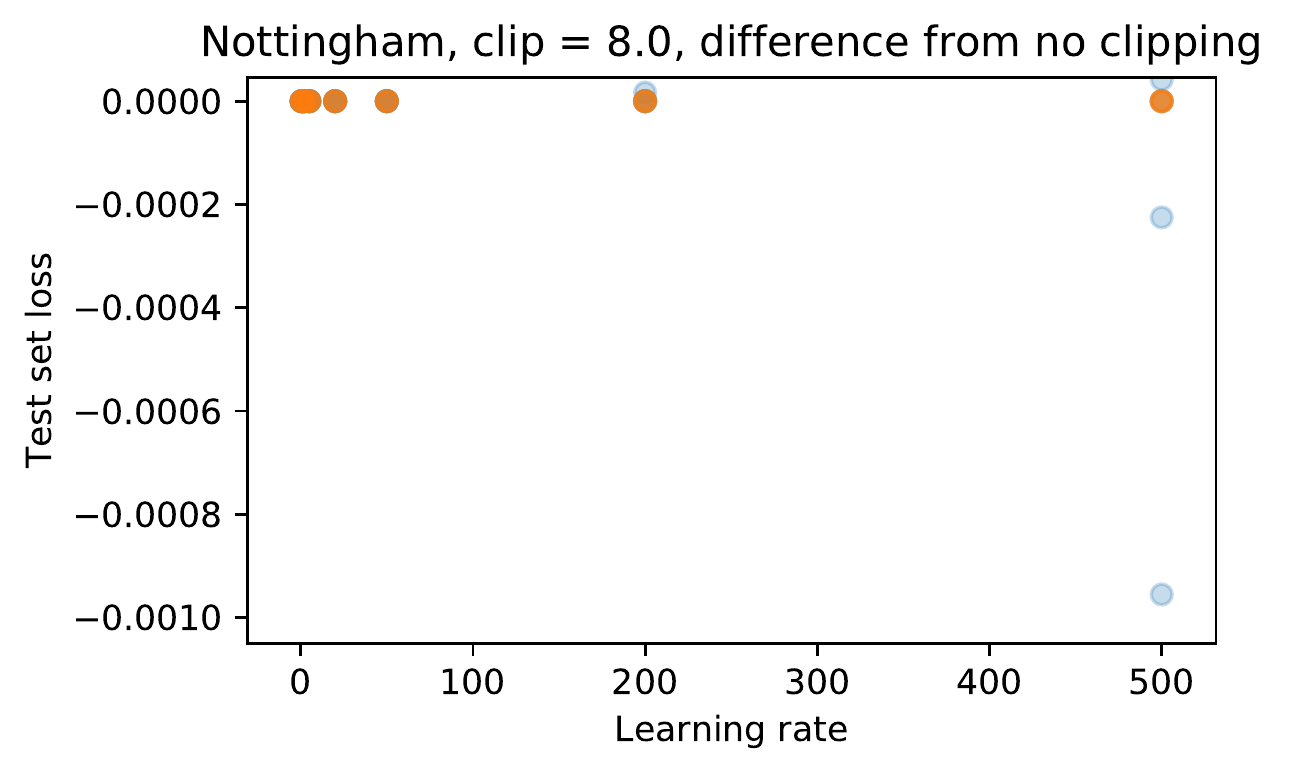}
\end{minipage}
\end{figure}

\clearpage
\subsection{Piano-midi.de}

\begin{figure}[h]
\centering
\begin{minipage}{.49\textwidth}
  \centering
  \includegraphics[width=.99\linewidth]{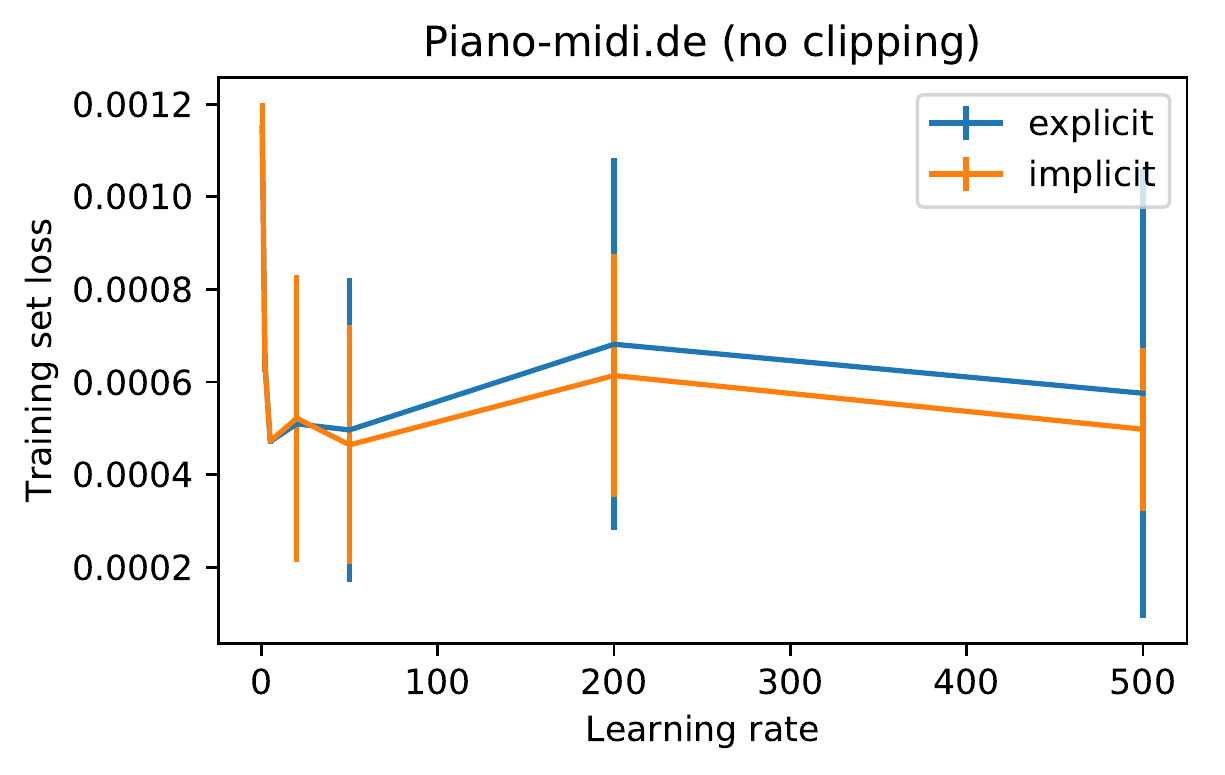}
\end{minipage}%
\hfill
\begin{minipage}{.49\textwidth}
  \centering
  \includegraphics[width=.99\linewidth]{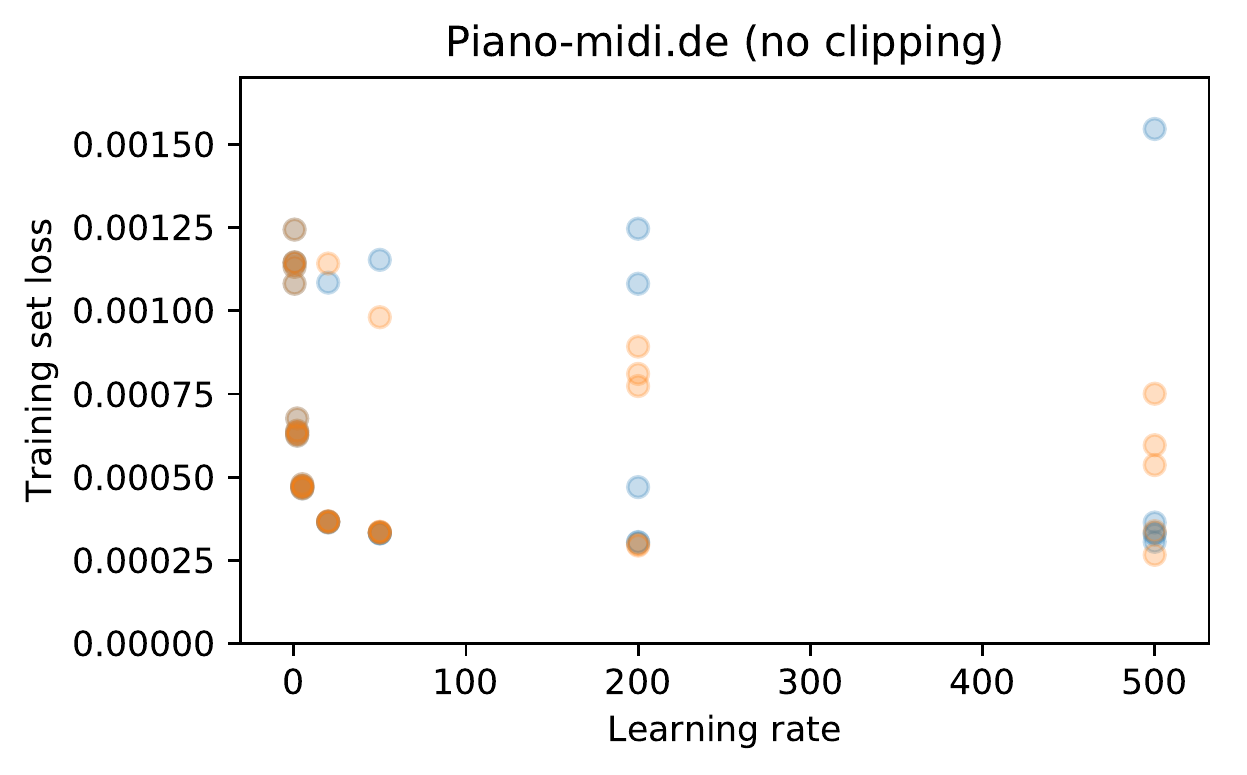}
\end{minipage}
\begin{minipage}{.49\textwidth}
  \centering
  \includegraphics[width=.99\linewidth]{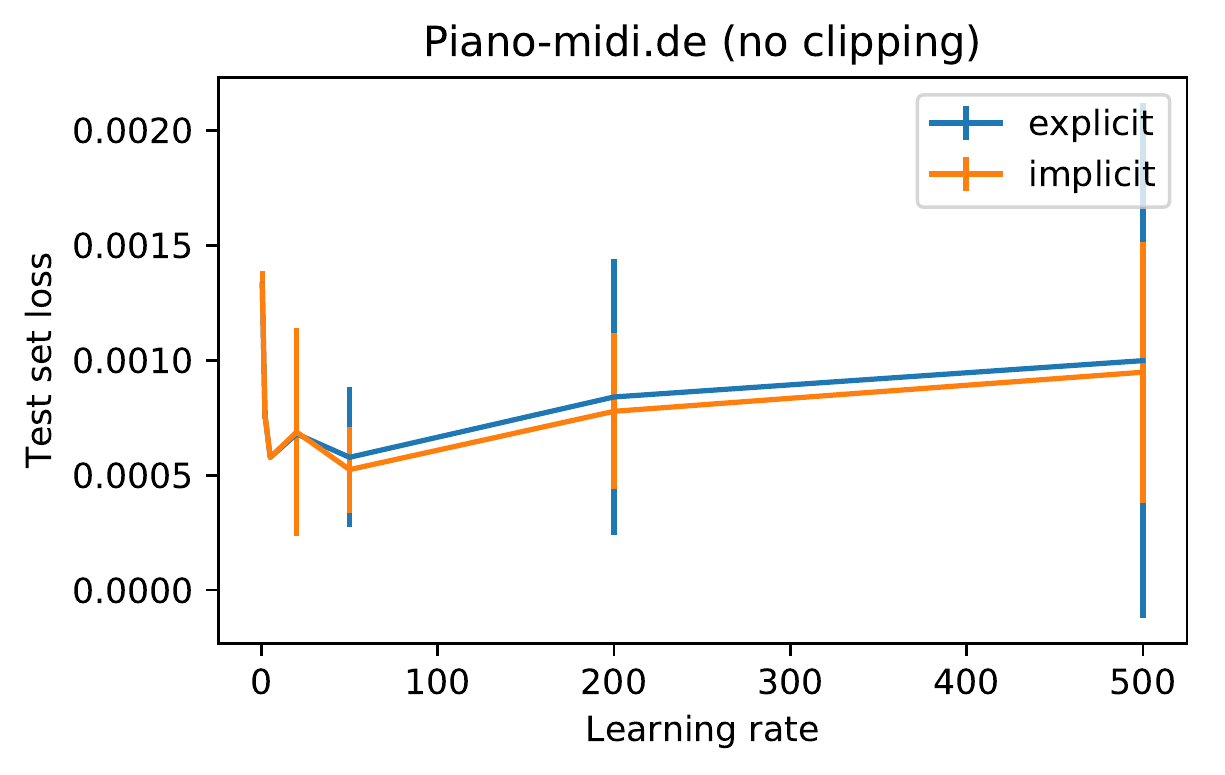}
\end{minipage}%
\hfill
\begin{minipage}{.49\textwidth}
  \centering
  \includegraphics[width=.99\linewidth]{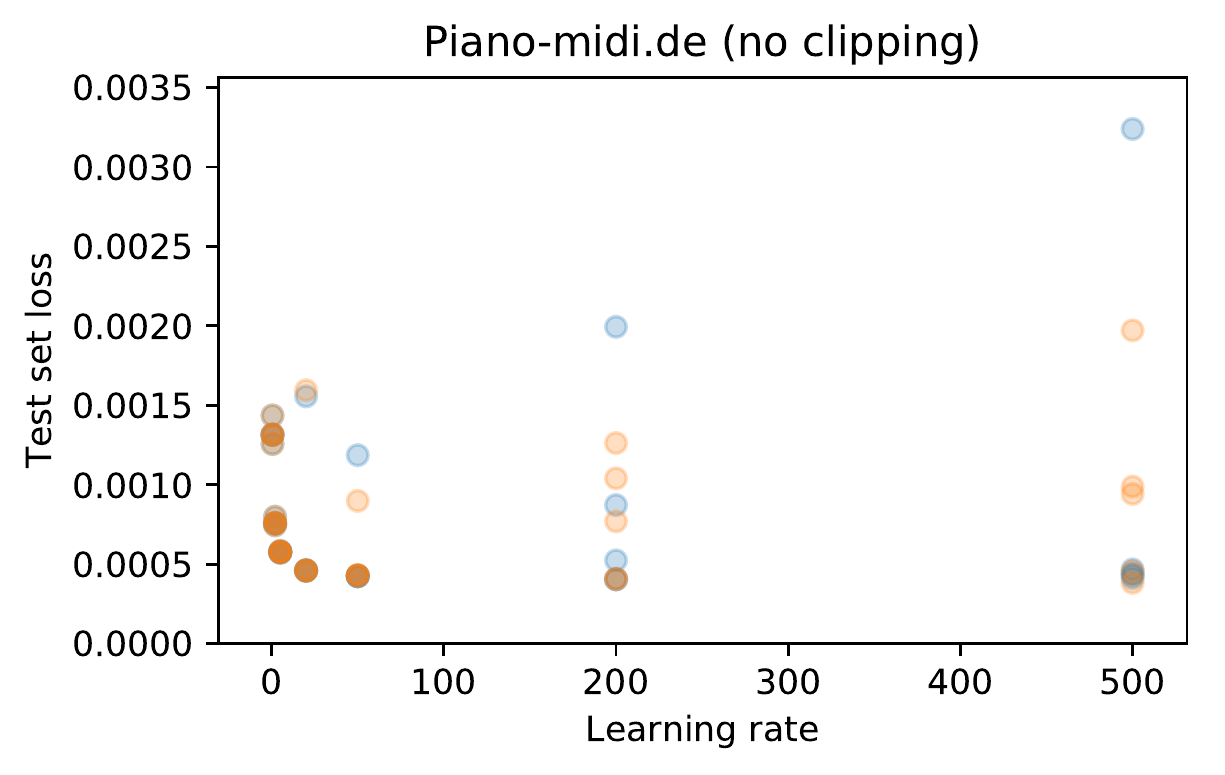}
\end{minipage}
\end{figure}

\begin{figure}[h]
\centering
\begin{minipage}{.49\textwidth}
  \centering
  \includegraphics[width=.99\linewidth]{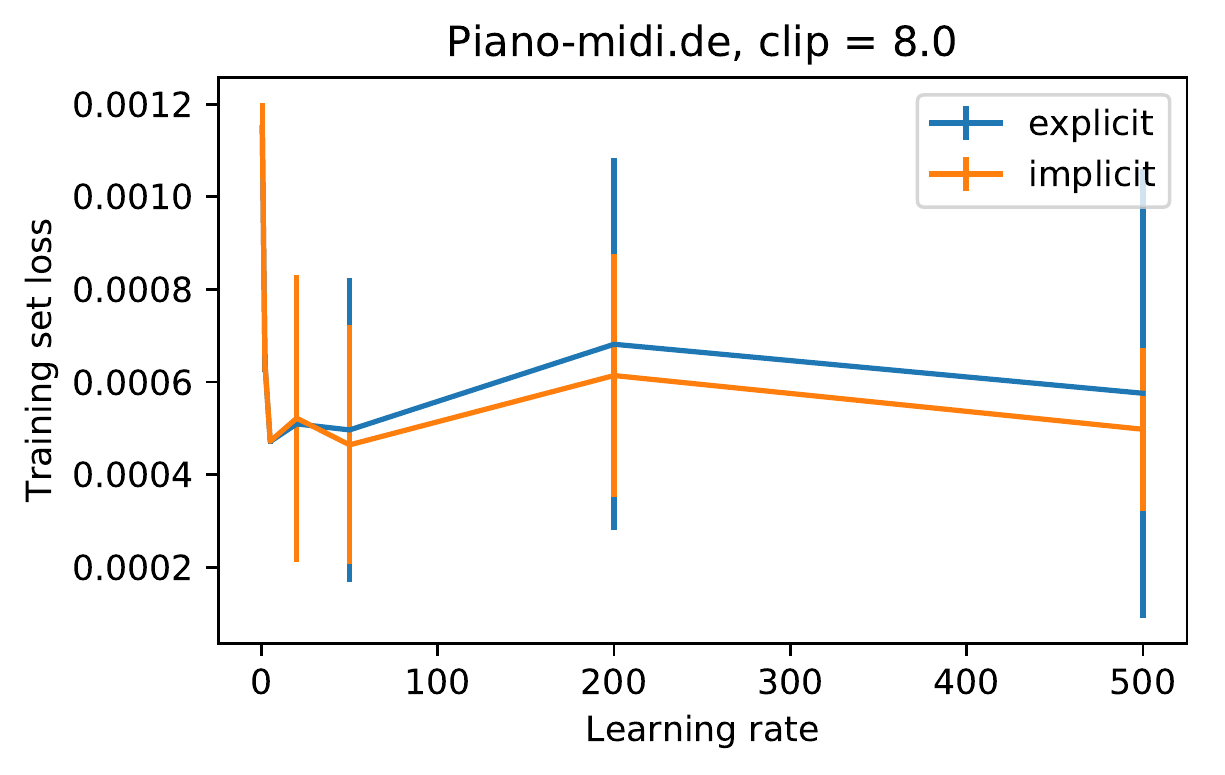}
\end{minipage}%
\hfill
\begin{minipage}{.49\textwidth}
  \centering
  \includegraphics[width=.99\linewidth]{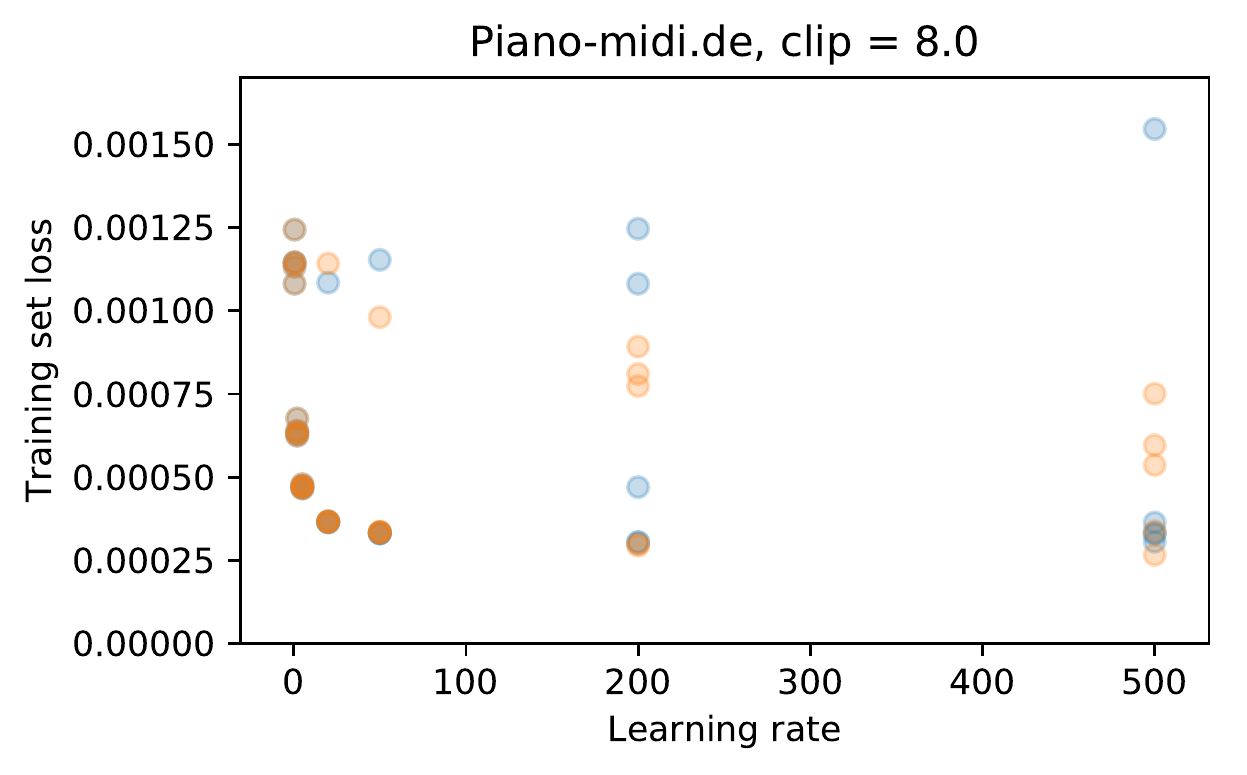}
\end{minipage}
\begin{minipage}{.49\textwidth}
  \centering
  \includegraphics[width=.99\linewidth]{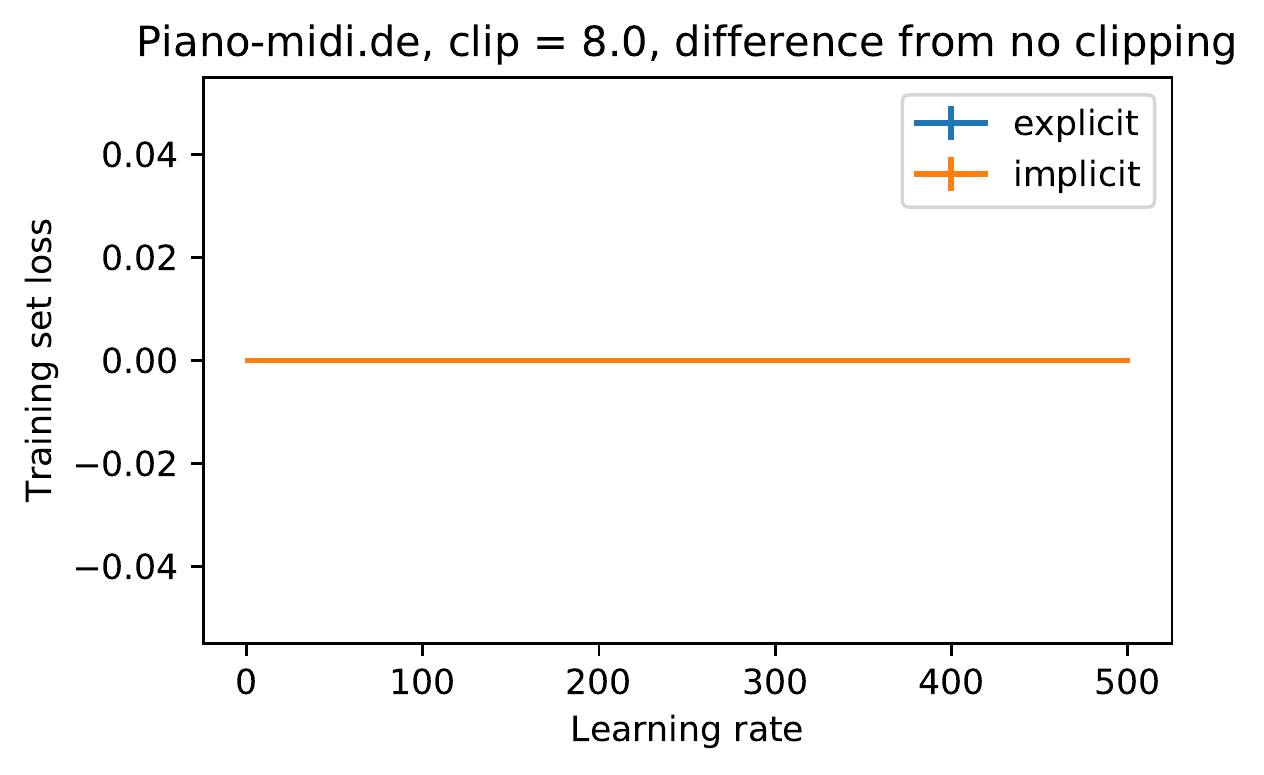}
\end{minipage}%
\hfill
\begin{minipage}{.49\textwidth}
  \centering
  \includegraphics[width=.99\linewidth]{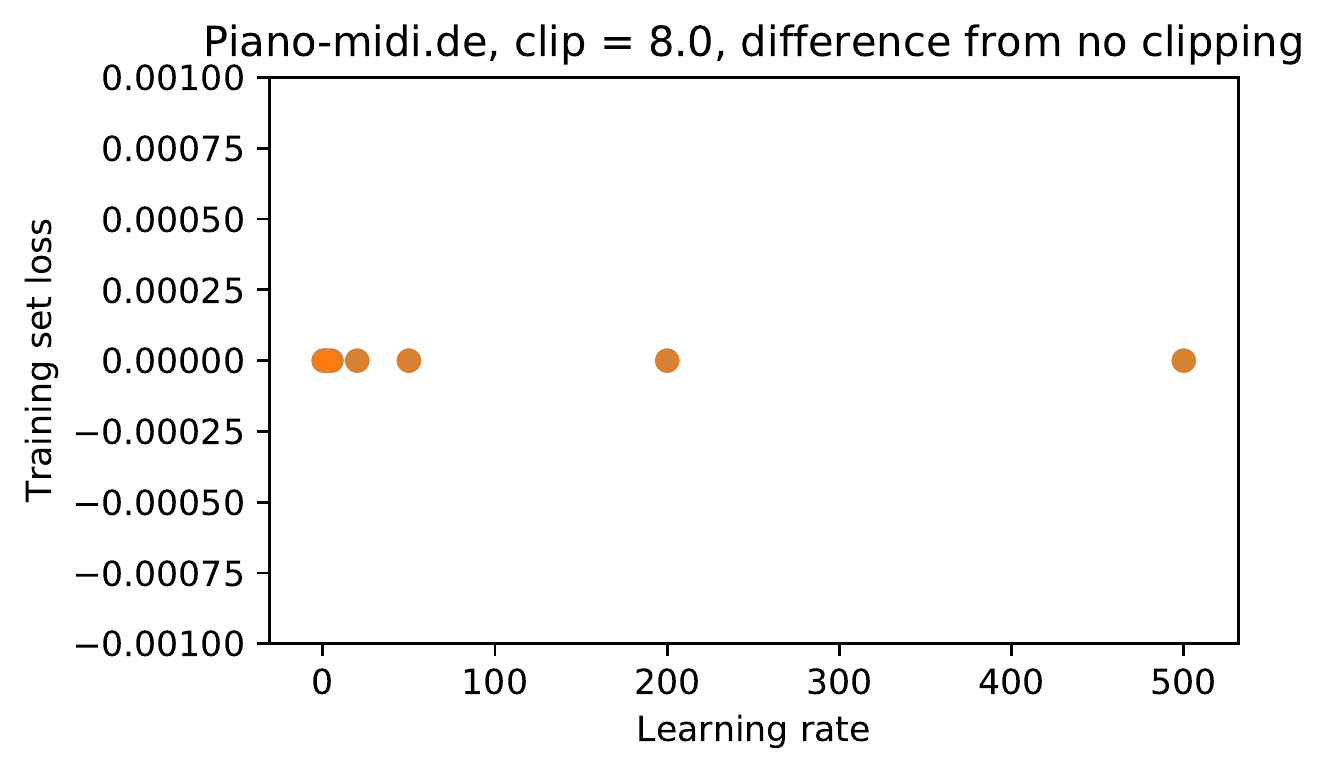}
\end{minipage}

\vspace{1cm}

\begin{minipage}{.49\textwidth}
  \centering
  \includegraphics[width=.99\linewidth]{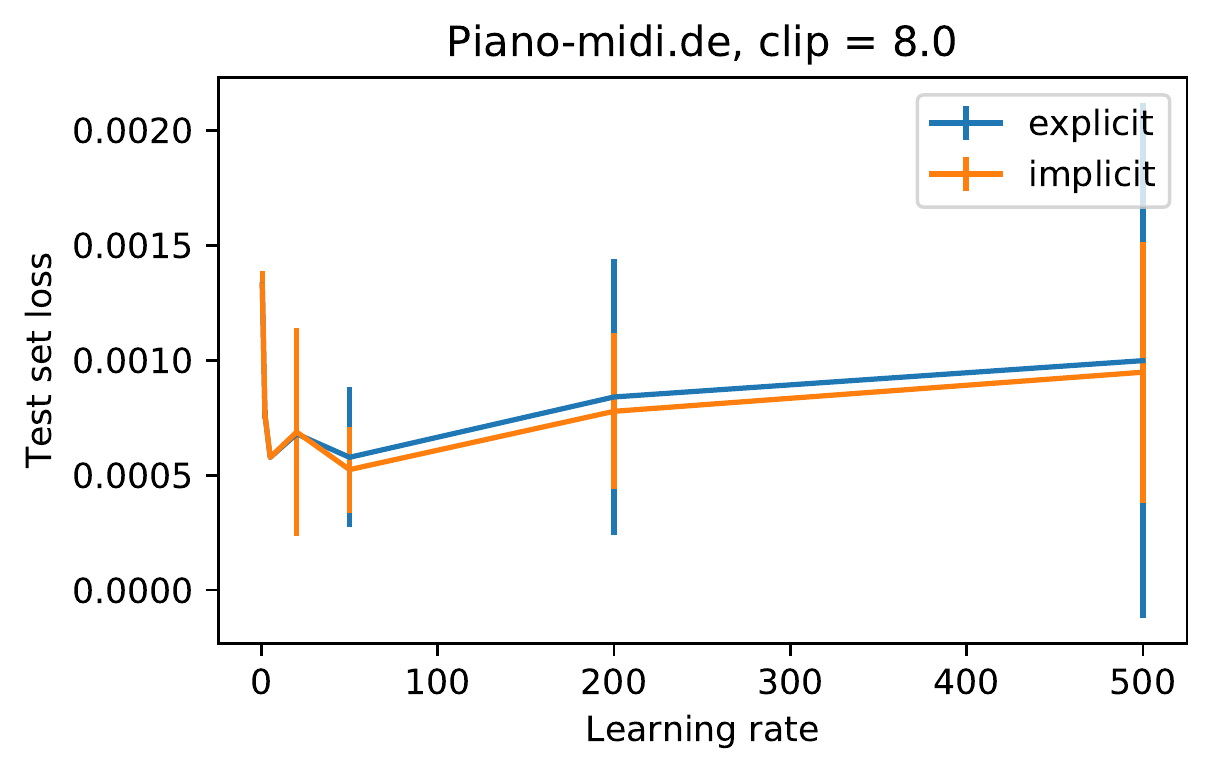}
\end{minipage}%
\hfill
\begin{minipage}{.49\textwidth}
  \centering
  \includegraphics[width=.99\linewidth]{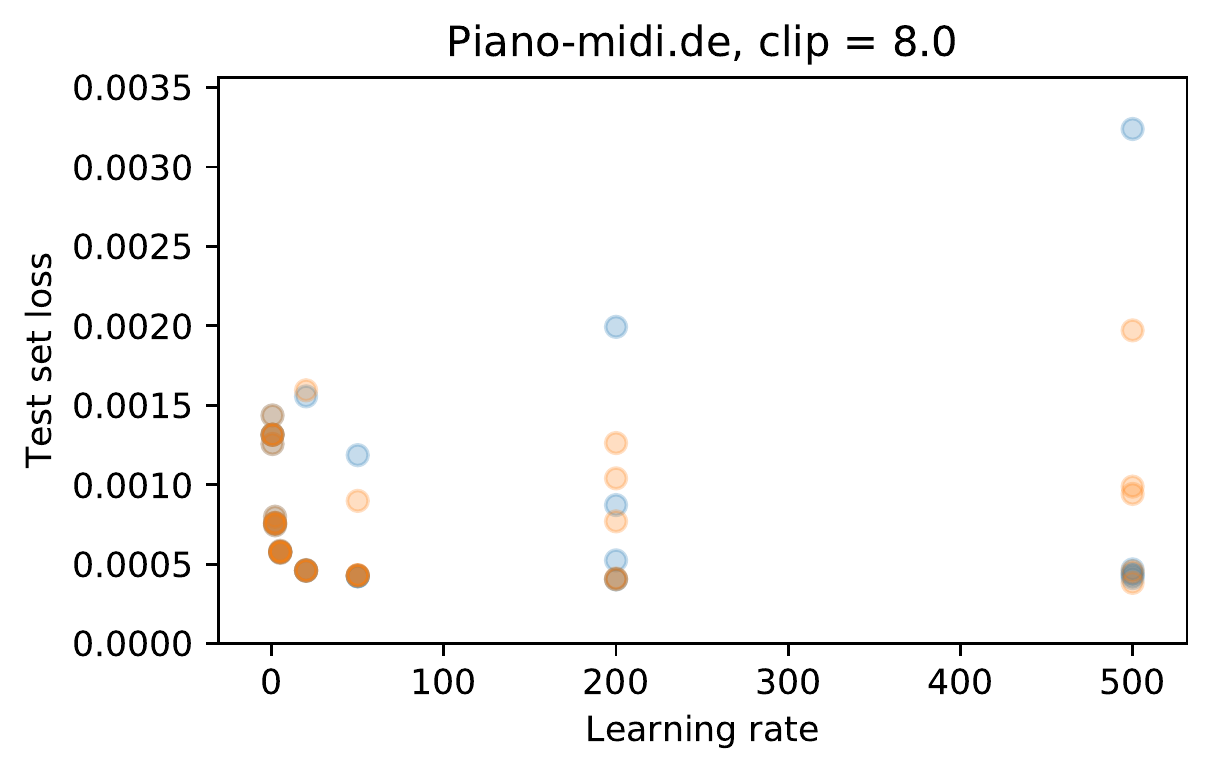}
\end{minipage}
\begin{minipage}{.49\textwidth}
  \centering
  \includegraphics[width=.99\linewidth]{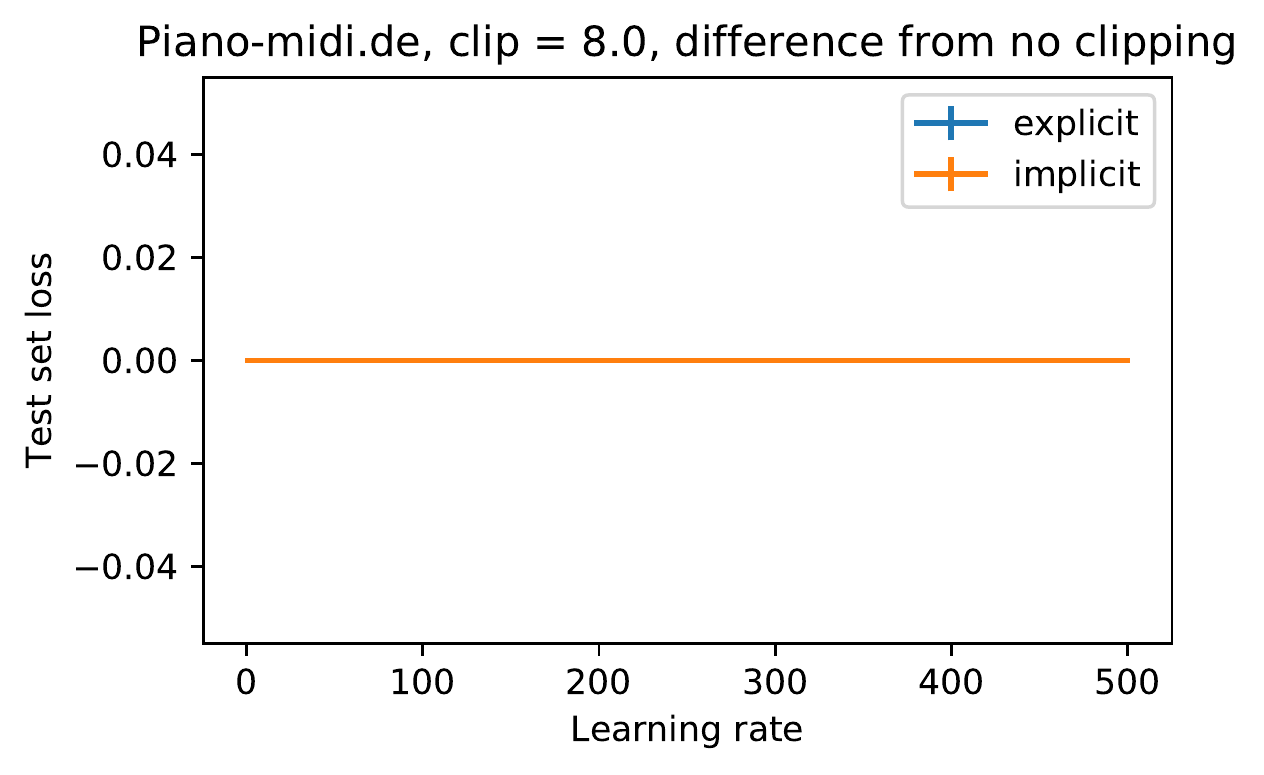}
\end{minipage}%
\hfill
\begin{minipage}{.49\textwidth}
  \centering
  \includegraphics[width=.99\linewidth]{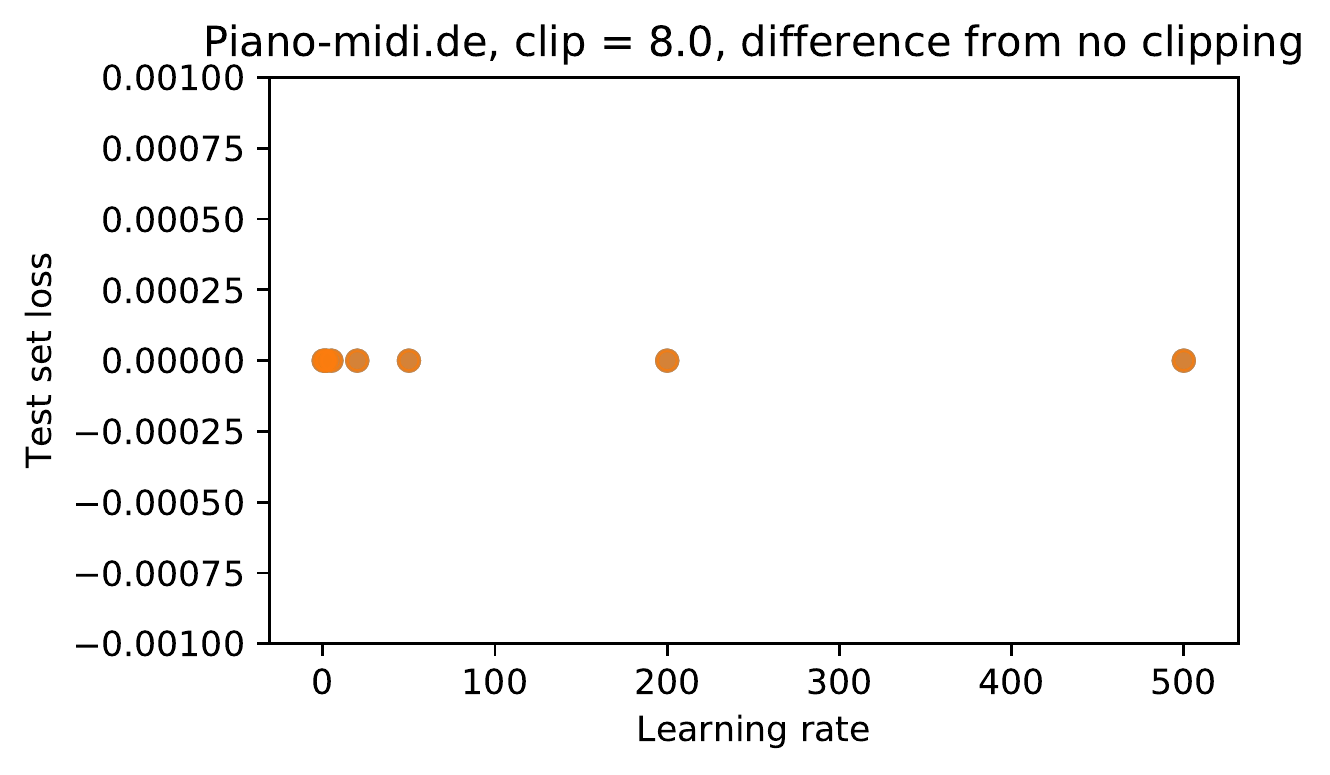}
\end{minipage}
\end{figure}

\begin{figure}[h]
\centering
\begin{minipage}{.49\textwidth}
  \centering
  \includegraphics[width=.99\linewidth]{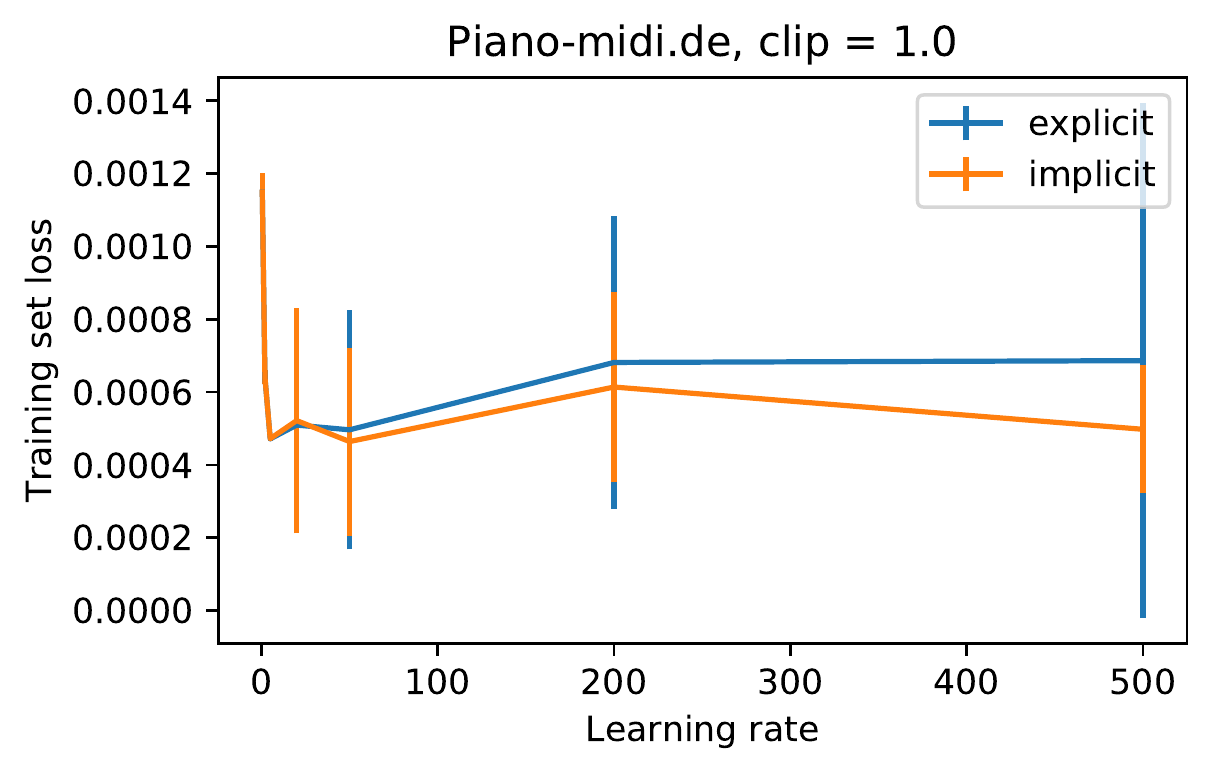}
\end{minipage}%
\hfill
\begin{minipage}{.49\textwidth}
  \centering
  \includegraphics[width=.99\linewidth]{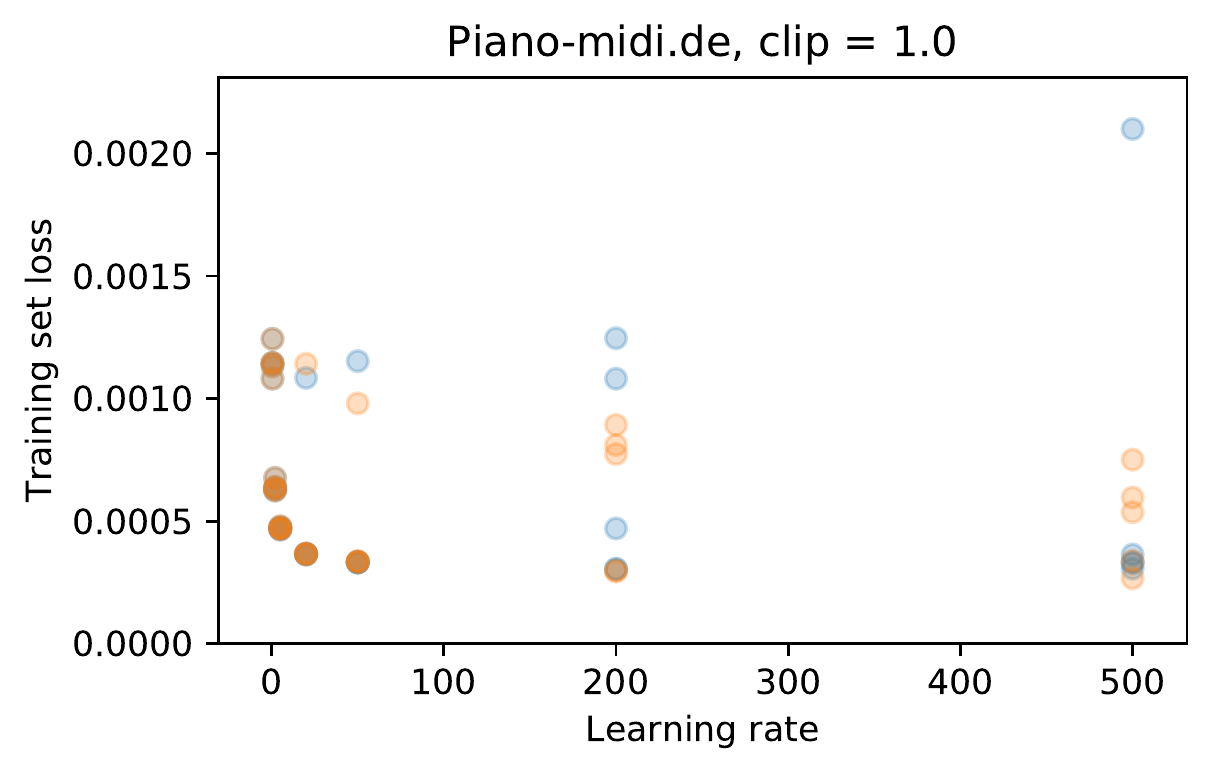}
\end{minipage}
\begin{minipage}{.49\textwidth}
  \centering
  \includegraphics[width=.99\linewidth]{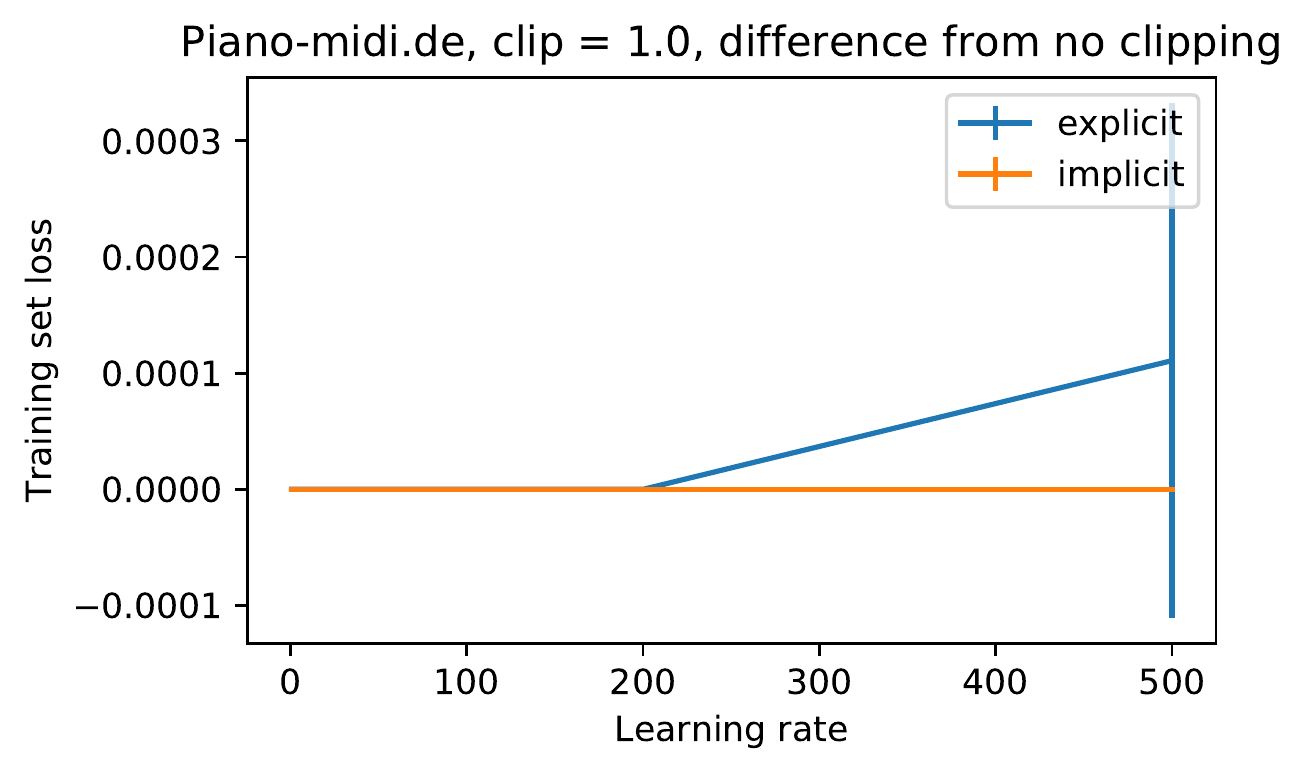}
\end{minipage}%
\hfill
\begin{minipage}{.49\textwidth}
  \centering
  \includegraphics[width=.99\linewidth]{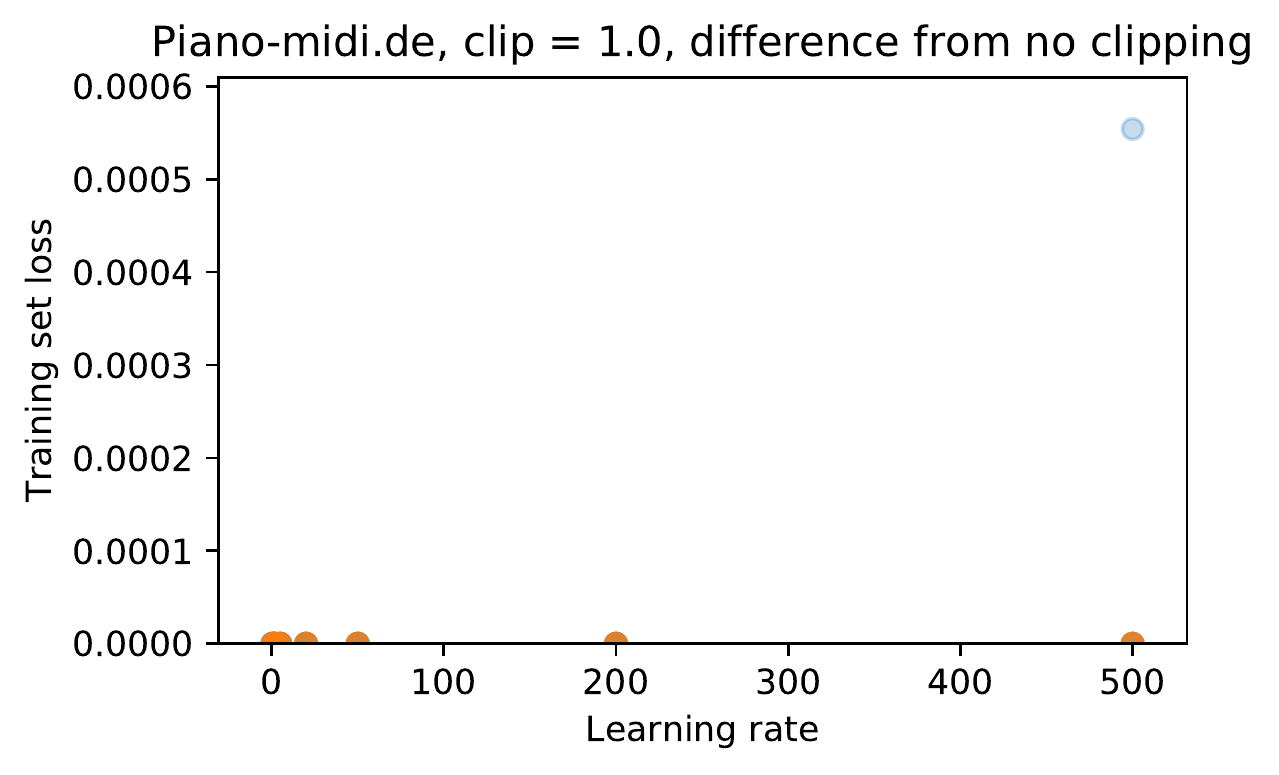}
\end{minipage}

\vspace{1cm}

\begin{minipage}{.49\textwidth}
  \centering
  \includegraphics[width=.99\linewidth]{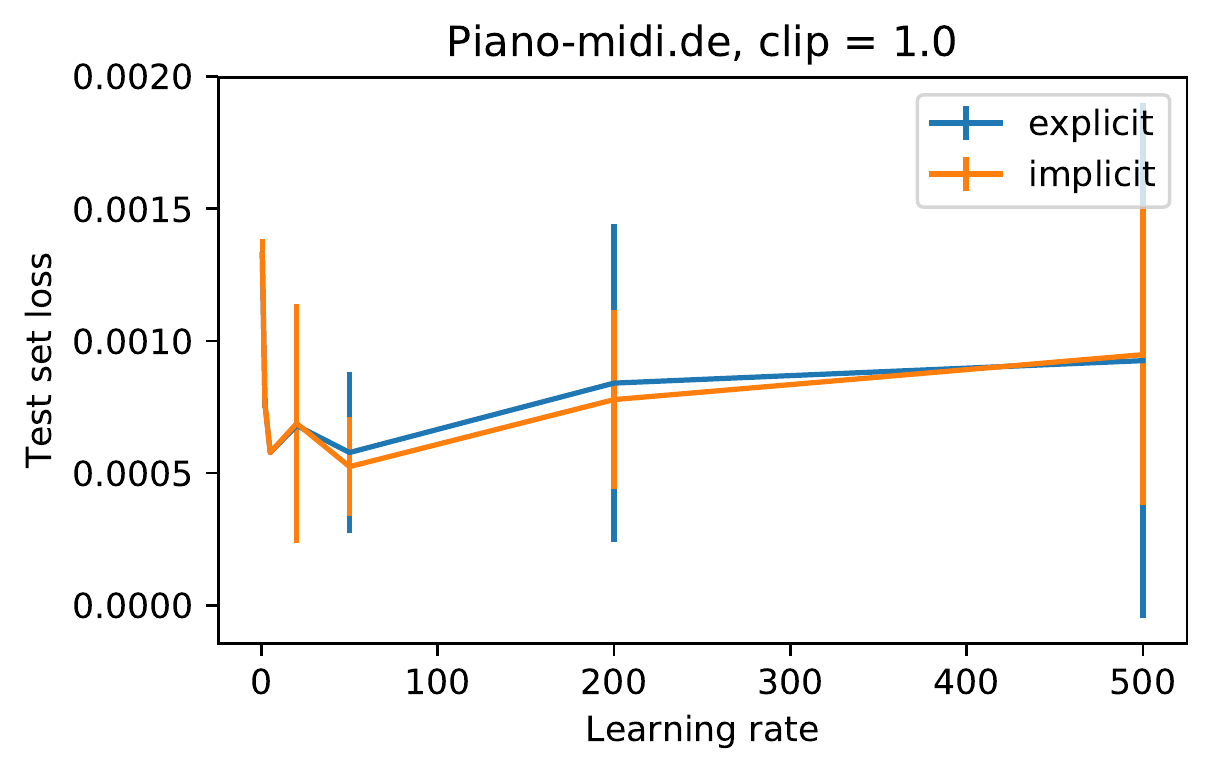}
\end{minipage}%
\hfill
\begin{minipage}{.49\textwidth}
  \centering
  \includegraphics[width=.99\linewidth]{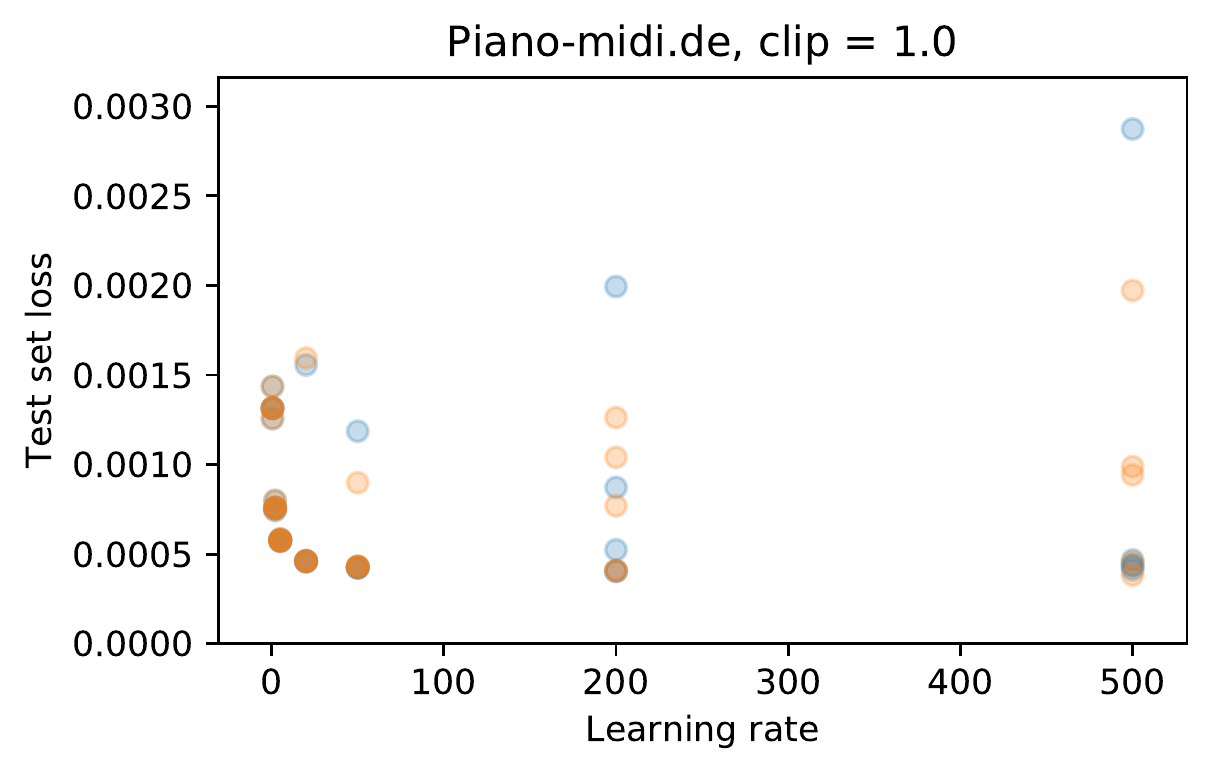}
\end{minipage}
\begin{minipage}{.49\textwidth}
  \centering
  \includegraphics[width=.99\linewidth]{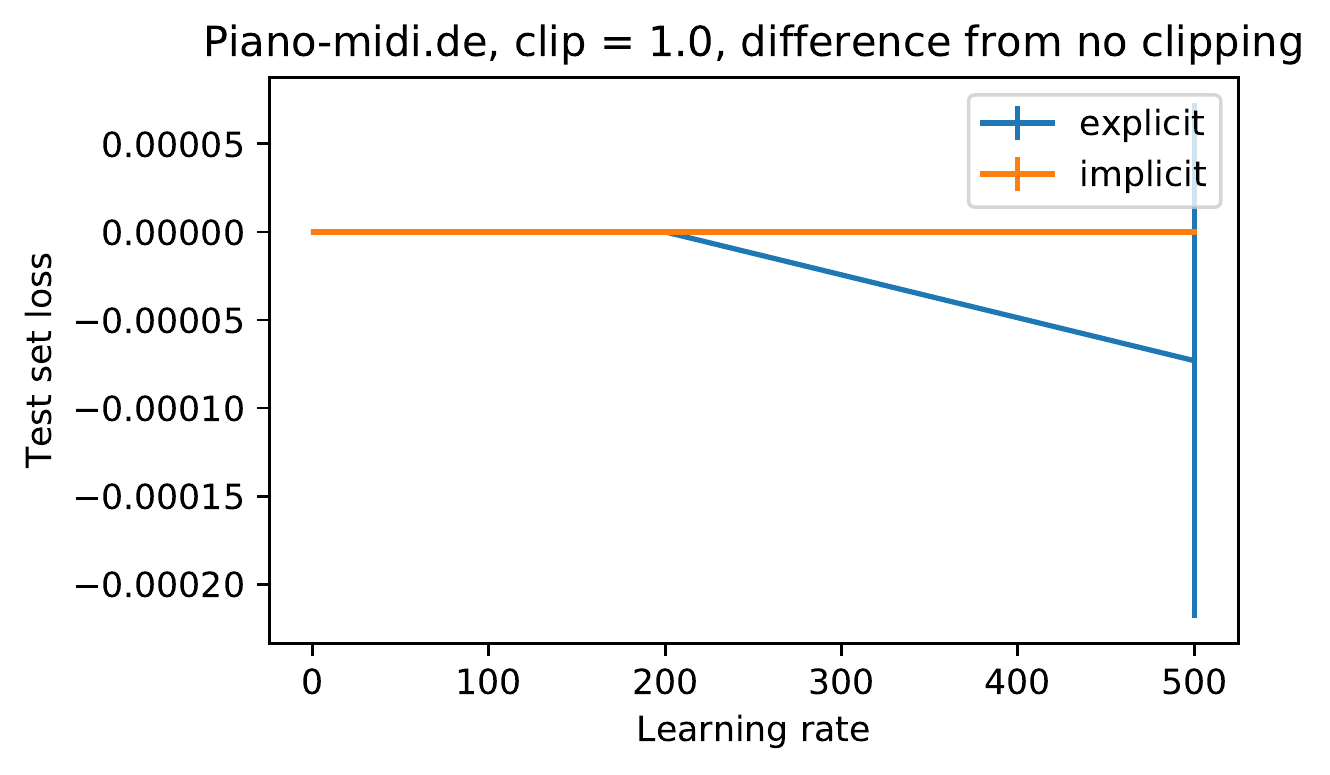}
\end{minipage}%
\hfill
\begin{minipage}{.49\textwidth}
  \centering
  \includegraphics[width=.99\linewidth]{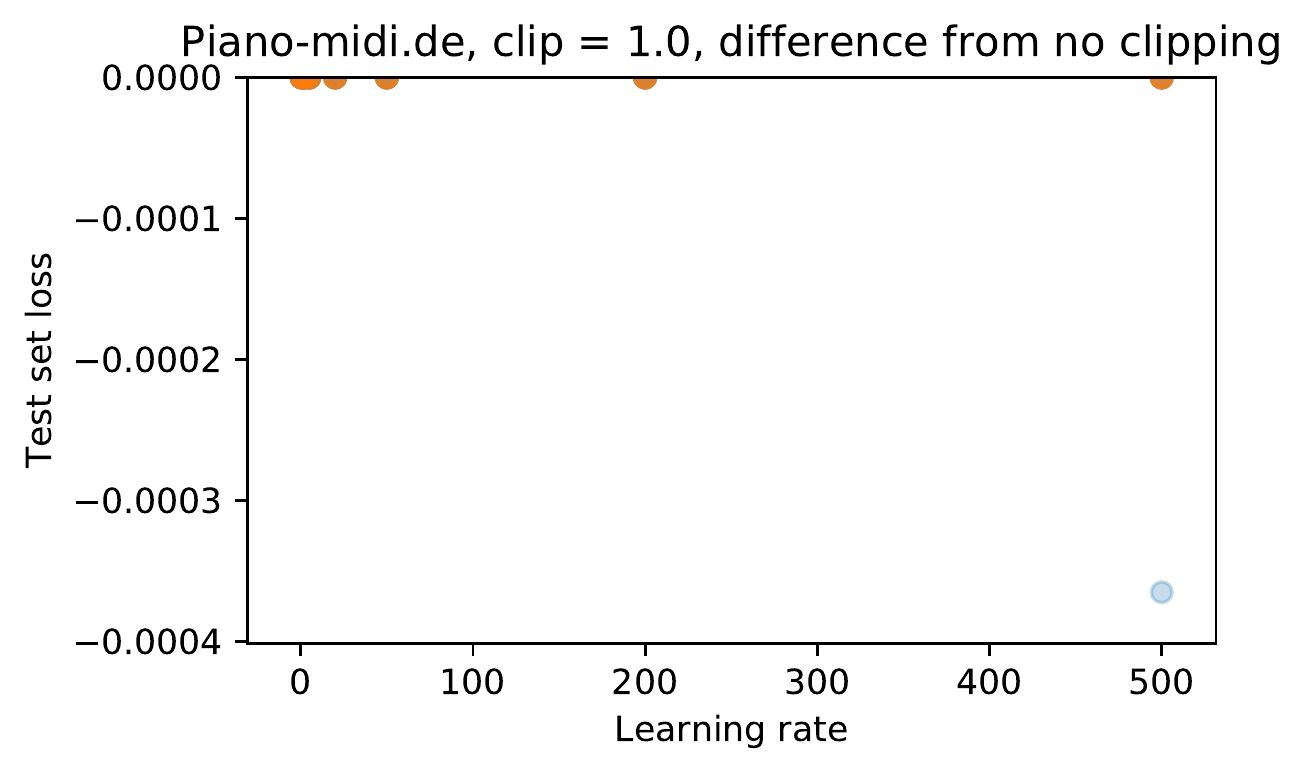}
\end{minipage}
\end{figure}

\begin{figure}[h]
\centering
\begin{minipage}{.49\textwidth}
  \centering
  \includegraphics[width=.99\linewidth]{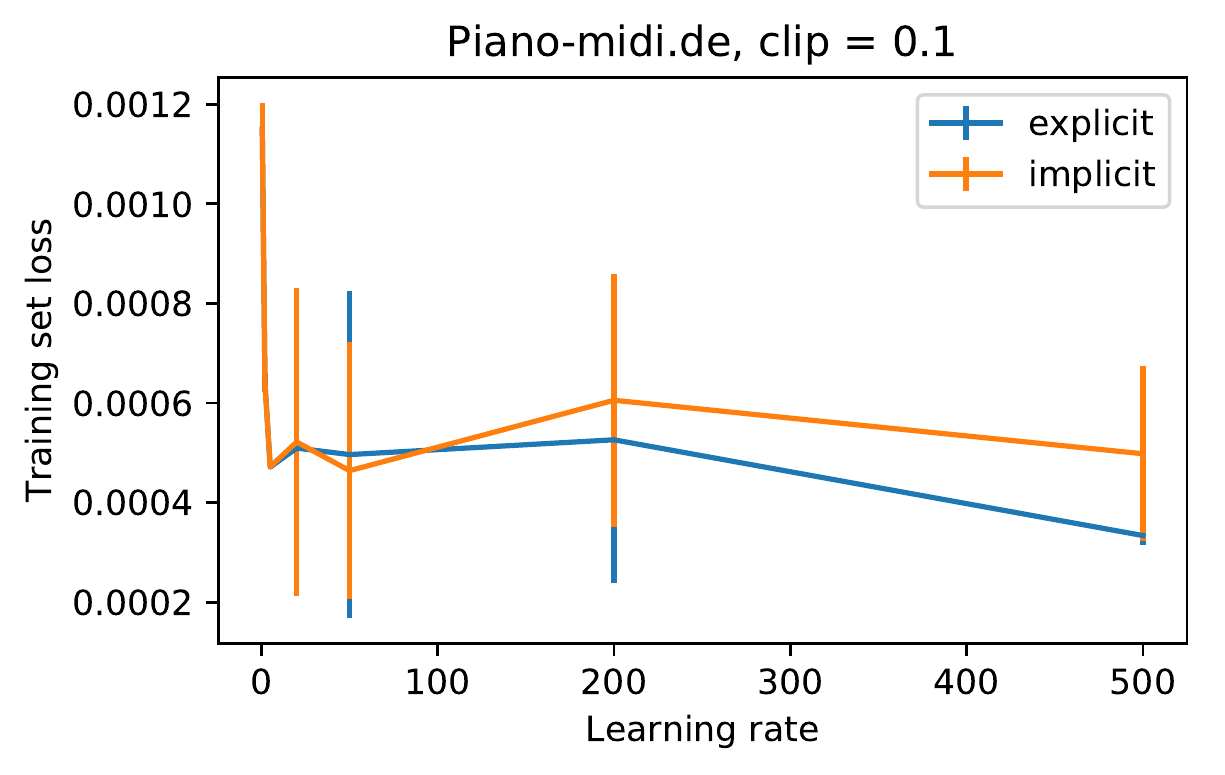}
\end{minipage}%
\hfill
\begin{minipage}{.49\textwidth}
  \centering
  \includegraphics[width=.99\linewidth]{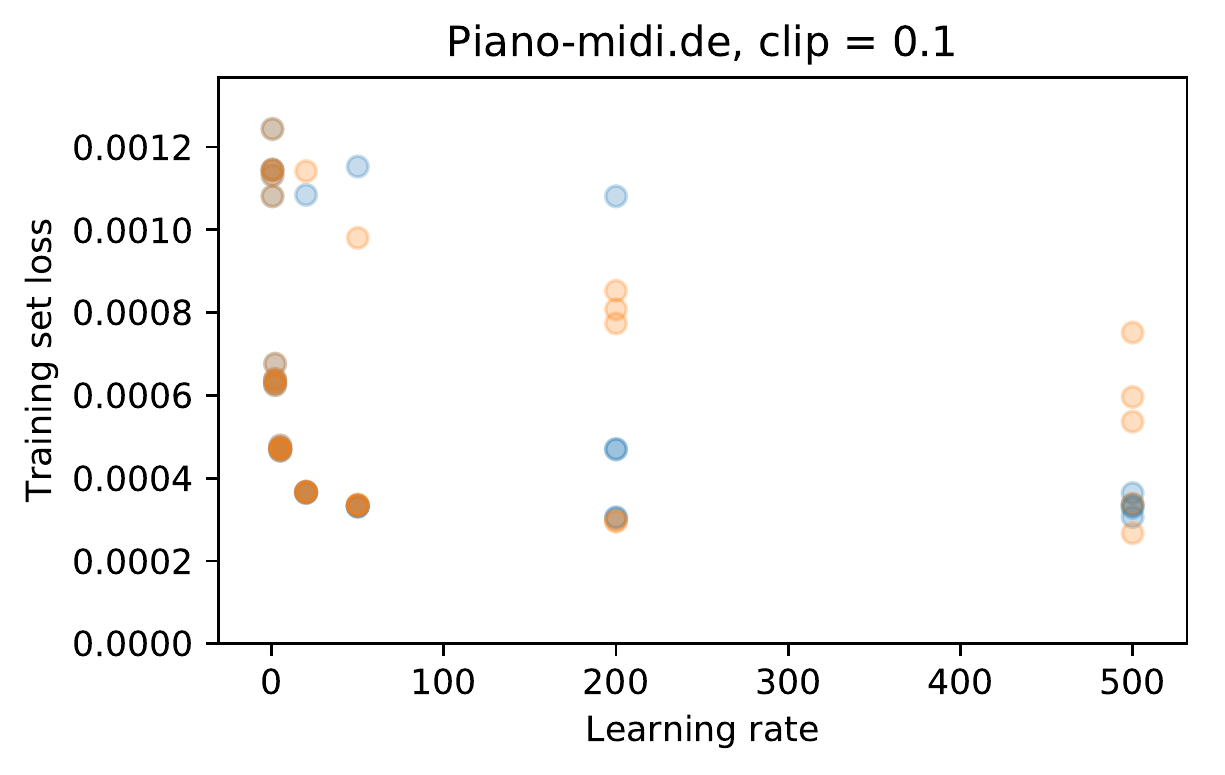}
\end{minipage}
\begin{minipage}{.49\textwidth}
  \centering
  \includegraphics[width=.99\linewidth]{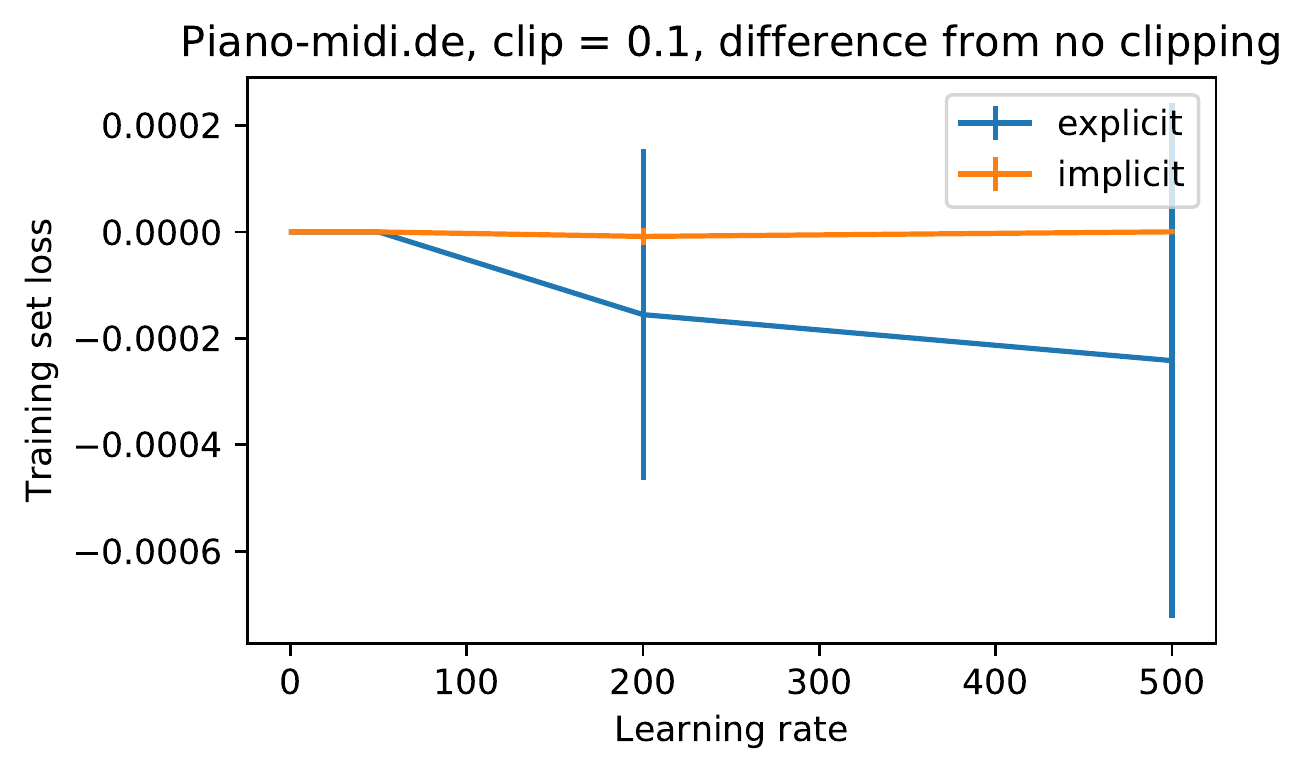}
\end{minipage}%
\hfill
\begin{minipage}{.49\textwidth}
  \centering
  \includegraphics[width=.99\linewidth]{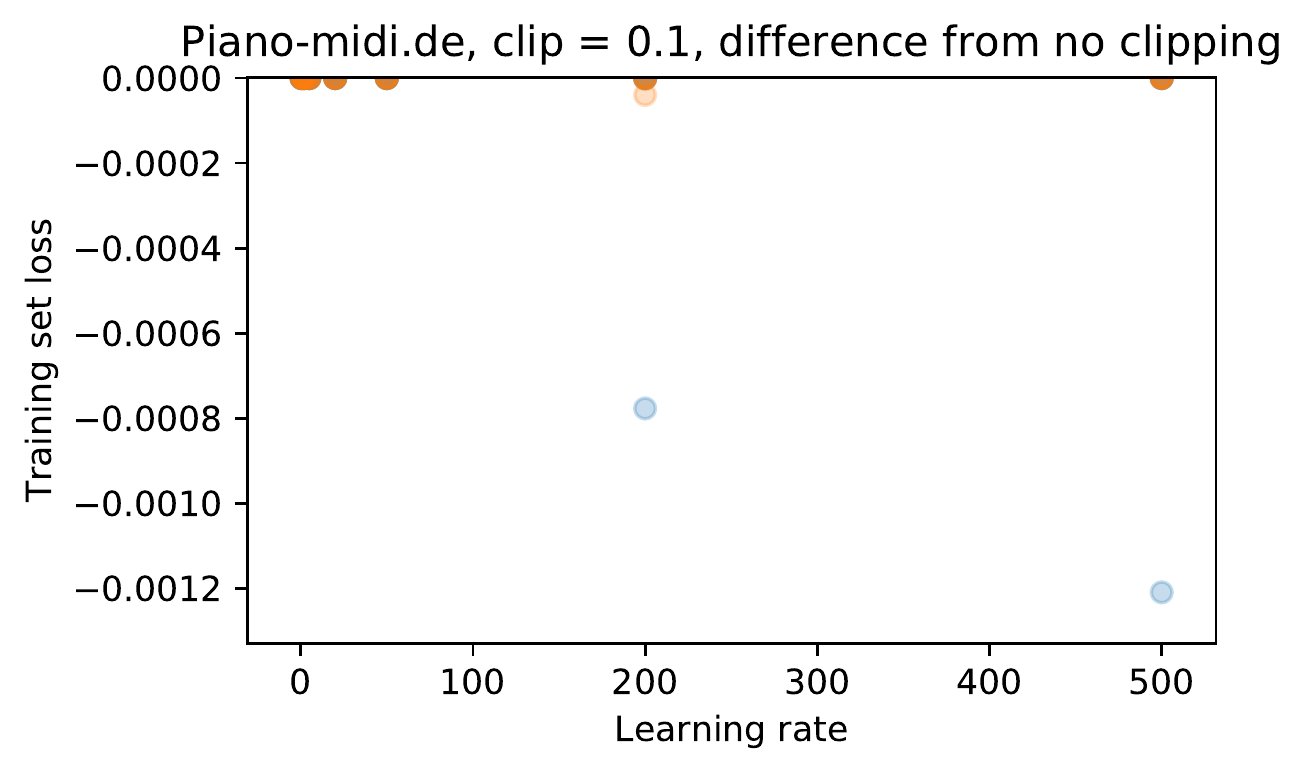}
\end{minipage}

\vspace{1cm}

\begin{minipage}{.49\textwidth}
  \centering
  \includegraphics[width=.99\linewidth]{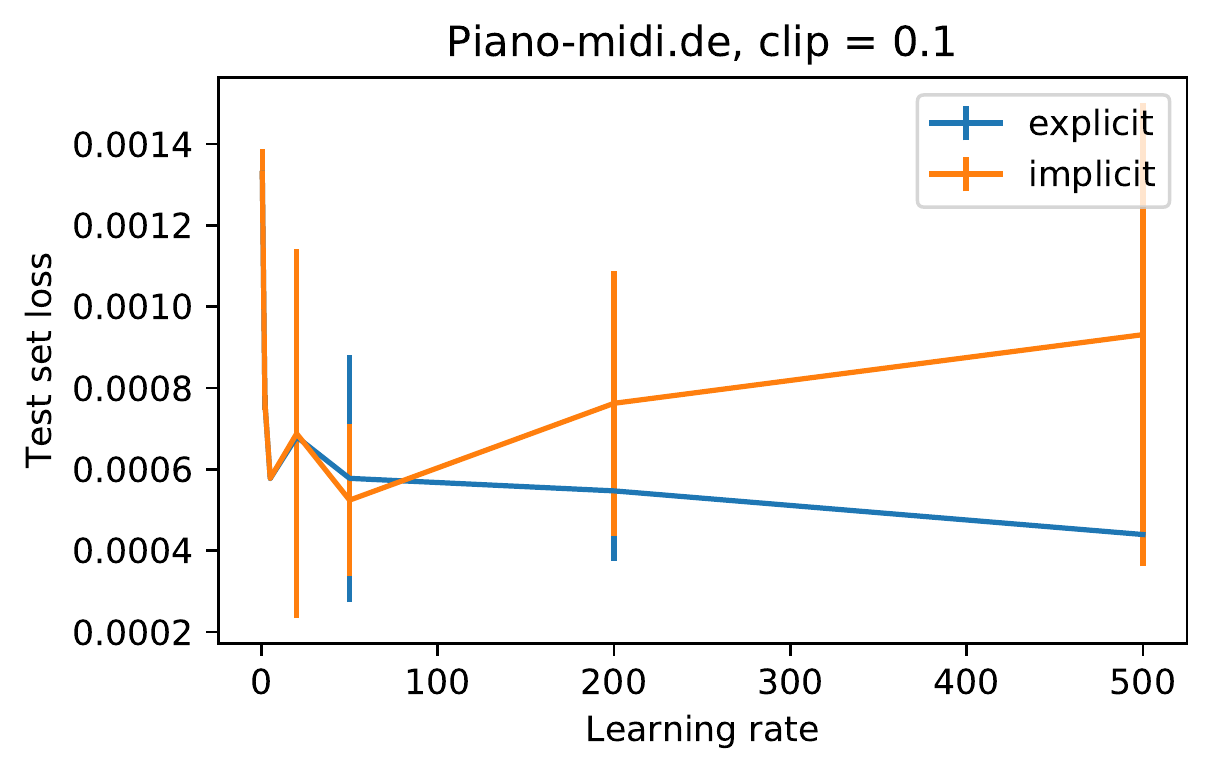}
\end{minipage}%
\hfill
\begin{minipage}{.49\textwidth}
  \centering
  \includegraphics[width=.99\linewidth]{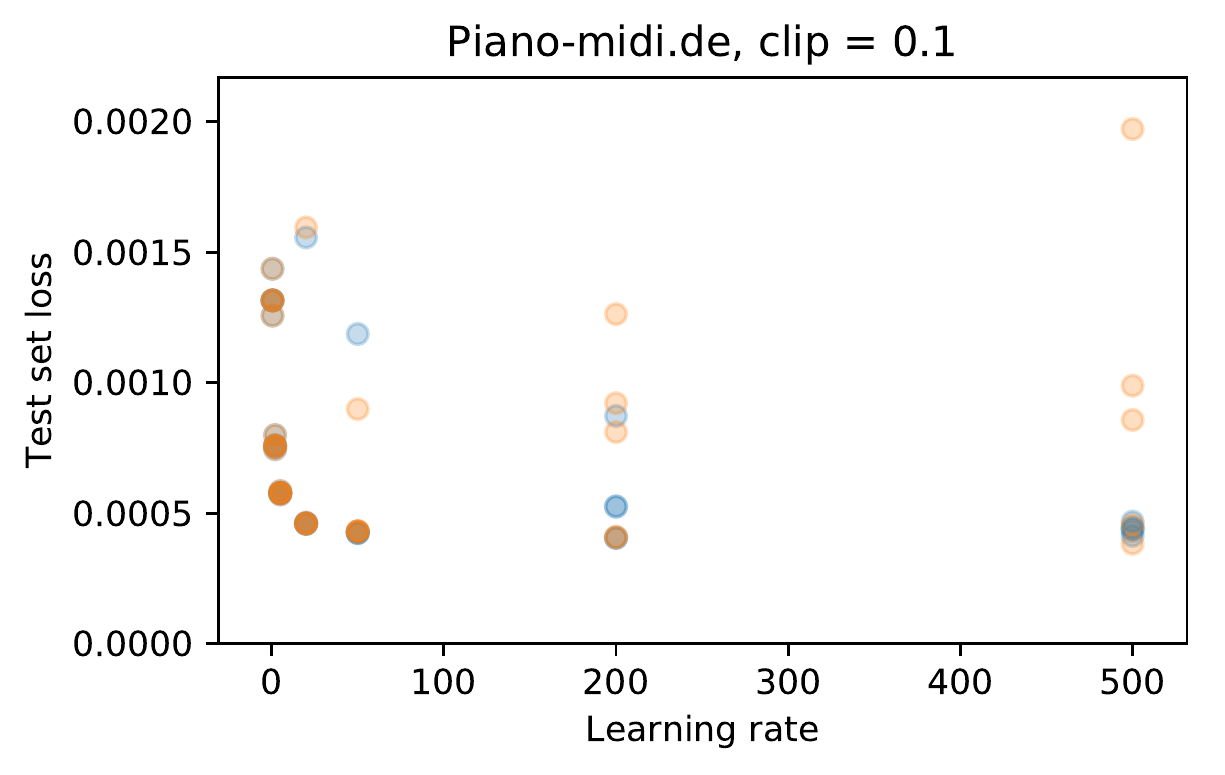}
\end{minipage}
\begin{minipage}{.49\textwidth}
  \centering
  \includegraphics[width=.99\linewidth]{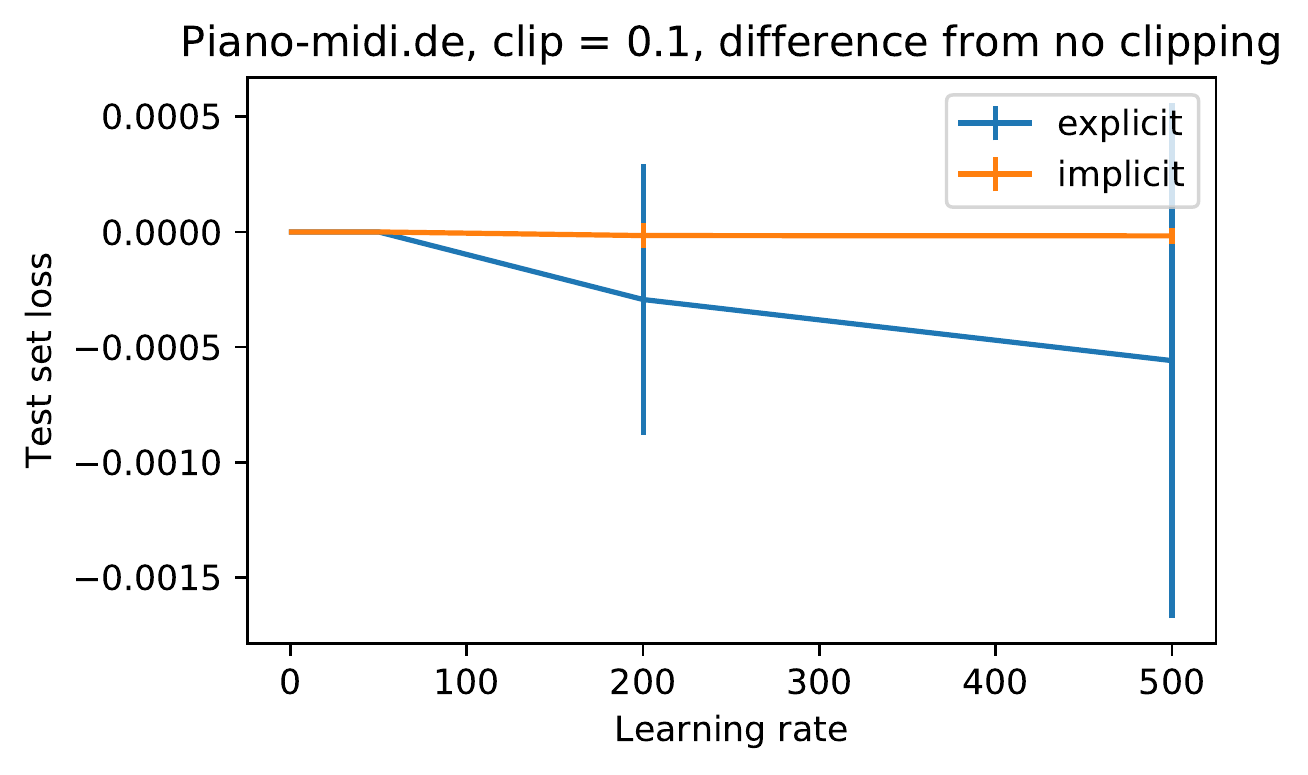}
\end{minipage}%
\hfill
\begin{minipage}{.49\textwidth}
  \centering
  \includegraphics[width=.99\linewidth]{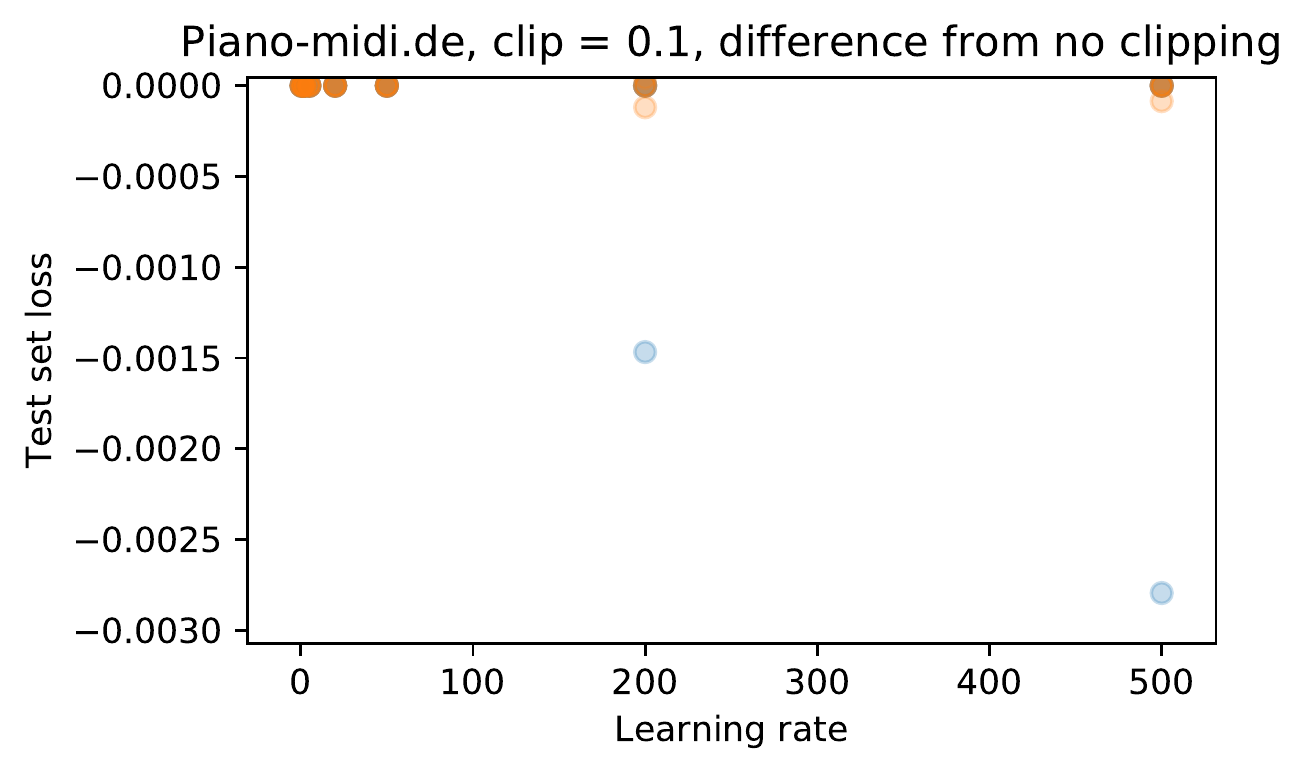}
\end{minipage}
\end{figure}

\begin{figure}[h]
\centering
\begin{minipage}{.49\textwidth}
  \centering
  \includegraphics[width=.99\linewidth]{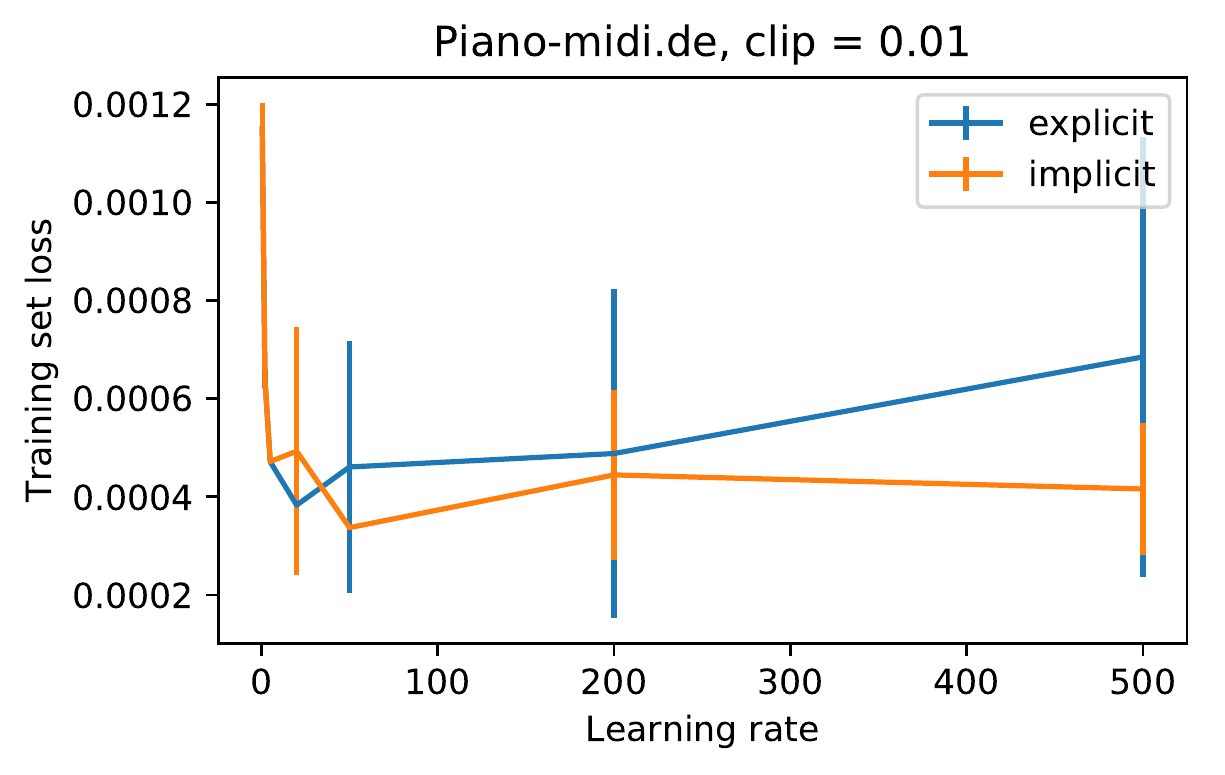}
\end{minipage}%
\hfill
\begin{minipage}{.49\textwidth}
  \centering
  \includegraphics[width=.99\linewidth]{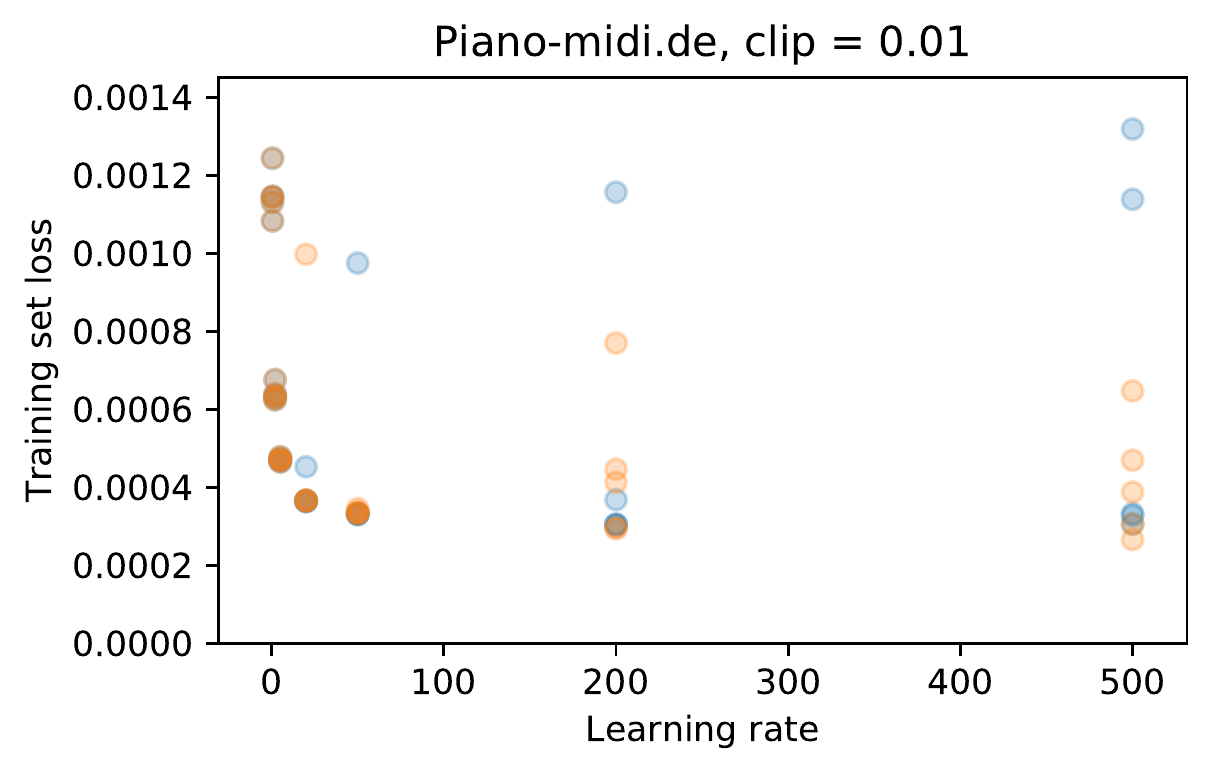}
\end{minipage}
\begin{minipage}{.49\textwidth}
  \centering
  \includegraphics[width=.99\linewidth]{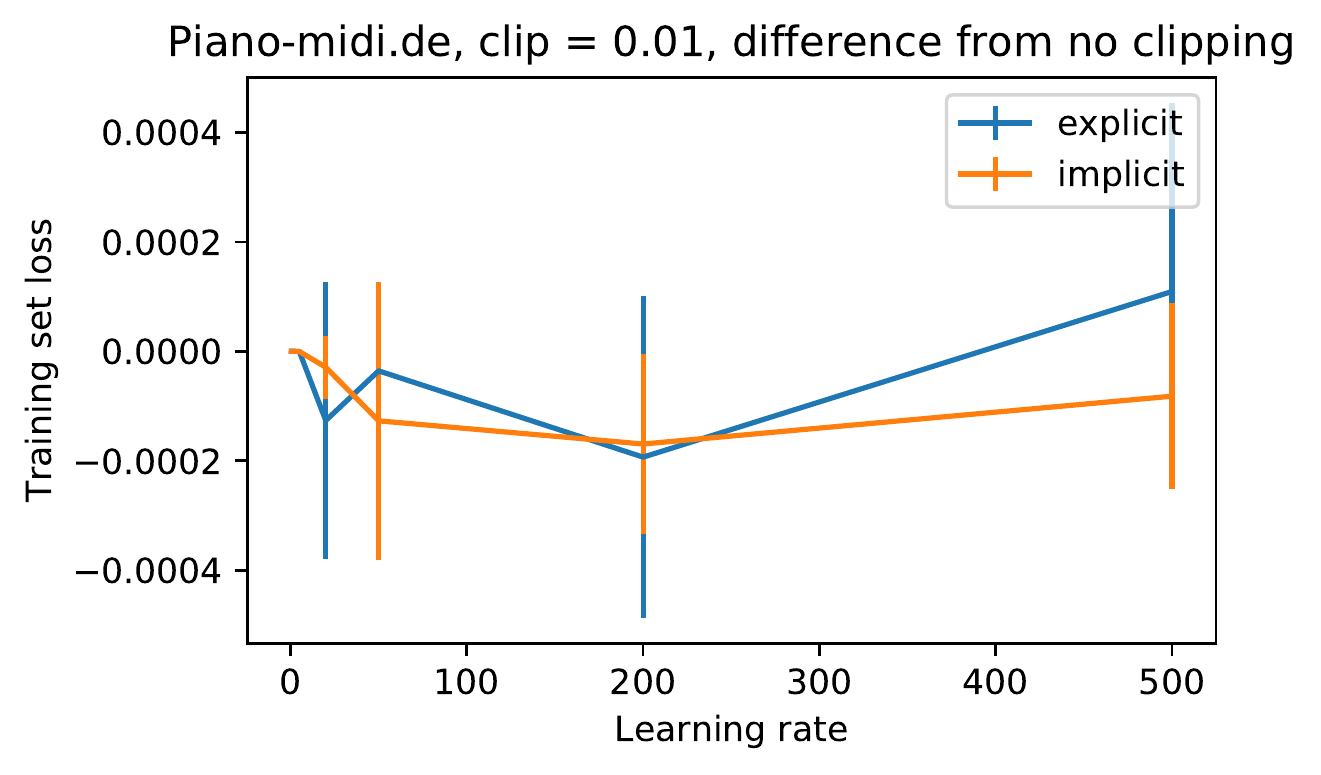}
\end{minipage}%
\hfill
\begin{minipage}{.49\textwidth}
  \centering
  \includegraphics[width=.99\linewidth]{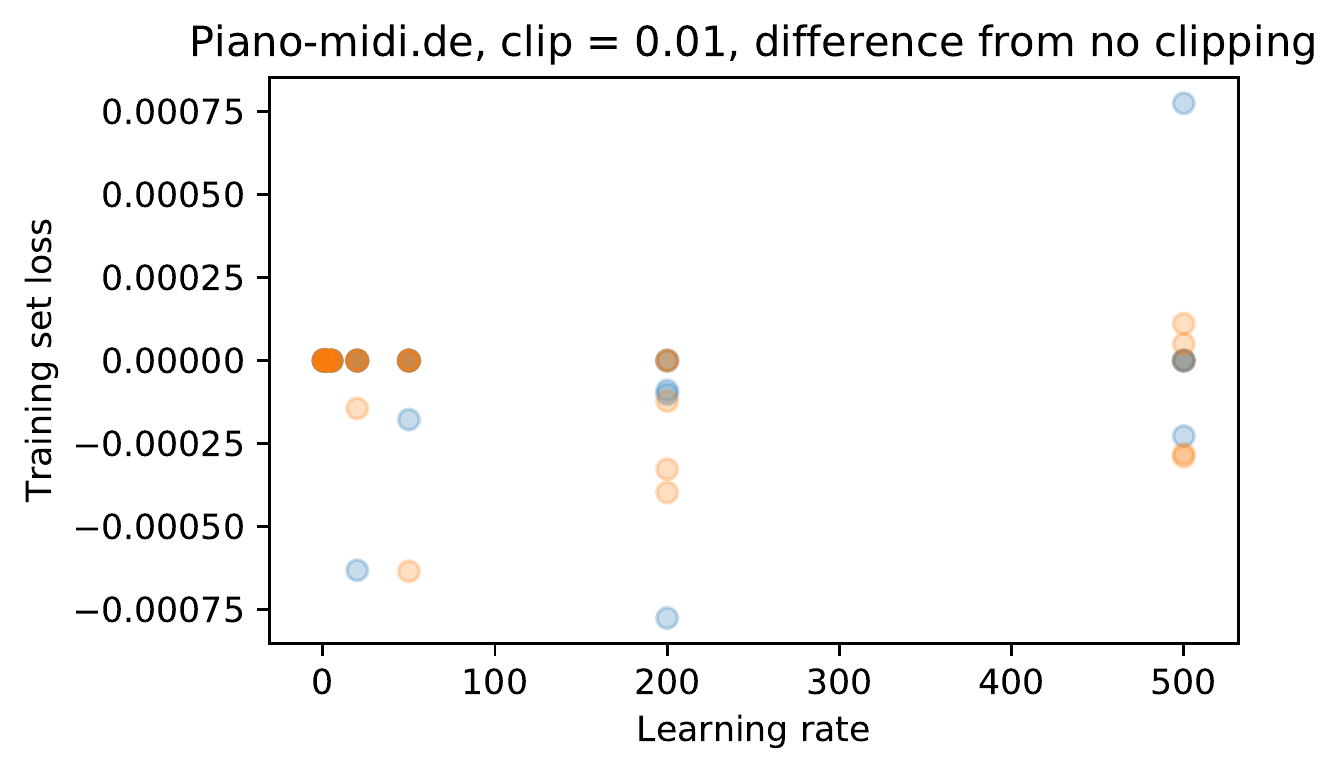}
\end{minipage}

\vspace{2cm}

\begin{minipage}{.49\textwidth}
  \centering
  \includegraphics[width=.99\linewidth]{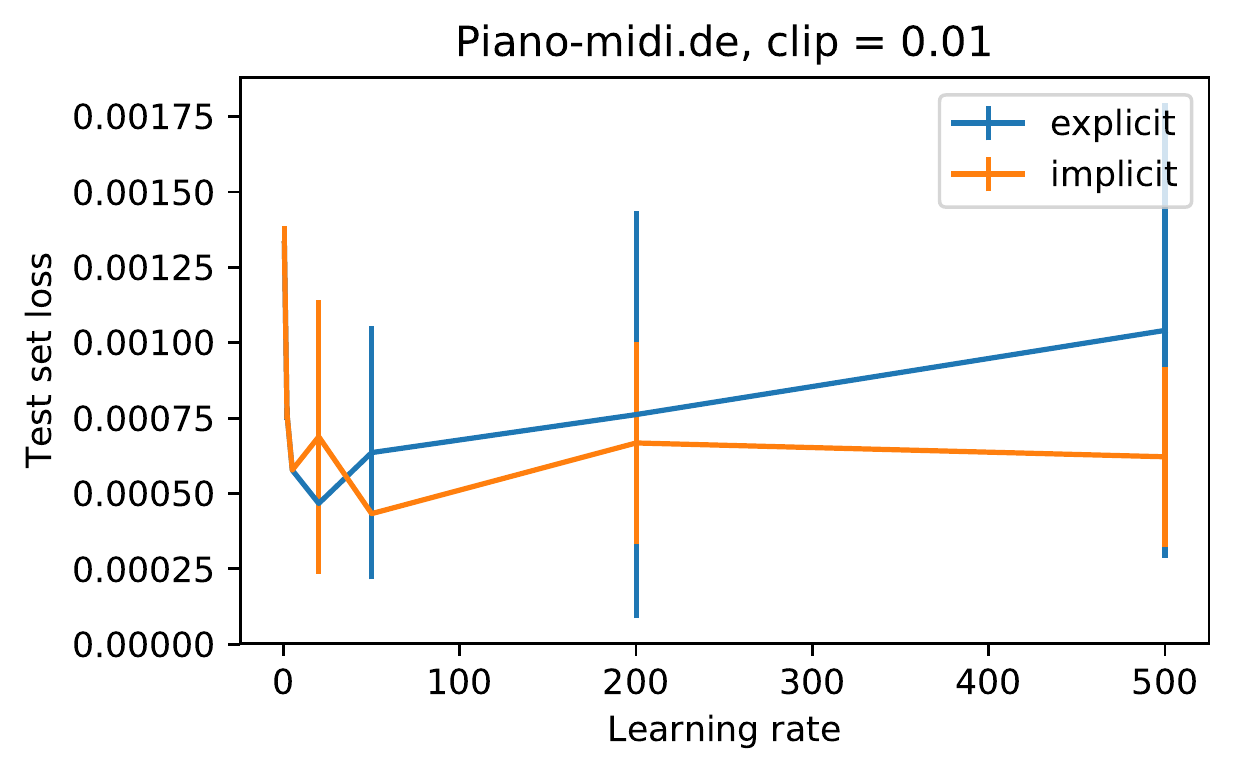}
\end{minipage}%
\hfill
\begin{minipage}{.49\textwidth}
  \centering
  \includegraphics[width=.99\linewidth]{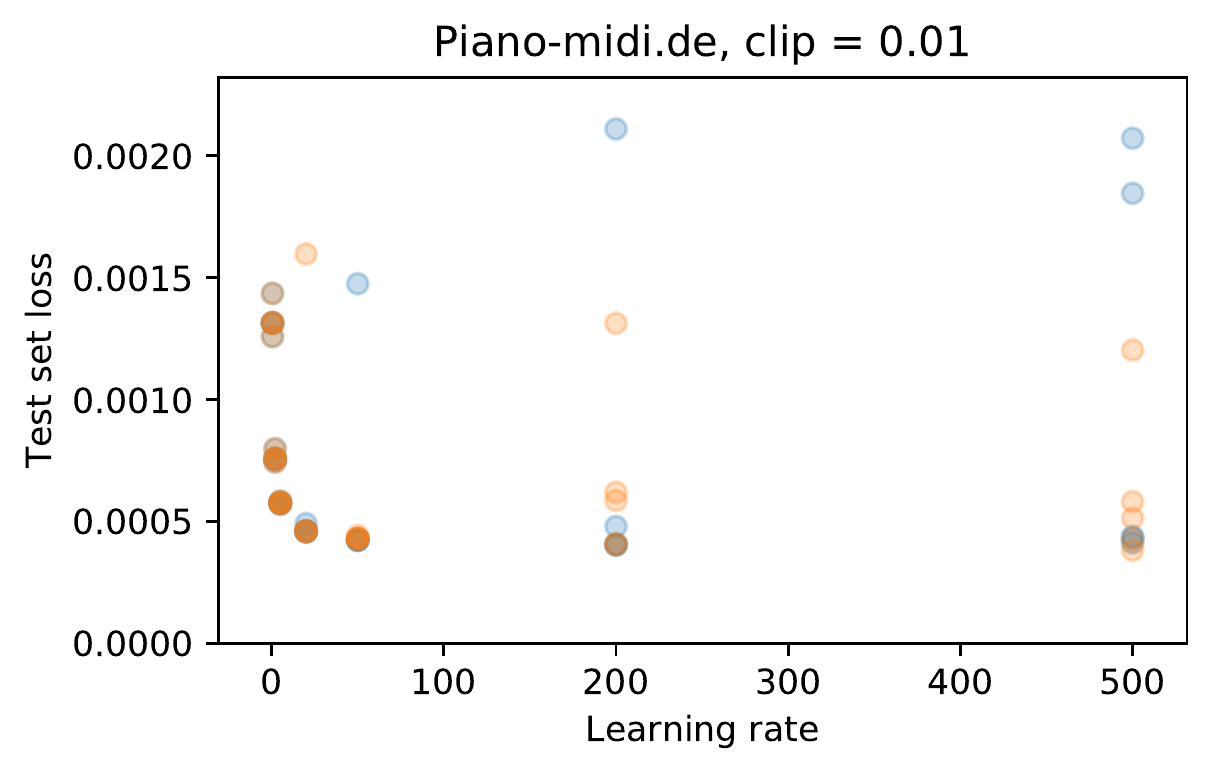}
\end{minipage}
\begin{minipage}{.49\textwidth}
  \centering
  \includegraphics[width=.99\linewidth]{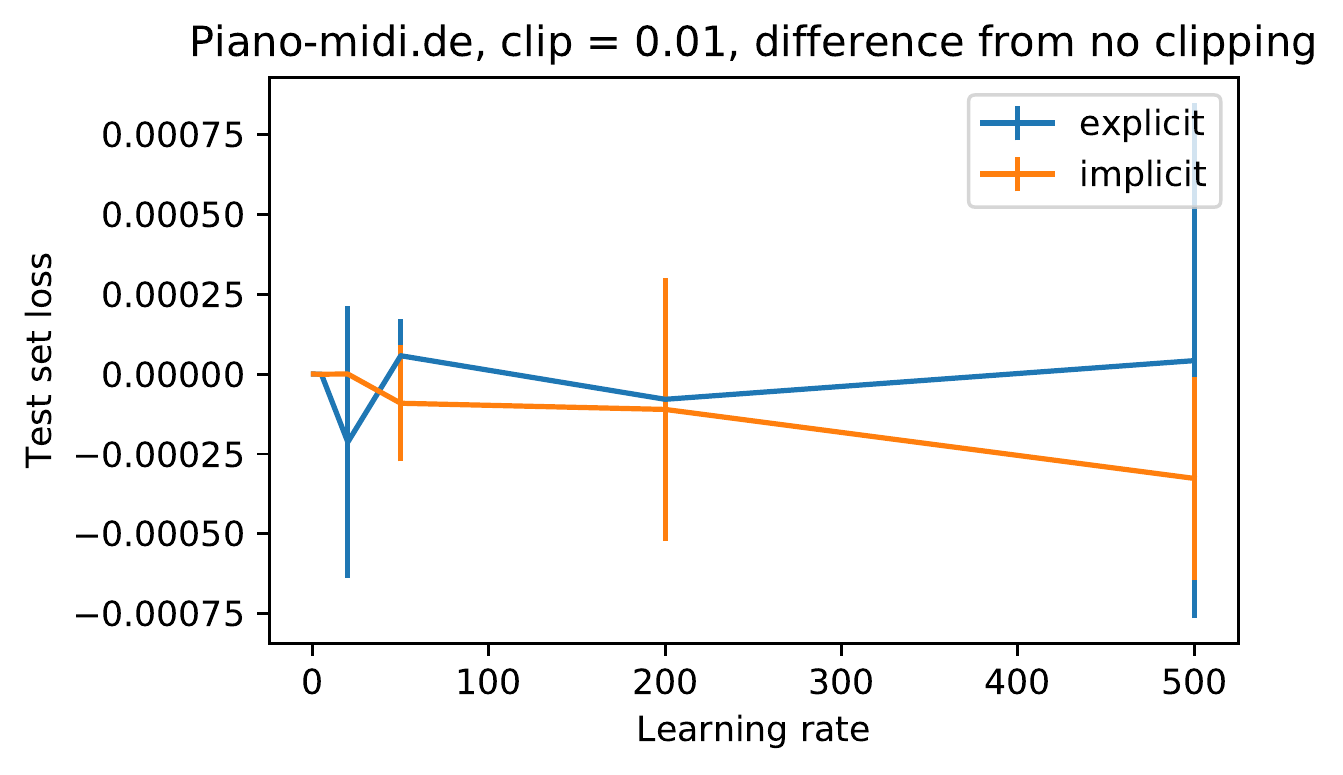}
\end{minipage}%
\hfill
\begin{minipage}{.49\textwidth}
  \centering
  \includegraphics[width=.99\linewidth]{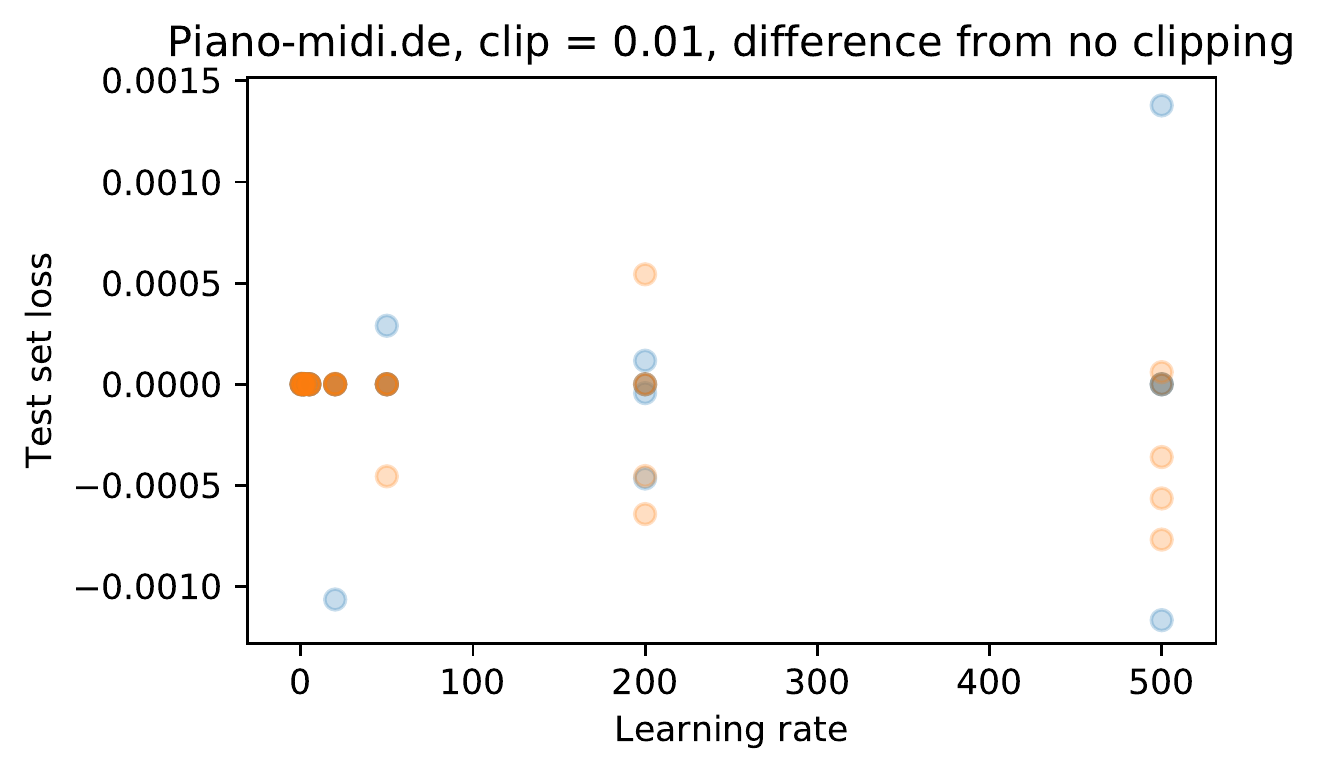}
\end{minipage}
\end{figure}


\clearpage
\subsection{UCI classification training losses}\label{app:classification_training_losses}
Here we display the training loss of IB and EB on each of the 121 UCI datasets at the end of the 10th epoch. We only ran each experiment with one random seed. The experiments are divided up into large (>1000 datapoints) and small ($\leq$1000 datapoints).

The performance of EB and IB are very similar for the smaller learning rates. This is still the case at the optimal learning rate for each dataset. However, once either EB or IB moves one learning rate higher, the loss tends to explode. On average it is 892 times higher than the for the optimal loss. This is typical behavior for learning rate sensitivity, as is shown in \cite[Fig 11.1, p. 425]{Goodfellow-et-al-2016}. The fact that there is such a large explosion in the loss indicates that the learning rate grid is too coarse to pick up the differences in behaviors in learning rates slightly larger than the optimal learning rate. Hence it struggles to systematically distinguish between the performance of EB and IB.

\begin{table}[h]
\scriptsize
\begin{tabular}{cccccccccccc}
\multirow{2}{*}{\textbf{Dataset} (large)} & \multicolumn{10}{c}{\textbf{Learning rate}} & \multirow{2}{*}{\textbf{SGD type}}\tabularnewline
 & 0.001 & 0.01 & 0.1 & 0.5 & 1.0 & 3.0 & 5.0 & 10.0 & 30.0 & 50.0 & \tabularnewline
\midrule
\multirow{2}{*}{abalone} & 1.11 & 1.10 & 0.83 & 0.77 & 0.76 & 12.30 & 6.13 & 25.20 & 57.09 & 212.61 & explicit\tabularnewline
 & 1.11 & 1.10 & 0.83 & 0.77 & 0.77 & 2.65 & 3.18 & 22.95 & 98.31 & 84.63 & implicit\tabularnewline
\hline
\multirow{2}{*}{adult} & 0.55 & 0.35 & 0.32 & 0.32 & 0.31 & 0.36 & 1.52 & 2.97 & 17.99 & 27.19 & explicit\tabularnewline
 & 0.55 & 0.35 & 0.32 & 0.32 & 0.32 & 0.34 & 1.15 & 4.31 & 12.22 & 20.27 & implicit\tabularnewline
\hline
\multirow{2}{*}{bank} & 0.60 & 0.36 & 0.25 & 0.27 & 0.30 & 0.52 & 1.29 & 14.74 & 27.23 & 113.73 & explicit\tabularnewline
 & 0.60 & 0.36 & 0.25 & 0.27 & 0.29 & 0.35 & 1.46 & 21.44 & 65.24 & 19.08 & implicit\tabularnewline
\hline
\multirow{2}{*}{car} & 1.18 & 1.01 & 0.85 & 0.50 & 0.39 & 0.59 & 5.28 & 4.55 & 53.97 & 106.93 & explicit\tabularnewline
 & 1.18 & 1.01 & 0.85 & 0.50 & 0.47 & 2.40 & 2.12 & 10.38 & 213.52 & 122.84 & implicit\tabularnewline
\hline
\multirow{2}{*}{cardiotocography-10clases} & 2.30 & 2.18 & 1.24 & 0.68 & 0.68 & 4.36 & 6.71 & 13.91 & 59.09 & 156.53 & explicit\tabularnewline
 & 2.30 & 2.18 & 1.24 & 0.69 & 0.68 & 5.01 & 21.19 & 12.75 & 119.45 & 113.76 & implicit\tabularnewline
\hline
\multirow{2}{*}{cardiotocography-3clases} & 1.09 & 0.74 & 0.30 & 0.21 & 0.21 & 0.67 & 5.59 & 19.91 & 23.95 & 173.83 & explicit\tabularnewline
 & 1.09 & 0.74 & 0.30 & 0.21 & 0.21 & 3.26 & 6.42 & 25.49 & 17.43 & 156.47 & implicit\tabularnewline
\hline
\multirow{2}{*}{chess-krvk} & 2.81 & 2.45 & 1.96 & 1.70 & 1.75 & 2.13 & 2.22 & 10.26 & 39.35 & 144.99 & explicit\tabularnewline
 & 2.81 & 2.45 & 1.97 & 1.72 & 1.75 & 1.94 & 3.92 & 72.35 & 55.65 & 136.72 & implicit\tabularnewline
\hline
\multirow{2}{*}{chess-krvkp} & 0.69 & 0.68 & 0.08 & 0.04 & 0.02 & 1.38 & 0.77 & 12.53 & 156.29 & 467.88 & explicit\tabularnewline
 & 0.69 & 0.68 & 0.08 & 0.07 & 0.07 & 0.33 & 2.80 & 12.91 & 26.88 & 56.18 & implicit\tabularnewline
\hline
\multirow{2}{*}{connect-4} & 0.56 & 0.54 & 0.48 & 0.40 & 0.38 & 2.36 & 4.59 & 9.25 & 34.85 & 59.43 & explicit\tabularnewline
 & 0.56 & 0.54 & 0.50 & 0.44 & 0.46 & 3.23 & 4.64 & 11.31 & 33.66 & 52.39 & implicit\tabularnewline
\hline
\multirow{2}{*}{contrac} & 1.08 & 1.07 & 1.07 & 0.99 & 0.96 & 1.00 & 5.34 & 17.32 & 149.70 & 133.28 & explicit\tabularnewline
 & 1.08 & 1.07 & 1.07 & 0.99 & 0.96 & 0.99 & 14.08 & 11.72 & 60.05 & 181.86 & implicit\tabularnewline
\hline
\multirow{2}{*}{hill-valley} & 0.70 & 0.69 & 0.69 & 0.69 & 0.69 & 84.43 & 38.06 & 943.16 & 3267.16 & 1005.30 & explicit\tabularnewline
 & 0.70 & 0.69 & 0.69 & 0.69 & 0.70 & 168.00 & 180.79 & 907.54 & 2014.49 & 90.33 & implicit\tabularnewline
\hline
\multirow{2}{*}{image-segmentation} & 1.95 & 1.94 & 0.95 & 0.20 & 0.12 & 9.13 & 11.97 & 13.45 & 95.62 & 72.58 & explicit\tabularnewline
 & 1.95 & 1.94 & 0.95 & 0.22 & 0.15 & 2.77 & 2.10 & 13.21 & 277.97 & 131.41 & implicit\tabularnewline
\hline
\multirow{2}{*}{led-display} & 2.32 & 2.32 & 2.30 & 1.73 & 1.25 & 1.63 & 4.25 & 7.59 & 109.59 & 209.45 & explicit\tabularnewline
 & 2.32 & 2.32 & 2.30 & 1.73 & 1.26 & 3.13 & 5.43 & 30.01 & 93.78 & 240.85 & implicit\tabularnewline
\hline
\multirow{2}{*}{letter} & 3.27 & 3.21 & 1.31 & 0.61 & 0.64 & 0.83 & 2.47 & 5.98 & 35.33 & 60.50 & explicit\tabularnewline
 & 3.27 & 3.21 & 1.31 & 0.59 & 0.68 & 0.84 & 1.85 & 5.92 & 27.24 & 39.06 & implicit\tabularnewline
\hline
\multirow{2}{*}{magic} & 0.65 & 0.53 & 0.47 & 0.41 & 0.41 & 0.46 & 3.59 & 24.93 & 40.47 & 57.64 & explicit\tabularnewline
 & 0.65 & 0.53 & 0.48 & 0.42 & 0.43 & 0.53 & 5.43 & 5.41 & 41.77 & 146.46 & implicit\tabularnewline
\hline
\multirow{2}{*}{miniboone} & 0.37 & 0.29 & 0.23 & 0.21 & 0.22 & 1.40 & 5.60 & 18.28 & 60.37 & 190.87 & explicit\tabularnewline
 & 0.37 & 0.29 & 0.22 & 0.22 & 0.22 & 1.12 & 2.71 & 3.39 & 20.09 & 32.84 & implicit\tabularnewline
\hline
\multirow{2}{*}{molec-biol-splice} & 1.05 & 0.98 & 0.36 & 0.31 & 0.25 & 5.32 & 9.87 & 30.26 & 146.88 & 227.42 & explicit\tabularnewline
 & 1.05 & 0.98 & 0.36 & 0.34 & 0.33 & 27.64 & 8.85 & 29.27 & 149.89 & 188.89 & implicit\tabularnewline
\hline
\multirow{2}{*}{mushroom} & 0.69 & 0.36 & 0.01 & 0.00 & 0.00 & 0.00 & 0.06 & 1.96 & 157.31 & 87.08 & explicit\tabularnewline
 & 0.69 & 0.36 & 0.01 & 0.00 & 0.00 & 0.00 & 0.05 & 4.86 & 11.92 & 42.45 & implicit\tabularnewline
\hline
\multirow{2}{*}{musk-2} & 0.49 & 0.19 & 0.08 & 0.03 & 2.75 & 34.14 & 112.76 & 228.01 & 436.03 & 288.58 & explicit\tabularnewline
 & 0.49 & 0.19 & 0.08 & 0.03 & 2.46 & 63.96 & 109.08 & 36.49 & 263.51 & 105.52 & implicit\tabularnewline
\hline
\multirow{2}{*}{nursery} & 1.47 & 1.21 & 0.51 & 0.12 & 0.14 & 0.31 & 0.87 & 6.99 & 63.94 & 195.34 & explicit\tabularnewline
 & 1.47 & 1.21 & 0.52 & 0.18 & 0.12 & 0.19 & 1.36 & 33.21 & 129.24 & 514.70 & implicit\tabularnewline
\hline
\multirow{2}{*}{oocytes-merluccius-nucleus-4d} & 0.65 & 0.63 & 0.61 & 0.69 & 0.73 & 18.79 & 19.89 & 82.75 & 767.96 & 917.21 & explicit\tabularnewline
 & 0.65 & 0.63 & 0.61 & 0.68 & 0.72 & 30.91 & 91.02 & 108.44 & 291.86 & 181.39 & implicit\tabularnewline
\hline
\multirow{2}{*}{oocytes-merluccius-states-2f} & 1.04 & 0.88 & 0.35 & 0.22 & 0.20 & 1.46 & 3.09 & 34.59 & 255.92 & 35.96 & explicit\tabularnewline
 & 1.04 & 0.88 & 0.35 & 0.22 & 0.21 & 3.73 & 7.49 & 42.19 & 53.97 & 86.39 & implicit\tabularnewline
\bottomrule
\end{tabular}
\end{table}

\begin{table}[h]
\scriptsize
\begin{tabular}{cccccccccccc}
\multirow{2}{*}{\textbf{Dataset} (large)} & \multicolumn{10}{c}{\textbf{Learning rate}} & \multirow{2}{*}{\textbf{SGD type}}\tabularnewline
 & 0.001 & 0.01 & 0.1 & 0.5 & 1.0 & 3.0 & 5.0 & 10.0 & 30.0 & 50.0 & \tabularnewline
\midrule
\multirow{2}{*}{optical} & 2.27 & 1.55 & 0.10 & 0.02 & 0.01 & 0.23 & 1.59 & 10.73 & 51.20 & 121.22 & explicit\tabularnewline
 & 2.27 & 1.55 & 0.10 & 0.03 & 0.01 & 0.22 & 0.42 & 3.97 & 8.87 & 15.60 & implicit\tabularnewline
\hline
\multirow{2}{*}{ozone} & 0.51 & 0.18 & 0.09 & 0.09 & 0.10 & 9.91 & 11.00 & 8.83 & 115.42 & 30.89 & explicit\tabularnewline
 & 0.51 & 0.18 & 0.09 & 0.09 & 0.10 & 2.46 & 2.67 & 57.39 & 13.99 & 79.68 & implicit\tabularnewline
\hline
\multirow{2}{*}{page-blocks} & 1.30 & 0.47 & 0.25 & 0.16 & 0.13 & 0.13 & 1.36 & 6.97 & 7.46 & 23.52 & explicit\tabularnewline
 & 1.30 & 0.47 & 0.25 & 0.16 & 0.15 & 0.13 & 2.79 & 24.66 & 43.27 & 111.66 & implicit\tabularnewline
\hline
\multirow{2}{*}{pendigits} & 2.30 & 2.22 & 0.28 & 0.07 & 0.06 & 0.10 & 0.95 & 4.35 & 20.55 & 66.20 & explicit\tabularnewline
 & 2.30 & 2.22 & 0.28 & 0.05 & 0.05 & 0.18 & 0.26 & 0.72 & 12.03 & 32.69 & implicit\tabularnewline
\hline
\multirow{2}{*}{plant-margin} & 4.60 & 4.59 & 3.91 & 1.39 & 0.82 & 6.24 & 16.50 & 52.11 & 152.54 & 299.85 & explicit\tabularnewline
 & 4.60 & 4.59 & 3.91 & 1.39 & 0.86 & 21.08 & 28.08 & 46.19 & 117.16 & 208.09 & implicit\tabularnewline
\hline
\multirow{2}{*}{plant-shape} & 4.60 & 4.58 & 3.95 & 3.00 & 2.89 & 23.00 & 42.98 & 79.10 & 293.86 & 496.63 & explicit\tabularnewline
 & 4.60 & 4.58 & 3.95 & 3.00 & 2.87 & 19.54 & 43.80 & 85.42 & 234.65 & 572.32 & implicit\tabularnewline
\hline
\multirow{2}{*}{plant-texture} & 4.61 & 4.59 & 4.17 & 1.20 & 0.55 & 0.34 & 16.13 & 39.06 & 123.57 & 193.97 & explicit\tabularnewline
 & 4.61 & 4.59 & 4.17 & 1.20 & 0.57 & 0.84 & 11.28 & 17.27 & 68.12 & 208.77 & implicit\tabularnewline
\hline
\multirow{2}{*}{ringnorm} & 0.69 & 0.69 & 0.46 & 0.19 & 0.24 & 3.51 & 4.33 & 11.12 & 33.76 & 68.65 & explicit\tabularnewline
 & 0.69 & 0.69 & 0.47 & 0.22 & 0.21 & 2.64 & 5.10 & 53.34 & 29.11 & 72.79 & implicit\tabularnewline
\hline
\multirow{2}{*}{semeion} & 2.18 & 1.04 & 0.08 & 0.01 & 0.02 & 5.35 & 21.05 & 92.46 & 978.63 & 1565.21 & explicit\tabularnewline
 & 2.18 & 1.04 & 0.08 & 0.01 & 0.00 & 90.42 & 1.71 & 7.28 & 149.68 & 559.76 & implicit\tabularnewline
\hline
\multirow{2}{*}{spambase} & 0.67 & 0.36 & 0.18 & 0.16 & 0.17 & 8.17 & 2.06 & 20.93 & 67.37 & 102.59 & explicit\tabularnewline
 & 0.67 & 0.36 & 0.18 & 0.16 & 0.17 & 9.62 & 3.12 & 11.46 & 80.37 & 161.66 & implicit\tabularnewline
\hline
\multirow{2}{*}{statlog-german-credit} & 0.73 & 0.65 & 0.59 & 0.46 & 0.47 & 0.54 & 19.04 & 28.39 & 223.84 & 279.24 & explicit\tabularnewline
 & 0.73 & 0.65 & 0.59 & 0.46 & 0.47 & 0.55 & 12.99 & 32.00 & 36.53 & 281.25 & implicit\tabularnewline
\hline
\multirow{2}{*}{statlog-image} & 1.95 & 1.94 & 0.98 & 0.20 & 0.13 & 2.63 & 10.29 & 16.16 & 104.27 & 86.13 & explicit\tabularnewline
 & 1.95 & 1.94 & 0.98 & 0.20 & 0.18 & 3.61 & 7.49 & 62.98 & 65.65 & 62.29 & implicit\tabularnewline
\hline
\multirow{2}{*}{statlog-landsat} & 1.73 & 0.90 & 0.37 & 0.32 & 0.30 & 3.21 & 7.15 & 25.85 & 51.35 & 53.25 & explicit\tabularnewline
 & 1.73 & 0.90 & 0.37 & 0.33 & 0.33 & 5.02 & 8.10 & 6.22 & 34.73 & 126.06 & implicit\tabularnewline
\hline
\multirow{2}{*}{statlog-shuttle} & 0.80 & 0.18 & 0.03 & 0.04 & 0.01 & 0.02 & 0.54 & 0.71 & 40.30 & 4.51 & explicit\tabularnewline
 & 0.80 & 0.18 & 0.03 & 0.02 & 0.02 & 0.02 & 0.12 & 0.26 & 7.60 & 13.13 & implicit\tabularnewline
\hline
\multirow{2}{*}{steel-plates} & 1.91 & 1.80 & 1.22 & 0.72 & 0.69 & 3.44 & 14.42 & 12.85 & 99.23 & 87.18 & explicit\tabularnewline
 & 1.91 & 1.80 & 1.22 & 0.73 & 0.70 & 68.19 & 11.35 & 169.10 & 136.68 & 340.84 & implicit\tabularnewline
\hline
\multirow{2}{*}{thyroid} & 0.68 & 0.30 & 0.13 & 0.07 & 0.06 & 0.06 & 1.37 & 3.02 & 8.48 & 52.41 & explicit\tabularnewline
 & 0.68 & 0.30 & 0.13 & 0.07 & 0.07 & 0.09 & 2.71 & 6.23 & 33.58 & 51.21 & implicit\tabularnewline
\hline
\multirow{2}{*}{titanic} & 0.64 & 0.63 & 0.63 & 0.52 & 0.52 & 0.59 & 0.98 & 10.04 & 61.18 & 87.35 & explicit\tabularnewline
 & 0.64 & 0.63 & 0.63 & 0.52 & 0.53 & 0.61 & 4.44 & 14.18 & 81.66 & 58.33 & implicit\tabularnewline
\hline
\multirow{2}{*}{twonorm} & 0.69 & 0.55 & 0.06 & 0.06 & 0.06 & 0.07 & 0.59 & 5.49 & 23.61 & 26.69 & explicit\tabularnewline
 & 0.69 & 0.55 & 0.06 & 0.06 & 0.07 & 0.08 & 2.58 & 3.33 & 4.87 & 21.40 & implicit\tabularnewline
\hline
\multirow{2}{*}{wall-following} & 1.38 & 1.18 & 0.67 & 0.33 & 0.34 & 1.56 & 6.02 & 19.54 & 63.01 & 149.71 & explicit\tabularnewline
 & 1.38 & 1.18 & 0.68 & 0.34 & 0.32 & 3.83 & 5.84 & 11.73 & 33.44 & 367.02 & implicit\tabularnewline
\hline
\multirow{2}{*}{waveform} & 1.09 & 0.89 & 0.28 & 0.28 & 0.29 & 1.43 & 5.87 & 16.95 & 43.49 & 37.55 & explicit\tabularnewline
 & 1.09 & 0.89 & 0.28 & 0.28 & 0.29 & 1.81 & 6.47 & 9.13 & 32.50 & 30.61 & implicit\tabularnewline
\hline
\multirow{2}{*}{waveform-noise} & 1.09 & 0.82 & 0.29 & 0.33 & 0.29 & 10.64 & 21.44 & 41.78 & 107.99 & 249.30 & explicit\tabularnewline
 & 1.09 & 0.82 & 0.29 & 0.33 & 0.31 & 7.92 & 9.56 & 10.69 & 87.91 & 35.45 & implicit\tabularnewline
\hline
\multirow{2}{*}{wine-quality-red} & 1.78 & 1.55 & 1.20 & 0.96 & 0.95 & 1.01 & 4.72 & 16.11 & 72.08 & 206.43 & explicit\tabularnewline
 & 1.78 & 1.55 & 1.20 & 0.97 & 0.96 & 0.99 & 7.27 & 24.77 & 71.79 & 183.35 & implicit\tabularnewline
\hline
\multirow{2}{*}{wine-quality-white} & 1.96 & 1.42 & 1.18 & 1.12 & 1.14 & 1.36 & 5.40 & 28.62 & 95.23 & 209.29 & explicit\tabularnewline
 & 1.96 & 1.42 & 1.18 & 1.12 & 1.15 & 1.30 & 22.65 & 22.89 & 267.55 & 226.40 & implicit\tabularnewline
\hline
\multirow{2}{*}{yeast} & 2.29 & 2.07 & 1.73 & 1.21 & 1.21 & 2.11 & 5.53 & 16.11 & 40.73 & 139.51 & explicit\tabularnewline
 & 2.29 & 2.07 & 1.73 & 1.21 & 1.29 & 2.06 & 18.38 & 30.82 & 100.81 & 318.28 & implicit\tabularnewline
\bottomrule
\end{tabular}
\end{table}

\begin{table}[h]
\scriptsize
\begin{tabular}{cccccccccccc}
\multirow{2}{*}{\textbf{Dataset} (small)} & \multicolumn{10}{c}{\textbf{Learning rate}} & \multirow{2}{*}{\textbf{SGD type}}\tabularnewline
 & 0.001 & 0.01 & 0.1 & 0.5 & 1.0 & 3.0 & 5.0 & 10.0 & 30.0 & 50.0 & \tabularnewline
\midrule
\multirow{2}{*}{acute-inflammation} & 0.69 & 0.69 & 0.69 & 0.69 & 0.67 & 0.45 & 9.87 & 26.80 & 58.27 & 30.87 & explicit\tabularnewline
 & 0.69 & 0.69 & 0.69 & 0.69 & 0.67 & 0.45 & 28.65 & 19.13 & 11.30 & 7.00 & implicit\tabularnewline
\hline
\multirow{2}{*}{acute-nephritis} & 0.68 & 0.68 & 0.66 & 0.65 & 0.64 & 0.02 & 12.88 & 0.73 & 7.07 & 28.32 & explicit\tabularnewline
 & 0.68 & 0.68 & 0.66 & 0.65 & 0.64 & 0.02 & 9.84 & 0.00 & 13.72 & 21.91 & implicit\tabularnewline
\hline
\multirow{2}{*}{annealing} & 1.52 & 1.20 & 0.73 & 0.49 & 0.57 & 2.24 & 10.88 & 38.62 & 183.52 & 209.44 & explicit\tabularnewline
 & 1.52 & 1.20 & 0.73 & 0.55 & 0.58 & 12.99 & 18.48 & 21.23 & 180.15 & 250.37 & implicit\tabularnewline
\hline
\multirow{2}{*}{arrhythmia} & 2.46 & 1.92 & 0.63 & 0.25 & 24.92 & 57.87 & 162.63 & 256.21 & 507.11 & 1280.03 & explicit\tabularnewline
 & 2.46 & 1.92 & 0.63 & 0.23 & 30.57 & 48.03 & 111.98 & 329.23 & 1401.02 & 2270.51 & implicit\tabularnewline
\hline
\multirow{2}{*}{audiology-std} & 2.89 & 2.85 & 2.55 & 0.93 & 0.42 & 28.32 & 27.10 & 57.28 & 165.67 & 270.77 & explicit\tabularnewline
 & 2.89 & 2.85 & 2.55 & 0.93 & 0.42 & 49.83 & 75.19 & 136.42 & 247.31 & 433.71 & implicit\tabularnewline
\hline
\multirow{2}{*}{balance-scale} & 1.13 & 1.07 & 0.92 & 0.91 & 0.91 & 0.32 & 2.48 & 7.48 & 75.97 & 111.11 & explicit\tabularnewline
 & 1.13 & 1.07 & 0.92 & 0.91 & 0.91 & 0.32 & 19.65 & 30.05 & 151.05 & 267.63 & implicit\tabularnewline
\hline
\multirow{2}{*}{balloons} & 0.70 & 0.69 & 0.68 & 0.68 & 0.68 & 0.68 & 6.96 & 43.68 & 25.14 & 88.85 & explicit\tabularnewline
 & 0.70 & 0.69 & 0.68 & 0.68 & 0.68 & 0.68 & 19.71 & 17.06 & 33.36 & 256.16 & implicit\tabularnewline
\hline
\multirow{2}{*}{blood} & 0.69 & 0.62 & 0.54 & 0.54 & 0.54 & 0.54 & 0.54 & 17.12 & 67.33 & 85.15 & explicit\tabularnewline
 & 0.69 & 0.62 & 0.54 & 0.54 & 0.54 & 0.54 & 0.54 & 28.88 & 49.44 & 143.05 & implicit\tabularnewline
\hline
\multirow{2}{*}{breast-cancer} & 0.71 & 0.68 & 0.63 & 0.62 & 0.59 & 0.61 & 20.42 & 33.63 & 38.61 & 97.66 & explicit\tabularnewline
 & 0.71 & 0.68 & 0.63 & 0.62 & 0.59 & 0.61 & 9.89 & 48.75 & 21.05 & 177.86 & implicit\tabularnewline
\hline
\multirow{2}{*}{breast-cancer-wisc} & 0.70 & 0.67 & 0.64 & 0.10 & 0.09 & 0.11 & 2.13 & 3.60 & 3.59 & 16.08 & explicit\tabularnewline
 & 0.70 & 0.67 & 0.64 & 0.10 & 0.09 & 0.11 & 3.99 & 2.78 & 36.46 & 43.01 & implicit\tabularnewline
\hline
\multirow{2}{*}{breast-cancer-wisc-diag} & 0.71 & 0.67 & 0.11 & 0.05 & 0.05 & 2.33 & 1.64 & 3.06 & 51.97 & 1399.89 & explicit\tabularnewline
 & 0.71 & 0.67 & 0.11 & 0.05 & 0.05 & 6.14 & 6.91 & 15.54 & 44.17 & 95.69 & implicit\tabularnewline
\hline
\multirow{2}{*}{breast-cancer-wisc-prog} & 0.69 & 0.65 & 0.53 & 0.45 & 0.37 & 0.53 & 125.74 & 28.90 & 97.62 & 72.58 & explicit\tabularnewline
 & 0.69 & 0.65 & 0.53 & 0.45 & 0.37 & 0.49 & 25.71 & 36.27 & 54.02 & 1589.74 & implicit\tabularnewline
\hline
\multirow{2}{*}{breast-tissue} & 1.78 & 1.78 & 1.78 & 1.77 & 1.75 & 1.18 & 6.23 & 35.95 & 69.66 & 510.77 & explicit\tabularnewline
 & 1.78 & 1.78 & 1.78 & 1.77 & 1.75 & 1.18 & 6.97 & 28.55 & 38.37 & 179.07 & implicit\tabularnewline
\hline
\multirow{2}{*}{congressional-voting} & 0.71 & 0.69 & 0.67 & 0.66 & 0.63 & 0.68 & 13.56 & 69.97 & 139.57 & 503.67 & explicit\tabularnewline
 & 0.71 & 0.69 & 0.67 & 0.66 & 0.63 & 0.68 & 39.21 & 52.27 & 166.97 & 265.89 & implicit\tabularnewline
\hline
\multirow{2}{*}{conn-bench-sonar-mines-rocks} & 0.69 & 0.69 & 0.59 & 0.30 & 0.32 & 75.14 & 93.29 & 204.05 & 363.93 & 1474.36 & explicit\tabularnewline
 & 0.69 & 0.69 & 0.59 & 0.30 & 0.33 & 33.36 & 51.40 & 62.84 & 447.59 & 387.23 & implicit\tabularnewline
\hline
\multirow{2}{*}{conn-bench-vowel-deterding} & 2.41 & 2.41 & 2.40 & 1.33 & 1.02 & 1.00 & 4.54 & 13.65 & 105.52 & 270.46 & explicit\tabularnewline
 & 2.41 & 2.41 & 2.40 & 1.33 & 1.05 & 1.20 & 10.81 & 30.59 & 244.61 & 175.47 & implicit\tabularnewline
\hline
\multirow{2}{*}{credit-approval} & 0.69 & 0.69 & 0.67 & 0.33 & 0.33 & 16.51 & 7.90 & 30.36 & 46.97 & 58.09 & explicit\tabularnewline
 & 0.69 & 0.69 & 0.67 & 0.34 & 0.33 & 1.66 & 8.54 & 42.96 & 31.86 & 35.54 & implicit\tabularnewline
\hline
\multirow{2}{*}{cylinder-bands} & 0.71 & 0.68 & 0.64 & 0.49 & 0.50 & 5.17 & 20.72 & 75.39 & 299.83 & 175.94 & explicit\tabularnewline
 & 0.71 & 0.68 & 0.64 & 0.50 & 0.52 & 31.31 & 71.16 & 57.63 & 68.50 & 386.99 & implicit\tabularnewline
\hline
\multirow{2}{*}{dermatology} & 1.81 & 1.78 & 1.32 & 0.23 & 0.09 & 5.17 & 1.43 & 28.02 & 62.58 & 374.63 & explicit\tabularnewline
 & 1.81 & 1.78 & 1.32 & 0.23 & 0.08 & 4.62 & 17.01 & 26.38 & 53.13 & 170.26 & implicit\tabularnewline
\hline
\multirow{2}{*}{echocardiogram} & 0.67 & 0.65 & 0.59 & 0.25 & 0.19 & 0.19 & 12.02 & 73.34 & 95.15 & 136.16 & explicit\tabularnewline
 & 0.67 & 0.65 & 0.59 & 0.25 & 0.19 & 0.21 & 11.35 & 9.94 & 80.01 & 34.53 & implicit\tabularnewline
\hline
\multirow{2}{*}{ecoli} & 2.11 & 2.06 & 1.79 & 1.19 & 0.96 & 0.97 & 23.58 & 13.47 & 90.18 & 77.99 & explicit\tabularnewline
 & 2.11 & 2.06 & 1.79 & 1.19 & 0.96 & 1.02 & 9.14 & 38.29 & 107.86 & 64.26 & implicit\tabularnewline
\hline
\multirow{2}{*}{energy-y1} & 1.15 & 1.10 & 1.01 & 0.43 & 0.43 & 5.66 & 9.66 & 16.76 & 63.23 & 115.98 & explicit\tabularnewline
 & 1.15 & 1.10 & 1.01 & 0.45 & 0.49 & 4.34 & 1.61 & 17.47 & 44.28 & 59.35 & implicit\tabularnewline
\hline
\multirow{2}{*}{energy-y2} & 1.18 & 1.11 & 1.02 & 0.40 & 0.39 & 0.44 & 1.36 & 5.68 & 14.01 & 38.33 & explicit\tabularnewline
 & 1.18 & 1.11 & 1.02 & 0.41 & 0.47 & 0.94 & 1.09 & 8.01 & 12.52 & 48.55 & implicit\tabularnewline
\hline
\multirow{2}{*}{fertility} & 0.72 & 0.69 & 0.48 & 0.35 & 0.35 & 0.35 & 0.35 & 6.90 & 34.50 & 131.21 & explicit\tabularnewline
 & 0.72 & 0.69 & 0.48 & 0.35 & 0.35 & 0.35 & 0.35 & 62.82 & 11.22 & 17.92 & implicit\tabularnewline
\hline
\multirow{2}{*}{flags} & 2.09 & 2.07 & 1.92 & 1.62 & 1.40 & 3.40 & 57.37 & 38.55 & 70.61 & 328.58 & explicit\tabularnewline
 & 2.09 & 2.07 & 1.92 & 1.62 & 1.40 & 19.65 & 38.18 & 103.61 & 150.33 & 631.23 & implicit\tabularnewline
\hline
\multirow{2}{*}{glass} & 1.76 & 1.74 & 1.60 & 1.50 & 1.50 & 1.57 & 39.15 & 41.59 & 255.73 & 224.50 & explicit\tabularnewline
 & 1.76 & 1.74 & 1.60 & 1.50 & 1.50 & 1.59 & 11.74 & 27.31 & 140.43 & 378.32 & implicit\tabularnewline
\hline
\multirow{2}{*}{haberman-survival} & 0.62 & 0.60 & 0.58 & 0.58 & 0.58 & 0.59 & 2.22 & 9.43 & 59.62 & 151.06 & explicit\tabularnewline
 & 0.62 & 0.60 & 0.58 & 0.57 & 0.58 & 0.59 & 2.16 & 11.38 & 35.12 & 166.79 & implicit\tabularnewline
\hline
\multirow{2}{*}{hayes-roth} & 1.14 & 1.12 & 1.04 & 1.02 & 1.02 & 1.02 & 7.22 & 10.04 & 104.41 & 229.25 & explicit\tabularnewline
 & 1.14 & 1.12 & 1.04 & 1.02 & 1.02 & 1.02 & 8.16 & 26.48 & 31.05 & 127.38 & implicit\tabularnewline
\hline
\multirow{2}{*}{heart-cleveland} & 1.64 & 1.58 & 1.30 & 0.91 & 0.90 & 0.90 & 7.38 & 16.45 & 108.54 & 99.98 & explicit\tabularnewline
 & 1.64 & 1.58 & 1.30 & 0.91 & 0.90 & 0.91 & 12.00 & 18.19 & 143.55 & 81.04 & implicit\tabularnewline
\hline
\multirow{2}{*}{heart-hungarian} & 0.66 & 0.66 & 0.66 & 0.31 & 0.30 & 0.39 & 20.01 & 39.19 & 49.91 & 151.12 & explicit\tabularnewline
 & 0.66 & 0.66 & 0.66 & 0.31 & 0.30 & 0.37 & 18.31 & 5.64 & 211.14 & 81.71 & implicit\tabularnewline
\hline
\multirow{2}{*}{heart-switzerland} & 1.65 & 1.63 & 1.51 & 1.35 & 1.33 & 1.32 & 42.48 & 47.01 & 296.09 & 295.82 & explicit\tabularnewline
 & 1.65 & 1.63 & 1.51 & 1.35 & 1.33 & 1.32 & 32.38 & 135.06 & 444.06 & 270.40 & implicit\tabularnewline
\hline
\multirow{2}{*}{heart-va} & 1.65 & 1.63 & 1.55 & 1.50 & 1.48 & 1.39 & 24.42 & 59.53 & 434.70 & 186.13 & explicit\tabularnewline
 & 1.65 & 1.63 & 1.55 & 1.50 & 1.48 & 1.38 & 48.08 & 62.68 & 171.65 & 209.72 & implicit\tabularnewline
\hline
\multirow{2}{*}{hepatitis} & 0.72 & 0.68 & 0.53 & 0.39 & 0.24 & 0.32 & 21.86 & 29.30 & 37.44 & 118.80 & explicit\tabularnewline
 & 0.72 & 0.68 & 0.53 & 0.39 & 0.24 & 0.35 & 32.01 & 19.00 & 41.54 & 111.05 & implicit\tabularnewline
\hline
\multirow{2}{*}{horse-colic} & 0.70 & 0.68 & 0.65 & 0.34 & 0.34 & 0.32 & 24.44 & 60.72 & 42.39 & 335.00 & explicit\tabularnewline
 & 0.70 & 0.68 & 0.65 & 0.34 & 0.34 & 0.33 & 8.99 & 118.06 & 347.70 & 60.39 & implicit\tabularnewline
\hline
\multirow{2}{*}{ilpd-indian-liver} & 0.70 & 0.66 & 0.60 & 0.59 & 0.55 & 0.53 & 16.40 & 55.36 & 45.47 & 299.05 & explicit\tabularnewline
 & 0.70 & 0.66 & 0.60 & 0.59 & 0.55 & 0.57 & 11.63 & 85.04 & 125.33 & 158.17 & implicit\tabularnewline
\hline
\multirow{2}{*}{ionosphere} & 0.70 & 0.68 & 0.61 & 0.20 & 0.20 & 6.67 & 49.14 & 21.57 & 101.66 & 233.83 & explicit\tabularnewline
 & 0.70 & 0.68 & 0.61 & 0.21 & 0.22 & 34.03 & 9.97 & 64.38 & 768.47 & 572.70 & implicit\tabularnewline
\hline
\multirow{2}{*}{iris} & 1.11 & 1.11 & 1.10 & 1.09 & 1.09 & 0.90 & 7.90 & 14.65 & 14.59 & 257.50 & explicit\tabularnewline
 & 1.11 & 1.11 & 1.10 & 1.09 & 1.09 & 0.77 & 9.25 & 26.47 & 147.58 & 175.54 & implicit\tabularnewline
\hline
\multirow{2}{*}{lenses} & 1.09 & 1.07 & 0.98 & 0.91 & 0.91 & 0.91 & 10.79 & 13.42 & 38.90 & 30.30 & explicit\tabularnewline
 & 1.09 & 1.07 & 0.98 & 0.91 & 0.91 & 0.91 & 9.38 & 16.82 & 39.64 & 214.31 & implicit\tabularnewline
\hline
\multirow{2}{*}{libras} & 2.72 & 2.68 & 2.09 & 0.93 & 0.72 & 26.84 & 19.45 & 79.46 & 196.99 & 349.34 & explicit\tabularnewline
 & 2.72 & 2.68 & 2.09 & 0.94 & 0.76 & 28.88 & 85.69 & 101.16 & 474.27 & 1377.74 & implicit\tabularnewline
\bottomrule
\end{tabular}
\end{table}

\begin{table}[h]
\scriptsize
\begin{tabular}{cccccccccccc}
\multirow{2}{*}{\textbf{Dataset} (small)} & \multicolumn{10}{c}{\textbf{Learning rate}} & \multirow{2}{*}{\textbf{SGD type}}\tabularnewline
 & 0.001 & 0.01 & 0.1 & 0.5 & 1.0 & 3.0 & 5.0 & 10.0 & 30.0 & 50.0 & \tabularnewline
\midrule
\multirow{2}{*}{low-res-spect} & 2.19 & 1.78 & 0.43 & 0.25 & 11.66 & 14.48 & 53.18 & 222.26 & 605.63 & 449.56 & explicit\tabularnewline
 & 2.19 & 1.78 & 0.43 & 0.26 & 16.57 & 46.29 & 109.37 & 82.74 & 94.69 & 259.85 & implicit\tabularnewline
\hline
\multirow{2}{*}{lung-cancer} & 1.11 & 1.10 & 1.04 & 0.27 & 0.03 & 28.79 & 87.95 & 149.18 & 1071.40 & 298.13 & explicit\tabularnewline
 & 1.11 & 1.10 & 1.04 & 0.27 & 0.03 & 95.18 & 87.43 & 284.11 & 543.03 & 359.73 & implicit\tabularnewline
\hline
\multirow{2}{*}{lymphography} & 1.39 & 1.33 & 1.00 & 0.49 & 0.33 & 13.29 & 15.73 & 5.16 & 112.88 & 181.66 & explicit\tabularnewline
 & 1.39 & 1.33 & 1.00 & 0.49 & 0.34 & 6.11 & 14.23 & 10.87 & 34.13 & 261.76 & implicit\tabularnewline
\hline
\multirow{2}{*}{mammographic} & 0.71 & 0.70 & 0.69 & 0.44 & 0.42 & 0.46 & 12.34 & 4.14 & 145.04 & 36.88 & explicit\tabularnewline
 & 0.71 & 0.70 & 0.69 & 0.44 & 0.42 & 0.44 & 10.77 & 9.93 & 116.31 & 244.93 & implicit\tabularnewline
\hline
\multirow{2}{*}{molec-biol-promoter} & 0.69 & 0.69 & 0.65 & 0.11 & 0.01 & 36.10 & 27.78 & 109.47 & 258.44 & 974.85 & explicit\tabularnewline
 & 0.69 & 0.69 & 0.65 & 0.11 & 0.01 & 97.50 & 201.21 & 551.19 & 642.16 & 863.06 & implicit\tabularnewline
\hline
\multirow{2}{*}{monks-1} & 0.69 & 0.69 & 0.69 & 0.69 & 0.62 & 0.71 & 17.77 & 17.11 & 89.50 & 138.78 & explicit\tabularnewline
 & 0.69 & 0.69 & 0.69 & 0.69 & 0.62 & 0.71 & 15.65 & 14.64 & 52.53 & 179.91 & implicit\tabularnewline
\hline
\multirow{2}{*}{monks-2} & 0.68 & 0.66 & 0.65 & 0.65 & 0.65 & 0.65 & 4.11 & 7.96 & 135.92 & 180.88 & explicit\tabularnewline
 & 0.68 & 0.66 & 0.65 & 0.65 & 0.65 & 0.65 & 11.61 & 6.54 & 25.41 & 200.93 & implicit\tabularnewline
\hline
\multirow{2}{*}{monks-3} & 0.70 & 0.69 & 0.69 & 0.68 & 0.51 & 0.76 & 4.16 & 11.69 & 45.17 & 222.38 & explicit\tabularnewline
 & 0.70 & 0.69 & 0.69 & 0.68 & 0.51 & 0.76 & 4.11 & 63.57 & 268.51 & 85.94 & implicit\tabularnewline
\hline
\multirow{2}{*}{musk-1} & 0.69 & 0.63 & 0.32 & 0.59 & 35.26 & 48.17 & 358.12 & 325.26 & 2522.81 & 5693.28 & explicit\tabularnewline
 & 0.69 & 0.63 & 0.32 & 0.27 & 26.13 & 83.83 & 373.71 & 1592.32 & 5190.91 & 8195.80 & implicit\tabularnewline
\hline
\multirow{2}{*}{oocytes-trisopterus-nucleus-2f} & 0.69 & 0.68 & 0.65 & 0.53 & 0.50 & 16.61 & 24.44 & 108.62 & 133.12 & 854.10 & explicit\tabularnewline
 & 0.69 & 0.68 & 0.65 & 0.53 & 0.52 & 22.13 & 109.88 & 27.77 & 631.40 & 859.68 & implicit\tabularnewline
\hline
\multirow{2}{*}{oocytes-trisopterus-states-5b} & 1.00 & 0.89 & 0.41 & 0.24 & 0.21 & 7.33 & 15.53 & 14.13 & 42.37 & 516.20 & explicit\tabularnewline
 & 1.00 & 0.89 & 0.41 & 0.24 & 0.21 & 49.55 & 19.40 & 24.82 & 154.19 & 116.00 & implicit\tabularnewline
\hline
\multirow{2}{*}{parkinsons} & 0.67 & 0.65 & 0.55 & 0.33 & 0.31 & 2.05 & 3.80 & 24.16 & 51.71 & 52.80 & explicit\tabularnewline
 & 0.67 & 0.65 & 0.55 & 0.33 & 0.31 & 20.58 & 14.53 & 90.82 & 37.16 & 468.79 & implicit\tabularnewline
\hline
\multirow{2}{*}{pima} & 0.66 & 0.66 & 0.66 & 0.64 & 0.51 & 0.53 & 8.14 & 27.66 & 112.66 & 79.41 & explicit\tabularnewline
 & 0.66 & 0.66 & 0.66 & 0.64 & 0.51 & 0.52 & 2.47 & 69.41 & 143.07 & 249.05 & implicit\tabularnewline
\hline
\multirow{2}{*}{pittsburg-bridges-MATERIAL} & 1.19 & 1.15 & 0.90 & 0.77 & 0.77 & 0.74 & 5.41 & 33.57 & 193.72 & 360.10 & explicit\tabularnewline
 & 1.19 & 1.15 & 0.90 & 0.77 & 0.77 & 0.74 & 15.25 & 37.79 & 203.28 & 303.36 & implicit\tabularnewline
\hline
\multirow{2}{*}{pittsburg-bridges-REL-L} & 1.19 & 1.17 & 1.06 & 1.00 & 1.00 & 0.97 & 7.93 & 33.08 & 87.76 & 122.83 & explicit\tabularnewline
 & 1.19 & 1.17 & 1.06 & 1.00 & 1.00 & 0.97 & 11.10 & 25.15 & 124.80 & 330.23 & implicit\tabularnewline
\hline
\multirow{2}{*}{pittsburg-bridges-SPAN} & 1.05 & 1.05 & 1.04 & 1.03 & 1.03 & 0.93 & 14.34 & 36.34 & 172.06 & 135.45 & explicit\tabularnewline
 & 1.05 & 1.05 & 1.04 & 1.03 & 1.03 & 0.93 & 2.40 & 45.68 & 166.73 & 279.46 & implicit\tabularnewline
\hline
\multirow{2}{*}{pittsburg-bridges-T-OR-D} & 0.83 & 0.78 & 0.47 & 0.35 & 0.35 & 0.35 & 0.35 & 5.05 & 27.71 & 124.39 & explicit\tabularnewline
 & 0.83 & 0.78 & 0.47 & 0.35 & 0.35 & 0.35 & 0.35 & 21.00 & 38.70 & 23.87 & implicit\tabularnewline
\hline
\multirow{2}{*}{pittsburg-bridges-TYPE} & 2.00 & 1.98 & 1.80 & 1.59 & 1.58 & 1.56 & 2.51 & 77.31 & 186.86 & 290.00 & explicit\tabularnewline
 & 2.00 & 1.98 & 1.80 & 1.59 & 1.58 & 1.56 & 5.37 & 82.37 & 272.11 & 350.51 & implicit\tabularnewline
\hline
\multirow{2}{*}{planning} & 0.62 & 0.61 & 0.58 & 0.58 & 0.58 & 0.59 & 34.64 & 19.23 & 30.94 & 453.85 & explicit\tabularnewline
 & 0.62 & 0.61 & 0.58 & 0.58 & 0.58 & 0.59 & 19.01 & 15.86 & 166.14 & 261.24 & implicit\tabularnewline
\hline
\multirow{2}{*}{post-operative} & 1.21 & 1.17 & 0.94 & 0.71 & 0.69 & 0.68 & 25.50 & 73.03 & 60.28 & 83.88 & explicit\tabularnewline
 & 1.21 & 1.17 & 0.94 & 0.71 & 0.69 & 0.68 & 25.03 & 50.13 & 90.23 & 71.85 & implicit\tabularnewline
\hline
\multirow{2}{*}{primary-tumor} & 2.70 & 2.68 & 2.50 & 2.15 & 1.83 & 13.21 & 17.81 & 41.57 & 99.73 & 94.32 & explicit\tabularnewline
 & 2.70 & 2.68 & 2.50 & 2.15 & 1.83 & 10.32 & 15.55 & 35.04 & 70.16 & 326.01 & implicit\tabularnewline
\hline
\multirow{2}{*}{seeds} & 1.11 & 1.10 & 1.09 & 0.62 & 0.47 & 0.68 & 8.84 & 13.44 & 28.41 & 32.21 & explicit\tabularnewline
 & 1.11 & 1.10 & 1.09 & 0.62 & 0.47 & 7.20 & 5.44 & 28.78 & 87.84 & 85.50 & implicit\tabularnewline
\hline
\multirow{2}{*}{soybean} & 2.90 & 2.86 & 2.43 & 0.65 & 0.21 & 2.12 & 5.69 & 33.16 & 151.98 & 119.04 & explicit\tabularnewline
 & 2.90 & 2.86 & 2.43 & 0.66 & 0.22 & 9.36 & 20.48 & 32.44 & 161.14 & 288.93 & implicit\tabularnewline
\hline
\multirow{2}{*}{spect} & 0.70 & 0.69 & 0.67 & 0.44 & 0.45 & 0.57 & 54.69 & 105.29 & 263.16 & 159.21 & explicit\tabularnewline
 & 0.70 & 0.69 & 0.67 & 0.44 & 0.45 & 0.47 & 18.18 & 159.16 & 21.86 & 567.59 & implicit\tabularnewline
\hline
\multirow{2}{*}{spectf} & 0.69 & 0.64 & 0.49 & 0.38 & 0.42 & 20.77 & 6.86 & 28.35 & 622.85 & 585.05 & explicit\tabularnewline
 & 0.69 & 0.64 & 0.49 & 0.37 & 0.51 & 80.29 & 132.31 & 27.83 & 130.49 & 1229.86 & implicit\tabularnewline
\hline
\multirow{2}{*}{statlog-australian-credit} & 0.68 & 0.66 & 0.64 & 0.64 & 0.64 & 0.64 & 19.39 & 35.72 & 31.33 & 78.80 & explicit\tabularnewline
 & 0.68 & 0.66 & 0.64 & 0.64 & 0.64 & 0.64 & 73.71 & 98.74 & 228.37 & 448.01 & implicit\tabularnewline
\hline
\multirow{2}{*}{statlog-heart} & 0.69 & 0.69 & 0.69 & 0.32 & 0.33 & 6.74 & 4.19 & 9.82 & 143.11 & 37.86 & explicit\tabularnewline
 & 0.69 & 0.69 & 0.69 & 0.32 & 0.33 & 15.30 & 0.98 & 28.84 & 157.01 & 541.23 & implicit\tabularnewline
\hline
\multirow{2}{*}{statlog-vehicle} & 1.39 & 1.39 & 1.31 & 0.90 & 0.78 & 16.15 & 9.25 & 42.65 & 163.61 & 403.98 & explicit\tabularnewline
 & 1.39 & 1.39 & 1.31 & 0.89 & 0.70 & 32.42 & 20.88 & 93.00 & 74.45 & 656.19 & implicit\tabularnewline
\hline
\multirow{2}{*}{synthetic-control} & 1.80 & 1.72 & 0.70 & 0.06 & 5.96 & 24.28 & 43.88 & 113.24 & 69.17 & 482.32 & explicit\tabularnewline
 & 1.80 & 1.72 & 0.70 & 0.06 & 9.76 & 24.30 & 50.20 & 110.53 & 164.13 & 692.48 & implicit\tabularnewline
\hline
\multirow{2}{*}{teaching} & 1.11 & 1.11 & 1.09 & 1.09 & 1.09 & 1.11 & 16.23 & 47.19 & 43.59 & 92.10 & explicit\tabularnewline
 & 1.11 & 1.11 & 1.09 & 1.09 & 1.09 & 1.11 & 12.49 & 21.53 & 266.23 & 199.40 & implicit\tabularnewline
\hline
\multirow{2}{*}{tic-tac-toe} & 0.68 & 0.66 & 0.64 & 0.52 & 0.10 & 2.05 & 4.13 & 45.16 & 156.90 & 81.64 & explicit\tabularnewline
 & 0.68 & 0.66 & 0.64 & 0.52 & 0.09 & 10.22 & 40.48 & 53.51 & 47.37 & 200.14 & implicit\tabularnewline
\hline
\multirow{2}{*}{trains} & 0.70 & 0.70 & 0.67 & 0.04 & 0.01 & 0.00 & 16.94 & 8.70 & 153.88 & 503.41 & explicit\tabularnewline
 & 0.70 & 0.70 & 0.67 & 0.04 & 0.01 & 0.00 & 55.28 & 83.36 & 0.00 & 0.00 & implicit\tabularnewline
\hline
\multirow{2}{*}{vertebral-column-2clases} & 0.68 & 0.66 & 0.63 & 0.58 & 0.37 & 6.00 & 13.78 & 9.92 & 23.02 & 189.26 & explicit\tabularnewline
 & 0.68 & 0.66 & 0.63 & 0.58 & 0.37 & 0.39 & 3.75 & 13.45 & 106.58 & 52.45 & implicit\tabularnewline
\hline
\multirow{2}{*}{vertebral-column-3clases} & 1.10 & 1.09 & 1.05 & 0.77 & 0.54 & 7.37 & 17.75 & 12.02 & 44.07 & 81.04 & explicit\tabularnewline
 & 1.10 & 1.09 & 1.05 & 0.77 & 0.53 & 2.55 & 9.28 & 7.79 & 42.32 & 40.59 & implicit\tabularnewline
\hline
\multirow{2}{*}{wine} & 1.12 & 1.12 & 1.09 & 0.53 & 0.46 & 5.43 & 3.44 & 9.22 & 43.31 & 62.74 & explicit\tabularnewline
 & 1.12 & 1.12 & 1.09 & 0.53 & 0.47 & 5.78 & 37.86 & 96.27 & 64.82 & 88.60 & implicit\tabularnewline
\hline
\multirow{2}{*}{zoo} & 1.92 & 1.91 & 1.83 & 1.57 & 1.01 & 6.70 & 6.23 & 33.11 & 172.96 & 380.68 & explicit\tabularnewline
 & 1.92 & 1.91 & 1.83 & 1.57 & 1.01 & 6.58 & 22.16 & 56.49 & 198.40 & 189.51 & implicit\tabularnewline
\bottomrule
\end{tabular}
\end{table}

\end{document}